\newif\ifarxiv 
\arxivtrue
\newif\ifcolt

\ifarxiv
\documentclass{article}
\usepackage[utf8]{inputenc} 

\usepackage[letterpaper, left=1in, right=1in, top=1in,
bottom=1in]{geometry}
\usepackage{parskip}

\usepackage[T1]{fontenc}    
\usepackage[hypertexnames=false]{hyperref}       

\usepackage{url}            
\usepackage{booktabs}       
\usepackage{amsfonts}       
\usepackage{latexsym,amssymb,amsthm,amscd,amsopn,amsmath}
\usepackage{bbm}
\usepackage{mathrsfs}
\usepackage{nicefrac}       
\usepackage{microtype}      
\usepackage{xcolor}         
\usepackage{algorithm}
\usepackage{verbatim}
\usepackage[noend]{algpseudocode}

\usepackage{xspace}
\usepackage{mathtools}
\usepackage{svg}
\usepackage[nameinlink,capitalize]{cleveref}
\usepackage{makecell}
\usepackage{setspace}
\usepackage{subfigure}
\usepackage{thm-restate}
\usepackage{multirow}
\usepackage{nicematrix}
\usepackage{hhline}

\definecolor{dgreen}{rgb}{0,0.5,0}
\hypersetup{
  colorlinks=true,
  linkcolor=blue,
  filecolor=blue,
  citecolor = dgreen,      
  urlcolor=cyan,
}

\crefformat{equation}{#2(#1)#3}
\Crefformat{equation}{#2(#1)#3}
\Crefname{construction}{Construction}{Constructions}

\Crefformat{figure}{#2Figure #1#3}
\Crefname{assumption}{Assumption}{Assumptions}
\Crefformat{assumption}{#2Assumption #1#3}
\Crefname{subsubsection}{Section}{Sections}
\crefformat{subsubsection}{#2Section #1#3}
\Crefformat{subsubsection}{#2Section #1#3}

\theoremstyle{plain}
\newtheorem{theorem}{Theorem}[section]
\newtheorem{lemma}[theorem]{Lemma}

\newtheorem{proposition}[theorem]{Proposition}

\newtheorem{assumption}[theorem]{Assumption}
\newtheorem{remark}[theorem]{Remark}
\theoremstyle{definition}
\newtheorem{definition}[theorem]{Definition}

\newtheorem{fact}[theorem]{Fact}

\newtheoremstyle{named}%
    {}{}{\itshape}{}{\bfseries}{.}{.5em}{\thmnote{#3}}
\theoremstyle{named}

\numberwithin{theorem}{section}

\newcommand{\nc}{\newcommand}
\nc{\DMO}{\DeclareMathOperator}
\newcount\Comments  
\DeclareMathOperator*{\argmin}{arg\,min} 
\DeclareMathOperator*{\argmax}{arg\,max}

\nc{\Moracle}{\MM^{\mathsf{oracle}}}
\nc{\tilMoracle}{\til{\MM}^{\mathsf{oracle}}}
\nc{\SimulateReduction}{\texttt{SimulateReduction}\xspace}
\nc{\TestReduction}{\texttt{DistinguishReduction}\xspace}
\nc{\SimulateSampling}{\texttt{SimulateSampling}\xspace}
\nc{\SimulateRegression}{\texttt{SimulateRegression}\xspace}
\nc{\Osample}{{\MO_{\mathsf{samp}}}}
\nc{\Oregress}{{\MO_{\mathsf{regress}}}}
\nc{\Oregressp}{{\MO'_{\mathsf{regress}}}}
\nc{\Obandits}{{\MO_{\mathsf{bandits}}}}
\nc{\dom}{\mathsf{dom}}
\nc{\SF}{\mathscr{F}}
\nc{\Fchernoff}{\MF^{\mathsf{chernoff}}}

\DMO{\prox}{prox}
\DMO{\Span}{span}
\DMO{\UCB}{UCB}
\DMO{\LCB}{LCB}
\nc{\expl}[2]{\ME^{#1}_{#2}}
\nc{\tilmdp}[1]{\til \MM({#1})}
\nc{\barpdp}[2]{\ol \MP_{#1}({#2})}
\nc{\barmdp}[1]{\ol \MM({#1})}
\nc{\hatmdp}[1]{\wh \MM({#1})}
\nc{\rem}[2]{\MR_{#1}({#2})}
\nc{\Pigen}{\Pi^{\rm gen}}
\nc{\Pidet}{\Pi^{\rm det}}
\nc{\PiZ}{\Pi_{\SZ}^{\rm markov}}
\nc{\und}[3]{\MU_{{#1}}^{{#2}}({#3})}
\nc{\zlow}[2]{\MZ_{{#1}}^\lowv({#2})}
\nc{\dg}{\dagger}
\nc{\bB}{\mathbf{B}}
\nc{\unif}{\mu_{\rm unif}}
\nc{\indsig}[2]{\mathcal{I}_{#1}({#2})}
\nc{\total}{{\rm fin}}
\nc{\early}{{\rm pre}}
\nc{\zsink}{z_{\rm sink}}
\nc{\lowv}{{\rm low}}
\nc{\oo}[1]{\texttt{o}({#1})}
\nc{\posnrm}[1]{\left[ {#1} \right]_+}
\nc{\negnrm}[1]{\left[ {#1} \right]_-}
\nc{\tvnrm}[1]{\left\| {#1} \right\|_1}
\nc{\absval}[1]{\left| {#1} \right|}
\nc{\normalize}[1]{\mathfrak{n}\left({#1}\right)}

\nc{\SZ}{\textsf{Z}}
\nc{\SO}{\textsf{O}}
\nc{\suff}[2]{{\rm suff}_{#1}({#2})}
\nc{\UPhi}{\mathscr{U}_{X,H}}
\nc{\UPhis}{\til{\mathscr{U}}_{X,H,\MF}}
\nc{\SV}{\mathscr{V}}
\nc{\Phiset}{\Phi_{X,H}}
\nc{\Phisets}{\til{\Phi}_{X,H,\MF}}
\nc{\Lyu}{{\mathtt{Lyu}}}
\nc{\wAlg}{{\widetilde \Alg}}

\nc{\ApproxMDP}{\texttt{ConstructMDP}\xspace}
\nc{\mainalg}{\texttt{BaSeCAMP}\xspace} 
\nc{\bspanner}{\texttt{BarySpannerPolicy}\xspace}

\nc{\gamvec}{\gamma}
\nc{\til}{\widetilde}
\nc{\td}{\tilde}
\nc{\wh}{\widehat}
\nc{\old}[1]{\ifnum\Comments=1 {\color{brown}  [COPIED: #1]}\fi}
\definecolor{darkgreen}{rgb}{0.0, 0.5, 0.0}
\nc{\noah}[1]{\ifnum\Comments=1 {\color{darkgreen} [ng: #1]}\fi}
\nc{\dhruv}[1]{\ifnum\Comments=1 {\color{purple} [dr: #1]}\fi}
\nc{\BP}{\mathbb{P}}
\nc{\BM}{\mathbb{M}}
\nc{\bbapx}{\bb^{\rm apx}}
\nc{\bbapxs}[1]{\bb^{\rm apx, {#1}}}

\nc{\fools}[3]{\MF_{#3}({#1}, {#2})}
\nc{\fool}[2]{\MF({#1},{#2})}
\nc{\clip}[2]{{\rm clip}\left[ \left. {#1} \right| {#2} \right]}
\nc{\imax}{\omega}
\DMO{\conv}{conv}
\nc{\MH}{\mathcal{H}}
\nc{\CH}{\mathscr{H}}
\nc{\CB}{\mathscr{B}}
\nc{\cD}{\mathscr{D}}
\nc{\MC}{\mathcal{C}}
\nc{\MV}{\mathcal{V}}
\nc{\Tclus}{\mathcal{T}_{\mathsf{clus}}}
\nc{\st}{\star}
\nc{\lng}{\langle}
\nc{\rng}{\rangle}
\DMO{\OOPT}{opt}
\nc{\dopt}[2]{\ell_{\OOPT}({#1},{#2})}
\nc{\grad}{\nabla}
\nc{\MG}{\mathcal{G}}
\nc{\MP}{\mathcal{P}}
\nc{\PP}{\mathbb{P}}
\nc{\TT}{\mathbb{T}}
\nc{\TTmax}{\TT_{\max}}
\DMO{\Ham}{Ham}
\DMO{\Gap}{Gap}
\DMO{\GD}{GD}
\DMO{\GDA}{GDA}
\DMO{\EG}{EG}
\DMO{\OGDA}{OGDA}
\DMO{\Unif}{Unif}
\DMO{\Tr}{Tr}
\nc{\ul}{\underline}
\nc{\ol}{\overline}
\nc{\Qu}{\ul{Q}}
\nc{\Qo}{\ol{Q}}
\nc{\Ro}{\ol{R}}
\nc{\Vu}{\ul{V}}
\nc{\Vo}{\ol{V}}
\nc{\RanQ}{\Delta Q}
\nc{\RanV}{\Delta V}
\nc{\clipQ}{\Delta \breve{Q}}
\nc{\frzQ}{\Delta \mathring{Q}}
\nc{\clipV}{\Delta \breve{V}}
\nc{\clipdelta}{\breve{\delta}}
\nc{\cliptheta}{\breve{\theta}}
\nc{\delmin}{\Delta_{{\rm min}}}
\nc{\delmins}[1]{\Delta_{{\rm min},{#1}}}
\nc{\gapfinal}[1]{\max \left\{ \frac{\frzQ_{{#1}}^{k^\st}(x,a)}{2H}, \frac{\delmin}{4H} \right\}}
\nc{\post}[2]{R({#1}; {#2})}
\nc{\posts}[3]{R_{#3}({#1}; {#2})}
\nc{\MAJ}{\mathsf{MAJ}}
\nc{\Dnull}{D^{\circ}}

\nc{\BKW}{\mathtt{BKW}}
\nc{\Dec}{\mathtt{Dec}}
\nc{\delreg}{\delta_{\mathsf{reg}}}
\nc{\delreal}{\delta_{\mathsf{real}}}
\nc{\Sreg}{S_{\mathsf{Reg}}}
\nc{\Treg}{T_{\mathsf{Reg}}}
\nc{\PC}{\texttt{PC}}
\nc{\PCO}{\texttt{PCE}}
\nc{\PCR}{\texttt{PCR}}
\nc{\VOX}{\texttt{VOX}}
\nc{\EPCO}{\texttt{EPCE}}
\nc{\EPCR}{\texttt{EPCR}}
\nc{\alphaPC}{\alpha_{\mathsf{PC}}}
\nc{\TPC}{T_{\mathsf{PC}}}
\nc{\SPC}{S_{\mathsf{PC}}}
\nc{\delsmall}{{\delta_{\mathsf{small}}}}
\nc{\CL}{\mathtt{ContrastLearn}}
\nc{\Select}{\mathtt{Select}}
\nc{\piunif}{{\pi_{\mathsf{unif}}}}
\nc{\picov}{{\pi_{\mathsf{cov}}}}
\nc{\ZZ}{\mathbb{Z}}
\nc{\sk}{{\mathsf{sk}}}
\nc{\Enc}{\mathtt{Enc}}
\nc{\EntLPN}{\mathtt{EntangleLPN}}
\nc{\mureg}{{\mu_{\mathsf{reg}}}}
\nc{\PPE}{\mathtt{PPE}}
\nc{\FQI}{\mathtt{FQI}}
\nc{\False}{\mathtt{False}}
\nc{\True}{\mathtt{True}}
\nc{\epreg}{\epsilon_{\mathsf{reg}}}
\nc{\algnst}[1]{\begin{align*}#1\end{align*}}
\nc{\algn}[1]{\begin{align}#1\end{align}}
\nc{\matx}[1]{\left(\begin{matrix}#1\end{matrix}\right)}

\nc{\pimix}{{\pi_{\mathsf{mix}}}}
\nc{\BPC}{{B_{\mathsf{PC}}}}
\nc{\size}{\mathrm{size}}
\nc{\OLIVE}{\texttt{OLIVE}}
\nc{\NP}{\textsf{NP}}
\nc{\RP}{\textsf{RP}}
\nc{\cprp}{c_{\mathsf{PRP}}}
\nc{\Mtoy}{\MM_{\mathsf{toy}}}
\nc{\Brute}{\mathtt{Brute}}

\nc{\nuu}{\nu}

\nc{\bel}[1]{\mathbf{b}({#1})}
\nc{\nbel}[1]{\bar{\mathbf{b}}({#1})}
\nc{\sbel}[2]{\mathbf{b}'_{#1}({#2})}
\nc{\nsbel}[2]{\bar{\mathbf{b}}'_{#1}({#2})}

\nc{\bv}{\mathbf{v}}
\nc{\bone}{\mathbf{1}}
\nc{\bX}{\mathbf{X}}
\nc{\be}{\mathbf{e}}
\nc{\bY}{\mathbf{Y}}
\nc{\bG}{\mathbf{G}}
\nc{\bz}{\mathbf{z}}
\nc{\bw}{\mathbf{w}}
\nc{\bA}{\mathbf{A}}
\nc{\bJ}{\mathbf{J}}
\nc{\bK}{\mathbf{K}}
\nc{\bb}{\mathbf{b}}
\nc{\ba}{\mathbf{a}}
\nc{\bs}{\mathbf{s}}
\nc{\bzero}{\mathbf{0}}
\nc{\bi}{\mathbf{i}}
\nc{\Edistinct}{\ME^{\mathsf{distinct}}}
\nc{\bc}{\mathbf{c}}
\nc{\bC}{\mathbf{C}}
\nc{\BR}{\mathbb R}
\nc{\BA}{\mathbb{A}}
\nc{\SA}{\mathscr{A}}
\nc{\BC}{\mathbb C}
\nc{\bx}{\mathbf{x}}
\nc{\bS}{\mathbf{S}}
\nc{\bM}{\mathbf{M}}
\nc{\bR}{\mathbf{R}}
\nc{\bN}{\mathbf{N}}
\nc{\by}{\mathbf{y}}
\nc{\sy}{y}
\nc{\sx}{x}

\nc{\MO}{\mathcal O}
\nc{\MQ}{\mathcal{Q}}
\nc{\CO}{\mathscr{O}}
\nc{\MU}{\mathcal{U}}
\nc{\ME}{\mathcal{E}}
\nc{\MN}{\mathcal{N}}
\nc{\MK}{\mathcal{K}}
\nc{\MM}{\mathcal{M}}
\nc{\MS}{\mathcal{S}}
\nc{\MT}{\mathcal{T}}
\nc{\BF}{\mathbb F}
\nc{\BQ}{\mathbb Q}
\nc{\MX}{\mathcal{X}}
\nc{\MA}{\mathcal{A}}
\nc{\MD}{\mathcal{D}}
\nc{\MB}{\mathcal{B}}
\nc{\MZ}{\mathcal{Z}}
\nc{\MJ}{\mathcal{J}}
\nc{\MW}{\mathcal{W}}
\nc{\MR}{\mathcal{R}}
\nc{\MY}{\mathcal{Y}}
\nc{\ML}{\mathcal{L}}
\nc{\BZ}{\mathbb Z}
\nc{\BN}{\mathbb N}
\nc{\ep}{\epsilon}
\nc{\gapfn}[1]{\varepsilon_{#1}}
\nc{\ggapfn}[2]{\varphi_{#1}({#2})}
\nc{\epsahk}{\gapfn{0}}
\nc{\BH}{\mathbb H}
\nc{\BG}{\mathbb{G}}
\nc{\D}{\Delta}
\nc{\MF}{\mathcal{F}}
\nc{\One}{\mathbbm{1}}
\nc{\bOne}{\mathbf{1}}
\nc{\Aopt}{\mathcal{A}^{\rm opt}}
\nc{\Amul}{\mathcal{A}^{\rm mul}}

\nc{\SP}{\mathsf P}
\nc{\SQ}{\mathsf Q}

\nc{\DO}{\accentset{\circ}{\D}}
\nc{\mf}{\mathfrak}
\nc{\mfp}{\mathfrak{p}}
\nc{\mfq}{\mf{q}}
\nc{\Sp}{\mbox{Spec}}
\nc{\Spm}{\mbox{Specm}}
\nc{\hookuparrow}{\mathrel{\rotatebox[origin=c]{90}{$\hookrightarrow$}}}
\nc{\hookdownarrow}{\mathrel{\rotatebox[origin=c]{-90}{$\hookrightarrow$}}}
\nc{\hra}{\hookrightarrow}
\nc{\tra}{\twoheadrightarrow}
\nc{\sgn}{{\rm sgn}}
\nc{\aut}{{\rm Aut}}
\nc{\Hom}{{\rm Hom}}
\nc{\img}{{\rm Im}}
\DMO{\id}{Id}
\DMO{\supp}{supp}
\DMO{\KL}{KL}
\nc{\kld}[2]{\KL({#1}||{#2})}
\nc{\ren}[2]{D_2({#1}||{#2})}
\nc{\chisq}[2]{\chi^2({#1}||{#2})}
\nc{\tvd}[2]{D_{\mathsf{TV}}\left({#1}, {#2}\right)}
\nc{\hell}[2]{H^2({#1}, {#2})}
\DMO{\BSS}{BSS}
\DMO{\BES}{BES}
\DMO{\BGS}{BGS}
\DMO{\poly}{poly}
\nc{\indep}{\perp}
\DMO{\sink}{sink}
\DMO{\nosink}{nosink}
\nc{\sinks}{s^{\sink}}
\nc{\sinkobs}{o^{\sink}}
\nc{\fp}[1]{\MP_1({#1})}
\nc{\BO}{\mathbb{O}}
\nc{\BT}{\mathbb{T}}

\nc{\RR}{\mathbb{R}}
\nc{\NN}{\mathbb{N}}
\nc{\Gradient}{\nabla}
\DMO{\diag}{diag}
\nc{\norm}[1]{\left \lVert #1 \right \rVert}
\DMO*{\EE}{\mathbb{E}}
\nc{\LPN}{\mathsf{LPN}}
\DMO{\Ber}{Ber}
\nc{\Regress}{\mathtt{Regress}}
\nc{\LFC}{\mathtt{LearnFromCorr}}
\nc{\RegressAlg}{\mathtt{RegressAlg}}
\nc{\DrawTraj}{\mathtt{DrawTrajectory}}
\nc{\pizero}{{\pi_{\mathsf{zero}}}}
\nc{\Tred}{{T_{\mathsf{red}}}}
\nc{\epred}{{\epsilon_{\mathsf{red}}}}
\nc{\TriAlg}{\mathtt{GenerateTriangleLPN}}
\nc{\Alg}{\mathtt{Alg}}
\nc{\AffSample}{\mathtt{AffSample}}

\nc{\br}{\mathbf{r}}
\nc{\TV}{{\mathsf{TV}}}
\DMO{\Law}{Law}
\DMO{\Sym}{Sym}
\nc{\bu}{\mathbf{u}}
\nc{\Reg}{\mathtt{Reg}}
\nc{\Breg}{B_{\mathsf{Reg}}}
\DMO{\dc}{dc}
\nc{\PSDP}{\texttt{PSDP}}

\DMO{\PR}{Pr}
\renewcommand{\Pr}{\PR}
\DMO*{\Prr}{Pr}
\nc{\E}{\mathbb{E}}
\nc{\ra}{\rightarrow}

\nc{\sups}[1]{^{\scriptscriptstyle{#1}}}
\nc{\bfr}{\mathbf{r}}
\nc{\Mbar}{\ol{M}}
\nc{\Sbar}{\ol{\MS}}
\nc{\Xbar}{\ol{\MX}}
\nc{\Pbar}{\ol{\BP}}
\nc{\term}{\mathfrak{t}}
\nc{\Srch}{\MS^{\mathsf{rch}}}
\nc{\Obar}{\ol{\BO}}
\nc{\tsmall}{\sigma_{\mathsf{bkup}}}
\nc{\trunc}{\sigma_{\mathsf{trunc}}}
\nc{\Nreg}{N_{\mathsf{reg}}}
\nc{\Srchhi}{\MS^{\mathsf{rch,0}}}
\nc{\Srchki}{\MS^{\mathsf{rch,1}}}
\nc{\Srchgam}{\MS^{\mathsf{rch,2}}}
\nc{\gamki}{\gamma^{\mathsf{KI}}}
\nc{\DKI}{D^{\mathsf{KI}}}
\nc{\lamrch}{\lambda^{\mathsf{rch}}}
\nc{\PSDPB}{\texttt{PSDP}}

\nc{\Oreg}{\MO_{\mathsf{reg}}}
\nc{\epa}{\varepsilon_A}
\nc{\epstat}{\epsilon_{\mathsf{stat}}}

\nc{\OneRed}{\texttt{OneRed}}
\nc{\NoiselessOneRed}{\texttt{NoiselessOneRed}}
\nc{\TwoRed}{\texttt{TwoRed}}
\nc{\OneTwo}{\texttt{OneTwo}}

\nc{\eprl}{\epsilon_{\mathsf{RL}}}
\nc{\delrl}{\delta_{\mathsf{RL}}}
\nc{\Nrl}{N_{\mathsf{RL}}}
\nc{\Krl}{K_{\mathsf{RL}}}
\nc{\epfinal}{{\varepsilon_{\mathsf{final}}}}

\nc{\Saug}{{\MS^{\mathsf{aug}}}}
\nc{\Xaug}{{\MX^{\mathsf{aug}}}}
\nc{\Phiaug}{{\Phi^{\mathsf{aug}}}}
\nc{\phiaug}{{\phi^{\mathsf{aug}}}}
\DMO{\aug}{aug}
\nc{\what}{\wh w}
\nc{\pihat}{\wh \pi}
\nc{\betahat}{\wh\beta}
\nc{\xbart}{\ol x_{h+1}^{(t)}}
\nc{\ccov}{c_{\mathsf{cov}}}
\DMO{\Val}{Val}
\DMO{\NH}{NH}
\DMO{\HH}{H}
\nc{\xbar}{\ol{x}}
\nc{\OneAug}{\texttt{OneAug}}
\nc{\TwoAug}{\texttt{TwoAug}}
\nc{\Igood}{\mathcal{I}_{\mathsf{good}}}
\DMO{\polylog}{polylog}

\makeatletter
  \renewenvironment{proof}[1][Proof]%
  {%
   \par\noindent{\bfseries\upshape {#1.}\ }%
  }%
  {\qed\newline}
  \makeatother

  \nc{\HOMER}{\texttt{HOMER}}
  
\newcommand{\citep}[1]{\cite{#1}}
\newcommand{\citet}[1]{\cite{#1}}
\fi

\ifcolt
\PassOptionsToPackage{dvipsnames}{xcolor}
\PassOptionsToPackage{hypertexnames=false}{hyperref}
\documentclass[anon,12pt,cleveref]{colt2025}
\usepackage{times}
\input{colt_commands}
\fi

\newcommand{\sssref}[1]{\texorpdfstring{\hyperref[#1]{\mbox{Section \ref*{#1}}}}{Section \ref*{#1}}}

\newcommand{\lineref}[1]{\texorpdfstring{\hyperref[#1]{\mbox{Line \ref*{#1}}}}{Line \ref*{#1}}}

\DeclarePairedDelimiter{\abs}{\lvert}{\rvert} %

\DeclarePairedDelimiter{\crl}{\{}{\}}

\let\Pr\undefined

\DeclareMathOperator{\Pr}{Pr}

\def\ddefloop#1{\ifx\ddefloop#1\else\ddef{#1}\expandafter\ddefloop\fi}
\def\ddef#1{\expandafter\def\csname bb#1\endcsname{\ensuremath{\mathbb{#1}}}}
\ddefloop ABCDEFGHIJKLMNOPQRSTUVWXYZ\ddefloop
\def\ddefloop#1{\ifx\ddefloop#1\else\ddef{#1}\expandafter\ddefloop\fi}
\def\ddef#1{\expandafter\def\csname b#1\endcsname{\ensuremath{\mathbf{#1}}}}
\ddefloop ABCDEFGHIJKLMNOPQRSTUVWXYZ\ddefloop
\def\ddef#1{\expandafter\def\csname sf#1\endcsname{\ensuremath{\mathsf{#1}}}}
\ddefloop ABCDEFGHIJKLMNOPQRSTUVWXYZ\ddefloop
\def\ddef#1{\expandafter\def\csname c#1\endcsname{\ensuremath{\mathcal{#1}}}}
\ddefloop ABCDEFGHIJKLMNOPQRSTUVWXYZ\ddefloop
\def\ddef#1{\expandafter\def\csname h#1\endcsname{\ensuremath{\widehat{#1}}}}
\ddefloop ABCDEFGHIJKLMNOPQRSTUVWXYZ\ddefloop
\def\ddef#1{\expandafter\def\csname hc#1\endcsname{\ensuremath{\widehat{\mathcal{#1}}}}}
\ddefloop ABCDEFGHIJKLMNOPQRSTUVWXYZ\ddefloop
\def\ddef#1{\expandafter\def\csname t#1\endcsname{\ensuremath{\widetilde{#1}}}}
\ddefloop ABCDEFGHIJKLMNOPQRSTUVWXYZ\ddefloop
\def\ddef#1{\expandafter\def\csname tc#1\endcsname{\ensuremath{\widetilde{\mathcal{#1}}}}}
\ddefloop ABCDEFGHIJKLMNOPQRSTUVWXYZ\ddefloop
\def\ddefloop#1{\ifx\ddefloop#1\else\ddef{#1}\expandafter\ddefloop\fi}
\def\ddef#1{\expandafter\def\csname scr#1\endcsname{\ensuremath{\mathscr{#1}}}}
\ddefloop ABCDEFGHIJKLMNOPQRSTUVWXYZ\ddefloop

\newcommand{\veps}{\varepsilon}

\newcommand{\ldef}{\vcentcolon=}

\usepackage{inconsolata}

\usepackage{etoolbox}
\usepackage{comment}
\newtoggle{colt}
\ifcolt 
\toggletrue{colt}
\fi
\newcommand{\colt}[1]{\iftoggle{colt}{#1}{}}
\newcommand{\arxiv}[1]{\iftoggle{colt}{}{#1}}
\newcommand{\loose}{\looseness=-1}

\usepackage{titletoc}
\usepackage{caption}

\usepackage{crossreftools}
  \newcommand{\creftitle}[1]{\crtcref{#1}}

\colt{\hbadness = 10000}
  
\usepackage{graphicx} 

 \usepackage[suppress]{color-edits}
\addauthor{df}{ForestGreen}

\addauthor{dr}{BurntOrange}

\colt{
\title[Toward A Computational Taxonomy For Reinforcement
Learning]{Necessary and Sufficient Oracles: Toward a Computational\\ Taxonomy For Reinforcement Learning}
}
\arxiv{ 
\title{Necessary and Sufficient Oracles: Toward a Computational\\ Taxonomy For Reinforcement Learning}
}
\date{\today}
\arxiv{
\author{
Dhruv Rohatgi\thanks{Email: \texttt{drohatgi@mit.edu}. This research was partially conducted during the author's internship at Microsoft Research.} \\ MIT  \and Dylan J. Foster\thanks{Email: \texttt{dylanfoster@microsoft.com}.} \\ Microsoft Research}
}

\begin{document}

\maketitle
\allowdisplaybreaks

\begin{abstract}
Algorithms for reinforcement learning (RL) in large state spaces
crucially rely on supervised learning subroutines to estimate objects
such as value functions or transition probabilities. Since only the
simplest supervised learning problems can be solved provably and
efficiently, practical performance of an RL algorithm depends on which
of these supervised learning ``oracles'' it assumes access to (and how they are implemented). But which oracles are better or worse? Is there a \emph{minimal} oracle?\loose

In this work, we clarify the impact of the choice of supervised
learning oracle on the computational complexity of RL, as quantified by the oracle strength. First, for the task of reward-free exploration in Block MDPs in the standard episodic access model---a ubiquitous setting for RL with function approximation---we identify \emph{two-context regression} as a minimal oracle, i.e. an oracle that is both necessary and sufficient (under a mild regularity assumption). Second, we identify \emph{one-context regression} as a near-minimal oracle in the stronger \emph{reset} access model, establishing a provable computational benefit of resets in the process. Third, we broaden our focus to \emph{Low-Rank MDPs}, where we give cryptographic evidence that the analogous oracle from the Block MDP setting is insufficient.\loose

\end{abstract}

\colt{
}

\arxiv{\newpage}
\section{Introduction}
\nc{\RVFS}{\texttt{RVFS}}

An overarching paradigm in modern machine learning is to reduce a complex task of interest to a simpler supervised learning task. Instances of this paradigm range from language modeling \citep{openai2023} and image generation \citep{song2019generative} to imitation learning \citep{bojarski2016end}. More broadly, the basic ansatz from supervised learning \--- that gradient descent finds good minimizers \--- underlies empirical progress in reinforcement learning \citep{mnih2015human}, artificial intelligence for games \citep{silver2018general}, and much more. State-of-the-art learning methods may not admit provable guarantees, but their empirical success is intuitively (to varying extents) justified by the basic ansatz.\loose

From a theoretical perspective, this ansatz can be exploited via \emph{oracle-efficient} algorithm design. Formally, an algorithm is oracle-efficient with respect to an oracle $\MO$ if it is computationally efficient and provably correct when given query access to $\MO$. Oracle-efficiency has been a cornerstone in the theory of online learning since the development of Follow-the-Perturbed-Leader \citep{kalai2005efficient}, which solves online linear optimization by reduction to an offline linear optimization oracle over the decision set. Recently, oracle-efficiency has become ubiquitous in theoretical reinforcement learning \citep{dann2018oracle,du2019provably,misra2020kinematic,foster2021statistical,mhammedi2023efficient,hussing2024oracle}, where end-to-end computational efficiency is often out-of-reach for settings with large observation spaces \citep{kane2022computational,golowich2024exploration}, yet heuristics based on deep learning (e.g., for estimation of value functions or transition dynamics) could plausibly work well on natural data. In reinforcement learning, as with online learning \citep{hazan2016computational,dudik2020oracle} and decision making \citep{agarwal2014taming,foster2020beyond}, the lens of oracle-efficiency has spurred numerous algorithmic improvements.

Yet, in online learning, the correct \emph{choice of oracle} was fairly clear: to solve an \emph{online} optimization problem, assume access to an oracle that solves the corresponding \emph{offline} optimization problem \citep{kalai2005efficient,hazan2016computational}. As observed by \cite{kalai2005efficient}, this assumption is essentially without loss of generality, since an online optimization algorithm must solve offline optimization as a special case. In contrast, for reinforcement learning---a more complex and structured setting, due to the interaction between the agent and environment---there is no such consensus about the ``right'' computational oracles, even for models that are by now well-established. Instead, oracle-efficiency has largely been used as a black-and-white prognostic, just separating ``reasonable'' algorithms from those requiring exhaustive enumeration \citep{dann2018oracle}.

In this work, we take a finer-grained view of oracle-efficiency---for example, are supervised learning oracles that perform regression onto value functions sufficient, or must we estimate more complex objects such as\arxiv{ (forward or inverse)} transition dynamics? Since access to an oracle is fundamentally an assumption, we are interested in the following question:
\colt{
\emph{What are the weakest computational oracles that suffice for oracle-efficient reinforcement learning?}}
\arxiv{ 
\begin{center} 
\emph{What are the weakest computational oracles that suffice for oracle-efficient reinforcement learning?}
\end{center}

}

In other words, while prior works have largely focused on the impact of differing structural assumptions on statistical complexity, we focus on the impact on computational complexity, as measured by the oracle strength. To begin this investigation, we study the task of exploration in \emph{Block Markov Decision Processes} \citep{du2019provably}---one of the most well-studied families of Markov Decision Processes (MDPs) with rich observation spaces. We identify the first \emph{minimal oracle} \citep{golowich2024exploration} for this task, under the standard episodic access model. We then consider the reset access model, and show that a strictly weaker oracle suffices. Moving beyond Block MDPs, we give cryptographic evidence of a qualitative computational separation between Block MDPs and the more general setting of \emph{Low-Rank} MDPs. 

\subsection{Background: Block MDPs and Computational Oracles}

A finite-horizon Markov Decision Process (MDP) is defined by a set of \emph{states}, a set of \emph{actions}, and an unknown transition function, which describes how the environment changes state as a result of the agent's actions. The agent learns by repeated interaction with the environment---the two most common interaction frameworks are episodic \citep{kearns2002near} and resets \citep{weisz2021query}; we study both.\footnote{We specify ``episodic RL'' to distinguish from ``RL with resets''; we say ``RL'' when the distinction is unimportant.} We focus on reward-free RL \citep{du2019provably,jin2020reward}, where the goal is \emph{exploration}: finding a set of \emph{policies} (i.e. mappings from states to actions) that cover the entire state space as well as possible. When the state space is extremely large, structural assumptions are needed to avoid statistical intractability: For most of this paper, we focus on the \emph{Block MDP} \citep{du2019provably,misra2020kinematic}, a canonical setting for RL with function approximation in which the rich observed dynamics are governed by a small (unobserved) \emph{latent state space}; just like in PAC learning, a concept class is required to model the mapping from observed states to latent states. 

\begin{definition}[Informal; see \cref{sec:prelim}]
  Let $\MX$, $\MS$, and $\MA$ be sets and let $\Phi$ be a \emph{concept class} of functions $\phi: \MX \to \MS$.\footnote{$\Phi$ is often referred to as a \emph{decoder class} in prior work; we use \emph{concept class} in analogy with PAC learning.}
  An MDP with state space $\MX$ and action space $\MA$ is a \textbf{$\Phi$-decodable Block MDP} with \emph{latent state space} $\MS$ if there is a function $\phi^\st\in\Phi$ such that for any two states $x,x'\in\MX$ and action $a \in\MA$, the transition probability from $x$ to $x'$ when the agent plays action $a$ depends only on $\phi^\st(x)$, $\phi^\st(x')$, and $a$; we refer to $\phi^{\st}(x)$ as the \emph{latent state}.\loose
\end{definition}

Henceforth, we refer to $\MX$ as the \emph{observed state space} or \emph{observation space} to distinguish from the latent state space. The concept class $\Phi$ is known to the learner, but the true decoder $\phi^{\star}\in\Phi$ is not. The statistical complexity of reward-free exploration scales polynomially in $|\MS|$, $|\MA|$, and $\log|\Phi|$ \citep{jiang2017contextual}---and crucially has no dependence on $|\MX|$, which should be thought of as exponentially large or even infinite. The \emph{computational} complexity is much more subtle. Initial algorithms required enumeration over $\Phi$ \citep{jiang2017contextual}; subsequent investigation identified oracle-efficient algorithms with respect to several different optimization oracles \citep{misra2020kinematic,zhang2022efficient,mhammedi2023representation,mhammedi2023efficient}, but no basis for comparison between these oracles has been proposed. The first works to raise the question of which oracles are \emph{necessary} were \citet{golowich2024exploring,golowich2024exploration}, who studied the \emph{one-context (realizable) regression} problem:\loose
\begin{definition}[informal; see \Cref{def:one-con-regression}]\label{def:ocr-informal}
Fix a concept class $\Phi$. Let $(x^{(i)},y^{(i)})_{i=1}^n$ be i.i.d., with $\EE[y^{(i)}\mid{}x^{(i)}] = f(\phi^\st(x^{(i)}))$ for some unknown $f: \MS \to [0,1]$ and $\phi^\st \in \Phi$. The goal of \textbf{one-context regression} is to compute a predictor $\MR: \MX \to [0,1]$ that approximates $x \mapsto f(\phi^\st(x))$.
\end{definition}

Intuitively, one-context regression can be thought of as a regression with a \emph{well-specified model} \citep{tsybakov2009nonparametric,wainwright2019high}, as the true target function depends only on the latent state $\phi^{\star}(x)$ (it can also be viewed as a generalization of PAC learning with random classification noise---see \cref{remark:ocr-to-pac}). This oracle is well-suited for estimating objects such as value functions---which depend only on the latent state in the Block MDP---and it has been implicitly used as a subroutine in many algorithms \citep{foster2020beyond,zhang2022efficient,mhammedi2023representation,mhammedi2023efficient}.\footnote{Information-theoretically,  $\text{MSE}\lesssim\frac{\abs{\cS} + \log(\abs{\Phi}\delta^{-1})}{n}$ is always possible, but not necessarily computationally efficiently.}

For any concept class $\Phi$, one-context regression is \emph{necessary} for episodic RL, i.e. there is a Cook reduction from regression to episodic RL \citep{golowich2024exploring}, so as an oracle assumption, it is without loss of generality. Unfortunately, one-context regression is also \emph{insufficient}: under a standard cryptographic assumption, there exists a concept class $\Phi$ for which there is no Cook reduction from episodic RL to regression \citep{golowich2024exploration}. Thus, one-context regression is not a minimal oracle for episodic RL.  This motivates us to consider the problem of \emph{two-context regression}. \loose

\begin{definition}[informal; see \Cref{def:two-con-regression}]\label{def:tcr-informal}
Fix a concept class $\Phi$. Let $(x_1^{(i)},x_2^{(i)},y^{(i)})_{i=1}^n$ be i.i.d., with $\EE[y^{(i)}\mid{}x_1^{(i)},x_2^{(i)}] = f(\phi^\st(x_1^{(i)}),\phi^\st(x_2^{(i)}))$ for some unknown $f: \MS\times\MS \to [0,1]$ and $\phi^\st \in \Phi$. The goal of \textbf{two-context regression} is to compute a predictor $\MR: \MX\times\MX \to [0,1]$ that approximates $(x_1,x_2)\mapsto f(\phi^\st(x_1),\phi^\st(x_2))$.
\end{definition}

Several RL algorithms use variants of this oracle \citep{misra2020kinematic,mhammedi2023representation}---roughly, to estimate (inverse) \emph{transition dynamics}---and it has been suggested that these variants may be essentially minimal \citep{golowich2024exploration}, but no evidence for this belief was known prior to this work. See \cref{sec:app_related} for a detailed discussion of prior work. 

\colt{
\begin{remark}\label{remark:oracle-subtleties}
The above oracles are phrased in terms of (1) statistical learning rather than worst-case optimization, and (2) improper learning rather than proper. These distinctions were rarely made in prior work, but are important from a complexity-theoretic lens---see \cref{remark:oracle-subtleties-app} for details.
\end{remark}
}

\arxiv{
\begin{remark}[Optimization vs. learning; proper vs. improper]\label{remark:oracle-subtleties-app}
  Many prior works in oracle-efficient RL assume access to \emph{optimization} oracles rather than statistical learning oracles. Informally, the former oracles require solving regression problems analogous to \cref{def:ocr-informal,def:tcr-informal} for \emph{arbitrary datasets} as opposed to i.i.d. and realizable datasets. This is primarily a conceptual distinction rather than technical, since in many cases a learning oracle can easily be substituted in \citep{misra2020kinematic,mhammedi2023representation}.
  However, it is important from a complexity-theoretic perspective, since statistical learning can often be substantially easier \citep{blum1998polynomial}. 
  
  A more technically salient distinction is that our definitions above allow for \emph{improper learning}, whereas almost all prior works in oracle-efficient RL for Block MDPs require \emph{proper} learning oracles\footnote{For example, the proper learning analogue of \Cref{def:ocr-informal} requires computing some predictor $\MR:\MX\to[0,1]$ with an explicit decomposition $\MR = \MR'\circ \phi$ for some $\phi \in \Phi$ and $\MR': \MS \to [0,1]$.}---an exception is the work of \cite{misra2020kinematic}, which our algorithmic results directly build on. There are many concept classes for which proper learning is $\NP$-hard, but it is considered unlikely for improper learning to be $\NP$-hard \citep{applebaum2008basing}. Since the goal of RL is to output policies, which are fundamentally improper, it seems unlikely that a proper supervised learning task could be reduced to RL (in the manner of results such as \cref{cor:regression-to-online-rl}).
\end{remark}
}

\subsection{Contributions}

We clarify the computational complexity of RL via the lens of oracle-efficiency, identifying (1) the first minimal oracle for episodic RL in Block MDPs, (2) provable benefits of \emph{reset access}, and (3) computational challenges of the more general \emph{Low-Rank} MDPs. See \cref{sec:discussion} for open questions.

\paragraph{A minimal oracle for episodic RL in Block MDPs (\Cref{sec:online}).} We show that for \emph{every} concept class $\Phi$, under a mild regularity condition, two-context regression is a minimal oracle---both \emph{sufficient} and \emph{necessary}---for reward-free episodic RL in $\Phi$-decodable Block MDPs. To show sufficiency, we generalize and simplify the algorithm $\HOMER$ of \cite{misra2020kinematic}, eliminating their reachability assumption and implementing their oracles with two-context regression. To show necessity, we give a novel reduction \emph{from} two-context regression \emph{to} reward-free RL, which simulates interaction with an appropriate MDP and ``stitches together'' the exploratory policies into a predictor. 


\paragraph{A provable computational benefit for reset access in Block MDPs (\Cref{sec:resets}).} Under the same regularity condition on $\Phi$ as before, we show that \emph{one-context regression} is a sufficient (and nearly necessary) oracle for reward-free RL in $\Phi$-decodable Block MDPs, given the additional ability to \emph{reset} to previously observed states \citep{li2021sample,yin2022efficient,
  mhammedi2024power}. Combined with the recent work of \cite{golowich2024exploration}, our result gives strong evidence that reset access has computational benefits over episodic access. 

Our algorithm uses a variant of the \emph{inverse kinematics} objective \citep{misra2020kinematic,mhammedi2023representation}, but exploits reset access to simplify the computational oracle. Previously, one-context regression was only known to be sufficient for Block MDPs with horizon $1$ or with deterministic dynamics \citep{golowich2024exploration}; the closest prior work is $\RVFS$ \citep{mhammedi2024power}, which solves RL with resets in general Block MDPs but requires an \emph{agnostic}/cost-sensitive regression oracle.\loose

\paragraph{A computational separation between Block MDPs and Low-Rank MDPs (\Cref{sec:lowrank}).} There has been recent progress on oracle-efficient algorithms for \emph{Low-Rank} MDPs \citep{modi2021model,zhang2022efficient,mhammedi2023efficient}, of which Block MDPs are a special case. However, these algorithms seemingly require a much more complex oracle. Do analogues of one- or two-context regression suffice for exploration in Low-Rank MDPs? We show that the analogue of one-context regression is cryptographically \emph{insufficient} for exploration in Low-Rank MDPs under reset access, thereby separating Low-Rank MDPs from Block MDPs. Conceptually, this separation arises from the same source as cryptographic hardness of agnostic halfspace learning \citep{tiegel2023hardness}, and points to the lack of \emph{weight function realizability} in Low-Rank MDPs as a potential computational barrier.





\section{Preliminaries}\label{sec:prelim}

To begin, we formally introduce Block MDPs, the episodic and reset access models for RL, and the computational problems: reward-free RL and one- and two-context regression. As basic notation, $[k]$ denotes the set of integers $\{1,\dots,k\}$, and $\Delta(\MZ)$ denotes the family of distributions over set $\MZ$.

\subsection{Block MDPs and Episodic RL}

For a set $\Phi \subseteq (\MX\to\MS)$, i.e. a set of functions $\phi:\MX\to\MS$, a (reward-free) $\Phi$-decodable \emph{Block MDP} \citep{du2019provably} is a tuple
$
M = (H, \MS, \MX, \MA, (\til \BP_h)_{h \in [H]}, (\til\BO_h)_{h\in[H]}, \phi^\st)
$
where $H \in \NN$ is the \emph{horizon}, $\MS$ is the \emph{latent state space}, $\MX$ is the \emph{observation space}, $\MA$ is the \emph{action set}, $\til \BP_1 \in \Delta(\MS)$ is the \emph{latent initial distribution}, $\til \BP_h: \MS \times \MA \to \MS$ is the \emph{latent transition distribution} into step $h$ (for any $h \in \{2,\dots,H\}$), $\til \BO_h: \MS \to \Delta(\MX)$ is the \emph{observation distribution} at step $h$ (for any $h \in [H]$), and $\phi^\st\in\Phi$ is the \emph{decoding function}. 
It is required that $\phi^\st(x_h) = s_h$ with probability $1$ over $x_h \sim \til \BO_h(\cdot\mid{}s_h)$, for all $h \in [H]$ and $s_h \in \MS$ (so in particular, $\til \BO_h(\cdot \mid{} s), \til \BO_h(\cdot \mid{} s')$ have disjoint supports for all $s \neq s'$).
For any $x,x'\in\MX$ and $a \in \MA$, we write $\BP_h(x'\mid{}x,a)$ to denote $\til\BP_h(\phi^\st(x')\mid{}\phi^\st(x),a) \til\BO_h(x'\mid{}\phi^\st(x'))$. We similarly define $\BP_1(x) = \til\BP_1(\phi^\st(x))\til\BO_1(x\mid{}\phi^\st(x))$. Observe that $(H, \MX, \MA,(\BP_h)_h)$ is a (reward-free) MDP, with the potentially large state space $\MX$.

\paragraph{Access model I: Episodic online RL.} Fix a Block MDP $M$ as specified above. We say that an algorithm $\Alg$ has (episodic) online access to $M$ to mean that $\Alg$ is executed in the following model. First, $\Alg$ is given $H$ and $\MA$ as input. At any time, $\Alg$ can request a new \emph{episode}. The model then draws $s_1 \sim \til \BP_1$ and $x_1 \sim \til\BO_1(\cdot\mid{}s_1)$, and sends $x_1$ to $\Alg$. The timestep of the episode is set to $h=1$. So long as $h \leq H$, the algorithm $\Alg$ can at any time play an action $a_h \in \MA$. If $h < H$, then the model draws $s_{h+1} \sim \til \BP_{h+1}(\cdot\mid{}s_h, a_h)$, and $x_{h+1} \sim \til\BO_{h+1}(\cdot\mid{}s_{h+1})$, and sends $x_{h+1}$ to $\Alg$ and increments $h$. Otherwise, the episode concludes. Note that $\Alg$ never observes the latent states $s_{1:H}$.

\paragraph{Access model II: (Episodic) online RL with resets.} We say that an algorithm $\Alg$ has reset access to $M$ to mean that $\Alg$ is given access to the following sampling oracles (in addition to $H$ and $\MA$, as before). The first sampling oracle draws $s_1 \sim \til\BP_1$ and $x_1 \sim \til\BO_1(\cdot\mid{}s_1)$, and outputs $x_1$. The second sampling oracle takes as input a step $h \in [H-1]$, any previously-seen observation $x_h \in \MX$, and an action $a_h \in \MA$; then, the oracle samples $s_{h+1} \sim \til\BP_{h+1}(\cdot\mid{}\phi^\st(x_h),a_h)$ and $x_{h+1} \sim \til\BO_{h+1}(\cdot\mid{}s_{h+1})$ and outputs $x_{h+1}$. Informally, this oracle allows the algorithm to not only sample independent episodes from $M$, but also to reset to any previously-seen observation.


\paragraph{Policies and visitations.} A (randomized) \emph{policy} $\pi = (\pi_h)_{h=1}^H$ is a collection of maps $\pi_h: \MX \to \Delta(\MA)$. We write $\Pi$ to denote the set of all policies. For $k \in [H]$, we write $d^{M,\pi}_k\in\Delta(\MS)$ to denote the distribution of $s_k$ in an episode of interaction with $M$ where $a_h\sim\pi_h(x_h)$ for each step $h$. 

\subsection{Computational Problems}

Fix sets $\MS,\MX$ and a concept class $\Phi \subseteq (\MX\to\MS)$. For the purposes of oracle reductions, an algorithm/oracle for a statistical learning problem is parametrized by its statistical efficiency (i.e. how many samples it needs in order to achieve certain accuracy) and, in the case of reward-free RL, the number of policies in the output. Throughout, we assume that the outputs of an algorithm/oracle (either policies or prediction functions) are succinctly described by circuits (and efficiently evaluatable).
For an algorithm $\Alg$ with access to an oracle $\MO$, the \emph{oracle time complexity} of $\Alg$ is the time complexity in the computational model where each query to $\MO$ takes linear time in the query length.

\begin{definition}[Reward-free RL \citep{du2019provably,jin2020reward}]\label{def:strong-rf-rl}
Fix\arxiv{ functions} $\Nrl, \Krl: (0,1/2)^2 \times \NN^2 \to \NN$. An interactive algorithm $\Alg$ is an $(\Nrl,\Krl)$-efficient \emph{reward-free (episodic/reset) RL algorithm} for $\Phi$ if the following holds. Fix $\epsilon,\delta \in (0,1/2)$,  $H \in \NN$, and a set $\MA$. Given (episodic/reset) access to a $\Phi$-decodable Block MDP $M$ with horizon $H$ and action set $\MA$, $\Alg(\epsilon,\delta, H,\MA)$ uses at most $\Nrl(\epsilon,\delta,H,|\MA|)$ (episodes/queries), and outputs a set of policies $\Psi$ with\arxiv{:
\begin{itemize}
    \item $|\Psi| \leq \Krl(\epsilon,\delta,H,|\MA|)$
    \item With probability at least $1-\delta$, it holds that for all $s \in \MS$ and $h \in [H]$, 
    \begin{equation} \max_{\pi \in \Psi} d^{M,\pi}_h(s) \geq \max_{\pi \in \Pi} d^{M,\pi}_h(s) - \epsilon.\label{eq:rfrl-pc}
    \end{equation}
\end{itemize}
}
\colt{ 
(1) $|\Psi| \leq \Krl(\epsilon,\delta,H,|\MA|)$, and (2) with probability at least $1-\delta$, for all $s \in \MS$ and $h \in [H]$, 
    \begin{equation} \max_{\pi \in \Psi} d^{M,\pi}_h(s) \geq \max_{\pi \in \Pi} d^{M,\pi}_h(s) - \epsilon.\label{eq:rfrl-pc}
    \end{equation}

}
\end{definition}

A set $\Psi$ that satisfies \cref{eq:rfrl-pc} is called a $(1,\epsilon)$-\emph{policy cover}; it formalizes the notion of reaching all latent states with near-maximal probability. While the ultimate goal in many applications is \emph{reward-directed RL}, it is straightforward to convert a reward-free RL algorithm into a reward-directed RL algorithm, and most existing oracle-efficient RL algorithms for Block MDPs (and more general classes) use some version of reward-free RL as a subroutine---see \cref{sec:app_related} for discussion and comparison to variants of \cref{def:strong-rf-rl}. We defer defining reward-free RL in more general settings to \cref{sec:lowrank}.\loose 


We now formally define the two notions of regression oracle we consider; versions of both (c.f. \cref{remark:oracle-subtleties-app}) have been used extensively throughout the reinforcement learning literature.

\begin{definition}[One-context regression]\label{def:one-con-regression}
Fix\arxiv{ a function} $\Nreg: (0,1/2)^2 \to \NN$. An algorithm $\Alg$ is an $\Nreg$-efficient one-context regression algorithm for $\Phi$ if the following holds. Fix $\epsilon,\delta \in (0,1/2)$, $n \in \NN$, $\phi \in \Phi$, $\MD \in \Delta(\MX)$, and $f: \MS \to [0,1]$. Let $(x^{(i)},y^{(i)})_{i=1}^n$ be i.i.d. samples with $x^{(i)} \sim \MD$, $y^{(i)} \in \{0,1\}$, and $\E[y^{(i)}\mid{}x^{(i)}] = f(\phi(x^{(i)}))$. If $n \geq \Nreg(\epsilon,\delta)$, then with probability at least $1-\delta$, the output of $\Alg((x^{(i)},y^{(i)})_{i=1}^n, \epsilon,\delta)$ is a circuit $\MR: \MX \to [0,1]$ satisfying \colt{$\EE_{x \sim \MD} (\MR(x) - f(\phi(x)))^2 \leq \epsilon.$}\arxiv{\[\EE_{x \sim \MD} (\MR(x) - f(\phi(x)))^2 \leq \epsilon.\]}
\end{definition}

See \cref{fig:ocr} for the graphical model structure satisfied by each sample. One-context regression is a natural oracle for estimating value functions and Bellman backups \citep{ernst2005tree,mhammedi2023efficient,golowich2024exploration}. In our definition, like that of \cite{golowich2024exploration}, the oracle is improper and only required to succeed on well-specified i.i.d. data.

\arxiv{
\begin{figure}[t]
\centering     
\subfigure[One-context regression]{\label{fig:ocr}\parbox{6cm}{\centering\includegraphics[height=3cm]{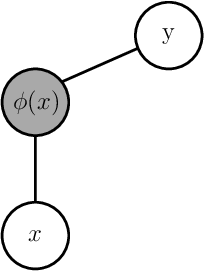}\vspace{1em}}}
\hspace{5em}
\subfigure[Two-context regression]{\label{fig:tcr}\parbox{6cm}{\centering\includegraphics[height=3cm]{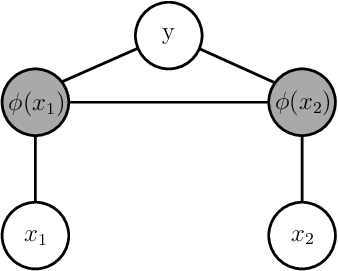}\vspace{1em}}}
\caption{Undirected graphical model representation for a single sample from \emph{(a)} one-context, or \emph{(b)} two-context regression. Note that the gray variables are unobserved.}
\end{figure}
}
\colt{ 
\begin{figure}[t]
\centering
\begin{minipage}{0.25\textwidth}
        \captionsetup{justification=raggedright, singlelinecheck=off} 

    \caption{Undirected graphical model representation for a single sample from \emph{(a)} one-context, or \emph{(b)} two-context regression. Note that the gray variables are unobserved.}
\end{minipage}%
\hspace{1em}
\begin{minipage}{0.7\textwidth}
    \subfigure[One-context regression]{\label{fig:ocr}\parbox{4.5cm}{\centering\includegraphics[height=3cm]{ocr.eps}\vspace{1em}}}
    \hspace{1em}
    \subfigure[Two-context regression]{\label{fig:tcr}\parbox{4.5cm}{\centering\includegraphics[height=3cm]{tcr.eps}\vspace{1em}}}
\end{minipage}
\end{figure}
}


\begin{definition}\label{def:realizable-distribution}
Let $\phi \in \Phi$. A distribution $\MD \in \Delta(\MX \times \MX)$ is $\phi$-realizable if $(X_1,X_2) \sim \MD$ satisfies $X_2 \perp X_1 \mid \phi(X_1)$ and $X_2 \perp X_1 \mid \phi(X_2)$.
\end{definition}

\begin{definition}[Two-context regression]\label{def:two-con-regression}
Fix\arxiv{ a function} $\Nreg: (0,1/2)^2 \to \NN$. An algorithm $\Alg$ is an $\Nreg$-efficient two-context regression algorithm for $\Phi$ if the following holds. Fix $\epsilon,\delta \in (0,1/2)$, $n \in \NN$, $\phi \in \Phi$, a $\phi$-realizable distribution $\MD \in \Delta(\MX\times\MX)$, and $f: \MS \times \MS \to [0,1]$. Let $(x_1^{(i)},x_2^{(i)},y^{(i)})_{i=1}^n$ be i.i.d. samples with $(x_1^{(i)},x_2^{(i)}) \sim \MD$, $y^{(i)} \in \{0,1\}$, and $\E[y^{(i)}\mid{}x_1^{(i)},x_2^{(i)}] = f(\phi(x_1^{(i)}),\phi(x_2^{(i)}))$. If $n \geq \Nreg(\epsilon,\delta)$, then with probability at least $1-\delta$, the output of $\Alg$ on input $(x_1^{(i)},x_2^{(i)},y^{(i)})_{i=1}^n$ is a circuit $\MR: \MX \times \MX \to [0,1]$ satisfying \colt{$\EE_{(x_1,x_2) \sim \MD} (\MR(x_1,x_2) - f(\phi(x_1),\phi(x_2)))^2 \leq \epsilon.$}\arxiv{\[\EE_{(x_1,x_2) \sim \MD} (\MR(x_1,x_2) - f(\phi(x_1),\phi(x_2)))^2 \leq \epsilon.\]}
\end{definition}

See \cref{fig:tcr} for the graphical model structure satisfied by each sample. Two-context regression is a natural oracle for estimating inverse kinematics---e.g., ``given data from two policies, predict which policy the sample came from.'' Variants of this oracle are widely-used in RL \citep{misra2020kinematic,lamb2022guaranteed,mhammedi2023representation}; see \cref{sec:app_related} for discussion.

\subsection{Additional Assumptions}

Our reductions \emph{from} regression \emph{to} RL require the following mild regularity condition on $\Phi$; essentially, it asserts that there are two latent states $\{0,1\}$ that are fully observable, irrespective of the decoding function $\phi^\st$. These extra states enable simulating ``reward'' states in the reductions. 

\begin{definition}[Regularity condition]\label{def:regular}
We say that a concept class $\Phi$ is \emph{regular} if there are two special states $\{0,1\} \in \MS \cap \MX$ that are fully observed, i.e. $\phi^{-1}(b) = \{b\}$ for all $\phi\in\Phi$ and $b \in \{0,1\}$.\loose
\end{definition}



\section{A Minimal Oracle for Episodic RL in Block MDPs}\label{sec:online}
\nc{\Nregc}{\Nreg^\circ}
\nc{\Creg}{C_{\mathsf{reg}}}

\nc{\Nrlc}{\Nrl^\circ}
\nc{\Krlc}{\Krl^\circ}
\nc{\Crl}{C_{\mathsf{RL}}}
\nc{\RegToRL}{\mathtt{RegToRL}}

In this section we show that two-context regression is a minimal oracle for reward-free episodic RL in Block MDPs. First we show that it is \emph{sufficient} (\cref{sec:episodic-suff}), then that it is \emph{necessary} (\cref{sec:episodic-nec}).

\subsection{Sufficiency: Reducing Episodic RL to Two-Context Regression}\label{sec:episodic-suff}

Our first result gives a computational reduction from reward-free RL (\cref{def:strong-rf-rl}) in a $\Phi$-decodable block MDP (with the episodic access model) to two-context regression for $\Phi$ (\cref{def:two-con-regression}). More precisely, we give a reward-free RL algorithm that requires access to a two-context regression oracle $\Reg$ (as well as episodic access to an MDP), and is oracle-efficient so long as $\Reg$ is sample-efficient. 

\begin{restatable}[Special case of \cref{thm:pco-app}]{theorem}{rltoregression}\label{cor:online-rl-to-regression}
There is a constant $C_{\ref{cor:online-rl-to-regression}}>0$ and an algorithm $\PCO$ (\cref{alg:pco}\colt{ in \cref{sec:app_online}}) so that the following holds. Let $\Phi \subseteq (\MX\to\MS)$ be any concept class, let $\Nregc,\Creg \in \NN$\arxiv{ be parameters}, and let $\Reg$ be a $\Nreg$-efficient two-context regression oracle for $\Phi$ with $\Nreg(\epsilon,\delta) := \Nregc/(\epsilon\delta)^{\Creg}$. Then $\PCO(\Reg,\Nreg,|\MS|,\cdot)$ 
is an $(\Nrl,\Krl)$-efficient reward-free RL algorithm for $\Phi$ in the \emph{episodic access model}, with $\Krl(\epsilon,\delta,H,|\MA|) \leq H^2|\MS|^2$ and $\Nrl(\epsilon,\delta,H,|\MA|) \leq \Nregc \cdot \left(\frac{H|\MA||\MS|}{\epsilon\delta}\right)^{C_{\ref{cor:online-rl-to-regression}}\Creg}.$
Moreover, the oracle time complexity of $\PCO$ is at most $\Nregc \cdot \left(\frac{H|\MA||\MS|}{\epsilon\delta}\right)^{C_{\ref{cor:online-rl-to-regression}}\Creg}.$
\end{restatable}

\arxiv{\begin{algorithm}[t]
	\caption{$\PCO$: Episodic Reward-free RL via Two-Context Regression}
	\label{alg:pco}
	\begin{algorithmic}[1]\onehalfspacing
          \State \textbf{input:} Two-context regression oracle $\Reg$; efficiency function $\Nreg: (0,1)\times(0,1)\to\NN$; \# of latent states $S$; final error tolerance $\epfinal \in (0,1)$; failure parameter $\delta \in (0,1)$; horizon $H$; action set $\MA$.
		%
        \State $N \gets \Nreg\left(\left(\frac{\epfinal\delta}{H|\MA|S}\right)^{C_{\ref{thm:pco-app}}},\left(\frac{\epfinal\delta}{H|\MA|S}\right)^{C_{\ref{thm:pco-app}}}\right)$.
        \State $\trunc \gets \epfinal/(4+HS)$.
        \State $R \gets SH$, $\tsmall \gets \frac{\trunc^2}{RS^2H^2}$, $\alpha \gets \frac{1-4\trunc}{S}$, $m \gets \frac{2}{\min(\alpha\trunc,\tsmall)}\log(S/\delta)$, $n \gets \frac{m|\MA|}{\trunc}$.
        \State $\epsilon \gets \min(\frac{\alpha^{32}\trunc^{64}\tsmall^{32}}{96^{16}H^{16}S^{16}|\MA|^{32}m^8}, \frac{\delta^4}{81n^4})$, $\gamma \gets \epsilon^{1/16}$, $\gamma' \gets 2\epsilon^{1/8}\sqrt{m|\MA|}$.
        \State $\Gamma^{(1)} \gets \emptyset$.
        \For{$1 \leq r \leq R$}
            \State $\Psi_1^{(r)} := \{\piunif\}$.
            \For{$1 \leq h < H$}
                \State $\Psi_{h+1}^{(r)} \gets \EPCO(\Reg, h, \Psi_{1:h}^{(r)}, \Gamma,n,m,N,\gamma,\gamma')$.\Comment{See \cref{alg:extend-pc-online}}
            \EndFor
            \State $\Gamma^{(r+1)} \gets \Gamma^{(r)} \cup \bigcup_{h \in [H]:|\Psi_h^{(r)}| \leq S} \Psi_h^{(r)}$.
        \EndFor
        \State \textbf{return} $\bigcup_{h \in [H]} \bigcup_{1 \leq r \leq R: |\Psi_h^{(r)}| \leq S} \Psi_h^{(r)}$.
	\end{algorithmic}
\end{algorithm}

\begin{algorithm}[t]
	\caption{$\EPCO$: Extend Policy Cover for Episodic RL}
	\label{alg:extend-pc-online}
	\begin{algorithmic}[1]\onehalfspacing
          \State \textbf{input:} Two-context regression oracle $\Reg$; step $h \in [H]$; policy covers $\Psi_{1:h}$; backup policy cover $\Gamma$; sample counts $n,m,N\in\NN$; tolerances $\gamma,\gamma' \in (0,1)$.
		%
        \State $\Reg'(\cdot) \gets \OneTwo(\Reg,\cdot)$ \Comment{See \cref{alg:onetwo}}
        \For{$a \in \MA$}
            \State $\MD_a \gets \emptyset$.
            \For{$N$ times}
                \State Sample trajectory $(x_1,a_1,\dots,x_h,a,x_{h+1}) \sim \frac{1}{2}\left(\Unif(\Psi_h) + \Unif(\Gamma)\right) \circ_h a$.
                \State Sample trajectory $(x'_1,a'_1,\dots,x'_h,a'_h,x'_{h+1}) \sim \frac{1}{2}\left(\Unif(\Psi_h) + \Unif(\Gamma)\right) \circ_h \Unif(\MA)$.
                \State Draw $y \sim \Ber(1/2)$.
                \State If $y = 1$, update dataset: $\MD_a \gets \MD_a \cup \{(x_h,x_{h+1},y)\}$.
                \State If $y = 0$, update dataset: $\MD_a \gets \MD_a \cup \{(x_h,x'_{h+1},y)\}$.
            \EndFor
            \State $\wh f_{h+1}(\cdot,\cdot;a) \gets \Reg(\MD_a)$. \label{line:hat-f-regression}
        \EndFor
        \For{$1 \leq i \leq m$}
            \State Sample trajectory $(x_1,a_1,\dots,x_h) \sim \frac{1}{2}(\Unif(\Psi_h)+\Unif(\Gamma))$.
            \State Set $x^{(i)}_h := x_h$.
        \EndFor
        \State $\Psi_{h+1} \gets \emptyset$, $\Tclus \gets \emptyset$.
        \For{$1 \leq t \leq n$}
            \State Draw $(x_1,a_1,\dots,x_h,a_h,x_{h+1}) \sim \frac{1}{2}(\Unif(\Psi_h)+\Unif(\Gamma)) \circ_h \Unif(\MA)$ and set $\ol x_{h+1}^{(t)} := x_{h+1}$.
            \State Define $\MR^{(t)}: \MX \to [0,1]$ by 
            \[\MR^{(t)}(x) := \max\left(0, 1-\frac{\max_{(i,a) \in [m]\times\MA} |\wh f_{h+1}(x_h^{(i)},\xbar_{h+1}^{(t)};a) - \wh f_{h+1}(x_h^{(i)}, x; a)|}{\gamma}\right).\] \label{line:rt-def-online}
            \If{$\max_{(i,a) \in [m]\times\MA} |\wh f_{h+1}(x_h^{(i)},\xbar_{h+1}^{(t)};a) - \wh f_{h+1}(x_h^{(i)},\xbar_{h+1}^{(t')};a)| > \gamma'$ for all $t' \in \Tclus$}\label{line:cluster-threshold-online}
                \State $\pihat^{(t)} \gets \PSDPB(h, \Reg', \MR^{(t)}, \Psi_{1:h}, \Gamma, N)$.\Comment{See \cref{alg:psdpb}}\label{line:psdp-call-online}
                \State Update $\Psi_{h+1} \gets \Psi_{h+1} \cup \{\pihat^{(t)}\}$.
                \State Update $\Tclus \gets \Tclus \cup \{t\}$.
            \EndIf
        \EndFor 
        \State \textbf{return} $\Psi_{h+1}$.
	\end{algorithmic}
\end{algorithm}
}

While \cref{cor:online-rl-to-regression} assumes a natural parametric scaling for $\Nreg$, the full result (\cref{thm:pco-app}) applies to any efficiency function. Note that $\Nregc,\Creg$ will naturally be larger for more complex concept classes $\Phi$, but there are no ``hidden'' dependencies on $\Phi$. Informally, \cref{cor:online-rl-to-regression} shows that two-context regression is a \emph{sufficient} oracle for reward-free episodic RL in block MDPs.

\paragraph{Proof overview.} The main subroutine of $\PCO$ is $\EPCO$, which strongly resembles the $\HOMER$ algorithm \citep{misra2020kinematic}.\footnote{Specifically, the ``non-quantized'' version of $\HOMER$ described in Appendix~E of \cite{misra2020kinematic}.} The basic idea of $\HOMER$ (and many other oracle-efficient RL algorithms \colt{\citep{du2019provably,mhammedi2023efficient}}\arxiv{\citep{du2019provably,mhammedi2023efficient,golowich2024exploring}}) is to iteratively learn policy covers $\Psi_{1:H}$ for each layer of the MDP. In $\HOMER$, given policy covers for layers $1,\dots,h$, a policy cover for layer $h+1$ is learned by applying the policy optimization method $\PSDP$ \citep{bagnell2003policy} to a set of carefully-designed internal reward functions at layer $h+1$. Ideally, the reward functions should incentivize reaching individual latent states; of course, latent states are not in general identifiable, so this criterion must be relaxed. Instead, the rewards are constructed in two steps. First, use two-context regression (with an appropriately-generated dataset, inspired by contrastive learning methods) to estimate the following \emph{kinematics function}:
\begin{equation} f_{h+1}(x_h,x_{h+1};a_h) := \frac{\BP_{h+1}(x_{h+1}\mid{}x_h,a_h)}{\BP_{h+1}(x_{h+1}\mid{}x_h,a_h) + F_{h+1}(x_{h+1})},\label{eq:kinematics-intro}\end{equation}
where $F_{h+1}$ is a certain normalization function. Second, sample a large number of ``cluster center'' observations $(\xbar_{h+1}^{(t)})$, 
and, for each, define a reward $\MR^{(t)}$ (derived from \cref{eq:kinematics-intro}) which is large precisely for those observations $x_{h+1}$ that have approximately the same kinematics as $\xbar^{(t)}_{h+1}$.

$\PCO$ follows the same blueprint, with two modifications. First, $\HOMER$ uses an offline cost-sensitive classification oracle for $\PSDP$. We use an alternative implementation of $\PSDP$ \citep{mhammedi2023representation} which can be implemented with one-context regression (\cref{lemma:psdp-trunc-online}) and hence two-context regression (via \cref{prop:onetwo}). Second, the analysis of $\HOMER$ assumes that all states are reachable with non-negligible probability. We remove this assumption via truncation arguments and an iterative discovery method \citep{golowich2024exploring}---this is the reason for the outer loop in $\PCO$. See \cref{sec:app_online} for the formal \colt{algorithm and }analysis.

\subsection{Necessity: Reducing Two-Context Regression to Episodic RL}\label{sec:episodic-nec}

Our second result provides a converse of \cref{cor:online-rl-to-regression}. We give a two-context regression algorithm that requires access to a reward-free episodic RL \emph{oracle} (\cref{def:strong-rf-rl}), and is oracle-efficient so long as the oracle is sample-efficient.\loose

\begin{restatable}{theorem}{regtorl}\label{cor:regression-to-online-rl}
There is a constant $C_{\ref{cor:regression-to-online-rl}}>0$ and an algorithm $\RegToRL$ (\cref{alg:regtorl} in \cref{sec:app_minimality}) so that the following holds. Let $\Phiaug \subseteq (\Xaug\to\Saug)$ be any regular concept class (\cref{def:regular}), let $\Nrlc,\Crl \in \NN$\arxiv{ be parameters}, and let $\MO$ be a $(\Nrl,\Krl)$-efficient reward-free episodic RL oracle for $\Phiaug$, with $\max(\Nrl(\epsilon,\delta,H,A),\Krl(\epsilon,\delta,H,A)) \leq \Nrlc \cdot (AH/\epsilon\delta)^{\Crl}$. Then $\RegToRL(\MO,\cdot)$ is an $\Nreg$-efficient two-context regression algorithm (\cref{def:two-con-regression}) for $\Phiaug$ with 
$\Nreg(\epsilon,\delta) \leq \Nrlc \left(|\MS|/(\epsilon\delta)\right)^{C_{\ref{cor:regression-to-online-rl}} \cdot \Crl}$ and with oracle time complexity at most $\Nrlc \left(|\MS|/(\epsilon\delta)\right)^{C_{\ref{cor:regression-to-online-rl}} \cdot \Crl}$.
\end{restatable}

\cref{cor:online-rl-to-regression} and \cref{cor:regression-to-online-rl} together show that two-context regression is a minimal oracle for reward-free episodic RL (for any regular $\Phi$). \cref{cor:regression-to-online-rl} strengthens \cite[Proposition B.2]{golowich2024exploration}, which reduces \emph{one}-context regression to RL---to our knowledge, the only prior result of this flavor. We require regularity\footnote{ \cite{golowich2024exploration} reduce to reward-directed RL and do not use regularity; however, under regularity it is simple to adapt their reduction to reward-free RL; this adaptation is the version we sketch below.} for the following technical reason: a concept class is regular if and only if it can be obtained by augmenting some base concept class $\Phi \subseteq (\MX\to\MS)$ with two fully observed states $\{0,1\}$ (see \cref{def:phiaug} for the formal definition). It is easy to reduce regression with $\Phiaug$ to regression with $\Phi$ (\cref{prop:twoaug}), so it suffices to reduce regression with $\Phi$ to reward-free RL with $\Phiaug$; the extra states provide useful flexibility since $\Phi$ is otherwise arbitrary.

We now sketch the prior reduction before discussing how to strengthen it.

\paragraph{Recap: Reducing one-context regression to RL.} Given a one-context regression dataset $(x^{(i)},y^{(i)})_{i=1}^n$ with samples in $\MX\times\{0,1\}$, consider simulating an MDP with horizon $H=2$, where the initial observation lies in $\MX$ and the second observation lies in $\{0,1\}$. The goal of this reduction is that a \emph{policy} that visits state $1$ at step $2$ with near-optimal probability corresponds to an accurate \emph{prediction function} for the regression. This can be achieved by defining the action space $\MA$ to be a discretization of the interval $[0,1]$, and requiring the policy to ``guess'' $y^{(i)}$ after observing $x^{(i)}$. 

More formally, the reduction uses a fresh datapoint $(x^{(i)},y^{(i)})$ to simulate each new episode. It first passes observation $x_1 := x^{(i)} \in \MX$ to the RL agent. When the agent plays an action $a_1 \in \MA \subset [0,1]$, the reduction passes the observation $x_2 \in \{0,1\}$ sampled from $\Ber(1 - (a - y^{(i)})^2).$ It can be checked that this procedure in fact simulates a $\Phi$-decodable MDP, so long as the regression dataset satisfied the desideratum that $\EE[y^{(i)}\mid{} x^{(i)}]$ only depends on $\phi(x^{(i)})$ for some $\phi \in \Phi$.

\paragraph{Reducing two-context regression to RL.} Can we generalize the above construction by increasing the horizon? Given a two-context regression sample $(x_1^{(i)}, x_2^{(i)},y^{(i)})$, consider passing both contexts to the RL agent one-by-one and then requiring the policy to ``guess'' $y^{(i)}$. Unfortunately, $y^{(i)}$ may depend on both $x_1^{(i)}$ and $x_2^{(i)}$, so the simulated decision process is non-Markovian. More broadly, this points to a representational obstacle: optimal policies for a Block MDP are Markovian and hence mappings $\MX \to \MA$. But for two-context regression, a predictor is a function on $\MX \times \MX$. Thus, any successful reduction will have to ``stitch together'' \emph{multiple} policies produced by the RL oracle.

To motivate our reduction, we recall why two-context regression was useful for RL in the first place: essentially, it was useful to estimate (some transformation of) transition probabilities between two consecutive states---see \cref{eq:kinematics-intro}. This suggests using the regression data to simulate an MDP where \emph{the probability of transitioning from $x_1^{(i)}$ to $x_2^{(i)}$ depends on $y^{(i)}$}.

Formally, our reduction simulates a horizon-$2$ MDP with first observation in $\MX$, second observation in $\MX\sqcup\{0\}$, and action space $\MA \subset [0,1]$. For each sample $(x_1^{(i)},x_2^{(i)},y^{(i)})$, the reduction simulates an episode of interaction with the MDP as follows:\arxiv{
\begin{enumerate}
\item First, the reduction passes observation $x_1 := x_1^{(i)} \in \MX$ to the RL agent.
\item Second, when the agent plays an action $a_1 \in \MA \subset [0,1]$, the reduction passes the observation $x_2 \in \MX$ sampled as follows: 
\[x_2 := \begin{cases} 
x_2^{(i)} & \text{ with probability } 1 - (a_1 - y^{(i)})^2 \\ 
0 & \text{ with probability } (a_1 - y^{(i)})^2 
\end{cases}.\]
\end{enumerate}
}\colt{ 
first, pass $x_1 := x_1^{(i)} \in \MX$ to the RL oracle. Second, when the oracle plays an action $a_1 \in \MA \subset [0,1]$, pass $x_2 \in \MX\sqcup\{0\}$ sampled as: 
\[x_2 := \begin{cases} x_2^{(i)} & \text{ with probability } 1 - (a_1 - y^{(i)})^2 \\ 0 & \text{ with probability } (a_1 - y^{(i)})^2 \end{cases}.\]
}
It can be checked that this procedure simulates a $\Phi$-decodable block MDP under the realizability assumptions on the regression dataset. Additionally, if we fix a latent state $s \in \MS$, then maximizing the probability of reaching state $s$ at step $2$ is equivalent to predicting $y^{(i)}$ conditioned on the observation $x_1^{(i)}$ \emph{and} the knowledge that $\phi^\st(x_2^{(i)}) = s$. Thus, any policy $\pi_s: \MX \to \MA$ that visits $s$ at step $2$ with near-maximal probability must approximately minimize a \emph{restricted} regression loss:
\[L_s(\pi_s) := \EE_{x_1,x_2} \left[\mathbbm{1}[\phi^\st(x_2) = s] (\pi_s(x_1) - f(\phi^\st(x_1),\phi^\st(x_2)))^2\right]\]
where $f: \MS\times\MS\to[0,1]$ is as in \cref{def:two-con-regression}. By assumption, the reward-free RL oracle will return a set of policies $\Psi$ that contains at least one such policy $\pi_s$ for each $s \in \MS$. 

The remaining challenge is how to stitch together these policies into a single predictor: given $(x_1,x_2) \in \MX\times\MX$, how do we use $x_2$ to identify the policy $\pi \in \Psi$ for which $\pi(x_1)$ is a good prediction? We accomplish this using \emph{one-context regression}, which (as discussed above) is reducible to reward-free RL. In particular, for each policy $\pi \in \Psi$, we construct datapoints of the form $(x_2^{(i)}, (\pi(x_1^{(i)}) - y^{(i)})^2)$. Applying one-context regression yields an estimate of (an appropriate transformation of) the map $x_2 \mapsto L_{\phi^\st(x_2)}(\pi)$. After learning this map, the final predictor $\MR:\MX\times\MX\to[0,1]$ is defined as follows: on input $(x_1,x_2)$, it outputs $\pihat^{x_2}(x_1)$, where $\pihat^{x_2}\in\Psi$ minimizes the estimated loss. See \cref{sec:app_minimality} for the formal reduction and analysis.

\section{A Simpler Oracle for RL in Block MDPs with Reset Access}\label{sec:resets}

We now turn to the \emph{RL with resets} access model, which is more permissive than episodic access. In this model, we give a reward-free RL algorithm $\PCR$ that only requires access to a \emph{one-context} regression oracle $\Reg$, and is oracle-efficient so long as $\Reg$ is sample-efficient.


\begin{theorem}[Special case of \cref{thm:pcr-app}]\label{cor:reset-rl-to-regression}
There is a constant $C_{\ref{cor:reset-rl-to-regression}}>0$ and an algorithm $\PCR$ (\cref{alg:pcr}\colt{ in \cref{sec:app_resets}}) so that the following holds. Fix $\Phi \subseteq (\MX\to\MS)$ and $\Nregc,\Creg \in \NN$. Let $\Reg$ be a $\Nreg$-efficient one-context regression oracle for $\Phi$ with $\Nreg(\epsilon,\delta) := \Nregc/(\epsilon\delta)^{\Creg}$.
Then $\PCR(\Reg,\Nreg,|\MS|,\cdot)$ 
is an $(\Nrl,\Krl)$-efficient reward-free RL algorithm for $\Phi$ in the \emph{reset access model}, with $\Krl(\epsilon,\delta,H,|\MA|) \leq H^2|\MS|^2$ and $\Nrl(\epsilon,\delta,H,|\MA|) \leq \Nregc \cdot \left(H|\MA||\MS|/(\epsilon\delta)\right)^{C_{\ref{cor:reset-rl-to-regression}}\Creg}.$
Moreover, the oracle time complexity of $\PCR$ is at most $\Nregc \cdot \left(H|\MA||\MS|/(\epsilon\delta)\right)^{C_{\ref{cor:reset-rl-to-regression}}\Creg}.$
\end{theorem}

\arxiv{\begin{algorithm}[t]
	\caption{$\PCR$: RL with Resets via One-Context Regression}
	\label{alg:pcr}
	\begin{algorithmic}[1]\onehalfspacing
		          \State \textbf{input:} One-context regression oracle $\Reg$; efficiency function $\Nreg: (0,1)\times(0,1)\to\NN$; \# of latent states $S$; final error tolerance $\epfinal \in (0,1)$; failure parameter $\delta \in (0,1)$; horizon $H$; action set $\MA$.
		%
        \State $N \gets \Nreg\left(\left(\frac{\epfinal\delta}{H|\MA|S}\right)^{C_{\ref{thm:pcr-app}}},\left(\frac{\epfinal\delta}{H|\MA|S}\right)^{C_{\ref{thm:pcr-app}}}\right)$.
        \State $\trunc \gets \epfinal/(4+HS)$.
        \State $R \gets SH$, $\tsmall \gets \frac{\trunc^2}{RS^2H^2}$, $\alpha \gets \frac{1-4\trunc}{S}$, $m \gets \frac{2}{\min(\alpha\trunc,\tsmall)}\log(S/\delta)$, $n \gets \frac{m|\MA|}{\trunc}\log(S/\delta)$.
        \State $\epsilon \gets \min(\frac{\alpha^8\trunc^{16}\tsmall^{8}}{6^{8}H^{8}m^{12}|\MA|^{12}}, \frac{\delta^4}{m^2|\MA|^2n^4})$, $\gamma \gets \epsilon^{1/8}$, $\gamma' \gets 2\epsilon^{1/4}$.
        \State $\Gamma^{(1)} \gets \emptyset$.
        \For{$1 \leq r \leq R$}
            \State $\Psi_1^{(r)} := \{\piunif\}$.
            \For{$1 \leq h < H$}
                \State $\Psi_{h+1}^{(r)} \gets \EPCR(\Reg, h, \Psi_{1:h}^{(r)}, \Gamma,n,m,N,\gamma,\gamma')$.
            \EndFor
            \State $\Gamma^{(r+1)} \gets \Gamma^{(r)} \cup \bigcup_{h \in [H]:|\Psi_h^{(r)}| \leq S} \Psi_h^{(r)}$.
        \EndFor
        \State \textbf{return:} $\bigcup_{h \in [H]} \bigcup_{1 \leq r \leq R: |\Psi_h^{(r)}| \leq S} \Psi_h^{(r)}$.
	\end{algorithmic}
\end{algorithm}

\begin{algorithm}[!htb]
	\caption{$\EPCR$: Extend Policy Cover in Reset Model}
	\label{alg:epcr}
	\begin{algorithmic}[1]\onehalfspacing
          \State \textbf{input:} One-context regression oracle $\Reg$; step $h \in [H]$; policy covers $\Psi_{1:h}$; backup policy cover $\Gamma$; sample counts $n,m,N\in\NN$; tolerances $\gamma,\gamma' \in (0,1)$.
		%
        \State $\MC \gets \emptyset$.
        \For{$1 \leq i \leq m$}
            \State Sample trajectory $(x_1,a_1,\dots,x_h) \sim \frac{1}{2}\left(\Unif(\Psi_h) + \Unif(\Gamma)\right)$.
            \State Set $x_h^{(i)} := x_h$.
            \State Update dataset: $\MC \gets \MC \cup \{x_h\}$.
        \EndFor
        \For{$1 \leq i \leq m$}
            \For{$a \in \MA$}
                \State $\MD_{i,a} \gets \emptyset$.
                \For{$N$ times}
                    \State Sample $j \sim \Unif([m])$ and $a_h \sim \Unif(\MA)$.
                    \State Reset to $x_h := x_h^{(j)}$ and sample $x_{h+1} \sim \BP^M_{h+1}(\cdot|x_h, a_h)$.
                    \State Update dataset: $\MD_{i,a} \gets \MD_{i,a} \cap \{(x_{h+1}, \mathbbm{1}[j=i \land a_h = a])\}$.
                \EndFor
                \State $\what_{h+1}(\cdot;i,a) \gets \Reg(\MD_{i,a})$. \Comment{$\what_{h+1}(\cdot; i,a): \MX \to [0,1]$}
            \EndFor
        \EndFor
        \State $\Psi_{h+1} \gets \emptyset$, $\Tclus \gets \emptyset$.
        \For{$1 \leq t \leq n$}
            \State Sample $j \sim \Unif([m])$.
            \State Reset to $x_h := x^{(j)}$ and sample trajectory $(x_h,a_h,x_{h+1}) \sim \Unif(\MA)$. Write $\xbar_{h+1}^{(t)} := x_{h+1}$.
            \State Define $\MR^{(t)}: \MX \to [0,1]$ by 
            \[\MR^{(t)}(x) := \max\left(0, 1-\frac{\max_{(i,a) \in [m]\times\MA} |\what_{h+1}(\xbar_{h+1}^{(t)};i,a) - \what_{h+1}(x;i,a)|}{\gamma}\right).\] \label{line:rt-def}
            \If{$\max_{(i,a) \in [m]\times\MA} |\what_{h+1}(\xbar_{h+1}^{(t)};i,a) - \what_{h+1}(\xbar_{h+1}^{(t')};i,a)| > \gamma'$ for all $t' \in \Tclus$}\label{line:cluster-test}
                \State $\pihat^{(t)} \gets \PSDPB(h, \Reg, \MR^{(t)}, \Psi_{1:h},\Gamma, N)$.\Comment{See \cref{alg:psdpb}}\label{line:psdp-call}
                \State Update $\Psi_{h+1} \gets \Psi_{h+1} \cup \{\pihat^{(t)}\}$.
                \State Update $\Tclus \gets \Tclus \cup \{t\}$.
            \EndIf
        \EndFor 
        \State \textbf{return:} $\Psi_{h+1}$.
	\end{algorithmic}
\end{algorithm}

}

Since RL in the \emph{episodic} access model is provably harder than one-context regression \citep{golowich2024exploration}, \cref{cor:reset-rl-to-regression} gives a provable computational benefit of reset access---to our knowledge, the first of its kind, though statistical benefits are known in different settings (\cref{sec:related-alg-block}).

\paragraph{Overview of algorithm design and proof.}
The $\PCR$ algorithm follows a similar blueprint to $\PCO$ (and $\HOMER$): given policy covers for the first $h$ layers, design internal reward functions for layer $h+1$ based on kinematics, and use $\PSDP$ to optimize these reward functions and obtain a policy cover for layer $h+1$ (recall that $\PSDP$ itself only requires a one-context regression oracle). The point of departure is in how the kinematics are estimated. In $\PCO$, two-context regression is applied on a dataset consisting of some pairs $(x_h^{(i)},x_{h+1}^{(i)})$ generated using the MDP's real transitions (and labeled $y^{(i)} = 1$), and some pairs $(x_h^{(j)}, x_{h+1}^{(j)})$ generated using ``fake'' transitions (and labeled $y^{(j)} = 0$). The Bayes optimal predictor $\EE[y\mid{}x_1,x_2]$ turns out to be exactly the kinematics function from \cref{eq:kinematics-intro}.

In the setting of \cref{cor:reset-rl-to-regression}, our algorithm $\PCR$ does not have access to two-context regression, but it does have the ability to ``reset'' the MDP to previously-seen observations. To use this, the algorithm first generates a large number of ``discriminator'' observations $(x_h^{(i)})_{i=1}^m$. For each fixed $i \in [m]$ and $a \in \MA$, the algorithm then learns to predict the following predicate: ``did the observation $x_{h+1}$ come from $x_h^{(i)}$ and action $a$, or from some other $x_h^{(j)}$ or other action?''. More precisely, using the power of resets, $\PCR$ can generate a large number of samples $(x,y)$ where $x \sim \BP_{h+1}(\cdot\mid{}x_h^{(j)},a_h)$ for $(j,a_h) \sim \Unif([m]\times\MA)$, and $y := \mathbbm{1}[j=i \land a_h=a]$. The Bayes optimal predictor $\EE[y\mid{}x]$ is then:\loose
\begin{equation} w_{h+1}(x_{h+1};i,a) := \frac{\BP_{h+1}(x_{h+1}\mid{}x_h^{(i)},a)}{\sum_{j,a_h} \BP_{h+1}(x_{h+1}\mid{}x_h^{(j)},a_h)} = \frac{\til\BP_{h+1}(\phi^\st(x_{h+1})\mid{}\phi^\st(x_h^{(i)}),a)}{\sum_{j,a_h} \til\BP_{h+1}(\phi^\st(x_{h+1})\mid{}\phi^\st(x_h^{(j)}),a_h)}.\label{eq:pcr-kinematics}\end{equation}
This function can be estimated via one-context regression, and then used to define internal reward functions at layer $h+1$, similar to $\PCO$. See \cref{sec:app_resets} for the formal \colt{algorithm and }analysis. 

\paragraph{Is one-context regression minimal?} 
By a variant of \cite[Proposition B.2]{golowich2024exploration}, the \emph{noiseless} version of one-context regression is necessary in the reset setting (\cref{prop:noiseless-onered}). While there are classical examples where noisy (but realizable) PAC learning is believed to be computationally harder than noiseless learning \citep{blum2003noise}, in many natural settings they are comparable---see e.g. halfspace learning \citep{blum1998polynomial,diakonikolas2023strongly} and the statistical query model \citep{kearns1998efficient}. In this sense, we expect that (noisy) one-context regression is \emph{nearly} minimal.\loose

\section{A Computational Separation for Low-Rank MDPs}\label{sec:lowrank}

\nc{\bphi}{\boldsymbol{\phi}}
\nc{\nunull}{\nu_{\mathsf{null}}}
\nc{\bast}{\ba^\st}
\nc{\Algbar}{\overline{\Alg}}
\nc{\MDnull}{\MD^\circ}
\nc{\epdisc}{\varepsilon_{\mathsf{disc}}}
\nc{\halfspacealg}{\mathtt{HalfspaceLearn}}
\nc{\cdisc}{c_{\mathsf{disc}}}
\nc{\WH}{(\fS^{t-1})^H}
\nc{\MObar}{\ol\MO}
\nc{\phist}{\phi^\st}
\nc{\must}{\mu^\st}
\nc{\Philin}{\Phi^{\mathsf{lin}}}
\nc{\philin}{\phi^{\mathsf{lin}}}
\nc{\fd}{\mathfrak{d}}
\nc{\Xgood}{\MX_{\mathsf{good}}}
\nc{\mulin}{\mu^{\mathsf{lin}}}

\nc{\nsmall}{t}
\nc{\nlarge}{n}
\nc{\fS}{\mathfrak{S}}

We now move beyond Block MDPs. \emph{Low-Rank} MDPs \citep{modi2021model} are perhaps the simplest commonly-studied model class that generalizes the Block MDP.
While Low-Rank MDPs admit oracle-efficient RL algorithms, the \emph{oracles} used seem significantly more complex---both in the episodic and reset access model---than the oracles used for Block MDPs. In particular, $\VOX$ \citep{mhammedi2023efficient}, which operates in the episodic access model, uses a proper min-max optimization oracle, and $\RVFS$ \citep{mhammedi2024power}, which requires reset access, uses an agnostic and cost-sensitive regression oracle. Comparing to our results for Block MDPs, it is natural to ask whether this apparent gulf in computational tractability is real, and if so, what the structural source is. In this section we make progress on this question, with a focus on the reset access model.\loose

\paragraph{Preliminaries.} An MDP $M$ is Low-Rank with rank $d$ if there are maps $\philin_h:\MX\times\MA\to\RR^d$ and $\mulin_{h+1}:\MX\to\RR^d$ such that the transitions have the following factorization:
\arxiv{\[\BP_{h+1}(x_{h+1}\mid{}x_h,a_h) = \langle \philin_h(x_h,a_h), \mulin_{h+1}(x_{h+1})\rangle.\]}
\colt{$\BP_{h+1}(x_{h+1}\mid{}x_h,a_h) = \langle \philin_h(x_h,a_h), \mulin_{h+1}(x_{h+1})\rangle.$}
Prior \dfedit{model-free} work assumes that $\philin_{1:H}$ lie in a known feature class $\Philin$, but the dual features $\mulin_{1:H}$ are arbitrary \citep{modi2021model,mhammedi2023efficient}.\footnote{Model-based works \citep{agarwal2020flambe,uehara2022representation} also assume that $\mulin_{1:H}$ lie in a known dual feature class.\loose} Given implicit access to $\Philin$ via some oracle(s), the typical goal is to design RL algorithms with time and oracle complexity polynomial in $d$, the horizon $H$, and the number of actions $\MA$.\arxiv{
 } Block MDPs with concept class $\Phi$ can be embedded in a class of Low-Rank MDPs with rank $d := |\MS||\MA|$: simply define $\Philin := \{\philin: \phi\in\Phi\}$ where $\philin:\MX\times\MA\to\RR^d$ \arxiv{is defined by mapping}\colt{maps} to an appropriate basis vector, i.e. $\philin(x,a) := e_{\phi(x),a}$. Analogously, there is a natural extension of one-context regression to Low-Rank MDPs.\loose
\begin{definition}[Informal]\label{Def:low-rank-reg}
The goal of \textbf{one-context low-rank regression} over $\Philin$ is the following: given an i.i.d. dataset $(x^{(i)},a^{(i)},y^{(i)})_i$ satisfying the realizability assumption $\EE[y^{(i)}\mid{}x^{(i)},a^{(i)}] = \langle \philin(x^{(i)},a^{(i)}), \theta\rangle$ for unknown $\philin\in\Philin$ and $\theta\in\RR^d$,
estimate the function $(x,a) \mapsto \EE[y\mid{}x,a]$.
\end{definition}
This oracle suffices to implement $\PSDP$ in Low-Rank MDPs---so \emph{given} a policy cover,\footnote{See \cite{mhammedi2023efficient} for the precise notion of a policy cover in Low-Rank MDPs.} we can optimize a reward \citep{mhammedi2023efficient}. In analogy with \cref{cor:reset-rl-to-regression}, we ask: is this oracle sufficient for the full task of reward-free RL with resets? We show that it is not, and shed light on why.\loose

\arxiv{\paragraph{Generalized Block MDPs.}}\colt{\vspace{0.3em}\noindent\textbf{Generalized Block MDPs.}} It is well known that Block MDPs are a special case of Low-Rank MDPs \citep{modi2021model,zhang2022efficient}. To highlight that there are \emph{multiple} important assumptions in the Block MDP definition, and to gain a finer understanding of how these assumptions interact with computational tractability, we introduce---and prove our hardness result in---an intermediate model class between Block MDPs and Low-Rank MDPs. Let $\Phi \subseteq (\MX\to\MS)$ be any concept class. A $\Phi$-decodable \emph{Generalized} Block MDP is an MDP $M = (H, \MX,\MA,(\BP_h)_h)$ with the property that there exist $\phist_1,\dots,\phist_H \in \Phi$ so that $\BP_{h+1}(x_{h+1}\mid{}x_h,a_h)$ is a function of $x_{h+1}$, $\phist_h(x_h)$, and $a_h$.

\begin{figure}[t]
\centering
\includegraphics[width=0.4\textwidth]{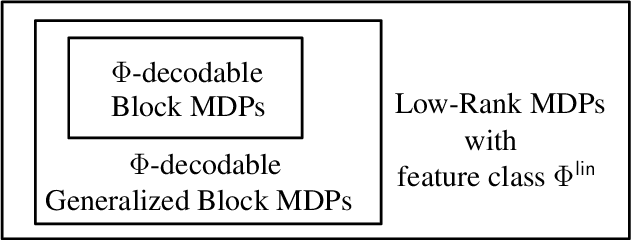}
\caption{Containment diagram for model classes discussed in \cref{sec:lowrank}. Here, $\Phi \subseteq (\MX\to\MS)$ is any concept class, and $\Philin := \{\philin:\phi\in\Phi\}$ where $\philin(x,a) := e_{\phi(x),a}\in\RR^{|\MS||\MA|}$.}\label{fig:classes}
\end{figure}

Generalized Block MDPs are still Low-Rank MDPs via the same embedding discussed above (\cref{prop:genblock-is-lowrank})---the features are appropriate basis vectors. However, compared to standard Block MDPs, the transition probability $\BP_{h+1}(x_{h+1}\mid{}x_h,a_h)$ can now depend arbitrarily on $x_{h+1}$.\footnote{We also defined standard Block MDPs to have a single decoding function $\phi^\st$ rather than one per layer, but this is was a superficial choice made for notational simplicity; our RL algorithms from prior sections still apply if there is one function per layer, and the one-context regression oracle needed by $\PCR$ doesn't change.} Moreover, in Generalized Block MDPs, \cref{Def:low-rank-reg} simplifies back to \cref{def:one-con-regression}, our original definition of one-context regression (\cref{prop:gen-block-ocr}). These connections are summarized in \cref{fig:classes}. 

\arxiv{\paragraph{Hardness for Generalized Block MDPs.}}\colt{\vspace{0.3em}\noindent\textbf{Hardness for Generalized Block MDPs.}}
Unfortunately, the arbitrary dependence of the transition probability on $x_{h+1}$ \arxiv{in Generalized Block MDPs} defeats algorithms such as $\PCR$, since the ideal kinematics (see \cref{eq:pcr-kinematics}) are no longer a function of just $\phi^\st(x_{h+1})$. But could there be a more clever algorithm that avoids needing to estimate such quantities? In \cref{thm:halfspace-main}, we show that there is an inherent computational barrier. 

For $n \in \NN$, let $\Phi_n := \{\phi^\theta:\bbR^{n}\to\crl{0,1}\mid{}\theta\in\RR^n\}$ be the concept class of linear threshold\arxiv{ functions}\colt{s}: $\phi^\theta(x) \ldef \mathbbm{1}[\langle x,\theta\rangle \geq 0]$. We prove that one-context regression for $\Phi_n$ is computationally tractable, but reward-free RL with resets for $\Phi_n$-decodable Generalized Block MDPs is cryptographically hard.

\begin{theorem}[\cref{thm:halfspace-reg}+\cref{thm:halfspace-rl-hard}]\label{thm:halfspace-main}
There is an algorithm for one-context regression with concept class $\Phi_n$, that achieves error $\epsilon$ with probability at least $1-\delta$ and has time complexity $\poly(n,1/\epsilon,1/\delta)$.
In contrast, suppose there exists an $\Alg^M$ that---given interactive reset access to any $\Phi_n$-decodable Generalized Block MDP $M$ with horizon $H$, observation space $\MX=\RR^n$, latent state space $\MS=\{0,1\}$, and action space $\MA$---has time complexity $\poly(n,H,|\MA|)$ and produces a set of policies $\Psi$ satisfying the following guarantee with probability at least $1/2$:
\begin{equation} \forall s \in \MS: \max_{\pi\in\Psi} d^{M,\pi}_H(s) \geq \frac{1}{\poly(|\MA|, |\MS|, H)} \left(\max_{\pi\in\Pi} d^{M,\pi}_H(s) - \frac{1}{8}\right).\label{eq:gen-block-rfrl}\end{equation}
Then the Continuous Learning With Errors (cLWE) hardness assumption (\cref{ass:clwe}) is false. 
\end{theorem}
\cref{thm:halfspace-main} rules out any reduction from reward-free RL with resets in $\Phi_n$-decodable Generalized Block MDPs to one-context regression, where the oracle time complexity of the reduction is allowed to scale polynomially in all relevant parameters ($H$, $|\MS|$, $|\MA|$, $1/\epsilon$, $1/\delta$, and the description length $n$ of an observation). By \cref{cor:reset-rl-to-regression}, this separates standard Block MDPs from Low-Rank MDPs. \arxiv{Additionally, the cLWE assumption can in turn be based on classical LWE and therefore on worst-case hardness of approximate shortest vector on lattices---see \cite[Corollary 3]{gupte2022continuous} and references.}\colt{We remark that the cLWE assumption can be based on classical LWE \citep{gupte2022continuous}.\loose
}



\begin{remark}[On the reward-free solution concept]
Comparing \cref{eq:gen-block-rfrl} to \cref{def:strong-rf-rl}, we have \emph{relaxed} the notion of exploration to allow for multiplicative error in the visitation probability. In other words, we are only asking for an $(\alpha,\epsilon)$-policy cover where $\alpha = 1/\poly(H,|\MA|,|\MS|)$. The reason is that for Generalized Block MDPs, a $(1,\epsilon)$-policy cover may not even exist. Fortunately, an $(\alpha,\epsilon)$-policy cover does always exist, and finding it is statistically tractable (\cref{prop:halfspace-rl-statistical}), indicating that the source of the hardness in \cref{thm:halfspace-rl-hard} is indeed purely computational.
\end{remark}

\paragraph{Proof overview.} In \cref{thm:halfspace-main} (proven in \cref{app:lowrank}), the algorithmic result is a (simple, but not immediate) consequence of seminal work on learning halfspaces \citep{blum2003noise,diakonikolas2023strongly}. The proof of the hardness result builds on recent work by \cite{tiegel2023hardness} on the cryptographic hardness of agnostic learning of halfspaces. Fix $n \in \NN$ and let $\fS^{t-1}$ denote the unit sphere in $t \approx \polylog(n)$ dimensions. \cite{tiegel2023hardness} shows that under cLWE, there are two families  $\{\nu_{w,0}: w \in \fS^{t-1}\}$ and $\{\nu_{w,1}: w \in \fS^{t-1}\}$ of distributions on $\RR^n$ such that for any $w$, the distributions $\nu_{w,0}$ and $\nu_{w,1}$ are \emph{approximately} separated by a hyperplane $\phi^{\theta(w)}$, but for unknown $w \sim \Unif(\fS^{t-1})$, distinguishing either $\nu_{w,0}$ or $\nu_{w,1}$ from a certain null distribution $\nunull$ is computationally hard.\loose

We use these distributions to construct a family of \emph{approximate} combination lock MDPs. These MDPs are parametrized by hidden vectors $w_1,\dots,w_H \in \fS^{t-1}$ and hidden actions $a^\st_1,\dots,a^\st_H \in \MA$. The transition distribution at any state $x_h\in\bbR^n$ and action $a_h$ is either $\nu_{w_{h+1},0}$ or $\nu_{w_{h+1},1}$, depending on the current latent state $\phi^{\theta(w_h)}(x_h)$ and whether $a_h$ equals $a^\st_h$. By a hybrid argument, we prove that these MDPs are indistinguishable from the MDP where all transitions follow $\nunull$, and hence learning $a^\st$ is impossible. However, we prove that learning $a^\st$ is necessary to solve the reward-free RL task.\loose

\paragraph{Takeaway: a computational role for weight function realizability?} The immediate reason why $\PCR$ fails for Low-Rank MDPs is that \emph{weight functions}, i.e. ratios $\BP_{h+1}(x_{h+1}\mid{}x,a) / \BP_{h+1}(x_{h+1}\mid{}x',a')$, are not realizable as linear functions in the feature mapping. In contrast, for Block MDPs the analogous realizability does hold, \dfedit{and plays a central role in\arxiv{ the design of} algorithms like \texttt{HOMER} and \texttt{MusIK}.} \cref{thm:halfspace-main} presents evidence that this distinction is computationally important, and perhaps suggests \emph{regression oracles for weight functions} as a route for generalizing our theory of computational tractability beyond Block MDPs. \arxiv{On the algorithmic front}\colt{Algorithmically}, methods from \citet{amortila2024scalable}---which apply to MDPs with low \emph{coverability} (subsuming Low-Rank MDPs)---offer some initial hope in this direction.\loose



\arxiv{

\section{Discussion and Future Work}\label{sec:discussion}
This work takes a step towards understanding the minimal computational oracles needed for reinforcement learning; however, much remains unclear, both for the specific tasks discussed in this paper, and beyond. Below, we discuss some particularly notable questions, ordered roughly from narrowest to broadest.

\paragraph{Pinning down a minimal oracle in the reset access model?} We showed in \cref{cor:reset-rl-to-regression} that one-context regression is sufficient for reward-free RL in Block MDPs with reset access. It's known to be necessary for reward-free RL in Block MDPs with \emph{episodic} access \citep{golowich2024exploration}, but for the reset access model we only show that noiseless one-context regression is necessary (\cref{prop:noiseless-onered}). Likely, the same argument shows that \emph{one-context regression with conditional sampling queries} is necessary (which slightly strengthens noiseless one-context regression), but even this oracle does not seem sufficient for learning inverse dynamics.

\paragraph{Reward-free vs reward-directed?} Our work focuses on reward-free RL, and in particular we mostly focus on the strongest formulation: computing a $(1,\epsilon)$-policy cover. For Block MDPs, relaxing to $(\alpha,\epsilon)$-policy covers or reward-directed RL does not appear to lead to simpler algorithms, but there are also technical obstacles to extending our results on necessity of two-context regression (e.g. \cref{cor:regression-to-online-rl}) to these settings. Most notably, while computing a $(1,\epsilon)$-policy cover can be hard for even $H = 2$, hardness for these weaker tasks only arises when $H = \omega(1)$, and it's unclear how to use two-context regression data to usefully simulate interaction with an MDP with longer horizon. Existing hardness results that exploit long horizon include \cref{thm:halfspace-main} and the main result of \cite{golowich2024exploration}, but both require specific concept classes with specific structure, whereas \cref{cor:regression-to-online-rl} applies to almost any concept class.

This technical difficulty leaves an open question: are there substantive computational differences between these tasks? This question may be of conceptual interest since $(1,\epsilon)$-policy covers are not meaningful beyond Block MDPs; a better understanding of weaker notions of exploration seems essential for more general characterizations.

\paragraph{Stronger hardness for Low-Rank MDPs?} In \cref{thm:halfspace-main}, we show that one-context low-rank regression is insufficient for Low-Rank MDPs in the reset access model, which indicates a qualitative computational difference with Block MDPs. However, this leaves several questions:
\begin{itemize} 
\item Is \emph{agnostic} one-context regression necessary? The source of computational hardness in \cref{thm:halfspace-main} is the same as the source for agnostic halfspace learning \citep{tiegel2023hardness}, but the result does not constitute a general-purpose reduction.

\item Is two-context low-rank regression insufficient for Low-Rank MDPs in the episodic access model? In a sense, \cref{thm:halfspace-main} gives weak positive evidence: the RL hardness result still holds in the episodic access model, and two-context regression for the concept class of halfspaces is likely reducible to PAC learning an intersection of two halfspaces (with random classification noise). So if two-context regression were sufficient, then it would likely imply cryptographic hardness of learning an intersection of two halfspaces under LWE, which currently seems out of reach \citep{tiegel2024improved}.

\item More broadly, is there \emph{any} oracle-efficient algorithm for RL in Low-Rank MDPs in the episodic access model that only uses a ``minimization'' oracle? In stark contrast with Block MDPs, the only existing algorithm for RL in Low-Rank MDPs requires min-max optimization, and it would be interesting to understand if there is an inherent computational barrier.
\end{itemize}

\paragraph{A computational taxonomy for RL?} Perhaps the most interesting question is whether our methods can be extended to develop a theory of computational tractability in more general interactive decision-making problems (i.e. beyond Block MDPs and Low-Rank MDPs), in analogy with the theory of statistical tractability developed by \cite{foster2021statistical}. As we discussed in \cref{sec:lowrank}, we view \emph{regression oracles for weight functions} as one promising direction, but a satisfying answer to this question could require developing new oracles and modeling assumptions.

}

\ifarxiv \bibliographystyle{alpha} \fi
\bibliography{bib}

\newpage

\appendix

\renewcommand{\contentsname}{Contents of Appendix}
		%
\crefalias{section}{appendix}

\startcontents[appendices]
\printcontents[appendices]{l}{1}{\section*{Contents of Appendix}\setcounter{tocdepth}{2}}

\newpage
\section{Additional Related Work}\label{sec:app_related}

In this section we discuss related work in more detail. Most prior work in theoretical reinforcement learning consists of algorithm design with one of three goals, reflecting varying levels of concern with computational complexity:
\begin{enumerate}
\item statistical efficiency with arbitrary computation \citep{jiang2017contextual,foster2021statistical};
\item oracle-efficiency, but the choice of oracle is not the focus \citep{dann2018oracle};
\item end-to-end computational efficiency \citep{kearns1998efficient,jin2020provably,golowich2024exploring}.
\end{enumerate}

Our work fits between the second and third levels, since we are interested in the \emph{minimal} oracles that are necessary and sufficient for oracle-efficient RL---and we are interested in settings where end-to-end efficiency is unlikely. We therefore omit discussion of non-oracle-efficient algorithms. In \cref{sec:related-alg-block}, we survey prior oracle-efficient algorithms for RL in Block MDPs, with a focus on which oracles they require. In \cref{sec:related-alg-low-rank}, we give a partial survey of oracle-efficient algorithms for RL in Low-Rank MDPs and beyond. In \cref{sec:related-ocr}, we discuss previous applications of one-context regression in theoretical RL. Finally, in \cref{sec:related-hardness} we discuss prior works on computational hardness in RL.

\colt{
In the following remark, we highlight two conceptual distinctions that are important when attempting to study RL from a complexity-theoretic perspective.

\begin{remark}[Optimization vs. learning; proper vs. improper]\label{remark:oracle-subtleties-app}
  Many prior works in oracle-efficient RL assume access to \emph{optimization} oracles rather than statistical learning oracles. Informally, the former oracles require solving regression problems analogous to \cref{def:ocr-informal,def:tcr-informal} for \emph{arbitrary datasets} as opposed to i.i.d. and realizable datasets. This is primarily a conceptual distinction rather than technical, since in many cases a learning oracle can easily be substituted in \citep{misra2020kinematic,mhammedi2023representation}.
  However, it is important from a complexity-theoretic perspective, since statistical learning can often be substantially easier \citep{blum1998polynomial}. 
  
  A more technically salient distinction is that our definitions above allow for \emph{improper learning}, whereas almost all prior works in oracle-efficient RL for Block MDPs require \emph{proper} learning oracles\footnote{For example, the proper learning analogue of \Cref{def:ocr-informal} requires computing some predictor $\MR:\MX\to[0,1]$ with an explicit decomposition $\MR = \MR'\circ \phi$ for some $\phi \in \Phi$ and $\MR': \MS \to [0,1]$.}---an exception is the work of \cite{misra2020kinematic}, which our algorithmic results directly build on. There are many concept classes for which proper learning is $\NP$-hard, but it is considered unlikely for improper learning to be $\NP$-hard \citep{applebaum2008basing}. Since the goal of RL is to output policies, which are fundamentally improper, it seems unlikely that a proper supervised learning task could be reduced to RL (in the manner of results such as \cref{cor:regression-to-online-rl}).
\end{remark}
}
\subsection{Algorithms for Block MDPs}\label{sec:related-alg-block}

\begin{table}[t]
  \centering
  \resizebox{.98\linewidth}{!}{
\renewcommand{\arraystretch}{1.8}

\begin{NiceTabular}{ccc}[hvlines]
\thead{Oracle(s)} & \thead{Additional assumptions} & \thead{Reference} \\

\makecell{
\small{$\argmin\limits_{\hat f: \MX\times\MX \to [0,1]} \sum\limits_{(x_1,x_2,y) \in \MD} (\hat f(x_1,x_2) - y)^2$}
}
 & Reachability & 
\cite{misra2020kinematic} \\

\Block{2-1}{ 
\makecell{
\small{$\argmin\limits_{\hat \phi\in\Phi} \max\limits_{\substack{f: \MS \to [0,1] \\ \phi\in\Phi}} \min\limits_{\hat f: \MS\times\MA \to [0,1]}$} \\ \small{$ \sum\limits_{(x,a)\in\MD} (\hat f(\hat\phi(x),a) - \EE_{x'|x,a} f(\phi(x')))^2$}
}
}

& Reachability & \cite{modi2021model} \\
& --- & \cite{zhang2022efficient} \\

\makecell{
\small{$\argmax\limits_{\substack{\mu: \MS^2 \to \Delta(\MA \times\MS)\\ \hat \phi\in\Phi}} \sum\limits_{(j,a,x,x') \in \MD} \log \mu((a,j)|\hat \phi(x),\hat \phi(x'))$}
}& --- & 
\makecell{\cite{mhammedi2023representation}} 
\\\hhline{|=|=|=|}
\makecell{\small{
$\sup_{\hat f:\MX\to[0,1]} \left|\sum_{x'\in\MC} \left(\hat V(x') - f(x')\right)\right|$} \\ 
\small{ 
s.t. $\sum_{x'\in\MD} (\hat V(x') - \hat f(x'))^2 \leq \veps$ 
}}
& \textbf{Reset access} & \cite{mhammedi2024power}
\\
    \end{NiceTabular}    }
    \caption{This table, adapted from \cite[Table 1]{golowich2024exploration}, gives an informal overview of the oracles used in oracle-efficient RL for $\Phi$-decodable block MDPs; the methods above the double line apply in the episodic access model. \cite{misra2020kinematic} additionally require a cost-sensitive classification oracle for $\PSDP$, though this can be replaced by one-context regression. The oracle in the second row can be implemented by the \texttt{RepLearn} algorithm \citep{modi2021model,zhang2022efficient,mhammedi2023efficient}, but this algorithm still uses a max-min oracle. The oracle in the third row can be replaced by an analogous squared-loss minimization oracle \cite[Footnote 5]{mhammedi2023representation}, equivalent to \emph{proper} two-context regression. \cite{mhammedi2024power} additionally requires a one-context regression oracle.
    }
    \label{tab:oracle-overview}
  \end{table}

\paragraph{Episodic RL.} In \cref{tab:oracle-overview} we give an informal overview of the oracles used in oracle-efficient RL algorithms for Block MDPs, adapted from \cite{golowich2024exploration}. For simplicity, we omit discussion of earlier works that required stronger assumptions---beyond reachability, which we view as largely technical--such as deterministic dynamics \citep{dann2018oracle}, separability \citep{du2019provably}, or function approximation for the observation distributions \citep{agarwal2020flambe}.

In \cref{tab:oracle-overview}, all works listed above the double line apply in the episodic access model. For our purposes, the most relevant work is that of \cite{misra2020kinematic}; their algorithm $\HOMER$ requires only a two-context regression oracle and $\PSDP$, which can be implemented with one-context regression \citep{mhammedi2023representation}. Notably, $\HOMER$ is the only prior algorithm which uses an \emph{improper} oracle; the others all seem to heavily rely on properness \citep{modi2021model,zhang2022efficient,mhammedi2023representation}. The only catch is that \cite{misra2020kinematic} analyze $\HOMER$ under a \emph{reachability} assumption---i.e., the sample complexity scales inverse-polynomially with the minimum visitation probability of any state $\eta := \min_{s\in\MS} \min_{h\in[H]} \max_{\pi\in\Pi} d^{M,\pi}_h(s)$. Avoiding this dependence is the primary technical contribution of \cref{cor:online-rl-to-regression}.

The conventional goal in theoretical RL is reward-directed RL, i.e. find a policy that has approximately maximal value with respect to an external reward function. The alternative goal is reward-free exploration \citep{jin2020reward}. For Block MDPs, this is typically formulated as the task of computing an $(\alpha,\epsilon)$-policy cover, for given $\epsilon>0$ and any $\alpha =  1/\poly(H,|\MA|,|\MS|)$, where $\Psi \subset \Pi$ is an $(\alpha,\epsilon)$-policy cover if, for every layer $h \in [H]$ and latent state $s \in \MS$,
\[\max_{\pi\in\Psi} d^{M,\pi}_h(s) \geq \alpha \cdot \max_{\pi\in\Pi} d^{M,\pi}_h(s) - \epsilon.\]
It's straightforward to see that reward-directed RL in Block MDPs is no harder than reward-free RL, since given a policy cover, $\PSDP$ can compute a near-optimal policy for any external reward function, using only a one-context regression oracle. It's unknown whether reward-directed RL is computationally \emph{easier} than reward-free RL. However, we remark that reward-free RL is a common building block for reward-directed RL for Block MDPs \citep{du2019provably,misra2020kinematic,modi2021model,mhammedi2023representation} and beyond \citep{golowich2022learning,golowich2024exploring,mhammedi2023efficient}. Moreover, the only prior algorithms for Block MDPs that do not solve reward-free RL as a byproduct \citep{modi2021model,zhang2022efficient} require a seemingly \emph{more} complicated oracle.

The particular goal that we study in this paper is computing a $(1,\epsilon)$-policy cover (\cref{def:strong-rf-rl}), which was previously studied by \cite{du2019provably}. This only strengthens our algorithmic results (\cref{cor:online-rl-to-regression} and \cref{cor:reset-rl-to-regression}), but it does leave an open question regarding our lower bound (\cref{cor:regression-to-online-rl}): can it be strengthened to apply to the problem of computing an $(\alpha,\epsilon)$-policy cover, or the problem of reward-directed RL, or is there an inherent computational gap? See \cref{sec:discussion} for further discussion.

\paragraph{RL with resets.} The reset access model augments the episodic access model by allowing the learner to revisit any previously-seen state. This access model has been extensively applied to RL settings with linear value function approximation \citep{weisz2021query,li2021sample,yin2022efficient}, for the purpose of circumventing \emph{sample complexity} lower bounds. These settings do not subsume the Block MDP setting, since they assume linearity of the value function(s) in a known feature mapping $\phi$. More relevant to us is the work of \cite{mhammedi2024power}, who give an oracle-efficient algorithm $\RVFS$ for Block MDPs (more generally, MDPs with low \emph{pushforward-coverability}) in the reset access model. As shown in \cref{tab:oracle-overview}, $\RVFS$ uses a disagreement-type optimization oracle. While it is improper, and has the same solution concept as one-context regression (i.e. a real-valued function on $\MX$), there is no obvious reduction to one-context regression, particularly one that ensures realizability. Thus, unlike our result, it's not directly apparent whether it yields a provable computational separation between the reset access model and the episodic access model (though it's certainly possible that it could imply a separation with additional work).\loose

On the other hand, our algorithm $\PCR$ is limited to the Block MDP model class, whereas $\RVFS$ applies in significantly greater generality. Thus, the guarantees are formally incomparable.

\subsection{Algorithms for Low-Rank MDPs and Beyond}\label{sec:related-alg-low-rank}

Recently, \cite{mhammedi2023efficient} developed an oracle-efficient algorithm for episodic RL in Low-Rank MDPs, which generalize Block MDPs. However, this algorithm relies on the analogue of the min-max oracle in \cref{tab:oracle-overview}. The $\RVFS$ algorithm of \cite{mhammedi2024power} solves RL in Low-Rank MDPs with reset access, but as discussed above, seems to require an agnostic optimization oracle. Beyond Low-Rank MDPs, \cite{amortila2024scalable} proposed an algorithm for exploring MDPs with low \emph{pushforward coverability} (which is satisfied by Low-Rank MDPs, and hence Block MDPs as well) assuming access to a certain policy optimization oracle as well as an optimization oracle over weight functions. 

\subsection{Algorithms Based on One-Context Regression}\label{sec:related-ocr}

One-context regression is a natural oracle for estimating value functions and Bellman backups \citep{ernst2005tree,mhammedi2023representation,golowich2024exploration}, though until recently, most applications were phrased in terms of the worst-case optimization oracle analogue. Prior to our work, versions of one-context regression were known to be sufficient for offline RL in Block MDPs, under all-policy concentrability \citep{chen2019information}; for (online) RL in Block MDPs with deterministic dynamics \citep{efroni2021provable}, and for (online) RL in horizon-one Block MDPs; these results were systematized under \cref{def:one-con-regression} by \cite{golowich2024exploration}. A variant of one-context regression (on a ``composed'' function class) was shown to be sufficient for Block MDPs under a separability condition \citep{du2019provably}. Recently, it was shown that (in a setting broader than Block MDPs), one-context regression is sufficient for competing with a weaker notion of optimal policy, called a ``max-following policy'' for a given policy ensemble \citep{hussing2024oracle}.

\subsection{Computational Hardness in RL}\label{sec:related-hardness}

Most relevant to our lower bounds are recent results by \cite{golowich2024exploring,golowich2023planning}, who showed that one-context regression is necessary for RL in Block MDPs (for any concept class $\Phi$) as well as insufficient (for a particular choice of $\Phi$), the latter under a cryptographic assumption. Other hardness results include hardness of learning Partially Observable MDPs \citep{papadimitriou1987complexity,jin2020sample,golowich2023planning} and hardness of learning MDPs with linear $Q^\st$ and $V^\st$ \citep{kane2022computational,liu2023exponential}, all of which hold under versions of the Exponential Time Hypothesis.


\newpage
\section{Additional Preliminaries}

\paragraph{Notation in proofs.} For any MDP $M$, any policy $\pi$ induces a distribution over \emph{trajectories} $(s_1,x_1,a_1,\dots,s_H,x_H,a_H)$, where the trajectory is sampled via an episode of online interaction with $M$, and at step $h$ the action $a_h$ is sampled from $\pi_h(x_h)$. We write $\Pr^{M,\pi}[\cdot]$ (respectively, $\EE^{M,\pi}[\cdot]$) to denote probability (respectively, expectation) over a trajectory sampled from $M$ with policy $\pi$. With this notation, for $h \in [H]$ and $s \in \MS$, we have $d^{M,\pi}_h(s) = \Pr^{M,\pi}[s_h=s]$. In a (minor) overload of notation, for $h \in [H]$ and $x \in \MX$ we also define the visitation probability of observation $x$ at step $h$ as $d^{M,\pi}_h(x) := \Pr^{M,\pi}[x_h = x]$. 

In the remaining appendices, we will be explicit about certain dependences on the MDP $M$ (since we will shortly be introducing truncations of $M$). In particular, unless the MDP $M$ is clear from context, we write $\til\BP^M_h$ to denote the latent transition distribution $\til \BP_h$ of $M$, and similarly write $\BP^M_h$ to denote the observed transition distribution.

\paragraph{Notation in pseudocode.} For a sampleable distribution over policies $\rho \in \Delta(\Pi)$, in the pseudocode for our algorithms we may write 
\[(x_1,a_1,\dots,x_k,a_k) \sim \rho\]
to denote sampling a policy $\pi \sim \rho$ and then a sampling (partial) trajectory via an episode of online interaction with the MDP $M$, using policy $\pi$. More generally, for two such distributions $\rho,\rho' \in \Delta(\Pi)$ and a step $h \in [H]$, we may write
\[(x_1,a_1,\dots,x_k,a_k) \sim \rho \circ_h \rho'\]
to denote sampling policies $\pi \sim \rho$ and $\pi' \sim \rho'$, and then sampling a partial trajectory from $M$ using the policy $\pi \circ_h \pi'$, i.e. playing $a_i \sim\pi_i(x_i)$ for each $i < h$ and $a_i \sim \pi'_i(x_i)$ for each $i \geq h$. In the above notation, for an action $a \in \MA$, we may overload $a$ to also denote the policy that deterministically plays action $a$.

\subsection{Truncated MDPs}\label{sec:truncated-mdps}
Fix a Block MDP $M = (H, \MS, \MX, \MA, (\til \BP_h)_{h \in [H]}, (\til\BO_h)_{h\in[H]}, \phi^\st)$. For purposes of analysis, it will be useful to define \emph{truncations} of $M$, as well as \emph{truncated} policy covers defined in terms of certain truncations. To understand the material in \cref{sec:app_online,sec:app_resets} in full generality, these definitions will be important; however, if one is willing to make a reachability assumption on the MDP $M$ \citep{misra2020kinematic}, then one may ignore these definitions and treat $\Mbar(\Gamma)$ and $\Mbar(\emptyset)$ as equivalent to $M$ itself.

In truncated Block MDPs, the latent state space and observation space are augmented with a terminal state/context $\term$, and transitions that would lead to ``difficult-to-reach'' states in $M$ instead lead to $\term$. The following preliminaries are adapted from \cite{golowich2024exploring}. For notational convenience, define $\phi^\star(\term) := \term$. 

\begin{definition}\label{def:truncated-bmdp}
Let $\Srch = (\Srch_1,\dots,\Srch_H)$ for some given sets $\Srch_1,\dots,\Srch_H \subseteq \MS$. The $\Srch$-truncation of $M$ is the Block MDP $(H, \Sbar,\Xbar,\MA,(\Pbar_h)_{h\in[H]},(\Obar_h)_{h\in[H]},\phi^\star)$ with latent state space $\Sbar := \MS \cup \{\term\}$, observation space $\Xbar := \MX \cup \{\term\}$, observation distribution
\[\Obar_h(x|s) := \begin{cases} \til\BO_h(x|s) & \text{ if } x \in \MX, s \in \MS \\ 1 & \text{ if } x=s=\term \\ 0 & \text{ otherwise} \end{cases},\]
initial state distribution 
\[\Pbar_1(s_1) := \begin{cases}
\til\BP_1(s_1) & \text{ if } s_1 \in \Srch_1 \\ 
0 & \text{ if } s_1 \in \MS \setminus \Srch_1 \\ 
\sum_{z \in \MS \setminus \Srch_1} \til\BP_1(z) & \text{ if } s_1 = \term 
\end{cases},\]
and transition distribution
\[\Pbar_h(s_h|s_{h-1},a) := \begin{cases} \til\BP_h(s_h|s_{h-1},a) & \text{ if } s_{h-1} \in \MS, s_h \in \Srch_h \\ 0 & \text{ if } s_{h-1} \in \MS, s_h \in \MS\setminus\Srch_h \\ \sum_{z \in \MS\setminus\Srch_h} \til\BP_h(z|s_{h-1},a) &\text{ if } s_{h-1} \in \MS, s_h = \term \\ \mathbbm{1}[s_h=\term] & \text{ if } s_{h-1} = \term \end{cases} \]
for each $h \in \{2,\dots,h\}$.
\end{definition}

\paragraph{Definition of $\Mbar(\emptyset)$ and $\Mbar(\Gamma$).} Fix parameters $\trunc \geq \tsmall > 0$ and a finite set of policies $\Gamma \subset \Pi$. We inductively define sets $\Srch_1(\Gamma),\dots,\Srch_H(\Gamma)$ and truncated Block MDPs $\Mbar_1(\Gamma),\dots,\Mbar_H(\Gamma)$ as follows. First, define 
\[\Srch_1(\Gamma) := \Srch_1(\emptyset) := \{s \in \MS: \til\BP_1(s) \geq \trunc\}\]
and let $\Mbar_1(\Gamma)$ be the $(\Srch_1(\Gamma),\MS,\dots,\MS)$-truncation of $M$. Next, for each $h \in \{2,\dots,H\}$, define \[\Srch_h(\emptyset) := \{s \in \MS: \max_{\pi\in\Pi} d^{\Mbar_{h-1}, \pi}_h(s) \geq \trunc\}\]
and, if $\Gamma \neq \emptyset$, \[\Srch_h(\Gamma) := \Srch_h(\emptyset) \cup \left\{s \in \MS: \EE_{\pi\sim\Unif(\Gamma)} d^{M,\pi}_h(s) \geq \tsmall\right\}.\]
Then we let $\Mbar_h(\Gamma)$ be the $(\Srch_2(\Gamma),\dots,\Srch_h(\Gamma),\MS,\dots,\MS)$-truncation of $M$. Finally, define $\Mbar(\Gamma) := \Mbar_H(\Gamma)$. As will be evident from the final parameter settings (in \cref{alg:pco,alg:pcr}), one should think of $\tsmall \ll \trunc$.

\paragraph{Truncated policy covers.} To avoid compounding errors in the analysis of algorithms that build policy covers layer-by-layer (such as $\PCO$ and $\PCR$), it is convenient to work with \emph{truncated} policy covers throughout the analysis (e.g. in the inductive hypothesis at layer $h$), and only convert to a standard policy cover (as in \cref{eq:rfrl-pc}) at the end of the analysis. Truncated policy covers are defined below:

\begin{definition}\label{defn:trunc-pc}
Let $\alpha \in (0,1)$. We say that a collection of policies $\Psi \subset \Pi$ is an \emph{$\alpha$-truncated policy cover (for $M$) at step $h\in [H]$} if for all $x \in \MX$,
\begin{align}
\frac{1}{|\Psi|} \sum_{\pi' \in \Psi} d_h^{M,\pi'}(x) \geq \alpha \cdot \max_{\pi \in \Pi} d_h^{\Mbar(\emptyset),\pi}(x)\label{eq:approx-pc}. 
\end{align}
\end{definition}
\begin{definition}\label{defn:trunc-max-pc}
Let $\alpha \in (0,1)$. We say that a collection of policies $\Psi \subset \Pi$ is an \emph{$\alpha$-truncated max policy cover at step $h\in [H]$} if for all $x \in \MX$,
\begin{align}
\max_{\pi'\in\Psi} d_h^{M,\pi'}(x) \geq \alpha \cdot \max_{\pi \in \Pi} d_h^{\Mbar(\emptyset),\pi}(x) \label{eq:approx-max-pc}.
\end{align}
\end{definition}

\subsubsection{Useful facts for Truncated MDPs}

Recall that we defined $\Mbar(\Gamma)$ as the final truncated MDP in an iterative process that produced intermediate truncated MDPs $\Mbar_1(\Gamma),\dots,\Mbar_H(\Gamma)$. These intermediate MDPs are only used for certain technical lemmas about $\Mbar(\Gamma)$; nonetheless it is useful to state several facts for later use. 

Recall that $\Mbar_h(\Gamma)$ is, essentially, truncated up to and including the transitions into step $h$. \cref{fact:trunc-dists} formalizes the fact that $\Mbar_h(\Gamma)$ agrees with $\Mbar(\Gamma)$ up to step $h$, and \cref{fact:trunc-trans} formalizes the fact that it agrees with $M$ after step $h$.

\begin{fact}\label{fact:trunc-dists}
For any finite set $\Gamma \subset \Pi$, integers $g,h \in [H]$ with $g \leq h$, policy $\pi \in \Pi$, and state $s \in \Sbar$, it holds that $d^{\Mbar(\Gamma),\pi}_g(s) = d^{\Mbar_h(\Gamma),\pi}_g(s)$.
\end{fact}

\begin{fact}\label{fact:trunc-trans}
For any finite set $\Gamma \subset \Pi$, integer $h \in \{1,\dots,H-1\}$, state $s \in \Sbar$, and action $a \in \MA$, it holds that $\til \BP^{\Mbar_h(\Gamma)}_{h+1}(\cdot|s,a) = \til\BP^M_{h+1}(\cdot|s,a)$.
\end{fact}

The most important property of a truncated MDP is that every state is either reachable (by some policy) with probability at least $\trunc$, or cannot be reached by any policy:

\begin{fact}\label{fact:trunc-reachability}
For every $h \in \{1,\dots,H\}$ and $s \in \MS$, we have $\max_{\pi\in\Pi} d^{\Mbar(\emptyset),\pi}_h(s) \geq \trunc$ if $s \in \Srch_h(\emptyset)$ and $\max_{\pi\in\Pi} d^{\Mbar(\emptyset),\pi}_h(s) = 0$ otherwise.
\end{fact}

The following facts formalize that $\Mbar_h(\Gamma)$ is ``more truncated'' than $\Mbar_{h-1}(\Gamma)$, and that $\Mbar(\Gamma)$ is ``less truncated'' than $\Mbar(\emptyset)$, respectively. The truncation process only takes mass away from non-terminal states.

\begin{fact}\label{fact:trunc-monotonicity}
Fix any finite set $\Gamma \subset \Pi$ and write $\Mbar_0(\Gamma) := M$. For any $h \in [H]$, $s \in \MS$, and $\pi \in \Pi$, it holds that $d^{\Mbar_h(\Gamma),\pi}_h(s) \leq d^{\Mbar_{h-1}(\Gamma),\pi}_h(s)$. Hence, $\E^{\Mbar_h(\Gamma),\pi}[f(x_h)] \leq \E^{\Mbar_{h-1}(\Gamma),\pi}[f(x_h)]$ for any $f: \Xbar \to \RR_{\geq 0}$ with $f(\term)=0$.
\end{fact}

\begin{fact}\label{fact:gamma-monotonicity}
For every $h \in [H]$, $s \in \MS$, $\Gamma \subset \Pi$, and $\pi \in \Pi$, we have $d^{\Mbar(\emptyset),\pi}_h(s) \leq d^{\Mbar(\Gamma),\pi}_h(s) \leq d^{M,\pi}_h(s)$.
\end{fact}

\subsubsection{Facts about truncated policy covers}

In our algorithms, we will often roll in to some step $h$ using a uniformly random policy from $\frac{1}{2}(\Unif(\Psi_h) + \Unif(\Gamma))$. The following lemma gives useful properties of the resulting visitation distribution, under the assumption that $\Psi_h$ is an $\alpha$-truncated policy cover at step $h$.

\begin{lemma}\label{lemma:srch-gamma-covering}
Fix $h \in \{1,\dots,H-1\}$ and $\alpha>0$. Let $\Psi_h, \Gamma \subset \Pi$ be finite sets of policies. Suppose that $\Psi_h$ is an $\alpha$-truncated policy cover at step $h$ (\cref{defn:trunc-pc}) for $M$. 
Then:
\begin{enumerate}
\item\label{item:h-cov-lb} For all $s_h \in \Srch_h(\Gamma)$, it holds that 
\[\EE_{\pi\sim\frac{1}{2}(\Unif(\Psi_h)+\Unif(\Gamma))} d^{M,\pi}_h(s_h) \geq \frac{\min(\alpha\trunc,\tsmall)}{2}.\]
\item\label{item:change-of-measure} For any $f:\Xbar\to\RR_{\geq 0}$ with $f(\term)=0$, it holds that 
\[\EE_{\pi\sim\frac{1}{2}(\Unif(\Psi_h)+\Unif(\Gamma))} \E^{M,\pi}[f(x_h)] \geq \frac{\min(\alpha\trunc,\tsmall)}{2} \max_{\pi\in\Pi} \E^{\Mbar(\Gamma),\pi}[f(x_h)].\]
\item\label{item:h-plus-one-cov-lb} For all $s_{h+1} \in \MS$,
\[\EE_{\pi\sim\frac{1}{2}(\Unif(\Psi_h)+\Unif(\Gamma))\circ_h \Unif(\MA)} d^{M,\pi}_{h+1}(s_{h+1}) \geq \frac{\min(\alpha\trunc,\tsmall)}{2|\MA|} \max_{\pi\in\Pi} d^{\Mbar(\Gamma),\pi}_{h+1}(s_{h+1}).\]
\end{enumerate}
\end{lemma}

\begin{proof}
We start with the first claim. Pick any $s_h \in \Srch_h(\Gamma)$. If $s_h \in \Srch_h(\emptyset)$, then $\max_{\pi\in\Pi} d^{\Mbar(\emptyset),\pi}_h(s) \geq \trunc$ (\cref{fact:trunc-reachability}). Thus, \[\EE_{\pi\sim\Unif(\Psi_h)} d^{M,\pi}_h(s_h) \geq \alpha\trunc\] by \cref{defn:trunc-pc}. On the other hand, if $s_h \in \Srch_h(\Gamma)\setminus\Srch_h(\emptyset)$, then $\E_{\pi\sim\Unif(\Gamma)} d^{M,\pi}_h(s) \geq \tsmall$ by definition of $\Srch_h(\Gamma)$. In either case, we have
\[\frac{1}{2}\EE_{\pi\sim\Unif(\Psi_h)} d^{M,\pi}_h(s_h) + \frac{1}{2}\EE_{\pi\sim\Unif(\Gamma)} d^{M,\pi}_h(s_h) \geq \frac{\min(\alpha\trunc,\tsmall)}{2}\]
which proves the first claim. Next, pick any $f:\Xbar\to\RR_{\geq 0}$ with $f(\term)=0$. We have
\begin{align*}
&\EE_{\pi\sim\frac{1}{2}(\Unif(\Psi_h)+\Unif(\Gamma))} \E^{M,\pi}[f(x_h)] \\ 
&= \sum_{s_h \in \MS} \EE_{\pi\sim\frac{1}{2}(\Unif(\Psi_h)+\Unif(\Gamma)} d^{M,\pi}_h(s_h) \EE_{x_h \sim \BO_h(\cdot|s_h)}[f(x_h)] \\ 
&\geq \sum_{s_h \in \Srch_h(\Gamma)} \EE_{\pi\sim\frac{1}{2}(\Unif(\Psi_h)+\Unif(\Gamma)} d^{M,\pi}_h(s_h) \EE_{x_h \sim \BO_h(\cdot|s_h)}[f(x_h)] \\ 
&\geq \frac{\min(\alpha\trunc,\tsmall)}{2} \sum_{s_h \in \Srch_h(\Gamma)} \max_{\pi\in\Pi} d^{\Mbar(\Gamma),\pi}_h(s_h) \EE_{x_h \sim \BO_h(\cdot|s_h)}[f(x_h)] \\ 
&\geq \frac{\min(\alpha\trunc,\tsmall)}{2} \max_{\pi\in\Pi} \E^{\Mbar(\Gamma),\pi}[f(x_h)]
\end{align*}
where the first inequality uses non-negativity of $f$; the second inequality uses \cref{item:h-cov-lb} together with the fact that $d^{\Mbar(\Gamma),\pi}_h(s_h) \leq 1$ for all $\pi,s_h$; and the third inequality uses the fact that $f(\term) = 0$. This proves the second claim. Finally, pick any $s_{h+1} \in \MS$. If $s_{h+1} \in \Srch_{h+1}(\Gamma)$, then for any $\pi \in \Pi$,
\begin{align*}
&\EE_{\pi\sim\frac{1}{2}(\Unif(\Psi_h)+\Unif(\Gamma))\circ_h \Unif(\MA)} d^{M,\pi}_{h+1}(s_{h+1}) \\ 
&= \sum_{s_h \in \MS} \left(\EE_{\pi\sim\frac{1}{2}(\Unif(\Psi_h)+\Unif(\Gamma)} d^{M,\pi}_h(s_h)\right) \frac{1}{|\MA|}\sum_{a \in \MA} \til\BP^M_{h+1}(s_{h+1}|s_h,a) \\ 
&\geq \sum_{s_h \in \MS} \left(\EE_{\pi\sim\frac{1}{2}(\Unif(\Psi_h)+\Unif(\Gamma)} d^{M,\pi}_h(s_h)\right) \frac{1}{|\MA|} \til\BP^M_{h+1}(s_{h+1}|s_h,\pi(s_h)) \\
&\geq \frac{\min(\alpha\trunc,\tsmall)}{2}\sum_{s_h\in\Srch_h(\Gamma)} d^{\Mbar(\Gamma),\pi}_h(s_h) \frac{1}{|\MA|} \til\BP^M_{h+1}(s_{h+1}|s_h,\pi(s_h)) \\
&= \frac{\min(\alpha\trunc,\tsmall)}{2|\MA|} d^{\Mbar(\Gamma),\pi}_{h+1}(s_{h+1})
\end{align*}
where the second inequality uses \cref{item:h-cov-lb} together with the fact that $d^{\Mbar(\Gamma),\pi}_h(s_h) \leq 1$ for all $s_h$. This proves the third claim whenever $s_{h+1} \in \Srch_{h+1}(\Gamma)$. Moreover, if $s_{h+1} \not \in \Srch_{h+1}(\Gamma)$ then $\max_{\pi\in\Pi} d^{\Mbar(\Gamma),\pi}_{h+1}(s_{h+1}) = 0$, so the inequality is vacuously true.
\end{proof}

\subsubsection{Additional facts about truncated MDPs}

The following lemma will be important in the win/win analyses for $\PCO$ and $\PCR$---specifically, it shows that if we find a policy that visits the terminal state $\term$ with reasonable probability in $\Mbar(\Gamma)$, then it must explore some hard-to-reach state at some earlier layer of $M$ (which is a form of progress).

\begin{lemma}\label{lemma:term-prob}
Let $\pi \in \Pi$ and $\Gamma \subset \Pi$. For any $h \in [H]$, it holds that \[d^{\Mbar(\Gamma),\pi}_h(\term) \leq \sum_{k=1}^h \sum_{s \in \MS\setminus \Srch_k(\Gamma)} d^{M,\pi}_k(s).\]
\end{lemma}

\begin{proof}
Observe that $d^{\Mbar(\Gamma),\pi}_1(\term) = \sum_{s \in \MS\setminus \Srch_1(\Gamma)} \til\BP^M_1(s) = \sum_{s \in \MS\setminus \Srch_1(\Gamma)} d^{M,\pi}_1(s)$ by construction. Moreover, for any $h \in \{2,\dots,H\}$, we have
\begin{align*}
d^{\Mbar(\Gamma),\pi}_h(\term)
&= d^{\Mbar(\Gamma),\pi}_{h-1}(\term) + \sum_{s \in \Srch_{h-1}(\Gamma)} d^{\Mbar(\Gamma),\pi}_{h-1}(s) \sum_{s' \in \MS \setminus \Srch_h(\Gamma)} \til\BP^M_h(s'|s,\pi(s)) \\ 
&\leq d^{\Mbar(\Gamma),\pi}_{h-1}(\term) + \sum_{s \in \MS} d^{M,\pi}_{h-1}(s) \sum_{s' \in \MS \setminus \Srch_h(\Gamma)} \til\BP^M_h(s'|s,\pi(s)) \\ 
&= d^{\Mbar(\Gamma),\pi}_{h-1}(\term) + \sum_{s' \in \MS \setminus \Srch_h(\Gamma)} d^{M,\pi}_h(s')
\end{align*}
where the inequality uses \cref{fact:gamma-monotonicity}. Inducting on $h$ completes the proof.
\end{proof}

The following lemma is in the analyses of $\PCO$ and $\PCR$, to show that a truncated max-policy cover is also a $(1,\epsilon)$-policy cover in the sense of \cref{eq:rfrl-pc}.

\begin{lemma}\label{lemma:term-ub}
Let $\pi \in \Pi$. For any $h \in [H]$ and $s \in \MS$, it holds that
\[d^{\Mbar(\emptyset),\pi}_h(s) \geq d^{M,\pi}_h(s) -  h|\MS|\trunc.\]
\end{lemma}

\begin{proof}
We have $d^{\Mbar(\emptyset),\pi}_h(\term) = d^{\Mbar_h(\emptyset),\pi}_h(\term)$ by \cref{fact:trunc-dists}. Moreover \[d^{\Mbar_1(\emptyset),\pi}_h(\term) = d^{\Mbar_1(\emptyset),\pi}_1(\term) = \sum_{z \in \MS \setminus \Srch_1(\emptyset)} \til\BP_1(z) \leq |\MS| \trunc.\] 
For any $2 \leq k \leq h$, we have
\begin{align*}
d^{\Mbar_k(\emptyset),\pi}_h(\term) - d^{\Mbar_{k-1}(\emptyset),\pi}_h(\term) 
&= d^{\Mbar_k(\emptyset),\pi}_k(\term) - d^{\Mbar_{k-1}(\emptyset),\pi}_k(\term) \\ 
&= \E^{\Mbar_{k-1}(\emptyset),\pi}  \left[\til\BP_k^{\Mbar_k(\emptyset)}(\term\mid{}s_{k-1},a_{k-1}) - \til\BP_k^{\Mbar_{k-1}(\emptyset)}(\term\mid{}s_{k-1},a_{k-1})\right] \\ 
&= \E^{\Mbar_{k-1}(\emptyset),\pi}\left[\sum_{z \in \MS \setminus \Srch_k} \til\BP^M_k(z\mid{}s_{k-1},a_{k-1})\right] \\ 
&= \sum_{z \in \MS \setminus \Srch_k} d^{\Mbar_{k-1}(\emptyset),\pi}_{k-1}(z) \\ 
&\leq |\MS|\trunc.
\end{align*}
Therefore $d^{\Mbar(\emptyset),\pi}_h(\term) \leq h|\MS|\trunc$ by telescoping. Now for any $s \in \MS$, this means that
\begin{align*}
d^{\Mbar(\emptyset),\pi}_h(s)
&= 1 - d^{\Mbar(\emptyset),\pi}_h(\term) - \sum_{s' \in \MS \setminus \{s\}} d^{\Mbar(\emptyset),\pi}_h(s') \\ 
&\geq 1 - h|\MS|\trunc - \sum_{s' \in \MS \setminus \{s\}} d^{M,\pi}_h(s') \\ 
&= d^{M,\pi}_h(s) - h|\MS|\trunc
\end{align*}
where the inequality also uses \cref{fact:gamma-monotonicity}.
\end{proof}

The following result, which will be used in the analysis of $\PSDP$ (\cref{lemma:psdp-trunc-online}), is a variant of the classical Performance Difference Lemma \citep{kakade2002approximately}; see also \cite{golowich2024exploring}.

\begin{lemma}[Performance Difference Lemma for truncated MDPs]\label{lemma:perf-diff-trunc}
Fix a finite set $\Gamma \subset \Pi$ and $k \in \{1,\dots,H-1\}$. Let $R:\Xbar\to[0,1]$ be a function with $R(\term)=0$. Then for any policies $\pi,\pi^\st \in \Pi$ it holds that
\[\E^{\Mbar(\Gamma),\pi}[R(x_{k+1})] - \E^{M,\pi}[R(x_{k+1})] \leq \sum_{h=1}^k \E^{\Mbar(\Gamma),\pi^\st}[Q^{M,\pi,\bfr}_h(x_h,a_h) - V^{M,\pi,\bfr}_h(x_h)]\]
where $\bfr=(\bfr_1,\dots,\bfr_H)$ is defined by $\bfr_{k+1}(x,a) = R(x)$ and $\bfr_h(x,a) = 0$ for all $h \neq k+1$, and we have defined $Q^{M,\pi,\bfr}_h(\term,a) := 0$ and $V^{M,\pi,\bfr}_h(\term) := 0$ for all $h \in [H]$ and $a \in \MA$.
\end{lemma}

\begin{proof}
Observe that for any $h \in \{1,\dots,k\}$, $x_h \in \Xbar$, $a_h \in \MA$, we have
\begin{equation}
Q^{M,\pi,\bfr}_h(x_h,a_h) 
= \EE_{x_{h+1} \sim \BP^M_{h+1}(\cdot|x_h,a_h)}[V^{M,\pi,\bfr}_{h+1}(x_{h+1})]
= \EE_{x_{h+1} \sim \BP^{\Mbar_h(\Gamma)}_{h+1}(\cdot|x_h,a_h)}[V^{M,\pi,\bfr}_{h+1}(x_{h+1})]
\label{eq:qv-equiv}
\end{equation}
by \cref{fact:trunc-trans}. For notation convenience, write $\Mbar_0(\Gamma) := M$. Now we have
\begin{align*}
&\E^{\Mbar(\Gamma),\pi^\st}[R(x_{k+1})] - \E^{M,\pi}[R(x_{h+1})] \\ 
&= \E^{\Mbar_{k+1}(\Gamma),\pi^\st}[Q^{M,\pi,\bfr}_{k+1}(x_{k+1},a_{k+1})] - \E^{M,\pi^\st}[V^{M,\pi,\bfr}_1(x_1)] \\ 
&= \E^{\Mbar_{k+1}(\Gamma),\pi^\st}[Q^{M,\pi,\bfr}_{k+1}(x_{k+1},a_{k+1})] - \E^{M,\pi^\st}[V^{M,\pi,\bfr}_1(x_1)] + \sum_{h=1}^k \E^{\Mbar_h(\Gamma),\pi^\st}[Q^{M,\pi,\bfr}_h(x_h,a_h) - V^{M,\pi,\bfr}_{h+1}(x_{h+1})] \\ 
&= \sum_{h=1}^{k+1}\left ( \E^{\Mbar_h(\Gamma),\pi^\st}[Q^{M,\pi,\bfr}_h(x_h,a_h)] - \E^{\Mbar_{h-1}(\Gamma),\pi^\st}[V^{M,\pi,\bfr}_h(x_h)]\right) \\ 
&\leq \sum_{h=1}^{k+1} \left(\E^{\Mbar_h(\Gamma),\pi^\st}[Q^{M,\pi,\bfr}_h(x_h,a_h)] - \E^{\Mbar_h(\Gamma),\pi^\st}[V^{M,\pi,\bfr}_h(x_h)]\right)
\end{align*}
where the first equality is by \cref{fact:trunc-dists} and the fact that $R(\term)=0$, the second equality is by \cref{eq:qv-equiv}, and the inequality is by \cref{fact:trunc-monotonicity} (along with the fact that $V^{M,\pi,\bfr}_h \geq 0$ and $V^{M,\pi,\bfr}_h(\term) = 0$). Finally, observe that $Q^{M,\pi,\bfr}_{k+1}(x_{k+1},a_{k+1}) = R(x_{k+1}) = V^{M,\pi,\bfr}_{k+1}(x_{k+1})$ for any $x_{k+1} \in \Xbar$, so the final term in the summation vanishes. We conclude that
\begin{align*}
&\E^{\Mbar(\Gamma),\pi^\st}[R(x_{k+1})] - \E^{M,\pi}[R(x_{h+1})] \\
&\leq \sum_{h=1}^{k} \left(\E^{\Mbar_h(\Gamma),\pi^\st}[Q^{M,\pi,\bfr}_h(x_h,a_h)] - \E^{\Mbar_h(\Gamma),\pi^\st}[V^{M,\pi,\bfr}_h(x_h)]\right) \\ 
&= \sum_{h=1}^k \left(\E^{\Mbar(\Gamma),\pi^\st}[Q^{M,\pi,\bfr}_h(x_h,a_h)] - \E^{\Mbar(\Gamma),\pi^\st}[V^{M,\pi,\bfr}_h(x_h)]\right)
\end{align*}
where the final equality uses \cref{fact:trunc-dists}.
\end{proof}

\newpage
\section{Proof of \creftitle{cor:online-rl-to-regression}}\label{sec:app_online}


In this section we prove that for any concept class $\Phi$, there is a reduction from reward-free RL (\cref{def:strong-rf-rl}) in the episodic access model to two-context regression (\cref{def:two-con-regression}). The formal statement is provided below.

\begin{theorem}[General version of \cref{cor:online-rl-to-regression}]\label{thm:pco-app}
There is a constant $C_{\ref{thm:pco-app}}>0$ and an algorithm $\PCO$ so that the following holds. Let $\Phi \subseteq (\MX\to\MS)$ be any concept class, and let $\Reg$ be a $\Nreg$-efficient two-context regression oracle for $\Phi$. Then $\PCO(\Reg,\Nreg,|\MS|,\cdot)$ 
is an $(\Nrl,\Krl)$-efficient reward-free RL algorithm for $\Phi$ in the episodic access model, with:
\begin{itemize}
    \item $\Krl(\epsilon,\delta,H,|\MA|) \leq H^2|\MS|^2$
    \item $\Nrl(\epsilon,\delta,H,|\MA|) \leq \left(\frac{H|\MA||\MS|}{\epsilon\delta}\right)^{C_{\ref{thm:pco-app}}} \Nreg\left(\left(\frac{\epsilon\delta}{H|\MA||\MS|}\right)^{C_{\ref{thm:pco-app}}},\left(\frac{\epsilon\delta}{H|\MA||\MS|}\right)^{C_{\ref{thm:pco-app}}}\right)$.
\end{itemize}
Moreover, the oracle time complexity of $\PCO$ is at most $\left(\frac{H|\MA||\MS|}{\epsilon\delta}\right)^{C_{\ref{thm:pco-app}}} \Nreg\left(\left(\frac{\epsilon\delta}{H|\MA||\MS|}\right)^{C_{\ref{thm:pco-app}}},\left(\frac{\epsilon\delta}{H|\MA||\MS|}\right)^{C_{\ref{thm:pco-app}}}\right).$
\end{theorem}

In particular, \cref{cor:online-rl-to-regression} follows from \cref{thm:pco-app} by substituting $\Nreg(\epsilon,\delta) := \Nregc/(\epsilon\delta)^{\Creg}$ into the above bounds. Henceforth, fix a concept class $\Phi$, a $\Nreg$-efficient two-context regression oracle $\Reg$, and a $\Phi$-decodable block MDP $M$ with horizon $H$, action set $\MA$, and unknown decoding function $\phi^\st \in \Phi$. We also define truncations of $M$ (see \cref{sec:truncated-mdps}), with the parameters $\trunc,\tsmall>0$ as defined in \cref{alg:pco}.

In \cref{sec:pco-overview}, we give pseudocode and an overview of $\PCO$ and its main subroutine $\EPCO$. In \cref{subsec:epco}, we formally analyze $\EPCO$. In \cref{sec:pco-analysis}, we formally analyze $\PCO$, completing the proof of \cref{thm:pco-app} (and hence \cref{cor:online-rl-to-regression}).

\subsection{$\PCO$ \colt{Pseudocode and }Overview}\label{sec:pco-overview}

We start by giving an overview of the algorithm $\PCO$ (\cref{alg:pco}). The main subroutine of this algorithm is $\EPCO$ (\cref{alg:extend-pc-online}), which is used to extend a set of policy covers from layers $1,\dots,h$ to layer $h+1$. We describe $\EPCO$ first, and then explain how it fits into $\PCO$. As discussed in \cref{sec:online}, $\PCO$ is a direct extension (and, from the perspective of oracles, a simplification) of $\HOMER$ \citep{misra2020kinematic}; we highlight relevant differences in the overview below.

\colt{}

\subsubsection{$\EPCO$: Extending a Policy Cover.} 

\paragraph{Algorithm overview.} As shown in \cref{alg:extend-pc-online}, $\EPCO$ takes as input a two-context regression oracle $\Reg$, a step $h \in [H]$, and a set of policy covers $\Psi_{1:h}$ (as well as several other inputs that will be discussed later as necessary---for now, consider the case $\Gamma = \emptyset$). The desired behavior of $\EPCO$ is that, if $\Psi_{1:h}$ are $(1,\epsilon)$-policy covers for layers $1,\dots,h$ of the MDP $M$ (as defined in \cref{eq:rfrl-pc}), then the output $\Psi_{h+1}$ should be a $(1,\epsilon)$-policy cover for layer $h+1$.

To this end, for each action $a \in \MA$, the algorithm first uses $\Psi_h$ to construct a two-context regression dataset $\MD_a$, via a contrastive learning approach where datapoints $(x_h,x_{h+1})$ with label $y=0$ are sampled independently, whereas datapoints with label $y=1$ are sampled dependently according to the transition dynamics, i.e. $x_{h+1} \sim \BP_{h+1}(\cdot\mid{}x_h,a)$ (in both cases, and throughout the rest of the algorithm, the algorithm rolls in to step $h$ with a random policy from $\Psi_h$). The algorithm then invokes $\Reg$ to compute a predictor $\wh f_{h+1}(\cdot,\cdot;a)$. It can be checked that the Bayes predictor $\EE[y\mid{}x_h,x_{h+1}]$ for this dataset is precisely the kinematics function from $f_{h+1}(\cdot,\cdot;a)$ from \cref{eq:kinematics-intro}. By the Block MDP assumption, this function only depends on $(x_h,x_{h+1})$ through $\phi^\st(x_h)$ and $\phi^\st(x_{h+1})$ (and the distribution of $(x_h,x_{h+1})$ is $\phi^\st$-realizable), so it follows from the guarantee of $\Reg$ that $\wh f_{h+1}(\cdot,\cdot;a)$ approximates the true kinematics function $f_{h+1}(\cdot,\cdot;a)$ with high probability.

Second, the algorithm samples $m$ observations $x_h^{(1)},\dots,x_h^{(m)}$ at step $h$, which will be used as ``test observations'' for evaluating $\wh f_{h+1}$: the idea is that if 
\[\wh f_{h+1}(x_h^{(i)}, x_{h+1};a) \approx \wh f_{h+1}(x_h^{(i)}, x'_{h+1};a)\]
for some $x_{h+1},x'_{h+1}\in\MX$ and all $i \in [m]$ and $a \in \MA$, then so long as these observations have appropriate coverage, in fact $\wh f_{h+1}(x_h,x_{h+1};a)$ should approximate $\wh f_{h+1}(x_h,x'_{h+1};a)$ for all $(x_h,a) \in \MX\times\MA$, i.e. $x_{h+1}$ and $x'_{h+1}$ should have approximately the same kinematics.

Third, the algorithm samples $n$ observations $\xbar_{h+1}^{(1)},\dots,\xbar_{h+1}^{(n)}$ at step $h+1$ (rolling in with a random policy from $\Psi_h$ followed by a random action at step $h$). These observations will serve as candidate ``cluster centers'' for defining internal reward functions. In particular, for each $t \in [n]$, the reward function $\MR^{(t)}:\MX\to[0,1]$ is defined in \lineref{line:rt-def-online} to be large precisely for those $x_{h+1}$ satisfying
\[\wh f_{h+1}(x_h^{(i)}, x_{h+1};a) \approx \wh f_{h+1}(x_h^{(i)}, \xbar_{h+1}^{(t)};a)\]
for all test observations $x_h^{(i)} \in \MX$ and actions $a \in \MA$.

Fourth, for each cluster center that has noticeably different kinematics from previous centers (measured in the same way as above), the algorithm invokes $\PSDP$ (\cref{alg:psdpb}) to compute a policy $\pihat^{(t)}$ that approximately maximizes the reward function $\MR^{(t)}$. This policy is added to $\Psi_{h+1}$.

\paragraph{Comparison with $\HOMER$.} The main differences between $\EPCO$ and the corresponding subroutine of $\HOMER$ arise in ensuring that $\EPCO$ is oracle-efficient with respect to two-context regression: in the first step, we perform an individual regression for each action $a$ (whereas $\HOMER$ performs a single regression joint across all actions), since in our definition of two-context regression, the action is not a covariate. Also, $\HOMER$ uses an implementation of $\PSDP$ with a cost-sensitive classification oracle; it is unclear how to reduce this to two-context regression, so we instead use an implementation of $\PSDP$ due to \cite{mhammedi2023representation}---see \cref{alg:psdpb}. For this implementation, a \emph{one}-context regression oracle suffices (\cref{lemma:psdp-trunc-online}), and this oracle in turn can easily be implemented---via the reduction $\OneTwo$ (\cref{alg:onetwo})---with two-context regression.

\paragraph{Proof outline.} Since $\EPCO$ is very similar (at a technical level) to the corresponding subroutine of $\HOMER$, we do not belabor the details of the proof in this overview. The basic reason why optimizing the reward functions $\MR^{(t)}$ is a good idea is the following. First, if two observations $x_{h+1},x'_{x+1}$ have the same latent state, then they have the same kinematics, i.e. $f_{h+1}(x_h,x_{h+1};a) = f_{h+1}(x_h,x_{h+1}';a)$ for all $x,a$. The converse is not necessarily true. However, something just as good is true: if two states $x_{h+1},x'_{h+1}$ have the same kinematics, then $d^{M,\pi}_{h+1}(x_{h+1}) = C \cdot d^{M,\pi}_{h+1}(x'_{h+1})$ for a constant $C$ that may depend on $x_{h+1}$ and $x'_{h+1}$ but \emph{does not depend on $\pi$}. In this sense, $x_{h+1}$ and $x'_{h+1}$ are ``kinematically inseparable'' \citep{misra2020kinematic}. It follows that the policy that maximizes $d^{M,\pi}_{h+1}(x_{h+1}) + d^{M,\pi}_{h+1}(x'_{h+1})$, i.e. the probability of visiting one of these two observations, also maximizes the probability of visiting either individual observation. More generally, for any set of kinematically inseparable observations, it suffices to maximize a reward function that rewards visiting any of these observations. 

Obviously, in the actual algorithm and actual reward functions there are statistical errors, but these can be handled under appropriate coverage conditions. Indeed, under a \emph{reachability assumption} on the MDP, \cite{misra2020kinematic} show that if $\Psi_{1:h}$ are $(1,\epsilon)$-policy covers for layers $1,\dots,h$, then with high probability $\Psi_{h+1}$ is a $(1,\epsilon)$-policy cover for layer $h+1$. Thus, they can simply run $\EPCO$ iteratively from $h=1,\dots,H$ and produce a set of policy covers $\Psi_{1:H}$. However, we want to avoid making a reachability assumption, so we require a more sophisticated algorithm (which still uses $\EPCO$ as a subroutine)---this is precisely $\PCO$.

We formally analyze $\EPCO$ in \cref{subsec:epco}; the main guarantee is \cref{thm:extend-pc-trunc-online}.

\subsubsection{$\PCO$: Handling Reachability Issues via Iterative Discovery}

\paragraph{Algorithm overview.} The basic idea of $\PCO$ (\cref{alg:pco}) is to essentially put an outer loop around the entire $\HOMER$/$\EPCO$ algorithm; this technique was previously used by \cite{golowich2024exploring} for the same reason of handling reachability issues, in the context of sparse linear MDPs. In particular, $\PCO$ proceeds in $R= |\MS| \cdot H$ rounds. In the first round, $\PCO$ runs $\EPCO$ iteratively from $h=1,\dots,H$ to construct a set of candidate policy covers $\Psi_{1:H}^{(1)}$. It then adds all of these computed policies to a ``backup policy cover'' $\Gamma^{(2)}$, and passes $\Gamma^{(2)}$ to $\EPCO$ in the next round. The outputs of $\EPCO$ are added to $\Gamma^{(3)}$, which is the backup policy cover for round $r=3$, and so forth. The final output of $\PCO$ is the union of all candidate policy covers (of bounded size) that were computed in all rounds. 

\paragraph{Proof outline.} We sketch why the outer loop in $\PCO$ is needed to avoid a reachability assumption (and why it works). Reachability assumptions are common in theoretical reinforcement learning \citep{du2019provably,misra2020kinematic}, but are often an artifact of the analysis. That is, they can often be avoided---without changing the algorithm---by analyzing \emph{truncated MDPs/policies} \citep{golowich2022learning,mhammedi2023efficient}, which essentially avoid issues of compounding errors on hard-to-reach states by truncating away such states. This approach would work in our setting if the reward functions $\MR^{(t)}$ were exact, and hence accurate on all states. However, in $\EPCO$/$\HOMER$, the reward functions are \emph{learned}, so they could be inaccurate on hard-to-reach states (in \cref{lemma:target-error-online}, notice that the error bound holds in expectation over the truncated MDP $\Mbar(\Gamma)$, which doesn't include the hard to reach states from the actual MDP $M$). This means that the policies computed by $\PSDP$ could obtain erroneously high reward without actually being optimal for the ``ideal'' reward functions.

The key idea is that the above pathology only occurs if one of the policies computed by $\PSDP$ ``discovers'' a state that the algorithm previously was unable to cover. This is the motivation for rerunning the entire algorithm with these policies mixed into all data collection procedures via the backup policy cover $\Gamma$. By a win/win argument \citep{golowich2024exploring}, after at most $R = H|\MS|$ rounds, there will be some round $r$ in which the algorithm does not discover any new states; it can be shown that the sets $\Psi^{(r)}_{1:H}$ constructed in this round are indeed policy covers. We discuss this win/win argument in more detail later.

We formally analyze $\PCO$ (and thereby prove \cref{thm:pco-app}) in \cref{sec:pco-analysis}.

\subsection{Analysis of $\EPCO$ (\creftitle{alg:extend-pc-online})}\label{subsec:epco}

We now prove the following guarantee for $\EPCO$ (\cref{alg:extend-pc-online}). Recall that we have fixed a concept class $\Phi$, a $\Nreg$-efficient two-context regression oracle $\Reg$, and a $\Phi$-decodable block MDP $M$ with horizon $H$, action set $\MA$, and unknown decoding function $\phi^\st \in \Phi$. We have also defined truncations of $M$ (see \cref{sec:truncated-mdps}), with the parameters $\trunc,\tsmall>0$ as defined in \cref{alg:pco}.

\begin{theorem}\label{thm:extend-pc-trunc-online}
Let $h \in \{1,\dots,H-1\}$. Let $\epsilon,\delta,\alpha > 0$ and $m,n,N \in \NN$. Let $\Gamma \subset \Pi$ be a finite set of policies. Suppose that $\Psi_{1:h}$ are $\alpha$-truncated policy covers (\cref{defn:trunc-pc}) for $M$ at steps $1,\dots,h$. Suppose that $m \geq \frac{2}{\min(\alpha\trunc,\tsmall)}\log(|\MS|/\delta)$, $n \geq \frac{2|\MA|}{\min(\alpha\trunc,\tsmall)\trunc} \log(|\MS|/\delta)$, $N \geq \Nreg(\epsilon,\delta)$, and 
\begin{equation} \epsilon^{1/16} \leq \frac{\alpha^2\trunc^4\tsmall^2}{96H|\MS||\MA|^2\sqrt{m}}.\label{eq:param-assm-online}
\end{equation}
Set $\gamma := \epsilon^{1/16}$ and $\gamma' := 2\epsilon^{1/8}\sqrt{m|\MA|}$ and let $\Psi_{h+1}$ denote the output of $\EPCO$ with inputs $\Reg,h,\Psi_{1:h},\Gamma,n,m,N,\gamma,\gamma'$. Then with probability at least $1 - (2+|\MA|+H|\MA|n)\delta - m|\MA|\epsilon^{1/2} - n\epsilon^{1/4}$, the following two properties hold:
\begin{itemize}
\item $|\Psi_{h+1}| \leq |\MS|$.
\item Either $\Psi_{h+1}$ is a $(1-4\trunc)$-truncated max policy cover (\cref{defn:trunc-max-pc}) for $M$ at step $h+1$, or $\max_{\pi\in\Psi_{h+1}} d^{\Mbar(\Gamma),\pi}_{h+1}(\term) \geq \trunc^2$.
\end{itemize}
\end{theorem}

Informally, $\Mbar(\Gamma)$ refers to a \emph{truncation} of the MDP $M$ in which all latent states that (a) cannot be reached by any policy with probability at least $\trunc$, and (b) cannot be reached by a uniformly random policy in $\Gamma$ with probability at least $\tsmall \ll \trunc$, are ``truncated away'' to an artificial terminal state $\term$. A truncated policy cover (\cref{defn:trunc-pc}) is essentially a set of policies $\Psi$ so that for each latent state $s$ that \emph{can} be reached with probability at least $\trunc$, a \emph{uniformly random} policy from $\Psi$ covers $s$. A truncated max-policy cover (\cref{defn:trunc-max-pc}) is essentially a set of policies so that for each latent state $s$ that \emph{can} be reached with probability at least $\trunc$, the \emph{best} policy from $\Psi$ covers $s$. See \cref{sec:truncated-mdps} for the formal definition of $\Mbar(\Gamma)$ and the truncated policy covers.

With this notation, \cref{thm:extend-pc-trunc-online} essentially asserts that if $\Psi_{1:h}$ are policy covers at steps $1,\dots,h$, then the output of $\EPCO$ is either a policy cover at step $h+1$, or else contains some policy $\pi$ that visits the terminal state in $\Mbar(\Gamma)$ with non-trivial probability. In this second case, it can be shown (via \cref{lemma:term-prob} and the definition of $\Srch_k(\Gamma)$) that $\pi$ visits some latent state in $M$ that was previously nearly-unexplored by all policies in $\Gamma$. Hence, when $\PCO$ adds $\pi$ to the backup policy cover $\Gamma^{(r+1)}$ in the next round, progress will have been made, and so this second case can only happen a bounded number of times---see the proof of \cref{thm:pco-app} in \cref{sec:pco-analysis}.


Let us fix the inputs to $\EPCO$: in addition to the two-context regression oracle $\Reg$ (\cref{def:two-con-regression}), we fix a layer $h \in [H-1]$, sets of policies $\Psi_1,\dots,\Psi_h$ and $\Gamma$, sample counts $n,m,N \in \NN$, and tolerances $\gamma,\gamma' \in (0,1)$. As discussed above, the main idea of $\EPCO$ is to use the oracle to estimate the kinematics $f_{h+1}: \MS\times\MS\times\MA \to [0,1]$ (defined informally in \cref{eq:kinematics-intro} and formally below), and then to apply the $\PSDP$ policy optimization method (\cref{alg:psdpb}) on internal reward functions constructed by clustering the kinematics.

\begin{definition}\label{def:kinematics}
For any $s,s' \in \MS$ and $a \in \MA$, define 
\[f_{h+1}(s,s';a) := \frac{\til\BP^M_{h+1}(s'\mid{}s,a)}{\til\BP^M_{h+1}(s'\mid{}s,a) + F_{h+1}(s')}\]
where 
\[F_{h+1}(s') := \EE_{\pi \sim \frac{1}{2}(\Unif(\Psi_h)+\Unif(\Gamma))\circ_h\Unif(\MA)} \E^{M,\pi}[\til\BP^M_{h+1}(s'\mid{}s_h,a_h)].\] 
\end{definition}

The main technical lemmas in the analysis are (1) \cref{lemma:regression-bound-online}, which shows that the regression problem solved in \lineref{line:hat-f-regression} is a realizable instance of two-context regression, and that the resulting estimator $\wh f_{h+1}$ is therefore a good estimate of $f_{h+1}$; and (2) \cref{lemma:target-error-online}, which shows that if $\wh f_{h+1}$ is close to $f_{h+1}$ then for any reachable latent state $s^\st$, there is some internal reward function $\MR^{(t)}$ computed by $\EPCO$ that approximately optimizes for visiting state $s^\st$. The following notation will be useful: 

\begin{definition}\label{def:marginals}
Let $\mu_{h+1}(a) \in \Delta(\MX\times\MX)$ be the marginal distribution of $(x_h,x_{h+1})$ where $(x_h,x_{h+1},y)$ is the first element of $\MD_a$. Let $\beta_h,\beta_{h+1} \in \Delta(\MS)$ be the marginal distributions of $s_h$ and $s_{h+1}$ respectively, for a trajectory $(s_1,x_1,a_1,\dots,s_{h+1},x_{h+1}) \sim \frac{1}{2}(\Unif(\Psi_h)+\Unif(\Gamma))\circ_h\Unif(\MA)$.
\end{definition}

In words, $\beta_h$ is the visitation distribution at step $h$ of a uniformly random policy $\pi \sim \frac{1}{2}(\Unif(\Psi_h)+\Unif(\Gamma))$, and $\beta_{h+1}$ is the visitation distribution at step $h+1$ obtained by sampling from $x_h \sim \beta_h$, $a_h \sim \Unif(\MA)$, and $x_{h+1} \sim \BP^M_{h+1}(\cdot\mid{}x_h,a_h)$.

We now prove that $\wh f_{h+1}$ approximates the true kinematics $f_{h+1}$ with high probability. This requires checking that each dataset $\MD_a$ constructed by $\EPCO$ satisfies the necessary realizability assumptions, specified in \cref{def:two-con-regression}, with respect to $\Phi$. \cref{item:reg-1-online} is then a direct consequence of the guarantee of $\Reg$ (together with a union bound over actions). \cref{item:reg-2-online} is a useful consequence, which asserts that if we plug in the $m$ ``test observations'' $x_h^{(1)},\dots,x_h^{(m)}$ sampled by $\EPCO$, and all $|\MA|$ actions, the resulting $m|\MA|$-dimensional vector $\wh f_{h+1}(x_h^{(i)},x_{h+1};a)_{i,a}$ is close to the corresponding ``true'' vector $f_{h+1}(\phi^\st(x_h^{(i)}),\phi^\st(x_{h+1});a)_{i,a}$ on average over $x_{h+1}$. This will be needed in \cref{lemma:target-error-online} since ultimately the reward functions $\MR^{(t)}$ are constructed by clustering these vectors.

\begin{lemma}\label{lemma:regression-bound-online}
Let $\epsilon,\delta,\delta' \in (0,1)$. Suppose that $\Reg$ is an $\Nreg$-efficient two-context regression oracle for $\Phi$, and $N \geq \Nreg(\epsilon,\delta)$. Then:
\begin{enumerate}
\item \label{item:reg-1-online} With probability at least $1-\delta|\MA|$, it holds that for all $a \in \MA$,
\[\EE_{(x_h,x_{h+1}) \sim \til\BO_h\beta_h \times \til\BO_{h+1}\beta_{h+1}}\left(\wh f_{h+1}(x_h,x_{h+1};a) - f_{h+1}(\phi^\st(x_h),\phi^\st(x_{h+1});a)\right)^2 \leq 2\epsilon.\]
\item \label{item:reg-2-online} With probability at least $1-\delta|\MA|-\delta'\cdot m|\MA|$, it holds that
\[\EE_{x_{h+1}\sim\til\BO_{h+1}\beta_{h+1}} \max_{(i,a)\in [m]\times\MA} \left(\wh f_{h+1}(x_h^{(i)},x_{h+1};a) - f_{h+1}(\phi^\st(x_h^{(i)}),\phi^\st(x_{h+1});a)\right)^2 \leq \frac{2\epsilon m|\MA|}{\delta'}.\]
\end{enumerate}
\end{lemma}

\begin{proof}
To prove the first claim, fix any $a \in \MA$. The dataset $\MD_a$ constructed by \cref{alg:extend-pc-online} consists of $N$ independent and identically distributed tuples $(x_h^{(i)},x_{h+1}^{(i)},y^{(i)})$. Fix any $i \in [N]$ and let $\mu_0 \in \Delta(\MX\times\MX)$ be the probability density function of $(x_h^{(i)},x_{h+1}^{(i)})$ conditioned on $y^{(i)} = 0$. Similarly, let $\mu_1 \in \Delta(\MX\times\MX)$ be the probability density function of $(x_h^{(i)},x_{h+1}^{(i)})$ conditioned on $y^{(i)} = 1$. For any $x,x' \in \MX$,
\begin{equation} \mu_0(x,x') = \beta_h(\phi^\st(x)) \til\BO_h(x\mid{}\phi^\st(x)) \EE_{(s,\tilde a)\sim \beta_h\times\Unif(\MA)} \til\BP^M_{h+1}(\phi^\st(x')\mid{}s,\tilde a)\til\BO_{h+1}(x'\mid{}\phi^\st(x'))\label{eq:mu0-density}\end{equation}
and
\[\mu_1(x,x') = \beta_h(\phi^\st(x)) \til\BO_h(x\mid{}\phi^\st(x)) \til\BP^M_{h+1}(\phi^\st(x')\mid{}\phi^\st(x),a) \til\BO_{h+1}(x'\mid{}\phi^\st(x')).\]
The unconditional probability density function of $(x_h^{(i)},x_{h+1}^{(i)})$ is therefore $\mu \in \Delta(\MX\times\MX)$ defined as
\begin{align*}
&\mu(x,x')  \\
&= \frac{\mu_0(x,x')+\mu_1(x,x')}{2} \\ 
&= \frac{\beta_h(\phi^\st(x))\til\BO_h(x\mid{}\phi^\st(x))}{2}\Bigg(\til\BP^M_{h+1}(\phi^\st(x')\mid{}\phi^\st(x),a) \til\BO_{h+1}(x'\mid{}\phi^\st(x')) \\
&\hspace{10em}+ \EE_{(s,\tilde a)\sim \beta_h\times\Unif(\MA)} \til\BP^M_{h+1}(\phi^\st(x')\mid{}s,\tilde a)\til\BO_{h+1}(x'\mid{}\phi^\st(x'))\Bigg) \\ 
&= \frac{\beta_h(\phi^\st(x))\til\BO_h(x\mid{}\phi^\st(x))}{2}\left(\til\BP^M_{h+1}(\phi^\st(x')\mid{}\phi^\st(x),a)  + \EE_{(s,\tilde a)\sim \beta_h\times\Unif(\MA)} \til\BP^M_{h+1}(\phi^\st(x')\mid{}s,\tilde a)\right) \\
&\hspace{10em} \cdot \til\BO_{h+1}(x'\mid{}\phi^\st(x')).
\end{align*}
From this expression it is clear that, for any $s \in \MS$, conditioned on the event $\phi^\st(x) = s$, $x_h^{(i)}$ and $x_{h+1}^{(i)}$ are independent. Moreover, for any $s \in \MS$, conditioned on the event $\phi^\st(x') = s$, $x_h^{(i)}$ and $x_{h+1}^{(i)}$ are independent. We conclude that $\mu$ is $\phi^\st$-realizable (\cref{def:realizable-distribution}). Next, for any $x,x' \in \MX$, note that
\begin{align*}
\E[y^{(i)}\mid{}x_h^{(i)}=x,x_{h+1}^{(i)}=x']
&= \frac{\mu_1(x,x')}{\mu_0(x,x') + \mu_1(x,x')} \\ 
&= \frac{\til\BP^M_{h+1}(\phi^\st(x')\mid{}\phi^\st(x),a)}{\til\BP^M_{h+1}(\phi^\st(x')\mid{}\phi^\st(x),a) + \EE_{(s,\tilde a)\sim\beta_h\times\Unif(\MA)} \til\BP^M_{h+1}(\phi^\st(x')\mid{}s,\tilde a)} \\ 
&= f_{h+1}(\phi^\st(x),\phi^\st(x');a)
\end{align*}
by \cref{def:kinematics}. Hence, we can apply the guarantee of $\Reg$ (\cref{def:two-con-regression}) with distribution $\mu$ and ground truth predictor $f_{h+1}$. We get that with probability at least $1-\delta$,
\[\EE_{(x_h,x_{h+1}) \sim \mu} \left(\wh f_{h+1}(x_h,x_{h+1};a) - f_{h+1}(\phi^\st(x_h),\phi^\st(x_{h+1});a)\right)^2 \leq \epsilon.\]
But $\mu(x,x') \geq \frac{1}{2}\mu_0(x,x') = \frac{1}{2}\BO_h \beta_h \times \BO_{h+1} \beta_{h+1}$ by \cref{eq:mu0-density} and definition of $\beta_{h+1}$. The first claim of the lemma statement follows.

In the event that the first claim holds, for each $i \in [m]$ and $a \in \MA$, since $x_h^{(i)}$ has distribution $\BO_h\beta_h$, 
Markov's inequality gives that with probability at least $1-\delta'$,
\[\EE_{x_{h+1} \sim \BO_{h+1}\beta_{h+1}}\left(\wh f_{h+1}(x_h^{(i)},x_{h+1};a) - f_{h+1}(\phi^\st(x_h^{(i)}),\phi^\st(x_{h+1});a)\right)^2 \leq \frac{2\epsilon}{\delta'}.\]
By a union bound, we have with probability at least $1-\delta'm|\MA|$ that 
\[\sum_{(i,a)\in[m]\times\MA} \EE_{x_{h+1} \sim \BO_{h+1}\beta_{h+1}}\left(\wh f_{h+1}(x_h^{(i)},x_{h+1};a) - f_{h+1}(\phi^\st(x_h^{(i)}),\phi^\st(x_{h+1});a)\right)^2 \leq \frac{2\epsilon m|\MA|}{\delta'}.\]
Exchanging the summation and expectation completes the proof of the second claim.
\end{proof}

\noindent To prove \cref{lemma:target-error-online}, we need the following preparatory results.

\begin{lemma}\label{lemma:f-ub}
Let $s,s' \in \MS$ and $a \in \MA$. Then 
\[1-f_{h+1}(s,s';a) \geq \frac{\beta_h(s)}{2|\MA|}.\]
\end{lemma}

\begin{proof}
Observe that
\begin{align*}
F_{h+1}(s')
&= \EE_{\pi \sim \frac{1}{2}(\Unif(\Psi_h)+\Unif(\Gamma))\circ_h\Unif(\MA)} \E^{M,\pi}[\til\BP^M_{h+1}(s'\mid{}s_h,a_h)] \\
&\geq \frac{1}{|\MA|} \EE_{\pi \sim \frac{1}{2}(\Unif(\Psi_h)+\Unif(\Gamma))} \E^{M,\pi}[\til\BP^M_{h+1}(s'\mid{}s_h,a)] \\ 
&= \frac{1}{|\MA|} \sum_{s_h \in \MS} \beta_h(s_h) \til\BP^M_{h+1}(s'\mid{}s_h,a) \\ 
&\geq \frac{\beta_h(s)}{|\MA|} \til\BP^M_{h+1}(s'\mid{}s,a).
\end{align*}
It follows that
\begin{align*}
1 - f_{h+1}(s,s';a)
&= \frac{F_{h+1}(s')}{\til\BP^M_{h+1}(s'\mid{}s,a) + F_{h+1}(s')} \\ 
&\geq \frac{\beta_h(s)}{|\MA|} f_{h+1}(s',s;a).
\end{align*}
If $f_{h+1}(s',s;a) \geq 1/2$ then the claim follows; otherwise $1-f_{h+1}(s',s;a) \geq 1/2 \geq \beta_h(s)/(2|\MA|)$ as well.
\end{proof}

\begin{lemma}\label{lemma:lipschitz-bound}
Let $x,y \in [0,1)$. Then
\[\left|\frac{x}{1-x} - \frac{y}{1-y}\right| \leq \frac{2|x-y|}{\min\{1-x,1-y\}^2}.\]
\end{lemma}

\begin{proof}
Define $g: [0,1) \to \RR$ by $g(x) = x/(1-x)$. Then 
\[g'(x) = \frac{1}{1-x} + \frac{x}{(1-x)^2} \leq \frac{2}{(1-x)^2}.\]
Hence, $g'(w) \leq 2/\min\{1-x,1-y\}^2$ for all $w$ in the interval between $x$ and $y$. The lemma follows.
\end{proof}

The following key lemma states for any reachable latent state $s^\st$, there is some reward function $\MR^{(t)}$ so that, under any policy $\pi$, the expected value of $\pi$ under this reward function is roughly proportional to the visitation probability of $s^\st$. Intuitively, this is important because optimizing with respect to $\MR^{(t)}$ then approximately optimizes the probability of reaching $s^\st$. The key to the proof is \cref{eq:dpi-expansion-online}, which shows that for any other state $s$ which has roughly the same \emph{kinematics} as $s^\st$, the probability that any policy $\pi$ visits $s$ is proportional to the probability that $\pi$ visits $s^\st$ (where the constant of proportionality may depend on $s$ and $s^\st$ but not $\pi$).

\begin{lemma}\label{lemma:target-error-online}
Let $\epreg>0$. Condition on $\wh f_{h+1}$ and $x_h^{(1)},\dots,x_h^{(m)}$, and suppose that 
\begin{equation}\EE_{x_{h+1}\sim\BO_{h+1}\beta_{h+1}} \max_{(i,a)\in [m]\times\MA} \left(\wh f_{h+1}(x_h^{(i)},x_{h+1};a) - f_{h+1}(\phi^\st(x_h^{(i)}),\phi^\st(x_{h+1});a)\right)^2 \leq \epreg.\label{eq:reg-bound-applied}\end{equation}
Suppose that for each $s \in \Srch_h(\Gamma)$ there is some $i \in [m]$ with $\phi^\st(x^{(i)}) = s$. Suppose that $\beta_{h+1}(s) \geq \tilde \alpha \cdot \max_{\pi\in\Pi} d^{\Mbar(\Gamma),\pi}_{h+1}(s)$ for all $s \in \MS$. Then for any $t \in [n]$ and $s^\st \in \Srch_{h+1}(\Gamma)$, there is some $K = K(\Gamma,s^\st) \geq 1$ such that
\begin{align*}
&\max_{\pi\in\Pi} \left|\E^{\Mbar(\Gamma),\pi}[\MR^{(t)}(x_{h+1})] - K \cdot d^{\Mbar(\Gamma),\pi}_{h+1}(s^\st)\right| \\ 
&\leq \frac{8\gamma|\MS||\MA|^2}{\min_{s\in\Srch_h(\Gamma)}\beta_h(s)^2} + \frac{\sqrt{\epreg}}{\tilde \alpha \gamma} + \frac{\max_{(i,a)\in[m]\times\MA}\left|\wh f_{h+1}(x_h^{(i)},\bar x_{h+1}^{(t)};a) - f_{h+1}(\phi^\st(x_h^{(i)}),s^\st;a)\right|}{\gamma}
\end{align*}
where for notational convenience we take $\MR^{(t)}(\term) := 0$. 
\end{lemma}

\begin{proof}
First, observe that by the lemma assumption that $\beta_{h+1}(s) \geq \tilde \alpha \cdot \max_{\pi\in\Pi} d^{\Mbar(\Gamma),\pi}_{h+1}(s)$ for all $s \in \MS$, it follows that 
\begin{equation}
(\BO_{h+1}\beta_{h+1})(x) \geq \tilde \alpha \cdot \max_{\pi\in\Pi} d^{\Mbar(\Gamma),\pi}_{h+1}(x)
\label{eq:beta-x-lb}
\end{equation} 
for all $x \in \MX$. Now fix $\pi \in \Pi$. By definition of $\MR^{(t)}$ (\lineref{line:rt-def-online}), we have
\[\E^{\Mbar(\Gamma),\pi}[\MR^{(t)}(x_{h+1})] = \E^{\Mbar(\Gamma),\pi}\left[g\left(\max_{i,a}\left|\wh f_{h+1}(x_h^{(i)},\bar x_{h+1}^{(t)};a) - \wh f_{h+1}(x_h^{(i)},x_{h+1};a)\right|\right)\mathbbm{1}[s_{h+1} \in \MS]\right]\]
where $g(z) := \max(0,1-z/\gamma)$. For any $s,s' \in \MS$, define $\Delta(s,s') := \max_{i,a}|f(\phi^\st(x_h^{(i)}),s;a) - f(\phi^\st(x_h^{(i)}),s';a)|$. Then define 
\[W^\pi := \E^{\Mbar(\Gamma),\pi}[g(\Delta(s_{h+1},s^\st))\mathbbm{1}[s_{h+1}\in\MS]].\]
Then
\begin{align}
&\left|\E^{\Mbar(\Gamma),\pi}[\MR^{(t)}(x_{h+1})] - W^\pi\right| \nonumber\\ 
&\leq \frac{1}{\gamma} \E^{\Mbar(\Gamma),\pi}\left[\left|\max_{i,a}\left|\wh f_{h+1}(x_h^{(i)},\bar x_{h+1}^{(t)};a) - \wh f_{h+1}(x_h^{(i)},x_{h+1};a)\right|-\Delta(s_{h+1},s^\st)\right|\mathbbm{1}[s_{h+1}\in\MS]\right] \nonumber\\ 
&\leq \frac{1}{\gamma}\E^{\Mbar(\Gamma),\pi}\left[\max_{i,a}\left|\wh f_{h+1}(x_h^{(i)},x_{h+1};a) - f_{h+1}(\phi^\st(x_h^{(i)}),s_{h+1};a)\right| \mathbbm{1}[s_{h+1}\in\MS]\right] \nonumber\\ 
&\qquad+ \frac{\max_{i,a}\left|\wh f_{h+1}(x_h^{(i)},\bar x^{(t)}_{h+1};a) - f_{h+1}(\phi^\st(x_h^{(i)}),s^\st;a)\right|}{\gamma} \nonumber\\ 
&\leq \frac{1}{\tilde\alpha \gamma}\EE_{x_{h+1} \sim \BO_{h+1}\beta_{h+1}}\left[\max_{i,a}\left|\wh f_{h+1}(x_h^{(i)},x_{h+1};a) - f_{h+1}(\phi^\st(x_h^{(i)}),\phi^\st(x_{h+1});a)\right|\right] \nonumber\\ 
&\qquad+ \frac{\max_{i,a}\left|\wh f_{h+1}(x_h^{(i)},\bar x^{(t)}_{h+1};a) - f_{h+1}(\phi^\st(x_h^{(i)}),s^\st;a)\right|}{\gamma} \nonumber\\ 
&\leq \frac{\sqrt{\epreg}}{\tilde \alpha \gamma} + \frac{\max_{i,a}\left|\wh f_{h+1}(x_h^{(i)},\bar x^{(t)}_{h+1};a) - f_{h+1}(\phi^\st(x_h^{(i)}),s^\st;a)\right|}{\gamma},
\label{eq:rt-w-diff}
\end{align}
where the first inequality uses the fact that $g$ is $1/\gamma$-Lipschitz; the second inequality is by definition of $\Delta(s_{h+1},s^\st)$ and the triangle inequality; the third inequality is by \cref{eq:beta-x-lb}; and the fourth inequality is by \cref{eq:reg-bound-applied}. 

Next, for any fixed $s_{h+1} \in \Srch_{h+1}(\Gamma)$, recall that $\til\BP^{\Mbar(\Gamma)}_{h+1}(s_{h+1}\mid{}s_h,a_h) = \til\BP^M_{h+1}(s_{h+1}\mid{}s_h,a_h)$ and $\til\BP^{\Mbar(\Gamma)}_{h+1}(s_{h+1}\mid{}\term,a_h) = 0$ for all $s_h\in\MS$, $a_h \in \MA$ (by \cref{def:truncated-bmdp}); hence,
\begin{align}
d^{\Mbar(\Gamma),\pi}_{h+1}(s_{h+1}) 
&= \sum_{(s_h,a_h) \in \MS\times\MA} d^{\Mbar(\Gamma),\pi}_h(s_h,a_h)  \til\BP^M_{h+1}(s_{h+1}\mid{}s_h,a_h) \nonumber\\ 
&= \left(\sum_{(s_h,a_h)\in\MS\times\MA} d^{\Mbar(\Gamma),\pi}_h(s_h,a_h) \frac{f_{h+1}(s_h,s_{h+1};a_h)}{1-f_{h+1}(s_h,s_{h+1};a_h)}\right) F_{h+1}(s_{h+1})\label{eq:dpi-expansion-online}
\end{align} 
where the second equality is by definition of $f_{h+1},F_{h+1}$ (\cref{def:kinematics}). Substituting \cref{eq:dpi-expansion-online} into the definition of $W^\pi$ and using the fact that $d^{\Mbar(\Gamma),\pi}_{h+1}(s_{h+1}) = 0$ for all $s_{h+1} \in \MS\setminus \Srch_{h+1}(\Gamma)$, we get that 
\[W^\pi = \sum_{s_{h+1}\in\Srch_{h+1}(\Gamma)} \left(\sum_{(s_h,a_h)\in\MS\times\MA} d^{\Mbar(\Gamma),\pi}_h(s_h,a_h) \frac{f_{h+1}(s_h,s_{h+1};a_h)}{1-f_{h+1}(s_h,s_{h+1};a_h)}\right) F_{h+1}(s_{h+1}) g(\Delta(s_{h+1},s^\st)).\]
Define 
\[\widetilde W^\pi :=  \left(\sum_{(s_h,a_h)\in\MS\times\MA} d^{\Mbar(\Gamma),\pi}_h(s_h,a_h) \frac{f_{h+1}(s_h,s^\st;a_h)}{1-f_{h+1}(s_h,s^\st;a_h)}\right) \sum_{s_{h+1}\in\Srch_{h+1}(\Gamma)} F_{h+1}(s_{h+1}) g(\Delta(s_{h+1},s^\st)).\]
Then
\begin{align}
&|W^\pi - \til W^\pi| \nonumber\\ 
&\leq \sum_{s_{h+1}\in\Srch_{h+1}(\Gamma)} \left(\sum_{(s_h,a_h)\in\MS\times\MA} d^{\Mbar(\Gamma),\pi}_h(s_h,a_h) \left|\frac{f_{h+1}(s_h,s_{h+1};a_h)}{1-f_{h+1}(s_h,s_{h+1};a_h)} - \frac{f_{h+1}(s_h,s^\st;a_h)}{1-f_{h+1}(s_h,s^\st;a_h)}\right|\right)  \nonumber\\ &\qquad\qquad \cdot F_{h+1}(s_{h+1}) g(\Delta(s_{h+1},s^\st)) \nonumber\\ 
&\leq \sum_{s_{h+1}\in\Srch_{h+1}(\Gamma)} \left(\max_{(s_h,a_h) \in \Srch_h(\Gamma)\times\MA} \left|\frac{f_{h+1}(s_h,s_{h+1};a_h)}{1-f_{h+1}(s_h,s_{h+1};a_h)} - \frac{f_{h+1}(s_h,s^\st;a_h)}{1-f_{h+1}(s_h,s^\st;a_h)}\right|\right) F_{h+1}(s_{h+1}) g(\Delta(s_{h+1},s^\st)) \nonumber\\ 
&\leq 2\sum_{s_{h+1}\in\Srch_{h+1}(\Gamma)} \left(\frac{\max_{(s_h,a_h) \in \Srch_h(\Gamma)\times\MA} \left|f_{h+1}(s_h,s_{h+1};a_h) - f_{h+1}(s_h,s^\st;a_h)\right|}{\min_{(s_h,a_h,s) \in \Srch_h(\Gamma)\times\MA\times\MS} (1-f_{h+1}(s_h,s;a_h))^2}\right) F_{h+1}(s_{h+1}) g(\Delta(s_{h+1},s^\st)) \nonumber\\
&\leq 2\sum_{s_{h+1}\in\Srch_{h+1}(\Gamma)} \left(\frac{\Delta(s_{h+1},s^\st)}{\min_{(s_h,a_h,s) \in \Srch_h(\Gamma)\times\MA\times\MS} (1-f_{h+1}(s_h,s;a_h))^2}\right) F_{h+1}(s_{h+1}) g(\Delta(s_{h+1},s^\st)) \nonumber\\
&\leq 8|\MA|^2\sum_{s_{h+1}\in\Srch_{h+1}(\Gamma)} \left(\frac{\Delta(s_{h+1},s^\st)}{\min_{s \in \Srch_h(\Gamma)} \beta_h(s)^2}\right) F_{h+1}(s_{h+1}) g(\Delta(s_{h+1},s^\st)) \nonumber\\
&\leq \frac{8\gamma |\MS| |\MA|^2}{\min_{s \in \Srch_h(\Gamma)} \beta_h(s)^2}
\label{eq:w-wtilde-diff}
\end{align}
where the second inequality uses the fact that $d^{\Mbar(\Gamma),\pi}_h(s_h) = 0$ for all $s_h \in \MS \setminus \Srch_h(\Gamma)$; the third inequality uses \cref{lemma:lipschitz-bound}; the fourth inequality uses the definition of $\Delta(s_{h+1},s^\st)$ together with the lemma assumption that $\Srch_h(\Gamma) \subseteq \{\phi^\st(x^{(i)}_h):i\in[m]\}$; the fifth inequality uses \cref{lemma:f-ub}; and the sixth inequality uses the fact that $F_{h+1}(s_{h+1}) \leq 1$ for all $s_{h+1} \in \MS$ (\cref{def:kinematics}) together with the bound $z \cdot g(z) \leq \gamma$ for all $z \geq 0$. 

Finally, note that by definition of $\widetilde W^\pi$ and \cref{eq:dpi-expansion-online} applied to $s^\st$, we have
\begin{align}
\widetilde W^\pi 
&=  \left(\sum_{(s_h,a_h)\in\MS\times\MA} d^{\Mbar(\Gamma),\pi}_h(s_h,a_h) \frac{f_{h+1}(s_h,s^\st;a_h)}{1-f_{h+1}(s_h,s^\st;a_h)}\right) \sum_{s_{h+1}\in\Srch_{h+1}(\Gamma)} F_{h+1}(s_{h+1}) g(\Delta(s_{h+1},s^\st)) \nonumber \\ 
&= \frac{d^{\Mbar(\Gamma),\pi}_{h+1}(s^\st)}{ F_{h+1}(s^\st)} \sum_{s_{h+1}\in\Srch_{h+1}(\Gamma)} F_{h+1}(s_{h+1})g(\Delta(s_{h+1},s^\st)) \nonumber \\ 
&= K(\Gamma,s^\st) \cdot d^{\Mbar(\Gamma),\pi}_{h+1}(s^\st) \label{eq:wtilde-expansion-online}
\end{align}
where \[K(\Gamma,s^\st) := \frac{\sum_{s_{h+1}\in\Srch_{h+1}(\Gamma)}  F_{h+1}(s_{h+1})g(\Delta(s_{h+1},s^\st))}{F_{h+1}(s^\st)}.\]
Since $s^\st \in \Srch_{h+1}(\Gamma)$ and $g(\Delta(s^\st,s^\st)) = g(0) = 1$, we have $K(\Gamma,s^\st) \geq 1$. Combining \cref{eq:rt-w-diff,eq:w-wtilde-diff,eq:wtilde-expansion-online} yields the lemma claim.
\end{proof}

We now prove \cref{thm:extend-pc-trunc-online} by combining \cref{lemma:regression-bound-online,lemma:target-error-online} with a standard guarantee for $\PSDP$ (\cref{lemma:psdp-trunc-online}). We use the assumption that $\Psi_1,\dots,\Psi_h$ are truncateed policy covers for steps $1,\dots,h$ to show that the hypotheses for \cref{lemma:target-error-online} are satisfied with high probability. We remark that $\PSDP$ naturally uses a one-context regression oracle and not two; this is why we invoke it with a one-context regression oracle $\Reg'$ obtained by reduction to two-context regression (\cref{alg:onetwo}; see \cref{prop:onetwo}).

\vspace{1em}

\begin{proof}[Proof of \cref{thm:extend-pc-trunc-online}]
For any fixed $s \in \Srch_h(\Gamma)$ and $i \in [m]$, since $\phi^\st(x_h^{(i)}) \sim \beta_h$ and $\Psi_h$ is an $\alpha$-truncated policy cover for $M$ at step $h$, we have by \cref{item:h-cov-lb} of \cref{lemma:srch-gamma-covering} that
\[\Pr[\phi^\st(x_h^{(i)}) = s] \geq \frac{\min(\alpha\trunc,\tsmall)}{2}.\]
Let $\ME_1$ be the event that for each $s \in \Srch_h(\Gamma)$ there is some $i \in [m]$ with $\phi^\st(x^{(i)}) = s$. Then
\[\Pr[\ME_1] \geq 1 - |\MS|\left(1 - \frac{\min(\alpha\trunc,\tsmall)}{2}\right)^m \geq 1-\delta\]
by the theorem assumption that $m \geq \frac{2}{\min(\alpha\trunc,\tsmall)} \log(|\MS|/\delta)$. Henceforth condition on $x_h^{(1)},\dots,x_h^{(m)}$ and suppose that $\ME_1$ holds.

Let $\ME_2$ be the event that
\[\EE_{x_{h+1}\sim\til\BO_{h+1}\beta_{h+1}} \max_{(i,a)\in [m]\times\MA} \left(\wh f_{h+1}(x_h^{(i)},x_{h+1};a) - f_{h+1}(\phi^\st(x_h^{(i)}),\phi^\st(x_{h+1});a)\right)^2 \leq 2\sqrt{\epsilon} m|\MA|.\]
By the theorem assumptions that $\Reg$ is an $\Nreg$-efficient two-context regression oracle for $\Phi$, and $N \geq \Nreg(\epsilon,\delta)$, we may apply \cref{item:reg-2-online} of \cref{lemma:regression-bound-online} (with parameter $\delta' := \sqrt{\epsilon}$) to get that $\ME_2$ occurs with probability at least $1-\delta|\MA| - \sqrt{\epsilon} m|\MA|$ over the randomness of $(\MD_a)_{a \in \MA}$ and $\Reg$. Condition on this randomness (which determines $\wh f_{h+1}$) and suppose that $\ME_2$ holds.

For each $t \in [n]$, let $\ME_3^t$ be the event that \[\max_{(i,a)\in[m]\times\MA} \left(\wh f_{h+1}(x_h^{(i)},\bar x_{h+1}^{(t)};a) - f_{h+1}(\phi^\st(x_h^{(i)}),\phi^\st(\bar x_{h+1}^{(t)}); a)\right)^2 \leq 2\epsilon^{1/4}m|\MA|.\] Since $\bar x_{h+1}^{(t)} \sim \BO_{h+1}\beta_{h+1}$, we have by Markov's inequality and $\ME_2$ that $\Pr[\lnot \ME_3^t] \leq \epsilon^{1/4}$. Define $\ME_3 := \bigcap_{t=1}^n \ME_3^t$. By the union bound, $\ME_3$ occurs with probability at least $1-n\epsilon^{1/4}$ over the randomness of $\bar x_{h+1}^{(1)},\dots,\bar x_{h+1}^{(n)}$. 

Also, for each $s \in \Srch_{h+1}(\Gamma)$, let $t(s) \in [1,n] \cup \{\infty\}$ be the infimum over $t$ such that $\phi^\st(\bar x^{(t)}_{h+1}) = s$, and let $\ME_4^s$ be the event that $t(s) < \infty$. For each $t \in [1,n]$, since $\phi^\st(\bar x^{(t)}_{h+1})$ has distribution $\beta_{h+1}$, we have 
\begin{align*}
\Pr[\phi^\st(\bar x^{(t)}_{h+1})=s] 
&\geq \frac{\min(\alpha\trunc,\tsmall)}{2|\MA|} \max_{\pi\in\Pi} d^{\Mbar(\Gamma),\pi}_{h+1}(s) \\ 
&\geq \frac{\min(\alpha\trunc,\tsmall)}{2|\MA|}  \max_{\pi\in\Pi} d^{\Mbar(\emptyset),\pi}_{h+1}(s) \\ &\geq \frac{\min(\alpha\trunc,\tsmall)\trunc}{2|\MA|} \end{align*}
where the first inequality is by \cref{item:h-plus-one-cov-lb} of \cref{lemma:srch-gamma-covering}; the second inequality is by \cref{fact:gamma-monotonicity}; and the third inequality is by \cref{fact:trunc-reachability}.
Thus, $\Pr[\lnot \ME_4^s] \leq (1-\frac{\min(\alpha\trunc,\tsmall)\trunc}{2|\MA|})^n \leq \delta/|\MS|$ for any fixed $s \in \Srch_{h+1}$, by the theorem assumption that $n \geq \frac{2|\MA|}{\min(\alpha\trunc,\tsmall)\trunc} \log(|\MS|/\delta)$. Define $\ME_4 := \bigcap_{s \in \Srch_{h+1}} \ME_4^s$. By the union bound, $\ME_4$ occurs with probability at least $1-\delta$ over the randomness of $\bar x_{h+1}^{(1)},\dots,\bar x_{h+1}^{(n)}$. Condition on $\bar x_{h+1}^{(1)},\dots,\bar x_{h+1}^{(n)}$ and suppose that $\ME_3 \cap \ME_4$ holds. 

Finally, for each $t \in \Tclus$ let $\ME_5^t$ be the event that
\[\E^{M,\wh \pi^{(t)}}[\MR^{(t)}(x_{h+1})] \geq \max_{\pi\in\Pi} \E^{\Mbar(\Gamma),\pi}[\MR^{(t)}(x_{h+1})] - \frac{4H\sqrt{|\MA|\epsilon}}{\min(\alpha\trunc,\tsmall)}\]
where $\pihat^{(t)}$ is defined on \lineref{line:psdp-call-online} of \cref{alg:extend-pc-online}. Since $\Reg$ is an $\Nreg$-efficient two-context regression oracle, \cref{prop:onetwo} implies that $\Reg'$ is an $\Nreg$-efficient one-context regression oracle. Hence, by the theorem assumption on $\Psi_{1:h}$, and the fact that $N \geq \Nreg(\epsilon,\delta)$, \cref{lemma:psdp-trunc-online} gives that $\Pr[\ME_5^t] \geq 1-H|\MA|\delta$, so $\Pr[\ME_5] \geq 1-HAn\delta$ where $\ME_5 := \cap_{t \in [n]} \ME_5^t$. Condition on $\ME_5$. We have now restricted to an event of total probability at least $1 - (2+|\MA|+H|\MA|n)\delta - m|\MA|\epsilon^{1/2} - n\epsilon^{1/4}$; we argue that in this event, the properties claimed in the theorem statement hold.

\paragraph{Size of $\Psi_{h+1}$.} First, we argue that $|\Psi_{h+1}| \leq |\MS|$. Indeed, suppose that there are $t,t' \in \Tclus$ with $t<t'$ and $\phi^\st(\bar x_{h+1}^{(t)}) = \phi^\st(\bar x_{h+1}^{(t')})$. By \lineref{line:cluster-threshold-online} of \cref{alg:extend-pc-online}, and by choice of $\gamma'$, we know that 
\[\max_{(i,a) \in [m]\times\MA} \left|\wh f_{h+1}(x_h^{(i)},\bar x_{h+1}^{(t)};a) - \wh f_{h+1}(x_h^{(i)},\bar x_{h+1}^{(t')};a)\right| > \gamma' = 4\epsilon^{1/8}\sqrt{m|\MA|}.\] But this contradicts $\ME_3$ (in particular, the bounds implied by $\ME_3^t$ and $\ME_3^{t'}$ together with the triangle inequality and the fact that $\phi^\st(\bar x_{h+1}^{(t)}) = \phi^\st(\bar x_{h+1}^{(t')})$). We conclude that indeed $|\Psi_{h+1}| \leq |\MS|$.

\paragraph{Coverage of $\Psi_{h+1}$.} It remains to prove the second property of the theorem statement. Fix $s \in \Srch_{h+1}(\emptyset)$. Define $t^\st(s) := t(s)$ if $t(s) \in \Tclus$. Otherwise, let $t^\st(s)$ be the minimal $t \in \Tclus$ such that $\max_{i \in [m]} |\MR_i(\bar x_{h+1}^{(t(s))})-\MR_i(\bar x_{h+1}^{(t)})| \leq \gamma'$ (which exists by \lineref{line:cluster-threshold-online}). In either case, we know that $t^\st(s) \in \Tclus$, and \begin{equation}
\max_{(i,a) \in [m]\times\MA} |\wh f_{h+1}(x_h^{(i)},\bar x_{h+1}^{(t(s))};a)-\wh f_{h+1}(x_h^{(i)},\bar x_{h+1}^{(t^\st(s))};a)| \leq \gamma'.
\label{eq:t-tst-close}
\end{equation}

By $\ME_1$, $\ME_2$, and \cref{item:h-plus-one-cov-lb} of \cref{lemma:srch-gamma-covering}, we may apply \cref{lemma:target-error-online} with target state $s^\st := s$, index $t := t^\st(s)$, and parameters $\epreg := 2\sqrt{\epsilon}m|\MA|$ and $\tilde \alpha := \frac{\min(\alpha\trunc,\tsmall)}{2|\MA|}$. We get that there is some $K(\Gamma,s) \geq 1$ such that
\begin{align}
&\max_{\pi\in\Pi} \left|\E^{\Mbar(\Gamma),\pi}[\MR^{(t^\st(s))}(x_{h+1})] - K(\Gamma,s) \cdot d^{\Mbar(\Gamma),\pi}_{h+1}(s)\right|  \nonumber\\ 
&\leq \frac{8\gamma|\MS||\MA|^2}{\min_{s'\in\Srch_h(\Gamma)}\beta_h(s')^2} + \frac{\sqrt{\epreg}}{\tilde \alpha \gamma} + \frac{\max_{(i,a)\in[m]\times\MA}\left|\wh f_{h+1}(x_h^{(i)},\bar x_{h+1}^{(t^\st(s))};a) - f_{h+1}(\phi^\st(x_h^{(i)}),s;a)\right|}{\gamma} \nonumber\\
&\leq \frac{32\gamma|\MS||\MA|^2}{\min(\alpha^2\trunc^2,\tsmall^2)} + \frac{4\epsilon^{1/4}|\MA|\sqrt{m|\MA|}}{\gamma \cdot \min(\alpha\trunc,\tsmall)} \nonumber\\
&\qquad+ \frac{\max_{(i,a)\in[m]\times\MA}\left|\wh f_{h+1}(x_h^{(i)},\bar x^{(t^\st(s))}_{h+1};a) - f_{h+1}(\phi^\st(x_h^{(i)}),s;a)\right|}{\gamma} \nonumber \\
&\leq \frac{32\gamma|\MS||\MA|^2}{\min(\alpha^2\trunc^2,\tsmall^2)} + \frac{4\epsilon^{1/4}|\MA|\sqrt{m|\MA|}}{\gamma \cdot \min(\alpha\trunc,\tsmall)} \nonumber \\ 
&\qquad+ \frac{\epsilon^{1/8}\sqrt{2m|\MA|}+\max_{(i,a)\in[m]\times\MA}\left|\wh f_{h+1}(x_h^{(i)},\bar x^{(t^\st(s))}_{h+1};a) - \wh f_{h+1}(x_h^{(i)},\bar x_{h+1}^{(t(s))};a)\right|}{\gamma} \nonumber \\
&\leq \frac{32\gamma|\MS||\MA|^2}{\min(\alpha^2\trunc^2,\tsmall^2)} + \frac{4\epsilon^{1/4}|\MA|\sqrt{m|\MA|}}{\gamma \cdot \min(\alpha\trunc,\tsmall)} + \frac{\epsilon^{1/8}\sqrt{2m|\MA|}+\gamma'}{\gamma} \nonumber \\
&\leq \frac{32\epsilon^{1/16}|\MS||\MA|^2}{\min(\alpha^2\trunc^2,\tsmall^2)} + \frac{4\epsilon^{3/16}|\MA|\sqrt{m|\MA|}}{\min(\alpha\trunc,\tsmall)} + 6\epsilon^{1/16}\sqrt{m|\MA|} \nonumber \\
&\leq \trunc^2,\label{eq:rkd-error-online}
\end{align}
where the second inequality is by \cref{item:h-cov-lb} of \cref{lemma:srch-gamma-covering}; the third inequality is by $\ME_3$ and the fact that $s = \phi^\st(\bar x_{h+1}^{(t(s))})$; the fourth inequality is by \cref{eq:t-tst-close}; the fifth inequality is by choice of $\gamma,\gamma'$; and the final inequality is by \cref{eq:param-assm-online}. We now distinguish two cases:

\paragraph{Case I.} Suppose that 
\[\E^{M,\pihat^{(t^\st(s))}}[\MR^{(t^\st(s))}(x_{h+1})] \geq \E^{\Mbar(\Gamma),\pihat^{(t^\st(s))}}[\MR^{(t^\st(s))}(x_{h+1})] + \trunc^2.\]
Then
\begin{align*}
&d^{\Mbar(\Gamma),\pihat^{(t^\st(s))}}_{h+1}(\term) \\ 
&= \sum_{s \in \MS} \left(d^{M,\pihat^{(t^\st(s))}}_{h+1}(s) - d^{\Mbar(\Gamma),\pihat^{(t^\st(s))}}_{h+1}(s))\right) \\ 
&\geq \sum_{s \in \MS} \left(d^{M,\pihat^{(t^\st(s))}}_{h+1}(s) - d^{\Mbar(\Gamma),\pihat^{(t^\st(s))}}_{h+1}(s))\right) \EE_{x\sim\BO_{h+1}(\cdot|s)}[\MR^{(t^\st(s))}(x)] \\ 
&\geq \trunc^2
\end{align*}
where the equality is by the fact that $d^{M,\pihat^{(t^\st(s))}}_{h+1}(\cdot)$ is a distribution supported on $\MS$; the first inequality uses \cref{fact:gamma-monotonicity} and the fact that $\MR^{(t^\st(s))}(x) \leq 1$ for all $x \in \MX$. Thus $\max_{\pi\in\Psi_{h+1}} d^{\Mbar(\Gamma),\pi}(\term) \geq \trunc^2$, so the second property of the theorem statement is satisfied.

\paragraph{Case II.} Suppose that
\begin{equation} \E^{M,\pihat^{(t^\st(s))}}[\MR^{(t^\st(s))}(x_{h+1})] < \E^{\Mbar(\Gamma),\pihat^{(t^\st(s))}}[\MR^{(t^\st(s))}(x_{h+1})] + \trunc^2.\label{eq:m-mbar-closeness-online}\end{equation}
Now, 
\begin{align*}
K(\Gamma, s) \cdot d^{\Mbar(\Gamma), \pihat^{t^\st(s)}}_{h+1}(s) 
&\geq \E^{\Mbar(\Gamma),\pihat^{t^\st(s)}}[\MR^{(t^\st(s))}(x_{h+1})] - \trunc^2 \\ 
&\geq \E^{M,\pihat^{t^\st(s)}}[\MR^{(t^\st(s))}(x_{h+1})] - 2\trunc^2 \\ 
&\geq \max_{\pi\in\Pi} \E^{\Mbar(\Gamma),\pi}[\MR^{(t^\st(s))}(x_{h+1})] -  3\trunc^2 \\ 
&\geq \max_{\pi\in\Pi} K(\Gamma,s) \cdot d^{\Mbar(\Gamma),\pi}_{h+1}(s) -  4\trunc^2 \\
&\geq K(\Gamma,s) (1-4\trunc) \max_{\pi\in\Pi} d^{\Mbar(\Gamma),\pi}_{h+1}(s)
\end{align*}
where the first inequality is by \cref{eq:rkd-error-online}; the second inequality is by \cref{eq:m-mbar-closeness-online}; the third inequality is by $\ME_5$ and \cref{eq:param-assm-online}; the fourth inequality is by \cref{eq:rkd-error-online}; and the fifth inequality uses the fact that $\max_{\pi\in\Pi} d^{\Mbar(\Gamma),\pi}_{h+1}(s) \geq \trunc$ (\cref{fact:trunc-reachability} together with \cref{fact:gamma-monotonicity} and the fact that $s \in \Srch_{h+1}(\emptyset)$) and the bound $K(\Gamma,s) \geq 1$ (\cref{lemma:target-error-online}). Thus, since $\pihat^{(t^\st(s))} \in \Psi_{h+1}$, we have
\begin{align}
\max_{\pi\in\Psi_{h+1}} d^{M,\pi}_{h+1}(s) 
&\geq \max_{\pi\in\Psi_{h+1}} d^{\Mbar(\Gamma),\pi}_{h+1}(s) \nonumber\\ 
&\geq (1-4\trunc)\max_{\pi\in\Pi} d^{\Mbar(\Gamma),\pi}_{h+1}(s) \nonumber\\ 
&\geq (1-4\trunc)\max_{\pi\in\Pi} d^{\Mbar(\emptyset),\pi}_{h+1}(s)
\label{eq:coverage-final-online}
\end{align}
by two applications of \cref{fact:gamma-monotonicity}. Now recall that $s \in \Srch_{h+1}(\emptyset)$ was arbitrary. Moreover, if $s \in \MS\setminus \Srch_{h+1}(\emptyset)$ then $\max_{\pi\in\Pi} d^{\Mbar(\emptyset),\pi}_{h+1}(s) = 0$, so the inequality \cref{eq:coverage-final-online} still holds. We conclude that $\Psi_{h+1}$ is a $(1-4\trunc)$-truncated max policy cover for $M$ at step $h+1$ (\cref{defn:trunc-max-pc}), as needed.
\end{proof}

\subsection{Analysis of $\PCO$ (\creftitle{alg:pco})}\label{sec:pco-analysis}

We now complete the proof of \cref{thm:pco-app}. As previously discussed, \cref{thm:extend-pc-trunc-online} shows that in each round of $\PCO$, either a set of policy covers was constructed, or a new state was discovered. The latter can happen at most $H|\MS|-1$ times, so at least one round must construct a good set of policy covers. In the latter event, the union of all sets produced across all rounds is itself a good policy cover (so long as the individual sets have bounded size). 
Of course, since the guarantee of \cref{thm:extend-pc-trunc-online} is probabilistic, some additional care is needed. We make the argument formal below.

\vspace{1em}

\begin{proof}[Proof of \cref{thm:pco-app}]
Fix the remaining inputs $\epfinal,\delta>0$ to $\PCO(\Reg,\Nreg,|\MS|,\cdot)$. The oracle time complexity bound and bound on $\Nrl$ are clear from the parameter choices and pseudocode, so long as $C_{\ref{thm:pco-app}}$ is a sufficiently large constant. Moreover, it is immediate from the algorithm description that $|\Psi| \leq HR|\MS| \leq H^2|\MS|^2$. In order to show that the algorithm is $(\Nrl,\Krl)$-efficient, it remains to argue that with probability at least $1-\delta$, \cref{eq:rfrl-pc} holds for all $h \in [H]$ and $s \in \MS$, with parameter $\epfinal$.

Recall that $\trunc = \epfinal/(4+H|\MS|)$. Fix some $1 \leq r \leq R$. For convenience, write $\alpha := \frac{1-4\trunc}{|\MS|}$. For each $h \in [H]$, let $\ME_{h,r}$ be the event that $|\Psi_h^{(r)}| \leq |\MS|$ and $\Psi_h^{(r)}$ is a $(1-4\trunc)$-truncated max policy cover for $M$ at step $h$; let $\MF_{h,r}$ be the event that $|\Psi_{h}^{(r)}| \leq |\MS|$ and $\max_{\pi\in\Psi_{h}^{(r)}} d^{\Mbar(\Gamma^{(r)}),\pi}_{h}(\term) \geq \trunc^2$. It's clear that $\Pr[\ME_{1,r}] = 1$ (since $|\Psi_1^{(r)}| = 1$ and $d^{M,\piunif}_1(s) = d^{M,\pi}_1(s)$ for all $s\in\MS$ and $\pi \in \Pi$). Also, note that in the event $\ME_{k,r}$, we have that $\Psi_k^{(r)}$ is an $\alpha$-truncated policy cover for $M$ at step $k$. Thus, by \cref{thm:extend-pc-trunc-online} and choice of parameters (so long as $C_{\ref{thm:pco-app}}$ is a sufficiently large constant), we have for each $h \in \{2,\dots,H\}$ that
\begin{equation} \Pr\left[\lnot (\ME_{h,r} \cup \MF_{h,r}) \cap \bigcap_{1 \leq k < h} \ME_{k,r}\right] \leq \Pr\left[\lnot (\ME_{h,r} \cup \MF_{h,r}) \middle | \bigcap_{1 \leq k < h} \ME_{k,r}\right] \leq \frac{\delta}{HR}.\label{eq:epco-guarantee-symbolic}\end{equation}
In the event that the events $\MF_{1,r},\dots,\MF_{H,r}, \bigcap_{h\in[H]} \ME_{h,r}$ all fail, there is always some maximal $h \in [H]$ such that $\bigcap_{1 \leq k \leq h} \ME_{k,r}$ holds (since $\ME_{1,r}$ always holds); it must be that $1 \leq h<H$, and $\ME_{h+1,r}$ and $\MF_{h+1,r}$ both fail. Thus,
\[\Pr\left[\left(\lnot \bigcap_{h \in [H]} \ME_{h,r}\right) \cap \bigcap_{h \in [H]} \left(\lnot \MF_{h,r}\right)\right] \leq \sum_{h=2}^H \Pr\left[\lnot (\ME_{h,r} \cup \MF_{h,r}) \cap \bigcap_{1 \leq k < h} \ME_{k,r}\right] \leq \frac{\delta}{R}\]
where the final inequality is by \cref{eq:epco-guarantee-symbolic}. Let $\ME_r$ denote the complementary event (i.e. either $\bigcap_{h \in [H]} \ME_{h,r}$ holds, or there is some $h \in [H]$ such that $\MF_{h,r}$ holds), and let $\ME := \bigcap_{1 \leq r \leq R} \ME_r$; we have $\Pr[\ME] \geq 1-\delta$. We claim that $\PCO$ succeeds under event $\ME$. Indeed, there are two cases to consider.

\begin{enumerate}
\item In the first case, there is some $r \in [R]$ such that $\bigcap_{h \in [H]} \ME_{h,r}$ holds. Then for each $h \in [H]$, $|\Psi_h^{(r)}| \leq |\MS|$ and $\Psi_h^{(r)}$ is a $(1-4\trunc)$-truncated max policy cover for $M$ at step $h$. Thus, $\bigcup_{1 \leq r' \leq R: |\Psi_h^{(r')}| \leq |\MS|} \Psi_h^{(r')} \subseteq \Psi$ is also a $(1-4\trunc)$-truncated max policy cover for $M$ at step $h$. Therefore for each $h \in [H]$ and $s \in \MS$, we have by \cref{defn:trunc-max-pc} that
\begin{align*}
\max_{\pi'\in\Psi} d^{M,\pi'}_h(s)
&\geq (1-4\trunc) \max_{\pi\in\Pi} d^{\Mbar(\emptyset),\pi}_h(s) \\ 
&\geq \max_{\pi \in\Pi} d^{\Mbar(\emptyset),\pi}_h(s) - 4\trunc \\ 
&\geq \max_{\pi \in \Pi} d^{M,\pi}_h(s) - (4 + H|\MS|)\trunc
\end{align*}
where the final inequality is by \cref{lemma:term-ub}. Since $\trunc = \epfinal/(4+H|\MS|)$, this bound suffices.
\item In the second case, for each $r \in [R]$, there is some $h \in [H]$ such that $\MF_{h,r}$ holds. For each $r$, define
\[\MV^{(r)} := \left\{(s,h) \in \MS\times[H]: \max_{\pi\in\Gamma^{(r)}} d^{M,\pi}_h(s) \geq \frac{\trunc^2}{|\MS|H}\right\}.\]
Fix any $r \in [R]$. By assumption, there is some $h \in [H]$ so that $\MF_{h,r}$ holds. For this choice of $h$, we have by definition of $\MF_{h,r}$ that 
\[\max_{\pi \in \Psi_{h}^{(r)}} d^{\Mbar(\Gamma),\pi}_h(\term) \geq \trunc^2.\] 
By \cref{lemma:term-prob}, there is some $(s,k) \in (\MS\setminus \Srch_k(\Gamma))\times[h]$ such that \[\max_{\pi\in\Psi_h^{(r)}} d^{M,\pi}_k(s) \geq \frac{\trunc^2}{|\MS|H}.\] Thus $(s,k) \in \MV^{(r+1)}$. Moreover, since $s \not \in \Srch_k(\Gamma^{(r)})$, we have $\E_{\pi\sim\Unif(\Gamma^{(r)})} d^{M,\pi}_k(s) < \tsmall$. Using the fact that $|\Gamma^{(r)}| \leq RH|\MS|$ and choice of $\tsmall$, it follows that \[\max_{\pi\in\Gamma^{(r)}} d^{M,\pi}_k(s) < RH|\MS|\tsmall \leq \frac{\trunc^2}{|\MS|H}.\] So $(s,k) \not \in \MV^{(r)}$. We conclude that $|\MV^{(r+1)}| > |\MV^{(r)}|$. Since this inequality holds for all $r \in [R]$ and $|\MV^{(1)}| \geq H$, we get $|\MV^{(R)}| \geq H + R - 1 > |\MS|H$. Contradiction, so in fact this second case cannot occur.
\end{enumerate}
This completes the proof.
\end{proof}


\newpage

\section{Proof of \creftitle{cor:regression-to-online-rl}}\label{sec:app_minimality}


In this section we prove \cref{cor:regression-to-online-rl}, restated below. This theorem asserts that for any \emph{regular} (\cref{def:regular}) concept class $\Phiaug$, there is a reduction $\RegToRL$ (\cref{alg:regtorl}) from two-context regression to reward-free episodic RL.

\regtorl*

Any regular concept class can be defined by augmenting some base class $\Phi$ as specified in \cref{def:phiaug}. Accordingly, the main component of $\RegToRL$ is a reduction $\TwoRed$ (\cref{alg:twored}) from two-context regression over $\Phi$ to reward-free episodic RL over $\Phiaug$. The full reduction $\RegToRL$ simply applies $\TwoRed$ in conjunction with a reduction $\TwoAug$ (\cref{alg:twoaug}) from two-context regression over $\Phiaug$ to two-context regression over $\Phi$.

Henceforth, fix sets $\MS,\MX$ and a concept class $\Phi \subset (\MX\to\MS)$. We may define an augmented concept class $\Phiaug \subseteq (\Xaug\to\Saug)$ as follows (in \cref{sec:regtorl} we will formally argue that any regular concept class can be expressed in this way):

\begin{definition}[Augmented concept class]\label{def:phiaug}
Define augmented state space $\Saug := \MS \sqcup \{0,1\}$ and augmented observation space $\Xaug := \MX \sqcup \{0,1\}$. For each $\phi \in \Phi$ define $\aug(\phi): \Xaug \to \Saug$ by
\[\aug(\phi)(x) := \begin{cases} \phi(x) & \text { if } x \in \MX \\ x & \text { otherwise } \end{cases}.\]
Finally, define an extended function class $\Phiaug := \{\aug(\phi): \phi\in\Phi\}$.
\end{definition}

In \cref{sec:regtorl-overview} we give pseudocode and an overview of $\RegToRL$ and its main subroutine $\TwoRed$. In \cref{sec:twored} we formally analyze $\TwoRed$. In \cref{sec:regtorl} we use this to analyze $\RegToRL$, completing the proof of \cref{cor:regression-to-online-rl}. We defer the analysis of $\TwoAug$ to \sssref{sec:twoaug}.

\subsection{$\RegToRL$ Pseudocode and Overview}\label{sec:regtorl-overview}

We start by giving a brief overview of the reduction $\RegToRL$ (\cref{alg:regtorl}); for additional intuition, see also the discussion in \cref{sec:episodic-nec}. The main subroutine is $\TwoRed$ (\cref{alg:twored}), which is used to reduce two-context regression over $\Phi$ to reward-free RL over $\Phiaug$. We start with an overview of this subroutine, and then briefly discuss $\TwoAug$ (\cref{alg:twoaug}), the reduction from two-context regression over $\Phiaug$ to two-context regression over $\Phi$. The full reduction $\RegToRL$ is a direct combination of these two subroutines.

\subsubsection{$\TwoRed$: Simulating an RL Oracle for Two-Context Regression}

\paragraph{Algorithm overview.} As shown in \cref{alg:twored}, $\TwoRed$ takes as input a reward-free episodic RL oracle for $\Phiaug$, a dataset $(x_1^{(i)},x_2^{(i)},y^{(i)})_{i=1}^n$ for two-context regression over $\Phi$, and accuracy parameters $\epsilon,\delta>0$. The goal is to produce and estimate $\MR:\MX\times\MX\to[0,1]$ of the Bayes optimal predictor $\EE[y^{(i)}\mid{}x_1^{(i)},x_2^{(i)}]$ for the dataset. To this end, the first step is to simulate the RL oracle on an MDP with horizon $H=2$, initial observation space $\MX$, final observation space $\MX\sqcup\{0\}$, and action space $\MA \subset [0,1]$. In particular, $\TwoRed$ uses a new sample $(x_1^{(i)},x_2^{(i)},y^{(i)})$ from the dataset for each episode of interaction---this is where it is crucial that the RL oracle does not have reset access. The first observation is $x_1^{(i)}$. When the oracle returns an action $a \in \MA \subset [0,1]$, the second observation is $x_2^{(i)}$ with probability $1-(a-y^{(i)})^2$ and $0$ otherwise.

\begin{algorithm}[t]
    \caption{$\RegToRL(\MO,(x_1^{(i)},x_2^{(i)},y^{(i)})_{i=1}^n,\epsilon,\delta)$: Two-context regression-to-RL reduction}
    \label{alg:regtorl}
	\begin{algorithmic}[1]\onehalfspacing
          \State \textbf{input:} Reward-free episodic RL oracle $\MO$ for $\Phi$; samples $(x_1^{(i)},x_2^{(i)},y^{(i)})_{i=1}^n$; tolerances $\epsilon,\delta$.
		\State \textbf{return:} \[\MR \gets \TwoAug(\TwoRed(\MO,\cdot),(x_1^{(i)},x_2^{(i)},y^{(i)})_{i=1}^n,\epsilon,\delta).\]
	\end{algorithmic}
\end{algorithm}

\begin{algorithm}[t]
	\caption{$\TwoRed(\MO,(x_1^{(i)},x_2^{(i)},y^{(i)})_{i=1}^n,\epsilon,\delta)$: Main subroutine in $\RegToRL$}
	\label{alg:twored}
	\begin{algorithmic}[1]\onehalfspacing
          \State \textbf{input:} Reward-free episodic RL oracle $\MO$ for $\Phiaug$; samples $(x_1^{(i)},x_2^{(i)},y^{(i)})_{i=1}^n$; tolerances $\epsilon,\delta$.
		\State Set $\epa := \frac{\epsilon}{2|\MS|}$. Initialize $\MO$ with $H = 2$ and $\MA = \{0,\epa,\dots,1-\epa\}$ and tolerances $\frac{\epsilon^2}{4|\MS|^2},\delta/2$.\label{line:initialize-oracle}
        \Repeat
            \State Simulate episode $i$ of interaction with $\MO$ as follows: pass observation $x_1^{(i)}$ and receive action $a \in \MA$. With probability $(a - y^{(i)})^2$, pass observation $0$. Otherwise, pass observation $x_2^{(i)}$.
        \Until{$\MO$ returns policy cover $\Psi$}\label{line:oracle-policy-cover}
        \For{$\pi \in \Psi$}
            \State $\MC^\pi \gets \emptyset$
            \For{$n/2 + 1 \leq i \leq n$}
                \State Sample $z^{(i)} \sim \Ber((\pi(x_1^{(i)}) - y^{(i)})^2)$.
                \State $\MC^\pi \gets \MC^\pi \cup \{(x_2^{(i)}, z^{(i)}\}$.
            \EndFor 
            \State Compute $\MR^\pi \gets \OneRed(\MO, \MC^\pi)$.\Comment{See \cref{alg:onered}}
        \EndFor
        \State \textbf{return:} predictor $\MR: \MX \times \MX \to [0,1]$ defined as $\MR(x_1,x_2) := (\argmin_{\pi \in \Psi} \MR^\pi(x_2))(x_1)$, where ties are broken in some canonical fashion.
	\end{algorithmic}
\end{algorithm}

Eventually, the oracle produces a set of policies $\Psi$. Each policy is a map $\pi: \MX\to[0,1]$. To ``stitch'' these into a single predictor on $\MX\times\MX$, $\TwoRed$ estimates an error function for each. In particular, using fresh samples, the algorithm constructs a dataset consisting of samples $(x_2^{(i)}, z^{(i)})$ where $\EE[z^{(i)}\mid{} x_2^{(i)}] = \EE[(\pi(x_1^{(i)})-y^{(i)})^2\mid{} x_2^{(i)}]$ measures the error of $\pi$ on $x_1^{(i)}$, conditional on $x_2^{(i)}$. Applying a one-context regression oracle to this dataset (via \cref{alg:onered}, the reduction from one-context regression to reward-free RL) gives an estimated error function $\MR^\pi:\MX\to[0,1]$. 

Finally, the predictor $\MR$ output by $\TwoRed$ is defined as follows. Given $(x_1,x_2) \in \MX\times\MX$, identify the policy $\pi\in\Psi$ that minimizes $\MR^\pi(x_2)$, and return $\pi(x_1)$.

\paragraph{Proof outline.} As discussed in \cref{sec:episodic-nec}, the basic idea for the analysis of $\TwoRed$ is as follows. Let $M$ denote the simulated MDP, let $\phist$ denote the decoding function for the dataset, and let $f:\MS\times\MS\to[0,1]$ denote the true latent predictor for the dataset. It can be checked that $M$ satisfies $\Phiaug$-decodability (\cref{lemma:phiaug-decodable}). For any latent state $s \in \MS$, a policy $\pi$ that approximately maximizes $d^{M,\pi}_2(s)$ must optimally ``guess'' $y^{(i)}$ conditioned on both $x_1^{(i)}$ and the event $\phist(x_2^{(i)}) = s$; thus, as shown in \cref{lemma:visitation-diffs}, such a $\pi$ must approximately \emph{minimize} the following loss:
\[L_s(\pi) := \EE_{(x_1,x_2) \sim \MD} \left[\mathbbm{1}[\phi^\st(x_2) = s] (\pi(x_1) - f(\phi^\st(x_1),\phi^\st(x_2)))^2\right].\]
By the definition of reward-free RL, it follows that for each $s \in \MS$, there exists some $\pi \in \Psi$ such that $L_s(\pi)$ is small (\cref{lemma:psi-guarantee}), i.e. $\pi$ is a good approximation of $x_1 \mapsto f(\phist(x_1),s)$. 

It remains to argue that the policy selection procedure in the definition of the final predictor $\MR$ (i.e. picking the policy $\pihat^{x_2}$ that minimizes the estimated error $\MR^\pi(x_2)$) appropriately identifies \emph{which} policy is good for a given covariate $(x_1,x_2)$. Indeed, as shown in \cref{lemma:rpi-error}, for fixed $x_2$, $\MR^\pi(x_2)$ is monotonic in $L_{\phist(x_2)}(\pi)$ as $\pi$ varies. Thus, intuitively, the selection procedure makes sense. The remaining technical subtlety is that there is an apparent distributional mismatch. We know that the following loss is small, for each $x_2$ with $\phist(x_2)=s$:
\[L_s(\pihat^{x_2}) = \EE_{(x_1',x_2') \sim \MD} \left[\mathbbm{1}[\phi^\st(x_2') = s] (\pihat^{x_2}(x_1') - f(\phi^\st(x_1'),\phi^\st(x_2')))^2\right].\]
However, to bound the squared error of $\MR$, we would like to bound the following quantity:
\[\EE_{(x_1',x_2') \sim \MD}\left[ \mathbbm{1}[\phi^\st(x_2') = s] (\pihat^{x_2'}(x_1') - f(\phi^\st(x_1'),\phi^\st(x_2')))^2\right],\]
i.e. where context $x_2'$ in the squared error term is the same as the context used to select the policy $\pihat^{x_2'}$. However, it turns out that these two quantities can be related in expectation over $x_2$, using $\phist$-realizability of the context distribution (\cref{lemma:switching}).

We formally analyze $\TwoRed$ in \cref{sec:twored}; see \cref{thm:two-context-reduction} for the formal guarantee.

\subsubsection{$\TwoAug$: Augmenting the regression}

\paragraph{Algorithm overview.} As shown in \cref{alg:twoaug}, $\TwoAug$ takes as input a two-context regression oracle for $\Phi$, samples $(x_1^{(i)},x_2^{(i)},y^{(i)})_{i=1}^n$, and tolerance parameters $\epsilon,\delta\in(0,1/2).$ The goal is to solve two-context regression for $\Phiaug$. The basic idea is that since the extra states $\{0,1\}$ are fully observed, regression onto these states is actually simpler than regression onto the hidden states. In more detail, consider all indices $i$ for which $x_1^{(i)} = 0$ but $x_2^{(i)} \in \MX$. We can define a reduced dataset consisting of the samples $(\xbar, x_2^{(i)})$ for each such $i$, where $\xbar\in\MX$ is an arbitrary fixed observation. It's straightforward to check that this dataset is a valid two-context regression dataset for $\Phi$, so the oracle yields a good predictor for the corresponding conditional distribution (as long as there is enough data). We can do a similar argument for all other cases (e.g. $x_1^{(i)} \in \MX$ and $x_2^{(i)} = 1$, etc.), and it is straightforward to stitch together the predictors on subsets of $\Xaug\times\Xaug$ into a single predictor for the whole covariate space.

We defer the formal analysis of $\TwoAug$ to \sssref{sec:twoaug}; see \cref{prop:twoaug} for the formal guarantee.

\subsection{Analysis of $\TwoRed$ (\creftitle{alg:twored})}\label{sec:twored}

The following theorem states our main guarantee for $\TwoRed$. As discussed above, $\TwoRed$ uses the given dataset to simulate the RL oracle on a horizon-$2$ block MDP for which exploring a latent state $s$ with near-maximal probability corresponds to learning the regression function on a subset of the covariate distribution determined by $s$. Once the oracle produces a policy cover $\Psi$, $\TwoRed$ performs one-context regression to learn a loss function $\MR^\pi$ associated with each policy $\pi \in \Psi$, and then outputs a regressor obtained by stitching together the policies (according to whichever has the best loss on the given query).

\begin{theorem}\label{thm:two-context-reduction}
Suppose that $\MO$ is an $(\Nrl,\Krl)$-efficient reward-free episodic RL oracle (\cref{def:strong-rf-rl}) for $\Phiaug$. Then $\TwoRed(\MO,\cdot)$ is a $\Nreg$-efficient two-context regression algorithm (\cref{def:two-con-regression}) for $\Phi$ with
\[\Nreg(\epsilon,\delta) := 2\Nrl\left(\frac{\epsilon^4}{16|\MS|^4},\frac{\delta}{4K'},2,\frac{4|\MS|^2}{\epsilon^2}\right) + 64|\MS|^8\epsilon^{-8}\log\left(4K'\Krl\left(\frac{\epsilon^4}{16|\MS|^4},\frac{\delta}{4K'},2,\frac{4|\MS|^2}{\epsilon^2}\right)/\delta\right) \]
where $K' = \Krl(\frac{\epsilon^2}{4|\MS|^2},\frac{\delta}{2},2,\frac{2|\MS|}{\epsilon})$.
\end{theorem}

To prove \cref{thm:two-context-reduction}, we fix $\phi^\st \in \Phi$, a $\phi^\st$-realizable distribution $\MD \in \Delta(\MX\times\MX)$, and a function $f:\MS \times \MS \to [0,1]$. For some $n \geq \Nreg(\epsilon,\delta)$, we let $(x_1^{(i)}, x_2^{(i)},y^{(i)})_{i=1}^n$ be i.i.d. samples with $(x_1^{(i)},x_2^{(i)}) \sim \MD$ and $y^{(i)} \sim \Ber(f(\phi^\st(x_1^{(i)}),\phi^\st(x_2^{(i)})))$. In the remainder of the section, we analyze the execution of $\TwoRed(\MO,(x_1^{(i)}, x_2^{(i)},y^{(i)})_{i=1}^n,\epsilon,\delta)$. Ultimately, we must show that with probability at least $1-\delta$, the circuit $\MR$ produced by $\TwoRed$ satisfies 
\begin{equation} \EE_{(x_1,x_2) \sim \MD} \left(\MR(x_1,x_2) - f(\phi^\st(x_1),\phi^\st(x_2))\right)^2 \leq \epsilon.
\label{eq:twored-error-guarantee}
\end{equation}
To begin, in \cref{lemma:phiaug-decodable,lemma:simulate-trajectory} we show that $\TwoRed$ is invoking the RL oracle $\MO$ by simulating episodic access with the $\Phiaug$-decodable block MDP defined below:

\begin{definition}[Block MDP gadget for reduction]
Fix $\epa \in (0,1)$. We define a block MDP $M$ by
\[M = M_{\phi^\st,f,\MD,\epa} := (H, \Saug,\Xaug, \MA, (\til \BP_h)_{h\in [2]},(\til \BO)_{h \in [2]}, \aug(\phi^\st))\]
where $H := 2$, the action space is $\MA := \{0,\epa,\dots,\epa\lfloor1/\epa\rfloor\}$, and $\til \BP_1 \in \Delta(\Saug)$ is the marginal distribution of $\phi^\st(x_1)$ for $(x_1,x_2) \sim \MD$. The transition distribution $\til \BP_2$ is defined as follows. For initial state $s_1 \in \MS$, next state $s_2 \in \Saug$, and action $a \in \MA$,
\[\til\BP_2(s_2\mid{}s_1,a) := 
\begin{cases}
\Prr\limits_{(x_1,x_2) \sim \MD}[\phi^\st(x_2)=s_2\mid{}\phi^\st(x_1)=s_1] \EE\limits_{y \sim \Ber(f(s_1,s_2))}\left[1 - (a-y)^2\right] & \text { if } s_2 \in \MS \\ 
\sum\limits_{s_2' \in \MS} \Prr\limits_{(x_1,x_2) \sim \MD}[\phi^\st(x_2)=s_2'\mid{}\phi^\st(x_1)=s_1] \EE\limits_{y \sim \Ber(f(s_1,s_2'))}\left[(a-y)^2\right] & \text { if } s_2 = 0 \\ 
0 & \text { if } s_2 = 1
\end{cases}.\]
Note that $\til\BP_1$ is supported on $\MS$, so it is not necessary to define $\til \BP_2(s_2\mid{}s_1,a)$ for $s_1 \in \{0,1\}$. Finally, for $s \in \MS$ and $h \in [2]$, the observation distribution $\til \BO_h(\cdot|s)$ is defined as the conditional distribution of $x_h\mid{}\phi^\st(x_h)=s$ for $(x_1,x_2) \sim \MD$. The distributions $\til\BO_h(\cdot\mid{}0)$ are fully supported on $0$, and the distributions $\til\BO_h(\cdot\mid{}1)$ are fully supported on $1$. 
\end{definition}

In order to invoke the guarantees of the RL oracle, we need to verify that $M$ satisfies $\Phiaug$-decodability (\cref{lemma:phiaug-decodable}) and is in fact the MDP that is being simulated in $\TwoRed$ (\cref{lemma:simulate-trajectory}). 

\begin{lemma}\label{lemma:phiaug-decodable}
$M$ is a $\Phiaug$-decodable block MDP.
\end{lemma}

\begin{proof}
First we observe that $\til \BP_1, \til \BP_2$ are well-defined: by construction $\til \BP_1$ is a distribution. Moreover, for any $s_1 \in \MS$ and $a \in \MA$, it is clear that $\til \BP_2(\cdot|s_1,a)$ is non-negative and
\[\sum_{s_2 \in \MS} \til\BP_2(s_2\mid{}s_1,a) = \EE\limits_{\substack{(x_1,x_2) \sim \MD \\ y \sim \Ber(f(\phi^\st(x_1),\phi^\st(x_2)))}}\left[\left(1 - (a-y)^2\right)\middle| \phi^\st(x_1)=s_1\right] = 1 - \til\BP_2(0\mid{}s_1,a).\]
Thus, $\til \BP_2(\cdot\mid{}s_1,a)$ is a distribution. Finally, we observe that for any $h \in [2]$ and $s \in \MS$, the observation distribution $\til\BO_h(\cdot\mid{}s)$ is fully supported on $x_h \in \MX$ such that $\aug(\phi^\st)(x_h) = \phi^\st(x_h) = s$, by construction. Since $\aug(\phi^\st)(0) = 0$ and $\aug(\phi^\st)(1) = 1$, the same holds for $s \in \{0,1\}$.
\end{proof}

\begin{lemma}\label{lemma:simulate-trajectory}
Let $(x_1,x_2) \sim \MD$ and $y \sim \Ber(f(\phi^\st(x_1),\phi^\st(x_2)))$. The following process simulates an episode of interaction with $M$:
\begin{enumerate}
\item Pass observation $x_1$ and receive action $a \in \MA$.
\item With probability $(a-y)^2$, pass observation $0$. Otherwise, pass observation $x_2$.
\end{enumerate}
\end{lemma}

\begin{proof}
Let $p(x_1,x_2)$ denote the density of $\MD$, and observe that by $\phi^\st$-realizability \cref{def:realizable-distribution}) we can write
\[p(x_1,x_2) = \til p(\phi^\st(x_1),\phi^\st(x_2)) \til q_1(x_1\mid{}\phi^\st(x_1)) \til q_2(x_2\mid{}\phi^\st(x_2))\]
for some density $\til p$ and conditional densities $\til q_1,\til q_2$. Also let $\til p(s_2\mid{}s_1)$ denote the conditional density of $\phi^\st(x_2)\mid{}\phi^\st(x_1)$ under $(x_1,x_2) \sim \MD$. Now fix any policy $\pi:\MX \to \Delta(\MA)$ and trajectory $(s_1,x_1,a,s_2,x_2)$ where $x_2 \neq 0$. The likelihood of this trajectory under the described process is
\begin{align*}
&p(x_1,x_2) \mathbbm{1}[s_1=\phi^\st(x_1)]\mathbbm{1}[s_2=\phi^\st(x_2)] \pi(a\mid{}x_1) \EE_{y \sim \Ber(f(s_1,s_2))}[1 - (a - y)^2] \\ 
&= \til p(s_1,s_2) \til q_1(x_1\mid{}s_1) \til q_2(x_2\mid{}s_2) \pi(a\mid{}x_1)  \EE_{y \sim \Ber(f(s_1,s_2))}[1 - (a - y)^2] \\ 
&= \til\BP_1(s_1) \til p(s_2\mid{}s_1) \til\BO_1(x_1\mid{}s_1) \til\BO_2(x_2\mid{}s_2)\pi(a\mid{}x_1) \EE_{y \sim \Ber(f(s_1,s_2))}[1 - (a - y)^2] \\ 
&= \til\BP_1(s_1) \til\BO_1(x_1\mid{}s_1)\til\BO_2(x_2\mid{}s_2)\pi(a\mid{}x_1) \til\BP_2(s_2\mid{}s_1,a)
\end{align*}
which is precisely the likelihood of the trajectory under $M$. Similarly, the likelihood of any trajectory $(s_1,x_1,a,0,0)$ under the described process is
\begin{align*}
&\sum_{s_2 \in \MS} \til p(s_1,s_2) \til q_1(x_1\mid{}s_1) \pi(a\mid{}x_1) \EE_{y \sim \Ber(f(s_1,s_2))}[(a-y)^2] \\ 
&= \til\BP_1(s_1)\til\BO_1(x_1\mid{}s_1) \pi(a\mid{}x_1) \sum_{s_2 \in \MS} \til p(s_2\mid{}s_1) \EE_{y \sim \Ber(f(s_1,s_2))}[(a-y)^2] \\ 
&= \til\BP_1(s_1)\til\BO_1(x_1\mid{}s_1) \pi(a\mid{}x_1)\til \BP_2(s_2\mid{}s_1,a)
\end{align*}
as needed.
\end{proof}

Next, in \cref{lemma:psi-guarantee}, we show that for every latent state $s \in \MS$, the set of policies $\Psi$ computed by the RL oracle contains some approximate minimizer of the loss $L_s(\pi)$ defined below. To this end, the key lemma is \cref{lemma:visitation-diffs}, which characterizes the loss of a policy in terms of its visitation probability for state $s$.

\begin{definition}
For any policy $\pi \in \Pi$ and state $s \in \MS$, define
\[L_s(\pi) := \EE_{(x_1,x_2) \sim \MD} \left[\mathbbm{1}[\phi^\st(x_2) = s] (\pi(x_1) - f(\phi^\st(x_1),\phi^\st(x_2)))^2\right],\]
and for each $s \in \MS$ define
\[Z_s := \EE_{\substack{(x_1,x_2) \sim \MD \\ y \sim \Ber(f(\phi^\st(x_1),\phi^\st(x_2)))}}\left[ \mathbbm{1}[\phi^\st(x_2) = s] \left(y - f(\phi^\st(x_1),\phi^\st(x_2))^2\right)\right].\]
\end{definition}

\begin{lemma}\label{lemma:visitation-diffs}
For any policies $\pi,\pi' \in \Pi$ and state $s \in \MS$, we have $d_2^{M,\pi}(s) - d_2^{M,\pi'}(s) = L_s(\pi') - L_s(\pi)$.
\end{lemma}

\begin{proof}
By the characterization provided by \cref{lemma:simulate-trajectory}, for any policy $\pi$, the visitation distribution for a state $s_2 \in \MS$ at step $2$ is
\begin{align*}
d^{M,\pi}_2(s_2)
&= \EE_{\substack{(x_1,x_2) \sim \MD \\ y \sim \Ber(\phi^\st(x_1),\phi^\st(x_2))}} \left[\mathbbm{1}[\phi^\st(x_2)=s_2] \cdot (1 - (\pi(x_1) - y)^2)\right] \\ 
&= \EE_{\substack{(x_1,x_2) \sim \MD \\ y \sim \Ber(\phi^\st(x_1),\phi^\st(x_2))}}\left[\mathbbm{1}[\phi^\st(x_2)=s_2] \cdot \left(1 - (\pi(x_1) - f(\phi^\st(x_1),\phi^\st(x_2)) + f(\phi^\st(x_1),\phi^\st(x_2)) - y)^2\right)\right] \\ 
&= \Prr_{(x_1,x_2)\sim\MD}[\phi^\st(x_2)=s_2] - L_s(\pi) - \EE_{\substack{(x_1,x_2) \sim \MD \\ y \sim \Ber(\phi^\st(x_1),\phi^\st(x_2))}}\left[\mathbbm{1}[\phi^\st(x_2)=s_2](f(\phi^\st(x_1),\phi^\st(x_2)) - y)^2\right] \\ 
&\qquad - 2 \EE_{\substack{(x_1,x_2) \sim \MD \\ y \sim \Ber(\phi^\st(x_1),\phi^\st(x_2))}}\left[\mathbbm{1}[\phi^\st(x_2)=s_2] (\pi(x_1) - f(\phi^\st(x_1),\phi^\st(x_2)))( f(\phi^\st(x_1),\phi^\st(x_2)) - y)\right].
\end{align*}
The final term is $0$ since $\EE[y\mid{}x_1,x_2] = f(\phi^\st(x_1),\phi^\st(x_2))$. The remaining terms are all independent of $\pi$ except for $-L_s(\pi)$; hence, 
\[d^{M,\pi}_2(s_2) - d^{M,\pi'}_2(s_2) = L_s(\pi') - L_s(\pi)\]
as claimed.
\end{proof}

\begin{lemma}\label{lemma:psi-guarantee}
Suppose that $\MO$ is an $(\Nrl,\Krl)$-efficient reward-free RL oracle (\cref{def:strong-rf-rl}) for $\Phiaug$, and that $n/2 \geq \Nrl(\epsilon^2/(4|\MS|^2),\delta/2,2,2|\MS|/\epsilon)$. Then the set $\Psi$ computed in \lineref{line:oracle-policy-cover} of \cref{alg:twored} satisfies $|\Psi| \leq \Krl(\epsilon^2/(4|\MS|^2),\delta/2,2,2|\MS|/\epsilon)$. Moreover, it holds with probability at least $1-\delta/2$ that for every $s \in \MS$, there is some $\pi \in \Psi$ such that 
$L_s(\pi) \leq \epsilon^2/(2|\MS|^2).$
\end{lemma}

\begin{proof}
By \cref{lemma:simulate-trajectory}, $\MO$ is given interactive access to the MDP $M$. The claimed bound on $|\Psi|$ is immediate from \cref{def:strong-rf-rl} together with the initialization of $\MO$ (\lineref{line:initialize-oracle}) and the fact that $H=2$ and $|\MA| = 2|\MS|/\epsilon$. 

Next, by \cref{def:strong-rf-rl}, it holds with probability at least $1-\delta/2$ that for any $s \in \Saug$ and $h \in [2]$,
\begin{equation} \max_{\pi\in\Psi} d^{M,\pi}_h(s) \geq \max_{\pi \in \Pi} d^{M,\pi}_h(s) - \frac{\epsilon^2}{4|\MS|^2}.\label{eq:pc-applied}
\end{equation}
Condition on this event. Fix any $s \in \MS$ and define $\pi^\st:\Xaug \to \MA$ by $\pi^\st(x_1) := \epa\lfloor f(\phi^\st(x_1),s)/\epa\rfloor \in \MA$ for $x_1 \in \MX$ (define $\pi^\st(0),\pi^\st(1)$ arbitrarily). Note that $L_s(\pi^\st) \leq \epa^2$. By \eqref{eq:pc-applied}, there is some $\pi \in \Psi$ such that
\begin{align*}
\frac{\epsilon^2}{4|\MS|^2}
&\geq d^{M,\pi^\st}_2(s) - d^{M,\pi}_2(s) \\ 
&= L_s(\pi) - L_s(\pi^\st) \\ 
&\geq L_s(\pi) - \epa^2
\end{align*}
where the equality is by \cref{lemma:visitation-diffs}. The lemma follows from the definition of $\epa$.
\end{proof}

The following lemma shows that with high probability, each error function $\MR^\pi$ approximates an affine transformation of the loss $L_{\phist(x_2)}(\pi)$.

\begin{lemma}\label{lemma:rpi-error}
Suppose that $\MO$ is an $(\Nrl,\Krl)$-efficient reward-free RL oracle for $\Phiaug$ (\cref{def:strong-rf-rl}). Set $N = \Nrl(\epsilon^4/(16|\MS|^4), \delta/(4|\Psi|),2,4|\MS|^2/\epsilon^2)$ and $K = \Krl(\epsilon^4/(16|\MS|^4), \delta/(4|\Psi|),2,4|\MS|^2/\epsilon^2)$. If $n/2 \geq N + 64|\MS|^8\epsilon^{-8}\log(4|\Psi|K/\delta)$ then, with probability at least $1-\delta/2$, it holds for all $\pi \in \Psi$ that
\[\EE_{(x_1,x_2)\sim\MD} \left(\MR^\pi(x_2) - \frac{L_{\phi^\st(x_2)}(\pi) + Z_{\phi^\st(x_2)}}{\Prr_{(x_1',x_2') \sim \MD}[\phi^\st(x_2')=\phi^\st(x_2)]}\right)^2 \leq \frac{\epsilon^4}{4|\MS|^4}.\]
\end{lemma}

\begin{proof}
By \cref{prop:onered} and the assumption on $\MO$, we get that $\OneRed(\MO,\cdot)$ is an $\Nreg$-efficient one-context regression oracle for $\Phi$ (\cref{def:one-con-regression}) with 
\[\Nreg\left(\frac{\epsilon^4}{4|\MS|^4}, \frac{\delta}{2|\Psi|}\right) \leq N + \frac{64|\MS|^8 \log(4|\Psi| K/\delta)}{\epsilon^8} \leq n/2.\]
Fix any $\pi \in \Psi$. Observe that $\MC^\pi$ consists of $n/2$ i.i.d. samples $(x_2^{(i)}, z^{(i)})$ with $x_2^{(i)} \in \MX$ and $z^{(i)} \in \{0,1\}$. For any index $i$, by $\phi^\st$-realizability of $\MD$ and the law of $y^{(i)}$, we have that $x_2^{(i)} \perp (\pi(x_1^{(i)})-y^{(i)})^2 \mid{} \phi^\st(x_2^{(i)})$ and hence $x_2^{(i)} \perp z^{(i)} \mid{} \phi^\st(x_2^{(i)})$. Moreover, for any $x_2 \in \MX$ with $s := \phi^\st(x_2)$, we have
\begin{align*}
\EE[z^{(i)}\mid{}x_2^{(i)}=x_2] 
&= \EE[z^{(i)}\mid{}\phi^\st(x_2^{(i)}) = s] \\
&= \EE_{\substack{(x_1',x_2') \sim \MD \\ y \sim \Ber(f(\phi^\st(x_1'),\phi^\st(x_2')))}}[(\pi(x_1') - y')^2\mid{}\phi^\st(x_2')=s] \\ 
&= \EE_{(x_1',x_2') \sim \MD}[(\pi(x_1') - f(\phi^\st(x_1'),\phi^\st(x_2')))^2\mid{}\phi^\st(x_2')=s] \\ 
&\qquad+ \EE_{\substack{(x_1',x_2') \sim \MD \\ y \sim \Ber(f(\phi^\st(x_1'),\phi^\st(x_2')))}}[(f(\phi^\st(x_1'),\phi^\st(x_2')) - y)^2\mid{}\phi^\st(x_2')=s] \\ 
&= \frac{L_s(\pi) + Z_s}{\Pr_{(x_1',x_2') \sim \MD}[\phi^\st(x_2') = s]}
\end{align*}
where the penultimate equality uses that $\EE[y\mid{}x_1',x_2'] = f(\phi^\st(x_1'),\phi^\st(x_2'))$. The result now follows from \cref{def:one-con-regression} and a union bound over $\pi \in \Psi$.
\end{proof}

The following technical lemma is needed to handle the distribution mismatch discussed in \cref{sec:regtorl-overview}. Essentially, it asserts that if we sample $(x_1,x_2)$ and $(x_1',x_2')$ independently and then condition on $x_2, x_2'$ having the same latent state $s$, then $x_1$ and $x_1'$ are exchangeable, i.e. $(x_1,x_2)$ and $(x_1',x_2)$ have the same distribution.

\begin{lemma}\label{lemma:switching}
For any $s \in \MS$ and $g: \MX\times \MX \to \RR$, it holds that
\[\EE_{\substack{(x_1,x_2) \sim \MD \\ (x_1',x_2') \sim \MD}}[g(x_1,x_2)\mid{}\phi^\st(x_2)=\phi^\st(x_2')=s] = \EE_{\substack{(x_1,x_2) \sim \MD \\ (x_1',x_2') \sim \MD}}[g(x_1',x_2)\mid{}\phi^\st(x_2)=\phi^\st(x_2')=s]\]
where the random variables $(x_1,x_2)$ and $(x_1',x_2')$ in the expectation are independent draws from $\MD$.
\end{lemma}

\begin{proof}
By $\phi^\st$-realizability of $\MD$, note that $x_1,x_2$ are conditionally independent under the event $\phi^\st(x_2)=\phi^\st(x_2')=s$. Additionally, for any events $\ME,\ME' \subseteq \MX$, 
\begin{align*}
&\Pr[x_2 \in \ME \land x_1' \in \ME' \land \phi^\st(x_2) = s \land \phi^\st(x_2') = s] \\
&= \Pr[x_1 \in \ME \land \phi^\st(x_2) = s] \Pr[x_2' \in \ME' \land \phi^\st(x_2') = s] \\ 
&= \frac{\Pr[x_2 \in \ME \land \phi^\st(x_2) = s \land \phi^\st(x_2') = s] \Pr[x_1' \in \ME' \land \phi^\st(x_2) = s \land \phi^\st(x_2') = s]}{\Pr[\phi^\st(x_2) = s \land \phi^\st(x_2') = s]},
\end{align*}
so 
\begin{align*}
&\Pr[x_2 \in \ME \land x_1' \in \ME' \mid{} \phi^\st(x_2) = s \land \phi^\st(x_2') = s] \\ 
&= \Pr[x_2 \in \ME \mid{} \phi^\st(x_2) = s \land \phi^\st(x_2') = s]\Pr[x_1' \in \ME' \mid{} \phi^\st(x_2) = s \land \phi^\st(x_2') = s].
\end{align*}
Thus, $x_1',x_2$ are conditionally independent as well. But by symmetry, $x_1\mid{}\phi^\st(x_2)=\phi^\st(x_2') = s$ and $x_1'\mid{}\phi^\st(x_2)=\phi^\st(x_2') = s$ have identical distributions. Thus, $(x_1,x_2)\mid{}\phi^\st(x_2)=\phi^\st(x_2') = s$ and $(x_1',x_2)\mid{}\phi^\st(x_2)=\phi^\st(x_2') = s$ have identical distributions, which implies the claimed equality.
\end{proof}

With the above ingredients, we can now formally conclude the proof of \cref{thm:two-context-reduction}.

\vspace{1em}

\begin{proof}[Proof of \cref{thm:two-context-reduction}]
Consider the event that the claimed bounds of \cref{lemma:psi-guarantee} and \cref{lemma:rpi-error} both hold, which occurs with probability at least $1-\delta$. We argue that in this event, the desired error bound \eqref{eq:twored-error-guarantee} holds.

For each $s \in \MS$ fix any $\pi^s \in \argmax_{\pi \in \Psi} d^{M,\pi}_2(s)$ (i.e. $\pi^s$ is optimal for reaching state $s$, among policies in $\Psi$). For each $x_2 \in \MX$, let $\pihat^{x_2} \in \Psi$ be the policy minimizing $\MR^\pi(x_2)$ (breaking ties as in \cref{alg:twored}). We have by definition of $\MR$ that
\begin{align*}
&\EE_{(x_1,x_2) \sim \MD} \left[(\MR(x_1,x_2) - f(\phi^\st(x_1),\phi^\st(x_2))^2 \right]\\
&= \EE_{(x_1,x_2) \sim \MD} \left[(\pihat^{x_2}(x_1) - f(\phi^\st(x_1),\phi^\st(x_2))^2 \right]\\
&= \sum_{s \in \MS} \EE_{(x_1,x_2) \sim \MD} \left[\mathbbm{1}[\phi^\st(x_2) = s] \cdot (\pihat^{x_2}(x_1) - f(\phi^\st(x_1),\phi^\st(x_2))^2\right].
\end{align*}
Fix any $s \in \MS$. On the one hand, observe that
\begin{equation}
\EE_{(x_1,x_2) \sim \MD} \left[\mathbbm{1}[\phi^\st(x_2) = s]  (\pihat^{x_2}(x_1) - f(\phi^\st(x_1),\phi^\st(x_2)))^2\right]
\leq \Prr_{(x_1,x_2) \sim \MD}[\phi^\st(x_2) = s]
\label{eq:err-bound-1}
\end{equation}
since $\pihat^{x_2}(x_1) \in \MA \subset [0,1]$ and $f(\phi^\st(x_1),\phi^\st(x_2)) \in [0,1]$. On the other hand,
\begin{align}
&\Prr_{(x_1,x_2) \sim \MD}[\phi^\st(x_2) = s] \cdot \EE_{(x_1,x_2) \sim \MD}\left[ \mathbbm{1}[\phi^\st(x_2) = s] (\pihat^{x_2}(x_1) - f(\phi^\st(x_1),\phi^\st(x_2)))^2\right] \nonumber\\ 
&= \EE_{\substack{(x_1,x_2) \sim \MD \\ (x_1',x_2') \sim \MD}}\left[ \mathbbm{1}[\phi^\st(x_2)=s]\mathbbm{1}[\phi^\st(x_2') = s]  (\pihat^{x_2}(x_1) - f(\phi^\st(x_1),\phi^\st(x_2)))^2 \right]\nonumber \\ 
&= \EE_{\substack{(x_1,x_2) \sim \MD \\ (x_1',x_2') \sim \MD}} \left[\mathbbm{1}[\phi^\st(x_2)=s]\mathbbm{1}[\phi^\st(x_2') = s]  (\pihat^{x_2}(x_1') - f(\phi^\st(x_1'),\phi^\st(x_2)))^2\right] \nonumber \\ 
&= \EE_{(x_1,x_2) \sim \MD}\left[\mathbbm{1}[\phi^\st(x_2)=s] \cdot \EE_{(x_1',x_2') \sim \MD}\left[\mathbbm{1}[\phi^\st(x_2') = s]  (\pihat^{x_2}(x_1') - f(\phi^\st(x_1'),\phi^\st(x_2)))^2\right]\right] \nonumber \\
&= \EE_{(x_1,x_2) \sim \MD}\left[\mathbbm{1}[\phi^\st(x_2)=s] \cdot \EE_{(x_1',x_2') \sim \MD}\left[\mathbbm{1}[\phi^\st(x_2') = s]  (\pihat^{x_2}(x_1') - f(\phi^\st(x_1'),\phi^\st(x_2')))^2\right]\right] \nonumber \\
&= \EE_{(x_1,x_2) \sim \MD}\left[ \mathbbm{1}[\phi^\st(x_2)=s]\cdot L_s(\pihat^{x_2}) \right]\nonumber \\ 
&= \EE_{(x_1,x_2) \sim \MD} \left[\mathbbm{1}[\phi^\st(x_2)=s] \cdot \left(L_s(\pihat^{x_2}) - L_s(\pi^s)\right)\right] + \Prr_{(x_1,x_2) \sim \MD}[\phi^\st(x_2) = s]\cdot L_s(\pi^s)
\label{eq:err-bound-2}
\end{align}
where the second equality is by \cref{lemma:switching} with function $g(x_1,x_2) = (\pihat^{x_2}(x_1) - f(\phi^\st(x_1),\phi^\st(x_2)))^2$; the fourth equality is because $\phist(x_2) = \phist(x_2') = s$ (unless one of the indicator functions is zero); and the fifth inequality is by definition of $L_s$. Combining \cref{eq:err-bound-1,eq:err-bound-2}, we get that
\begin{align}
&\EE_{(x_1,x_2) \sim \MD}\left[ \mathbbm{1}[\phi^\st(x_2) = s] (\pihat^{x_2}(x_1) - f(\phi^\st(x_1),\phi^\st(x_2)))^2 \right]\nonumber \\ 
&\leq \sqrt{\EE_{(x_1,x_2) \sim \MD} \left[\mathbbm{1}[\phi^\st(x_2)=s] \cdot \left(L_s(\pihat^{x_2}) - L_s(\pi^s)\right) \right] + \Prr_{(x_1,x_2) \sim \MD}[\phi^\st(x_2) = s]\cdot L_s(\pi^s)}.
\label{eq:err-bound}
\end{align}
By the guarantee of \cref{lemma:psi-guarantee}, we have $L_s(\pi^s) \leq \epsilon^2/(2|\MS|^2)$. Also,
\begin{align*}
&\EE_{(x_1,x_2) \sim \MD} \left[\mathbbm{1}[\phi^\st(x_2)=s] (L_s(\pihat^{x_2}) - L_s(\pi^s))\right] \\ 
&\leq \EE_{(x_1,x_2) \sim \MD} \Big[\mathbbm{1}[\phi^\st(x_2)=s] \Big(L_s(\pihat^{x_2}) - \MR^{\pihat^{x_2}}(x_2)\Prr_{(x_1',x_2')\sim\MD}[\phi^\st(x_2')=s] \\
&\hspace{12em}+ \MR^{\pi^s}(x_2)\Prr_{(x_1',x_2')\sim\MD}[\phi^\st(x_2')=s] - L_s(\pi^s)\Big)\Big] \\ 
&\leq 2\max_{\pi\in\Psi} \EE_{x_1,x_2\sim\MD} \mathbbm{1}[\phi^\st(x_2)=s] \left| L_s(\pi)+Z_s - \MR^\pi(x_2)\Prr_{(x_1',x_2')\sim\MD}[\phi^\st(x_2')=s]\right| \\ 
&\leq 2\max_{\pi\in\Psi} \EE_{x_1,x_2 \sim \MD} \mathbbm{1}[\phi^\st(x_2)=s] \left| \frac{L_s(\pi) + Z_s}{\Prr_{(x_1',x_2')\sim\MD}[\phi^\st(x_2')=s]} - \MR^\pi(x_2)\right| \\ 
&\leq 2\max_{\pi\in\Psi} \sqrt{\EE_{(x_1,x_2) \sim \MD} \mathbbm{1}[\phi^\st(x_2)=s] \left( \frac{L_s(\pi) + Z_s}{\Prr_{(x_1',x_2')\sim\MD}[\phi^\st(x_2')=s]} - \MR^\pi(x_2)\right)^2}
\leq \frac{\epsilon^2}{2|\MS|^2}
\end{align*}
where the first inequality is by minimality of $\MR^{\pihat^{x_2}}(x_2)$ over all $\pi \in \Psi$, and the final inequality is by the guarantee of \cref{lemma:rpi-error}. 
Substituting into \cref{eq:err-bound} and summing over $s \in \MS$, we get 
\[\EE_{(x_1,x_2)\sim\MD} (\MR(x_1,x_2) - f(\phi^\st(x_1),\phi^\st(x_2)))^2 \leq \epsilon\]
as needed.
\end{proof}

\subsection{Analysis of $\RegToRL$ (\creftitle{alg:regtorl})}\label{sec:regtorl}

The proof of \cref{cor:regression-to-online-rl} is now straightforward from the analysis of $\TwoRed$ (\cref{thm:two-context-reduction}) and the analysis of $\TwoAug$ (\cref{prop:twoaug}).

\vspace{1em}

\begin{proof}[Proof of \cref{cor:regression-to-online-rl}]
Fix a regular concept class $\Phiaug \subseteq (\Xaug\to\Saug)$. By regularity (\cref{def:regular}), $\{0,1\} \subseteq \Xaug,\Saug$ so we can define $\MX := \Xaug \setminus \{0,1\}$ and $\MS := \Saug \setminus \{0,1\}$. Define a concept class $\Phi \subseteq (\MX \to \MS)$ by restricting each $\phi$ to domain $\MX$; regularity ensures that the range of each restricted map is contained in $\MS$, so this definition is well-defined. We now observe that the augmented concept class (\cref{def:phiaug}) with base class $\Phi$ is precisely $\Phiaug$.

Now suppose that $\MO$ is an $(\Nrl,\Krl)$-efficient reward-free episodic RL oracle for $\Phiaug$ with $\Nrl,\Krl$ bounded in terms of the parameters $\Nrlc,\Crl$ as specified in the theorem statement. By \cref{thm:two-context-reduction} and the assumed parametric bounds on $\Nrl,\Krl$, there is a constant $C>0$ so that $\TwoRed(\MO,\cdot)$ is an $\Nreg$-efficient two-context regression algorithm for $\Phi$ with
\[\Nreg(\epsilon,\delta) \leq \Nrlc\left(\frac{|\MS|}{\epsilon\delta}\right)^{C \cdot \Crl}.\]
It follows from \cref{prop:twoaug} that $\TwoAug(\TwoRed(\MO,\cdot),\cdot)$ is an $\Nreg'$-efficient two-context regression algorithm for $\Phiaug$ with 
\[\Nreg'(\epsilon,\delta) \leq \Nrlc \left(\frac{|\MS|}{\epsilon\delta}\right)^{C_{\ref{cor:regression-to-online-rl}} \cdot \Crl}\]
so long as $C_{\ref{cor:regression-to-online-rl}}>0$ is a sufficiently large constant. Note that we are using the fact that $\epsilon,\delta \in (0,1/2)$ to absorb constant factors.

To conclude the proof, we observe that the claimed oracle time complexity bound is immediate from the pseudocode of $\RegToRL$ and its subroutines.
\end{proof}

\newpage

\section{Proof of \creftitle{cor:reset-rl-to-regression}}\label{sec:app_resets}


In this section we prove that for any concept class $\Phi$, there is a reduction from reward-free RL (\cref{def:strong-rf-rl}) in the reset access model to one-context regression (\cref{def:one-con-regression}). The formal statement is provided below.

\nc{\phisi}{\phi^\st(x^{(i)}_h)}
\nc{\phisj}{\phi^\st(x^{(j)}_h)}

\begin{theorem}[General version of \cref{cor:reset-rl-to-regression}]\label{thm:pcr-app}
There is a constant $C_{\ref{thm:pcr-app}}>0$ and an algorithm $\PCR$ (\cref{alg:pcr}) so that the following holds. Let $\Phi \subseteq (\MX\to\MS)$ be any concept class, and let $\Reg$ be a $\Nreg$-efficient one-context regression oracle for $\Phi$. Then $\PCR(\Reg,\Nreg,|\MS|,\cdot)$ 
is an $(\Nrl,\Krl)$-efficient reward-free RL algorithm for $\Phi$ in the reset access model, with:
\begin{itemize}
    \item $\Krl(\epsilon,\delta,H,|\MA|) \leq H^2|\MS|^2$
    \item $\Nrl(\epsilon,\delta,H,|\MA|) \leq \left(\frac{H|\MA||\MS|}{\epsilon\delta}\right)^{C_{\ref{thm:pcr-app}}} \Nreg\left(\left(\frac{\epsilon\delta}{H|\MA||\MS|}\right)^{C_{\ref{thm:pcr-app}}},\left(\frac{\epsilon\delta}{H|\MA||\MS|}\right)^{C_{\ref{thm:pcr-app}}}\right)$.
\end{itemize}
Moreover, the oracle time complexity of $\PCR$ is at most $\left(\frac{H|\MA||\MS|}{\epsilon\delta}\right)^{C_{\ref{thm:pcr-app}}} \Nreg\left(\left(\frac{\epsilon\delta}{H|\MA||\MS|}\right)^{C_{\ref{thm:pcr-app}}},\left(\frac{\epsilon\delta}{H|\MA||\MS|}\right)^{C_{\ref{thm:pcr-app}}}\right).$
\end{theorem}

In particular, \cref{cor:reset-rl-to-regression} follows from \cref{thm:pcr-app} by substituting $\Nreg(\epsilon,\delta) := \Nregc/(\epsilon\delta)^{\Creg}$ into the above bounds. Henceforth, fix a concept class $\Phi$, a $\Nreg$-efficient one-context regression oracle $\Reg$, and a $\Phi$-decodable block MDP $M$ with horizon $H$, action set $\MA$, and unknown decoding function $\phi^\st \in \Phi$. We also define truncations of $M$ (see \cref{sec:truncated-mdps}), with the parameters $\trunc,\tsmall>0$ as defined in \cref{alg:pcr}.

\subsection{$\PCR$ \colt{Pseudocode and }Overview}\label{subsec:pcr-overview}

\colt{}

We start by giving an overview of the algorithm $\PCR$ (\cref{alg:pcr}). The structure is similar to that of $\PCO$ (\cref{alg:pco}), our algorithm from the episodic setting. In particular, the differences are entirely within the subroutine $\EPCR$ (\cref{alg:epcr}), the analogue of $\EPCO$ (\cref{alg:extend-pc-online}) from the episodic setting. Within $\EPCO$, the difference is in how the kinematics are estimated (and what precise kinematics function is estimated)---since we no longer have access to a two-context regression oracle, but we do have the ability to reset to any previously-seen state.

\paragraph{Overview of $\EPCR$.} As shown in \cref{alg:epcr}, $\EPCR$ takes as input a one-context regression oracle $\Reg$, a step $h \in [H]$, a set of policy covers $\Psi_{1:h}$, a backup policy cover $\Gamma$, and certain sample complexity and tolerance parameters. As with $\EPCO$, the goal is to produce a policy cover for step $h+1$, and to this end the algorithm estimates a certain kinematics function, defines internal reward functions by (implicitly) clustering together observations with similar kinematics, and uses $\PSDP$ (\cref{alg:psdpb}) to find a policy that $\pihat^{(t)}$ optimizes each reward function $\MR^{(t)}$.

As discussed in \cref{sec:resets}, the main difference compared to $\EPCO$ (and $\HOMER$) is in the estimation of kinematics. Rather than designing a dataset where the Bayes predictor must essentially distinguish between ``real'' and ``fake'' transitions (which inherently requires two contexts), $\EPCR$ samples $m$ discriminator observations $x_h^{(1)},\dots,x_h^{(m)}$ at step $h$ (by rolling in with the given policy cover $\Psi_h$ and backup policy cover $\Gamma$). For each of these observations $x_h^{(i)}$ and each action $a\in\MA$, $\EPCR$ uses reset access to draw many conditional samples from $\BP^M_{h+1}(\cdot\mid{}x_h^{(i)}, a)$, and then constructs a dataset $\MD_{i,a}$ where the Bayes predictor must essentially predict if an observation was conditionally sampled from $(x_h^{(i)},a)$ or from some other observation/action pair.

More formally, as shown in \cref{alg:epcr}, for each $i \in [m]$ and $a \in \MA$, $\EPCR$ constructs a dataset $\MD_{i,a}$ with samples from the following procedure. First, draw $j \sim \Unif([m])$ and $a_h \sim \Unif(\MA)$. Then reset to $x_h^{(j)}$, and sample $x_{h+1} \sim \BP^M_{h+1}(\cdot\mid{}x_h^{(j)},a_h)$. If $j = i$ and $a_h = a$, then add sample $(x_{h+1},1)$ to the dataset; otherwise add $(x_{h+1},0)$. It can be checked that the Bayes predictor $\EE[y\mid{}x_{h+1}]$ for this dataset is the kinematics function $w_{h+1}(\cdot;i,a)$ defined in \cref{eq:pcr-kinematics}.By the Block MDP assumption, $w_{h+1}(x_{h+1};i,a)$ only depends on $x_{h+1}$ through $\phi^\st(x_{h+1})$, so the guarantee of one-context regression applies, and invoking $\Reg$ on $\MD_{i,a}$ gives a good approximation $\wh w_{h+1}(\cdot;i,a)$ of $w_{h+1}(\cdot;i,a)$ with high probability.

This approximation is used in a similar fashion as $\wh f_{h+1}$ is used in $\EPCO$---the only difference henceforth is that there is no need to sample additional ``test observations'', since $x_h^{(1)},\dots,x_h^{(m)}$ serve this purposes.

We formally analyze $\EPCR$ in \cref{sec:epcr-analysis}; the main guarantee is \cref{thm:extend-pc-trunc}.

\paragraph{Overview of $\PCR$.} The algorithm $\PCR$ is identical to $\PCO$ aside from (1) various parameter choices, and (2) invoking $\EPCR$ rather than $\EPCO$. The algorithm proceeds in $R = |\MS|H$ rounds, and in each round $r$, it iteratively uses $\EPCR$ to construct sets $\Psi^{(r)}_{1:H}$, which are added to the backup policy cover for the next round. The output is the union of all sets (of bounded size) that were produced by $\EPCR$.

We formally analyze $\PCR$ (and thereby prove \cref{thm:pcr-app}) in \cref{sec:pcr-analysis}.


\subsection{Analysis of $\EPCR$ (\creftitle{alg:epcr})}\label{sec:epcr-analysis}

The following theorem is our main guarantee for $\EPCR$ (\cref{alg:epcr}). Recall that we have fixed a concept class $\Phi$, a $\Nreg$-efficient one-context regression oracle $\Reg$, and a $\Phi$-decodable block MDP $M$ with horizon $H$, action set $\MA$, and unknown decoding function $\phi^\st \in \Phi$. We have also defined truncations of $M$ (see \cref{sec:truncated-mdps}), with the parameters $\trunc,\tsmall>0$ as defined in \cref{alg:pcr}.

\begin{theorem}\label{thm:extend-pc-trunc}
Let $h \in \{1,\dots,H-1\}$. Let $\delta,\alpha > 0$ and $m,n,N \in \NN$. Let $\Gamma \subset \Pi$ be a finite set of policies. Suppose that $\Psi_{1:h}$ are $\alpha$-truncated policy covers (\cref{defn:trunc-pc}) for $M$ at steps $1,\dots,h$. Suppose that $m \geq \frac{2}{\min(\alpha\trunc,\tsmall)}\log(|\MS|/\delta)$, $n \geq m|\MA|\trunc^{-1}\log(|\MS|/\delta)$, $N \geq \Nreg(\epsilon,\delta)$, and
\begin{equation} \epsilon^{1/8} \leq \frac{\alpha \trunc^2 \tsmall}{6H(m|\MA|)^{3/2}}.\label{eq:param-assm}
\end{equation}
Consider the execution of $\EPCR$ with $\gamma := \epsilon^{1/8}$ and $\gamma' := 2\epsilon^{1/4}$. Then with probability at least $1-(2+m|\MA|+H|\MA|n)\delta-\sqrt{m|\MA|}n\epsilon^{1/4}$, the following two properties hold:
\begin{itemize}
\item $|\Psi_{h+1}| \leq |\MS|$.
\item Either $\Psi_{h+1}$ is a $(1-4\trunc)$-truncated max policy cover (\cref{defn:trunc-max-pc}) for $M$ at step $h+1$, or else \[\max_{\pi\in\Psi_{h+1}} d^{\Mbar(\Gamma),\pi}_{h+1}(\term) \geq \trunc^2.\]
\end{itemize}
\end{theorem}

See \cref{sec:truncated-mdps} for the definition of the truncated MDP $\Mbar(\Gamma)$ and the truncated policy covers. Like \cref{thm:extend-pc-trunc-online} (the analogous guarantee for $\EPCO$), this result shows that either $\EPCR$ produces a policy cover, or one of the policies in the output reaches the terminal state $\term$ in $\Mbar(\Gamma)$, which means that it discovered a state not well-covered by policies in $\Gamma$.

Let us fix the inputs to $\EPCR$: in addition to the one-context regression oracle $\Reg$ (\cref{def:two-con-regression}), we fix a layer $h \in [H-1]$, sets of policies $\Psi_1,\dots,\Psi_h$ and $\Gamma$, sample counts $n,m,N \in \NN$, and tolerances $\gamma,\gamma' \in (0,1)$. The first step of $\EPCR$ is to use one-context regression and reset access to estimate the kinematics function $w_{h+1}$ defined informally in \cref{eq:pcr-kinematics} and formally below. \cref{lemma:w-est-error} shows that with high probability, the estimate $\wh w_{h+1}(\cdot;i,a)$ is an accurate estimate of $w_{h+1}(\cdot;\phi^\st(x_h^{(i)}),a)$ for all actions $a \in \MA$ and ``test'' observations $x_h^{(1)},\dots,x_h^{(m)}$. Next, \cref{lemma:target-error-trunc} shows that the reward functions $\MR^{(t)}$ designed by clustering with respect to $\wh w_{h+1}$ approximately induce exploration of all latent reachable states. We then use these two lemmas to prove \cref{thm:extend-pc-trunc}.

\begin{definition}\label{def:beta-w}
For fixed $x^{(1)}_h,\dots,x^{(m)}_h$, we define distributions $\betahat_h,\betahat_{h+1} \in \Delta(\MS)$ by 
\[\betahat_h(s) := \frac{1}{m}\sum_{j=1}^m \mathbbm{1}[\phisj=s]\]
and
\[\betahat_{h+1}(s) := \frac{1}{m}\sum_{j=1}^m \EE_{a\sim\Unif(\MA)} \til\BP^M_{h+1}(s\mid{}\phisj, a).\]
For any $a_h \in \MA$ and $s_h,s_{h+1}\in\MS$ we also define
\[w_{h+1}(s_{h+1};s_h,a_h) := \frac{\til \BP^M_{h+1}(s_{h+1}\mid{}s_h,a_h)}{\sum_{j=1}^m\sum_{a \in \MA} \til\BP^M_{h+1}(s_{h+1}\mid{}\phisj,a)}.\]
\end{definition}

\begin{lemma}\label{lemma:w-est-error}
Fix a realization of $\MC = \{x_h^{(1)},\dots,x_h^{(m)}\}$. If $N \geq \Nreg(\epsilon,\delta)$, then it holds with probability at least $1-\delta m |\MA|$ that
\begin{equation}\EE_{x_{h+1} \sim \til\BO_{h+1} \betahat_{h+1}} \max_{(i,a) \in [m]\times\MA} \left|\what_{h+1}(x_{h+1};i,a) - w_{h+1}(\phi^\st(x_{h+1});\phi^\st(x_h^{(i)}),a) \right| \leq \sqrt{m|\MA|\epsilon}.
\label{eq:w-est-error}
\end{equation}
\end{lemma}

\begin{proof}
Fix $i \in [m]$ and $a \in \MA$. The dataset $\MD_{i,a}$ consists of $N$ i.i.d. samples. Let $(x_{h+1},y)$ denote the first sample, so that $x_{h+1} \sim \BP^M_{h+1}(\cdot\mid{}x^{(j)}_h,a_h)$ and $y = \mathbbm{1}[j=i \land a_h=a]$ for latent random variables $j \sim \Unif([m])$ and $a_h \sim \Unif(\MA)$. The marginal distribution of $x_{h+1}$ is exactly $\til\BO_{h+1}\betahat_{h+1}$. Also,
\begin{align*}
\EE[y\mid{}x_{h+1}=x]
&= \Pr[j=i \land a_h=a\mid{}x_{h+1}=x] \\ 
&= \frac{\Pr[x_{h+1}=x\mid{}j=i\land a_h=a]\Pr[j=i \land a_h=a]}{\Pr[x_{h+1}=x]} \\ 
&= \frac{\frac{1}{m|\MA|} \til \BO^M_{h+1}(x\mid{}\phi^\st(x)) \til\BP^M_{h+1}(\phi^\st(x)\mid{}\phi^\st(x_h^{(i)}), a)}{\frac{1}{m|\MA|} \sum_{j=1}^m \sum_{a \in \MA} \til \BO^M_{h+1}(x_{h+1}\mid{}\phi^\st(x_{h+1})) \til\BP^M_{h+1}(\phi^\st(x_{h+1})\mid{}\phi^\st(x_h^{(j)}), a)} \\ 
&= \frac{\til\BP^M_{h+1}(\phi^\st(x)\mid{}\phi^\st(x_h^{(i)}), a)}{ \sum_{j=1}^m \sum_{a \in \MA} \til\BP^M_{h+1}(\phi^\st(x)\mid{}\phi^\st(x_h^{(j)}), a)} \\ 
&= w_{h+1}(\phi^\st(x);\phisi,a)
\end{align*}
by \cref{def:beta-w}. Hence, we can apply the guarantee of $\Reg$ (\cref{def:one-con-regression}) with distribution $\til\BO_{h+1}\betahat_{h+1}$ and ground truth predictor $s_{h+1} \mapsto w_{h+1}(s_{h+1};\phisi,a)$. Since $N \geq \Nreg(\epsilon,\delta)$, we get that with probability at least $1-\delta$, the output $\what(\cdot;i,a)$ of $\Reg(\MD_{i,a})$ satisfies
\[\EE_{x_{h+1}\sim\til\BO_{h+1}\betahat_{h+1}} \left(\what_{h+1}(x_{h+1};i,a) - w_{h+1}(\phi^\st(x_{h+1});\phisi,a)\right)^2 \leq \epsilon.\]
Condition on the event that this bound holds for all $i \in [m]$ and $a \in \MA$, which occurs with probability at least $1-\delta m |\MA|$. Then
\begin{align*}
&\EE_{x_{h+1} \sim \til\BO_{h+1} \betahat_{h+1}} \max_{(i,a) \in [m]\times\MA} \left|\what_{h+1}(x_{h+1};i,a) - w_{h+1}(\phi^\st(x_{h+1});\phi^\st(x_h^{(i)}),a) \right|\\
&\leq \sqrt{\EE_{x_{h+1} \sim \til\BO_{h+1} \betahat_{h+1}} \max_{(i,a) \in [m]\times\MA} \left(\what_{h+1}(x_{h+1};i,a) - w_{h+1}(\phi^\st(x_{h+1});\phi^\st(x_h^{(i)}),a) \right)^2} \\ 
&\leq \sqrt{\sum_{(i,a) \in [m]\times\MA}\EE_{x_{h+1} \sim \til\BO_{h+1} \betahat_{h+1}} \left(\what_{h+1}(x_{h+1};i,a) - w_{h+1}(\phi^\st(x_{h+1});\phi^\st(x_h^{(i)}),a) \right)^2} \\ 
&\leq \sqrt{m|\MA|\epsilon}
\end{align*}
as claimed.
\end{proof}



\begin{lemma}\label{lemma:target-error-trunc}
Suppose that the event of \cref{lemma:w-est-error} holds. Let $\Gamma \subset \Pi$ be a finite set of policies, and let $\ccov > 0$. Suppose that $\{\phi^\st(x_h^{(1)}),\dots,\phi^\st(x_h^{(m)})\} \supseteq \Srch_h(\Gamma)$. Also suppose that $\betahat_{h+1}(s) \geq \ccov \cdot \max_{\pi\in\Pi} d^{\Mbar(\Gamma),\pi}_{h+1}(s)$ for all $s \in \MS$. Then for any $t \in [n]$ and $s^\st \in \Srch_{h+1}(\Gamma)$, there is some $K = K(\Gamma,s^\st) \geq 1$ such that
\begin{align*}
&\max_{\pi\in\Pi} \left|\E^{\Mbar(\Gamma),\pi}[\MR^{(t)}(x_{h+1})] - K \cdot d^{\Mbar(\Gamma),\pi}_{h+1}(s^\st)\right| \\
&\leq m\gamma + \frac{\sqrt{m|\MA|\epsilon}}{\ccov \gamma} + \frac{\max_{(i,a) \in [m]\times\MA} \left|\what_{h+1}(\xbart;i,a) - w_{h+1}(s^\st; \phi^\st(x_h^{(i)}),a)\right|}{\gamma}
\end{align*}
where $\MR^{(t)}, \xbart$ are as defined in \cref{alg:epcr}, and we let $\MR^{(t)}(\term):=0$.
\end{lemma}

\begin{proof}
Fix $\pi \in \Pi$. For $s,s' \in \MS$, define \[\Delta(s,s') := \max_{(i,a) \in [m]\times\MA} |w_{h+1}(s;\phisi,a)-w_{h+1}(s';\phisi,a)|\] and
\[W^\pi := \E^{\Mbar(\Gamma),\pi} \left[g(\Delta(s_{h+1}, s^\st))\mathbbm{1}[s_{h+1}\in\MS]\right].\]
We start by proving that $\EE^{\Mbar(\Gamma),\pi}[\MR^{(t)}(x_{h+1})]$ is close to $W^\pi$ for all $\pi \in \Pi$. Recall from \lineref{line:rt-def} that for any $x_{h+1} \in \MX$, we have \[\MR^{(t)}(x_{h+1}) = g\left(\max_{(i,a) \in [m]\times\MA} |\what_{h+1}(\xbart;i,a) - \what_{h+1}(x_{h+1};i,a)|\right)\]
where $g(z) = \max(0, 1-z/\gamma)$. Thus, for any $s_{h+1} \in \MS$, we have
\begin{align}
&\left|\EE_{x_{h+1}\sim\BO_{h+1}(\cdot|s_{h+1})}\left[ \MR^{(t)}(x_{h+1}) - g\left(\Delta(s_{h+1},s^\st)\right)\right]\right|\nonumber \\ 
&\leq \frac{1}{\gamma}\EE_{x_{h+1}\sim\BO_{h+1}(\cdot\mid{}s_{h+1})} \Bigg[\max_{(i,a) \in [m]\times\MA} \bigg| \left|\what_{h+1}(x_{h+1};i,a) - \what_{h+1}(\xbart;i,a)\right| \nonumber\\
&\hspace{14.3em}- \left| w_{h+1}(s_{h+1};\phisi,a) - w_{h+1}(s^\st;\phisi,a)\right|\bigg|\Bigg] \nonumber\\ 
&\leq \frac{1}{\gamma}\underbrace{\EE_{x_{h+1}\sim\BO_{h+1}(\cdot\mid{}s_{h+1})} \left[\max_{(i,a) \in [m]\times\MA} \left| \what_{h+1}(x_{h+1};i,a) - w_{h+1}(s_{h+1};\phisi,a)\right|\right]}_{E_1(s_{h+1})} \nonumber\\
&\qquad+ \frac{1}{\gamma} \underbrace{\max_{(i,a) \in [m]\times\MA} \left| \what_{h+1}(\xbart;i,a) - w_{h+1}(s^\st;\phisi,a)\right|}_{E_2} \label{eq:what-w-single}
\end{align}
where the first inequality uses that $g$ is $1/\gamma$-Lipschitz. It follows that
\begin{align}
\left| \E^{\Mbar(\Gamma),\pi}[\MR^{(t)}(x_{h+1})] - W^\pi\right|
&\leq \frac{1}{\gamma} \E^{\Mbar(\Gamma),\pi}[E_1(s_{h+1})] + \frac{E_2}{\gamma} \nonumber\\ 
&\leq \frac{1}{\ccov \gamma} \EE_{s_{h+1} \sim \wh\beta_{h+1}}[E_1(s_{h+1})] + \frac{E_2}{\gamma} \nonumber\\ 
&\leq \frac{\sqrt{m|\MA|\epsilon}}{\ccov \gamma} + \frac{E_2}{\gamma} \label{eq:rt-wpi-diff}
\end{align}
where the first inequality is by \cref{eq:what-w-single}, the second inequality uses the assumption on $\betahat_{h+1}$ and nonnegativity of $E_1$, and the third inequality uses \cref{eq:w-est-error}. 

Next, for any $s_{h+1} \in \Srch_{h+1}(\Gamma)$, note that $\BP^{\Mbar(\Gamma)}_{h+1}(s_{h+1}\mid{}s_h,a_h) = \BP^M_{h+1}(s_{h+1}\mid{}s_h,a_h)$ and $\BP^{\Mbar(\Gamma)}_{h+1}(s_{h+1}\mid{}\term,a_h) = 0$ for all $s_h\in\MS$, $a_h \in \MA$. Thus, the following equality holds for all $s_{h+1} \in \Srch_{h+1}(\Gamma)$:
\begin{align}
d^{\Mbar(\Gamma),\pi}_{h+1}(s_{h+1}) 
&= \sum_{(s_h,a_h) \in \MS\times\MA} d^{\Mbar(\Gamma),\pi}_h(s_h,a_h) \til\BP^M_{h+1}(s_{h+1}\mid{}s_h,a_h) \nonumber\\ 
&= \left(\sum_{(s_h,a_h)\in\MS\times\MA} d^{\Mbar(\Gamma),\pi}_h(s_h,a_h)w_{h+1}(s_{h+1};s_h,a_h)\right)\sum_{j=1}^m\sum_{a\in\MA}\til\BP^M_{h+1}(s_{h+1}\mid{}\phisj,a)
\label{eq:dpi-expansion}
\end{align}
where the final equality is by definition of $w_h(s_{h+1};s_h,a_h)$ (\cref{def:beta-w}). It follows from \cref{eq:dpi-expansion} and the fact that $d^{\Mbar(\Gamma),\pi}_{h+1}(s_{h+1}) = 0$ for all $s_{h+1} \in \MS\setminus\Srch_{h+1}(\Gamma)$ that
\begin{align*}
W^\pi &= \sum_{s_{h+1}\in\Srch_{h+1}(\Gamma)} g(\Delta(s_{h+1},s^\st))\left(\sum_{(s_h,a_h)\in\MS\times\MA} d^{\Mbar(\Gamma),\pi}_h(s_h,a_h)w_{h+1}(s_{h+1};s_h,a_h)\right)\\ 
&\hspace{10em}\cdot\left( \sum_{j=1}^m\sum_{a\in\MA}\til\BP^M_{h+1}(s_{h+1}\mid{}\phi^\st(x_h^{(j)}),a)\right).
\end{align*}
Motivated by this expression, define
\begin{align*}
\til W^\pi &:= \sum_{s_{h+1}\in\Srch_{h+1}(\Gamma)} g(\Delta(s_{h+1},s^\st))\left(\sum_{(s_h,a_h)\in\MS\times\MA} d^{\Mbar(\Gamma),\pi}_h(s_h,a_h)w_{h+1}(s^\st;s_h,a_h)\right)\\ 
&\hspace{10em}\cdot\left( \sum_{j=1}^m\sum_{a\in\MA}\til\BP^M_{h+1}(s_{h+1}\mid{}\phi^\st(x_h^{(j)}),a)\right).
\end{align*}
Observe that
\begin{align}
&\sum_{(s_h,a_h)\in\MS\times\MA} d^{\Mbar(\Gamma),\pi}_h(s_h,a_h)|w_{h+1}(s_{h+1};s_h,a) - w_{h+1}(s^\st;s_h,a)| \nonumber\\
&\leq \max_{s_h \in \Srch_h(\Gamma)} \max_{a \in \MA} |w_{h+1}(s_{h+1};s_h,a) - w_{h+1}(s^\st;s_h,a)|\nonumber\\ 
&\leq \max_{(i,a) \in [m]\times\MA}|w_{h+1}(s_{h+1};\phi^\st(x^{(i)}),a)-w_{h+1}(s^\st;\phi^\st(x^{(i)}),a)| \nonumber\\
&= \Delta(s_{h+1},s^\st),\label{eq:d-sh-st-bound}
\end{align}
where the first inequality uses the fact that $\sum_{(s_h,a_h)\in\MS\times\MA} d^{\Mbar(\Gamma),\pi}_h(s_h,a_h) \leq 1$ and $d^{\Mbar(\Gamma),\pi}_h(s_h,a_h) = 0$ if $s \in \MS \setminus \Srch_h(\Gamma)$, and the second inequality uses the assumption that $\Srch_h(\Gamma) \subseteq \{\phi^\st(x_h^{(1)}),\dots,\phi^\st(x_h^{(m)})\}$. Using \cref{eq:d-sh-st-bound}, we get
\begin{align}
& \left|W^\pi - \til W^\pi \right| \nonumber\\ 
&\leq \sum_{s_{h+1}\in\Srch_{h+1}(\Gamma)} g(\Delta(s_{h+1},s^\st))\Delta(s_{h+1},s^\st) \sum_{j=1}^m \sum_{a\in\MA} \til\BP^M_{h+1}(s_{h+1}\mid{}\phi^\st(x_h^{(j)}),a) \nonumber\\ 
&\leq \gamma \sum_{s_{h+1}\in\Srch_{h+1}(\Gamma)} \sum_{j=1}^m \sum_{a\in\MA}\til\BP^M_{h+1}(s_{h+1}\mid{}\phi^\st(x_h^{(j)}),a) \nonumber \\ 
&= \gamma m \label{eq:wpi-kd-diff}
\end{align}
where the first inequality is by \cref{eq:d-sh-st-bound} and the second inequality uses the fact that $z\cdot g(z) \leq \gamma$ for all $z \geq 0$. Now from the definition of $\til W^\pi$,
\begin{align}
\til W^\pi 
&= \left(\sum_{(s_h,a_h)\in\MS\times\MA} d^{\Mbar(\Gamma),\pi}_h(s_h,a_h)w_{h+1}(s^\st;s_h,a_h)\right) \sum_{s_{h+1} \in \Srch_{h+1}(\Gamma)} g(\Delta(s_{h+1},s^\st))\sum_{j=1}^m \sum_{a\in\MA} \til\BP^M_{h+1}(s_{h+1}\mid{}\phi^\st(x_h^{(j)}),a) \nonumber\\ 
&= \frac{d^{\Mbar(\Gamma),\pi}_{h+1}(s^\st)}{\sum_{j=1}^m \sum_{a\in\MA}\til\BP^M_{h+1}(s^\st\mid{}\phi^\st(x_h^{(j)}),a)} \cdot \sum_{s_{h+1}\in\Srch_{h+1}(\Gamma)}g(\Delta(s_{h+1},s^\st))\sum_{j=1}^m \sum_{a\in\MA} \til\BP^M_{h+1}(s_{h+1}\mid{}\phi^\st(x_h^{(j)}),a)  \nonumber \\ 
&= K(\Gamma,s^\st) \cdot d^{\Mbar(\Gamma),\pi}_{h+1}(s^\st) \label{eq:kd-expr}
\end{align}
where the second equality uses \cref{eq:dpi-expansion} together with the assumption that $s^\st \in \Srch_{h+1}(\Gamma)$, and in the final equality we have defined
\begin{align*}
K(\Gamma,s^\st)
&:= \frac{\sum_{s_{h+1}\in\Srch_{h+1}(\Gamma)}g(\Delta(s_{h+1},s^\st))\sum_{j=1}^m \sum_{a\in\MA} \til\BP^M_{h+1}(s_{h+1}\mid{}\phi^\st(x_h^{(j)}),a)}{\sum_{j=1}^m \sum_{a\in\MA}\til\BP^M_{h+1}(s^\st\mid{}\phi^\st(x_h^{(j)}),a)}.
\end{align*}
Note that $K(\Gamma,s^\st) \geq 1$, since $s^\st \in \Srch_{h+1}(\Gamma)$ and $g(\Delta(s^\st,s^\st)) = g(0) = 1$. Combining \cref{eq:rt-wpi-diff}, \cref{eq:wpi-kd-diff}, and \cref{eq:kd-expr}, we get that
\[\left|\E^{\Mbar(\Gamma),\pi}[\MR^{(t)}(x_{h+1})] - K(\Gamma,s^\st) \cdot d^{\Mbar(\Gamma),\pi}_{h+1}(s^\st)\right| \leq \frac{\sqrt{m|\MA|\epsilon}}{\ccov \gamma} + \frac{E_2}{\gamma} + \gamma m.\]
Substituting in the definition of $E_2$ yields the result.
\end{proof}

\begin{proof}[Proof of \cref{thm:extend-pc-trunc}]
For any fixed $s \in \Srch_h(\Gamma)$ and $i \in [m]$, we have by \cref{item:h-cov-lb} of \cref{lemma:srch-gamma-covering} that
\[\Pr[\phi^\st(x_h^{(i)}) = s] \geq \frac{\min(\alpha\trunc,\tsmall)}{2}.\]
Let $\ME_1$ be the event that for each $s \in \Srch_h(\Gamma)$ there is some $i \in [m]$ with $\phi^\st(x_h^{(i)}) = s$. Then
\[\Pr[\ME_1] \geq 1 - |\MS|\left(1 - \frac{\min(\alpha\trunc,\tsmall)}{2}\right)^m \geq 1-\delta\]
by the assumption that $m \geq \frac{2}{\min(\alpha\trunc,\tsmall)} \log(|\MS|/\delta)$. Henceforth condition on $x_h^{(1)},\dots,x_h^{(m)}$ and suppose that $\ME_1$ holds.

For any $s_{h+1} \in \Srch_{h+1}(\Gamma)$, we have by definition of $\betahat_{h+1}$ (\cref{def:beta-w}) that
\begin{align}
\betahat_{h+1}(s_{h+1}) 
&= \frac{1}{m}\sum_{j=1}^m \EE_{a\sim\Unif(\MA)} \til\BP^M_{h+1}(s_{h+1}|\phi^\st(x^{(j)}),a) \nonumber\\ 
&\geq \frac{1}{m} \max_{s_h \in \Srch_h(\Gamma)} \EE_{a\sim\Unif(\MA)} \til\BP^M_{h+1}(s_{h+1}|s_h,a) \nonumber\\ 
&\geq \frac{1}{m|\MA|} \max_{(s_h,a_h) \in \Srch_h(\Gamma)\times\MA} \til\BP^M_{h+1}(s_{h+1}|s_h,a_h) \nonumber\\ 
&\geq \frac{1}{m|\MA|} \max_{\pi\in\Pi} d^{\Mbar(\Gamma),\pi}_{h+1}(s_{h+1}) \label{eq:betahat-coverage}
\end{align}
where the first inequality uses $\ME_1$ and the final inequality uses the fact that \[d^{\Mbar(\Gamma),\pi}_{h+1}(s_{h+1}) = \sum_{(s_h,a_h) \in \Srch_h(\Gamma)\times\MA} d^{\Mbar(\Gamma),\pi}_h(s_h,a_h)\til\BP^M_{h+1}(s_{h+1}|s_h,a_h).\] 
Since $\max_{\pi\in\Pi} d^{\Mbar(\Gamma),\pi}_{h+1}(s_{h+1}) = 0$ for all $s_{h+1} \in \MS\setminus\Srch_{h+1}(\Gamma)$, in fact the bound
\begin{equation} 
\betahat_{h+1}(s_{h+1}) \geq \frac{1}{m|\MA|} \max_{\pi\in\Pi} d^{\Mbar(\Gamma),\pi}_{h+1}(s_{h+1}) \geq \frac{1}{m|\MA|} \max_{\pi\in\Pi} d^{\Mbar(\emptyset),\pi}_{h+1}(s_{h+1})
\label{eq:betahat-coverage-2}
\end{equation} holds for all $s_{h+1} \in \MS$, where the last inequality is by \cref{fact:gamma-monotonicity}.

Next, let $\ME_2$ be the event that 
\begin{equation}
\EE_{x_{h+1} \sim \til\BO_{h+1} \betahat_{h+1}} \max_{(i,a) \in [m]\times\MA} \left|\what_{h+1}(x_{h+1};i,a) - w_{h+1}(\phi^\st(x_{h+1});\phi^\st(x_h^{(i)}),a) \right| \leq \sqrt{m|\MA|\epsilon}.
\label{eq:w-west-applied}
\end{equation}
By \cref{lemma:w-est-error} and the assumption that $N \geq \Nreg(\epsilon,\delta)$, the event $\ME_2$ occurs with probability at least $1-m|\MA|\delta$ over the randomness of $(\MD_{i,a})_{i,a}$ and $\Reg$. Condition on $\what_{h+1}(\cdot;\cdot,\cdot)$ and suppose that $\ME_2$ holds.

For each $t \in [n]$, let $\ME_3^t$ be the event that 
\[\max_{(i,a)\in[m]\times\MA} |\what_{h+1}(\xbart;i,a) - w_{h+1}(\phi^\st(\xbart); \phi^\st(x_h^{(i)}),a)| \leq \epsilon^{1/4}.\] Since $\xbart \sim \til\BO_{h+1}\betahat_{h+1}$, we have by Markov's inequality and \cref{eq:w-west-applied} that $\Pr[\lnot \ME_3^t] \leq \sqrt{m|\MA|}\epsilon^{1/4}$. Define $\ME_3 := \bigcap_{t=1}^n \ME_3^t$. By the union bound, $\ME_3$ occurs with probability at least $1-\sqrt{m|\MA|}n\epsilon^{1/4}$ over the randomness of $\ol x_{h+1}^{(1)},\dots,\ol x_{h+1}^{(n)}$. 

Also, for each $s \in \Srch_{h+1}(\emptyset)$, let $t(s) \in [1,n] \cup \{\infty\}$ be the infimum over $t$ such that $\phi^\st(\xbart) = s$, and let $\ME_4^s$ be the event that $t(s) < \infty$. For each $t \in [1,n]$, since $\phi^\st(\xbart)$ has distribution $\betahat_{h+1}$, we have 
\[\Pr[\phi^\st(\xbart)=s] = \betahat(s)  \geq \frac{1}{m|\MA|} \max_{\pi\in\Pi} d^{\Mbar(\emptyset),\pi}_{h+1}(s) \geq \frac{\trunc}{m|\MA|}\]
by \cref{eq:betahat-coverage-2} and \cref{fact:trunc-reachability}.
Thus, $\Pr[\lnot \ME_4^s] \leq (1-\trunc/(m|\MA|))^n \leq \delta/|\MS|$ for any fixed $s \in \Srch_{h+1}(\emptyset)$, by the assumption that $n \geq m|\MA|\trunc^{-1} \log(|\MS|/\delta)$. Define $\ME_4 := \bigcap_{s \in \Srch_{h+1}(\emptyset)} \ME_4^s$. By the union bound, $\ME_4$ occurs with probability at least $1-\delta$ over the randomness of $\ol x_{h+1}^{(1)},\dots,\ol x_{h+1}^{(n)}$.

Condition on $\ol x_{h+1}^{(1)},\dots,\ol x_{h+1}^{(n)}$ and suppose that $\ME_3 \cap \ME_4$ holds. Note that the set $\Tclus$ is now determined. For each $t \in \Tclus$ let $\ME_5^t$ be the event that
\[\E^{M,\pihat^{(t)}}[\MR^{(t)}(x_{h+1})] \geq \max_{\pi\in\Pi} \E^{\Mbar(\Gamma),\pi}[\MR^{(t)}(x_{h+1})] - \frac{4H\sqrt{|\MA|\epsilon}}{\min(\alpha\trunc,\tsmall)}\]
where $\pihat^{(t)}$ is defined in \lineref{line:psdp-call}. By the theorem assumptions on $\Reg$ and $\Psi_{1:h}$ and the fact that $N \geq \Nreg(\epsilon,\delta)$, \cref{lemma:psdp-trunc-online} gives $\Pr[\ME_5^t] \geq 1-H|\MA|\delta$, so $\Pr[\ME_5] \geq 1-H|\MA|n\delta$ where $\ME_5 := \cap_{t \in \Tclus} \ME_5^t$. Condition on $\ME_5$. We have now restricted to an event of total probability at least $1 - 2\delta - (m|\MA|+H|\MA|n)\delta - \sqrt{m|\MA|}n\epsilon^{1/4}$; we argue that in this event, the properties claimed in the theorem statement hold.

\paragraph{Size of $\Psi_{h+1}$.} First, we argue that $|\Psi_{h+1}| \leq |\MS|$. Indeed, suppose that there are $t < t'$ with $t,t' \in \Tclus$ and $\phi^\st(\ol x_{h+1}^{(t)}) = \phi^\st(\ol x_{h+1}^{(t')})$. By \lineref{line:cluster-test}, we know that 
\[\max_{(i,a) \in [m]\times\MA} |\what_{h+1}(\ol x_{h+1}^{(t)};i,a) - \what_{h+1}(\ol x_{h+1}^{(t')};i,a)| > \gamma' = 2\epsilon^{1/4}.\] 
But this contradicts $\ME_3$ (in particular, the bounds implied by $\ME_3^t$ and $\ME_3^{t'}$ in combination with the triangle inequality and the assumption that $\phi^\st(\xbart) = \phi^\st(\ol x_{h+1}^{(t')})$). We conclude that indeed $|\Psi_{h+1}| \leq |\MS|$.

\paragraph{Coverage of $\Psi_{h+1}$.} It remains to prove the second property of the theorem statement. Fix $s \in \Srch_{h+1}(\emptyset)$. Let $t^\st(s)$ be the minimal $t \in \Tclus$ such that 
\begin{equation}\max_{(i,a) \in [m]\times\MA} |\what_{h+1}(\ol x_{h+1}^{(t(s))};i,a)-\what(\ol x_{h+1}^{(t)};i,a)| \leq \gamma'.\label{eq:t-tst}\end{equation}
Note that $t^\st(s)$ necessarily exists, because either $t(s) \in \Tclus$ or, if not, \lineref{line:cluster-test} implies that some other $t \in \Tclus$ satisfies the above bound. Next, by $\ME_1$, \cref{eq:betahat-coverage-2}, and \cref{eq:w-west-applied}, we can apply \cref{lemma:target-error-trunc} with coverage parameter $\ccov := 1/(m|\MA|)$, target state $s^\st = s$, and index $t := t^\st(s)$. We get that
\begin{align}
&\max_{\pi\in\Pi} \left|\E^{\Mbar(\Gamma),\pi}[\MR^{(t^\st(s))}(x_{h+1})] - K(\Gamma,s) \cdot d^{\Mbar(\Gamma),\pi}_{h+1}(s)\right|  \nonumber\\ 
&\leq m\gamma + \frac{(m|\MA|)^{3/2}\sqrt{\epsilon}}{\gamma} + \frac{\max_{(i,a)\in[m]\times\MA}\left|\what_{h+1}(\ol x^{(t^\st(s))}_{h+1};i,a) - w_{h+1}(s;\phi^\st(x_h^{(i)}),a)\right|}{\gamma} \nonumber \\ 
&\leq m\gamma + \frac{(m|\MA|)^{3/2}\sqrt{\epsilon}}{\gamma} + \frac{\epsilon^{1/4} + \max_{(i,a)\in[m]\times\MA}\left|\what_{h+1}(\ol x^{(t^\st(s))}_{h+1};i,a) - \what_{h+1}(\ol x^{(t(s))}_{h+1};i,a)\right|}{\gamma} \nonumber \\ 
&\leq m\gamma + \frac{(m|\MA|)^{3/2}\sqrt{\epsilon}}{\gamma} + \frac{\epsilon^{1/4} + \gamma'}{\gamma} \nonumber\\
&\leq \trunc^2,\label{eq:rkd-error}
\end{align}
where the second inequality is by $\ME_3$ and the fact that $s = \phi^\st(\ol x_{h+1}^{(t(s))})$, the third inequality is by \cref{eq:t-tst}, and the final inequality is by choice of $\gamma,\gamma'$ and \cref{eq:param-assm}. We now distinguish two cases.

\paragraph{Case I.} Suppose that 
\[\E^{M,\pihat^{(t^\st(s))}}[\MR^{(t^\st(s))}(x_{h+1})] \geq \E^{\Mbar(\Gamma),\pihat^{(t^\st(s))}}[\MR^{(t^\st(s))}(x_{h+1})] + \trunc^2.\]
Then
\begin{align*}
&d^{\Mbar(\Gamma),\pihat^{(t^\st(s))}}_{h+1}(\term) \\ 
&= \sum_{s \in \MS} \left(d^{M,\pihat^{(t^\st(s))}}_{h+1}(s) - d^{\Mbar(\Gamma),\pihat^{(t^\st(s))}}(s))\right) \\ 
&\geq \sum_{s \in \MS} \left(d^{M,\pihat^{(t^\st(s))}}_{h+1}(s) - d^{\Mbar(\Gamma),\pihat^{(t^\st(s))}}_{h+1}(s))\right) \EE_{x\sim\BO_{h+1}(\cdot|s)}[\MR^{(t^\st(s))}(x)] \\ 
&\geq \trunc^2 
\end{align*}
where the equality is by the fact that $d^{M,\pihat^{(t^\st(s))}}_{h+1}(\cdot)$ is a distribution supported on $\MS$; the first inequality uses \cref{fact:gamma-monotonicity} and the fact that $\MR^{(t^\st(s))}(x) \leq 1$ for all $x \in \MX$. Thus $\max_{\pi\in\Psi_{h+1}} d^{\Mbar(\Gamma),\pi}(\term) \geq \trunc^2$, so the second property of the theorem statement is satisfied. 

\paragraph{Case II.} Suppose that
\begin{equation} \E^{M,\wh\pi^{(t^\st(s))}\circ_h\Unif(\MA)}[\MR^{(t^\st(s))}(x_{h+1})] < \E^{\Mbar(\Gamma),\wh\pi^{(t^\st(s))}\circ_h\Unif(\MA)}[\MR^{(t^\st(s))}(x_{h+1})] + \trunc^2.\label{eq:m-mbar-closeness}\end{equation}
Now, 
\begin{align*}
K(\Gamma, s) \cdot d^{\Mbar(\Gamma), \pihat^{(t^\st(s))}}_{h+1}(s) 
&\geq \E^{\Mbar(\Gamma),\pihat^{(t^\st(s))}}[\MR^{(t^\st(s))}(x_{h+1})] - \trunc^2 \\ 
&\geq \E^{M,\pihat^{(t^\st(s))}}[\MR^{(t^\st(s))}(x_{h+1})] - 2\trunc^2 \\ 
&\geq \max_{\pi\in\Pi} \E^{\Mbar(\Gamma),\pi}[\MR^{(t^\st(s))}(x_{h+1})] - 3\trunc^2 \\ 
&\geq \max_{\pi\in\Pi} K(\Gamma,s) \cdot d^{\Mbar(\Gamma),\pi}_{h+1}(s) -  4\trunc^2 \\ 
&\geq K(\Gamma,s)(1-4\trunc) \max_{\pi\in\Pi} d^{\Mbar(\Gamma),\pi}_{h+1}(s)
\end{align*}
where the first inequality is by \cref{eq:rkd-error}, the second inequality is by \cref{eq:m-mbar-closeness}, the third inequality is by $\ME_5$ and \cref{eq:param-assm}, the fourth inequality is by \cref{eq:rkd-error}, and the fifth inequality uses the fact that $\max_{\pi\in\Pi} d^{\Mbar(\Gamma),\pi}_{h+1}(s) \geq \trunc$ (\cref{fact:trunc-reachability} together with \cref{fact:gamma-monotonicity} and the fact that $s \in \Srch_{h+1}(\emptyset)$) and $K(\Gamma,s) \geq 1$. Thus, since $\pihat^{(t^\st(s))} \in \Psi_{h+1}$, we have 
\begin{align}
\max_{\pi\in\Psi_{h+1}} d^{M,\pi}_{h+1}(s)
&\geq \max_{\pi\in\Psi_{h+1}} d^{\Mbar(\Gamma),\pihat^{(t^\st(s))}}_{h+1}(s) \nonumber\\ 
&\geq (1-4\trunc)\max_{\pi\in\Pi} d^{\Mbar(\Gamma),\pi}_{h+1}(s) \nonumber\\ 
&\geq (1-4\trunc)\max_{\pi\in\Pi} d^{\Mbar(\emptyset),\pi}_{h+1}(s)
\label{eq:coverage-final-reset}
\end{align}
by two applications of \cref{fact:gamma-monotonicity}. Now recall that $s \in \Srch_{h+1}(\emptyset)$ was arbitrary. Moreover, if $s \in \MS\setminus \Srch_{h+1}(\emptyset)$ then $\max_{\pi\in\Pi} d^{\Mbar(\emptyset),\pi}_{h+1}(s) = 0$, so \cref{eq:coverage-final-reset} still holds. We conclude that $\Psi_{h+1}$ is a $(1-4\trunc)$-truncated max policy cover for $M$ at step $h+1$ (\cref{defn:trunc-max-pc}) as needed.
\end{proof}

\subsection{Analysis of $\PCR$ (\creftitle{alg:pcr})}\label{sec:pcr-analysis}

We can now use \cref{thm:extend-pc-trunc} to complete the analysis of $\PCR$, proving \cref{thm:pcr-app} (and thus \cref{cor:reset-rl-to-regression}). After modularizing out the analysis of $\EPCR$/$\EPCO$, the analyses of $\PCR$/$\PCO$ are essentially identical, so we omit the details here for brevity.

\vspace{1em}

\begin{proof}[Proof of \cref{thm:pcr-app}]
Fix the remaining inputs $\epfinal,\delta>0$ to $\PCO(\Reg,\Nreg,|\MS|,\cdot)$. The oracle time complexity bound and bound on $\Nrl$ are clear from the parameter choices and pseudocode, so long as $C_{\ref{thm:pcr-app}}$ is a sufficiently large constant. Moreover, it is immediate from the algorithm description that $|\Psi| \leq HR|\MS| \leq H^2|\MS|^2$. In order to show that the algorithm is $(\Nrl,\Krl)$-efficient, it remains to argue that with probability at least $1-\delta$, \cref{eq:rfrl-pc} holds for all $h \in [H]$ and $s \in \MS$, with parameter $\epfinal$.

Recall that $\trunc = \epfinal/(4+H|\MS|)$. Fix some $1 \leq r \leq R$. For convenience, write $\alpha := \frac{1-4\trunc}{|\MS|}$. For each $h \in [H]$, let $\ME_{h,r}$ be the event that $|\Psi_h^{(r)}| \leq |\MS|$ and $\Psi_h^{(r)}$ is a $(1-4\trunc)$-truncated max policy cover for $M$ at step $h$; let $\MF_{h,r}$ be the event that $|\Psi_{h}^{(r)}| \leq |\MS|$ and $\max_{\pi\in\Psi_{h}^{(r)}} d^{\Mbar(\Gamma^{(r)}),\pi}_{h}(\term) \geq \trunc^2$. It's clear that $\Pr[\ME_{1,r}] = 1$ (since $|\Psi_1^{(r)}| = 1$ and $d^{M,\piunif}_1(s) = d^{M,\pi}_1(s)$ for all $s\in\MS$ and $\pi \in \Pi$). Also, note that in the event $\ME_{k,r}$, we have that $\Psi_k^{(r)}$ is an $\alpha$-truncated policy cover for $M$ at step $k$. Thus, by \cref{thm:extend-pc-trunc} and choice of parameters (so long as $C$ is a sufficiently large constant), we have for each $h \in \{2,\dots,H\}$ that
\begin{equation} \Pr\left[\lnot (\ME_{h,r} \cup \MF_{h,r}) \cap \bigcap_{1 \leq k < h} \ME_{k,r}\right] \leq \Pr\left[\lnot (\ME_{h,r} \cup \MF_{h,r}) \middle | \bigcap_{1 \leq k < h} \ME_{k,r}\right] \leq \frac{\delta}{HR}.\label{eq:epcr-guarantee-symbolic}\end{equation}
The remainder of the proof is identical to that of \cref{thm:pco-app}. 
\end{proof}


\newpage

\section{Proofs from \creftitle{sec:lowrank}}\label{app:lowrank}

In this section, we re-introduce Generalized Block MDPs, and verify (\cref{prop:genblock-is-lowrank}) that generalized block MDPs are a special case of low-rank MDPs. Then, in \cref{sec:halfspace-ocr}, we prove \cref{thm:halfspace-reg}, which asserts that one-context regression is computationally tractable for the concept class $\Phi_n$ defined by halfspaces. In \cref{sec:halfspace-hardness}, we prove \cref{thm:halfspace-rl-hard}, which asserts that reward-free RL for $\Phi_n$-decodable generalized block MDPs is computationally \emph{hard}. Together, \cref{thm:halfspace-reg,thm:halfspace-rl-hard} immediately imply \cref{thm:halfspace-main}. Finally, in \cref{sec:halfspace-statistical}, we prove \cref{prop:halfspace-rl-statistical}, which asserts that reward-free RL in the family of Generalized Block MDPs (and hence Low-Rank MDPs) $\MM_n$ is \emph{statistically} tractable---and hence the preceding hardness result is purely a \emph{computational} phenomenon.

\subsection{Preliminaries}\label{sec:low-rank-prelim}

\paragraph{One-context low-rank regression.} We will not actually need a formal definition of one-context low-rank regression in our results, but we introduce it formally for the sake of discussion:

\begin{definition}\label{def:low-rank-reg-app}
Fix $\Nreg: (0,1/2) \to \NN$. An algorithm $\Alg$ is an $\Nreg$-efficient \emph{one-context low-rank regression} algorithm for a feature class $\Philin \subseteq (\MX \times \MA \to \RR^d)$ if the following holds. Fix $\epsilon,\delta \in (0,1/2)$, $n \in \NN$, $\philin\in\Philin$, $\MD\in\Delta(\MX\times\MA)$, and $\theta \in \RR^d$. Let $(x^{(i)},a^{(i)},y^{(i)})_{i=1}^n$ be i.i.d. samples with $(x^{(i)},a^{(i)})\sim\MD$, $y^{(i)} \in \{0,1\}$, and $\EE[y^{(i)}\mid{}x^{(i)},a^{(i)}] = \langle \philin(x^{(i)},a^{(i)}),\theta\rangle$. If $n \geq \Nreg(\epsilon,\delta)$, then with probability at least $1-\delta$, the output of $\Alg((x^{(i)},a^{(i)},y^{(i)})_{i=1}^n,\epsilon,\delta)$ is a circuit $\MR: \MX\times\MA\to[0,1]$ satisfying \[\EE_{(x,a)\sim\MD} (\MR(x,a) - \langle \philin(x,a),\theta\rangle)^2 \leq \epsilon.\]
\end{definition}

\paragraph{Generalized Block MDPs.} Recall from \cref{sec:lowrank} that  for sets $\MS,\MX$ and a concept class $\Phi \subseteq (\MX\to\MS)$, a \emph{Generalized} $\Phi$-decodable Block MDP is an MDP $M = (H, \MX,\MA,(\BP_h)_h)$ with the property that there exists some $\phist_1,\dots,\phist_H \in \Phi$ so that $\BP_{h+1}(x_{h+1}\mid{}x_h,a_h)$ is a function of $x_{h+1}$, $\phist_h(x_h)$, and $a_h$. 

The following proposition makes formal the observation from \cref{sec:lowrank} that Generalized Block MDPs (and, as a special case, Block MDPs themselves) are Low-Rank MDPs with an appropriate feature class:

\begin{proposition}\label{prop:genblock-is-lowrank}
Set $d = |\MS||\MA|$ and identify $[d] \equiv \MS\times\MA$. Define $\Philin := \{\philin: \phi \in \Phi\}$, where $\philin: \MX\times\MA \to \RR^d$ is defined by
\[\philin(x,a) := e_{\phi(x), a}.\]
Then any Generalized $\Phi$-decodable Block MDP $M$ is a Low-Rank MDP \citep{agarwal2020flambe,mhammedi2023efficient} with features $(\phist_h)_h \subset \Philin$ and dual features $(\must_h)_h$ where $(\must_h)_{s,a} \in \Delta(\MX)$ for all $h\in[H]$, $s\in\MS$, and $a\in\MA$.
\end{proposition}

\begin{proof}
Let $\phist_1,\dots,\phist_H$ be the decoding functions for $M$. For each $h \in [H-1]$ and $(s_h,a_h) \in \MS\times\MA$, we fix any $\xbar_h \in \MX$ with $\phist_h(\xbar_h)=s_h$ and define $\must_{h+1}(x_{h+1})_{s_h,a_h} = \BP_{h+1}(x_{h+1}\mid{}\xbar_h,a_h)$. Then for any $x_h,x_{h+1} \in \MX$ with $\phist_h(x_h) = s_h$, we have
\begin{align*}
\BP_{h+1}(x_{h+1}\mid{}x_h,a_h) 
&= \BP_{h+1}(x_{h+1}\mid{}\xbar_h,a_h) \\ 
&= \langle e_{s_h,a_h}, \must_{h+1}(x_{h+1})\rangle \\ 
&= \langle \philin_h(x_h,a_h), \must_{h+1}(x_{h+1})\rangle
\end{align*}
where $\philin_h\in\Philin$ is defined by $\philin_h(x,a) := e_{\phist_h(x),a}$. The first equality above used the Generalized Block MDP assumption.
\end{proof}

The following proposition asserts that for Generalized Block MDPs, when viewed as a special case of Low-Rank MDPs, the one-context \emph{low-rank} regression problem reduces to one-context regression. Since this fact is only important for purposes of motivation and discussion, we omit the formal statement (i.e. parameters of the reduction), and simply remark that the reduction proceeds by independently solving a one-context regression problem for each action. The proof of correctness would proceed similarly to e.g. that of \cref{prop:twoaug}.

\begin{proposition}[Informal]\label{prop:gen-block-ocr}
For any concept class $\Phi \subseteq (\MX\to\MS)$ and action space $\MA$, let $\Philin \subseteq (\MX\times\MA\to\RR^d)$ be the feature class defined in terms of $\Phi$ as specified in \cref{prop:genblock-is-lowrank}. Then there is a polynomial-time reduction from one-context low-rank regression for $\Philin$ (\cref{def:low-rank-reg-app}) to one-context regression for $\Phi$ (\cref{Def:low-rank-reg}).  
\end{proposition}

For the results in \cref{sec:halfspace-hardness}, it is convenient to introduce the following class of $\Phi_n$-decodable Generalized Block MDPs, where $\Phi_n$ is the class of linear threshold functions, as previously defined in \cref{sec:lowrank}.

\begin{definition}\label{def:halfspace-mdps-apx}
Fix $n \in \NN$. We define $\MM_n$ to be the family of generalized $\Phi_n$-decodable block MDPs with horizon $H := (\log n)^{\log \log n}$, observation space $\MX = \RR^n$, latent state space $\MS = \{0,1\}$, action space $\MA = \{0,1\}$, and feature class $\Phi_n = \{\phi^\theta: \theta \in \RR^n\}$ consisting of linear threshold functions, i.e. where 
\[\phi^\theta(x) := \mathbbm{1}[\langle x,\theta\rangle \geq 0].\]
\end{definition}

The hardness result in \cref{thm:halfspace-main} will be against (a subset of) $\MM_n$. Observe that $H \leq n$ and $|\MA| = 2$, so to rule out an RL algorithm for $\Phi_n$-decodable Generalized Block MDPs with time complexity $\poly(n,H,|\MA|)$, it suffices to rule out an RL algorithm for $\MM_n$ with time complexity $\poly(n)$.

\subsection{Tractability of One-Context Regression}\label{sec:halfspace-ocr}

We start by showing that one-context regression with concept class $\Phi_n$ can be solved with time complexity $\poly(n)$. Note that $n$ is the description length of an element of the observation space $\MX$ (ignoring issues of finite precision arithmetic). To prove \cref{thm:halfspace-reg}, recall that PAC learning of halfspaces with random classification noise is computationally tractable \citep{blum1998polynomial,diakonikolas2023strongly}. The only difference between this problem and one-context regression for $\Phi_n$ is that in the latter setting, the noise levels for the two classes may be different. To reduce the latter to the former, we apply a careful symmetrization step; a priori this seems to require knowing the noise levels, but this can be avoided by gridding over all possibilities.

\begin{theorem}\label{thm:halfspace-reg}
There is a constant $C_{\ref{thm:halfspace-reg}}>0$ so that the following holds. There is an $\Nreg$-efficient algorithm for one-context regression with concept class $\Phi_n$, where
\[\Nreg(\epsilon,\delta) = (n/\epsilon)^{C_{\ref{thm:halfspace-reg}}} \log^2(1/\delta).\]
Moreover, the time complexity of the algorithm with error tolerance $\epsilon$ and failure probability $\delta$ is $\poly(n,1/\epsilon,1/\delta)$.
\end{theorem}

\begin{proof}
We define a one-context regression algorithm $\Alg$ as follows. We are given, as input, samples $(x^{(i)},y^{(i)})_{i=1}^m$ with $x^{(i)} \in \RR^n$ and $y^{(i)} \in \{0,1\}$, and parameters $\epsilon,\delta\in(0,1/2)$. Set $\epdisc := \cdisc \delta/m$ for a universal constant $\cdisc>0$ that will be determined in the analysis. Let $\MG$ be a set of functions $g:\{0,1\} \to [0,1]$ such that $|\MG| \leq O(1/\epdisc^2)$ and for any $(a,b) \in [0,1]^2$ there is some $g \in \MG$ with $|a-g(0)|+|b-g(1)| \leq \epdisc$.

For each $g \in \MG$, we compute a predictor $\MR_g: \RR^n \to [0,1]$ as follows. First, we define $(y_g^{(i)})_{i=1}^{m/2}$ as follows:
\begin{itemize}
    \item If $g(0)+g(1) \geq 1$, then draw
    \[y_g^{(i)} \sim \begin{cases} \Ber\left(\frac{g(0)+g(1)-1}{g(0)+g(1)}\right) & \text{ if } y^{(i)} = 0 \\ 
    \Ber(1) & \text{ if } y^{(i)} = 1 
    \end{cases}.\]
    \item If $g(0) + g(1) < 1$, then draw 
    \[y_g^{(i)} \sim \begin{cases} 
    \Ber(0) & \text{ if } y^{(i)} = 0 \\ 
    \Ber\left(\frac{1-g(0)-g(1)}{2 - g(0)+g(1)}\right) & \text{ if } y^{(i)} = 1 
    \end{cases}.\]
\end{itemize}
Let $\halfspacealg$ denote the halfspace learning algorithm guaranteed by \cite[Theorem 1.8]{diakonikolas2023strongly}. Invoke $\halfspacealg$ on dataset $(x^{(i)}, y_g^{(i)})_{i=1}^{m/2}$, which returns a hypothesis $h_g^+: \RR^n \to \{0,1\}$. Define $\MR_g^+: \RR^n \to [0,1]$ by $\MR_g^-(x) = g(h_g^+(x))$. Similarly define $h_g^-$ by invoking $\halfspacealg$ on $(x^{(i)}, 1-y_g^{(i)})_{i=1}^{m/2}$, and define $\MR_g^-: \RR^n \to [0,1]$ by $\MR_g^-(x) = g(h_g^-(x))$.

Finally, output $\MR_{\wh g}^{\wh b}$ with $\wh g, \wh b$ defined by
\[(\wh g, \wh b) := \argmin_{g \in \MG, b \in \{-,+\}} \sum_{i=m/2+1}^m (\MR_g^b(x^{(i)}) - y^{(i)})^2.\]

\paragraph{Analysis.} Let $m \geq (n/\epsilon)^{C_{\ref{thm:halfspace-reg}}} \log^2(1/\delta)$. Suppose that $(x^{(i)},y^{(i)})_{i=1}^m$ are i.i.d. samples with $\EE[y^{(i)}|x^{(i)}] = f(\phi(x^{(i)}))$ for some $\phi \in \Phi$ and $f: \{0,1\} \to [0,1]$. Suppose that $f(0) + f(1) \geq 1$; the argument in the other case is similar. By construction, there is some $g \in \MG$ such that $g(0) + g(1) \geq 1$ and $|f(0)-g(0)| + |f(1)-g(1)| \leq O(\epdisc)$. We distinguish two cases. 
\begin{enumerate}
    \item First, if $|f(0) - f(1)| \leq \epsilon$, then $|g(0) - g(1)| \leq O(\epsilon)$ (since $\epdisc \leq \epsilon$), and thus \[|\MR_g^+(x) - f(\phi(x))| \leq \max_{b,b'\in\{0,1\}} |g(b) - f(b')| \leq O(\epsilon)\] for all $x$.

\item Second, suppose that $|f(0) - f(1)| \geq \epsilon$. Further suppose that $f(0) > f(1)$. Fix $i \in [m/2]$ and condition on $x^{(i)}$. If $\phi(x^{(i)}) = 0$, then 
\begin{align*}
\Pr[y_g^{(i)} = 0]
&= \frac{1}{g(0) + g(1)}\Pr[y^{(i)} = 1] \\ 
&= \frac{f(0)}{g(0) + g(1)}.
\end{align*}
If $\phi(x^{(i)}) = 1$, then
\begin{align*}
\Pr[y_g^{(i)} = 1]
&= \frac{g(0)+g(1)-1}{g(0)+g(1)}\Pr[y^{(i)}=1] + \Pr[y^{(i)}=0] \\ 
&= \frac{g(0)+g(1)-1}{g(0)+g(1)}f(1) + 1 - f(1) \\ 
&= \frac{g(0) + g(1) - f(1)}{g(0) + g(1)}.
\end{align*}
In particular, for each $b \in \{0,1\}$, if $\phi(x^{(i)}) = b$, then
\[\left|\Pr[y_g^{(i)}=b] - \frac{f(0)}{f(0) + f(1)}\right| \leq O(\epdisc).\]
Consider the alternative (idealized) dataset $(x^{(i)}, \til y^{(i)})_{i=1}^{m/2}$ where 
\[\Pr[\til y^{(i)} = \phi(x^{(i)}) \mid{} x^{(i)}] = \frac{f(0)}{f(0) + f(1)}.\]
Then the total variation distance between $\Law((x^{(i)},y_g^{(i)})_{i=1}^{m/2})$ and $\Law((x^{(i)},\til y^{(i)})_{i=1}^{m/2})$ is at most $O(\epdisc m) \leq \delta$, where the inequality holds by definition of $\epdisc$, so long as $\cdisc>0$ is chosen sufficiently small. Moreover, $(x^{(i)},\til y^{(i)})_{i=1}^{m/2}$ is an instance of the halfspace learning problem with noise level $\eta := \frac{f(1)}{f(0) + f(1)}$.  Let $\til h^+: \RR^n \to \{0,1\}$ denote the output of $\halfspacealg$ on $(x^{(i)}, \til y^{(i)})_{i=1}^{m/2}$. By the guarantee of \cite[Theorem 1.8]{diakonikolas2023strongly}, so long as $C_{\ref{thm:halfspace-reg}}$ is a sufficiently large constant, we have that with probability at least $1-\delta$,
\[\Pr_{x,\til y}[\til h(x) \neq \til y] \leq \eta + \frac{\epsilon^2}{4},\]
where $(x,\til y)$ is a fresh sample from the same distribution. Now
\begin{align*}
\Pr_{x,\til y}[\til h(x) \neq \til y]
&= (1-\eta) \Pr_x[\til h(x) \neq \phi(x)] + \eta (1 - \Pr_x[\til h(x) \neq \phi(x)]) \\ 
&= \eta + (1-2\eta) \Pr_x[\til h(x) \neq \phi(x)].
\end{align*}
Observe that $1 - 2\eta \geq \epsilon/2$ (since $f(0) - f(1) \geq \epsilon$). Therefore it holds with probability at least $1-\delta$ that
\[\Pr_x[\til h(x) \neq \phi(x)] \leq \frac{\Pr_{x,\til y}[\til h(x) \neq \til y] - \eta}{1-2\eta} \leq \epsilon.\]
It follows from the preceding total variation bound and the data processing inequality that in an event $\ME_g^+$ that occurs with probability at least $1-2\delta$, the estimator $h_g^+$ (which our algorithm actually computes) satisfies
\[\Pr_x[h_g^+(x) \neq \phi(x)] \leq \epsilon.\]
Moreover, in event $\ME_g^+$, we have
\[\EE_x (\MR_g^+(x) - f(\phi(x)))^2 = \EE_x (g(h_g^+(x)) - f(\phi(x)))^2 \leq O(\epdisc^2) + \epsilon \leq O(\epsilon).\]
Recall that this holds under the assumption that $f(0) > f(1)$. If $f(0) < f(1)$, then by a similar argument there is an event $\ME_g^-$ that occurs with probability at least $1-2\delta$, in which 
\[\EE_x (\MR_g^-(x) - f(\phi(x)))^2 \leq O(\epsilon).\]
\end{enumerate}
In either case, there is an event $\ME$ that holds with probability at least $1-2\delta$, in which 
\[\min_{g \in \MG, b \in \{-,+\}} \EE_x(\MR_g^b(x) - f(\phi(x)))^2 \leq O(\epsilon).\]
Next, notice that \[m \geq \epsilon^{-2} \log(4m^2/(\cdisc^2 \delta^3)) = \epsilon^{-2} \log(4|\MG|/\delta)\]
so long as $C_{\ref{thm:halfspace-reg}}$ is a sufficiently large constant. Therefore by Hoeffding's inequality and a union bound over $\MG \times \{-,+\}$, there is an event $\ME'$ that occurs with probability at least $1-\delta$ in which, for all $g \in \MG$ and $b \in \{-,+\}$,
\[\left| \frac{2}{m}\sum_{i=m/2+1}^m (\MR_g^b(x^{(i)} - y^{(i)}))^2 - \EE_{x,y}(\MR_g^b(x) - y)^2\right| \leq \epsilon.\]
Condition on the event $\ME \cap \ME'$. Then
\begin{align*}
\EE_x (\MR_{\wh g}^{\wh b}(x) - f(\phi(x)))^2 
&= \EE_{x,y} (\MR_{\wh g}^{\wh b}(x) - y)^2 - \EE_{x,y} (f(\phi(x)) - y)^2 \\ 
&\leq \epsilon + \frac{2}{m}\sum_{i=m/2+1}^m (\MR_{\wh g}^{\wh b}(x^{(i)} - y^{(i)}))^2 - \EE_{x,y} (f(\phi(x)) - y)^2 \\ 
&\leq \epsilon + \min_{g\in\MG,b\in\{-,+\}}\frac{2}{m}\sum_{i=m/2+1}^m (\MR_{g}^{b}(x^{(i)} - y^{(i)}))^2 - \EE_{x,y} (f(\phi(x)) - y)^2 \\ 
&\leq 2\epsilon + \min_{g\in\MG,b\in\{-,+\}} \EE_{x,y} (\MR_g^b(x) - y)^2 - \EE_{x,y} (f(\phi(x)) - y)^2 \\
&= 2\epsilon + \min_{g\in\MG,b\in\{-,+\}} \EE_{x,y} (\MR_g^b(x) - f(\phi(x)))^2 \\ 
&\leq O(\epsilon).
\end{align*}
Rescaling $\epsilon,\delta$ by the appropriate constants yields the desired bound.
\end{proof}

\begin{remark}\label{remark:ocr-to-pac}
In fact, the reduction in \cref{thm:halfspace-reg} does not use any special properties of halfspaces. Thus, it shows that for \emph{any} binary concept class $\Phi$, one-context regression is polynomial-time reducible to PAC learning $\Phi$ with random classification noise.
\end{remark}

\subsection{Hardness of Reward-Free RL with Resets}\label{sec:halfspace-hardness}

Next, we prove \cref{thm:halfspace-rl-hard}, which asserts that reward-free RL in $\MM_n$ with the reset access model is computationally intractable under the \emph{Continuous LWE} assumption formally defined below.

Together with \cref{thm:halfspace-reg}, this completes the proof of \cref{thm:halfspace-main}. Throughout this section, we write $\fS^{t-1}$ to denote the set of Euclidean unit vectors in $\RR^t$.

\begin{assumption}[Continuous LWE \citep{bruna2021continuous}]\label{ass:clwe}
Let $\delta \in (0,1)$. For $\beta = \beta(t) := 1/t$ and $\gamma = \gamma(t) := 2\sqrt{t}$, for any time-$2^{t^\delta}$ algorithm $\Alg^\MO$ with sampling oracle $\MO$ and outputs in $[0,1]$, it holds that
\[\left|\EE_{w\sim\Unif(\fS^{t-1})} \EE[\Alg^{C_{w,\beta,\gamma}}] - \EE[\Alg^{N(0,\frac{1}{2\pi}I_t)}]\right| \leq t^{-\omega(1)},\]
where $C_{w,\beta,\gamma}$ is the \emph{continuous LWE distribution} with secret $w \in \fS^{t-1}$ and parameters $\beta,\gamma>0$.
\end{assumption}

Recent work has shown that this assumption is well-founded---in particular, it can be based on hardness of discrete LWE \citep{gupte2022continuous} and therefore on worst-case hardness of the approximate shortest vector problem \citep{brakerski2013classical}.

\begin{theorem}[One-context regression is insufficient for Low-Rank MDPs with resets]
  \label{thm:halfspace-rl-hard}
Suppose that $\Alg^M$ is an algorithm that, given interactive reset access to any MDP $M \in \MM_n$, has time complexity $\poly(n)$ and produces a set of policies $\Psi$ satisfying the following guarantee with probability at least $1/2$:
\begin{equation}\forall s \in \MS: \max_{\pi\in\Psi} d^{M,\pi}_H(s) \geq \frac{1}{\poly(|\MA|, |\MS|, H)} \left(\max_{\pi\in\Pi} d^{M,\pi}_H(s) - \frac{1}{8}\right).\label{eq:genblock-cover}\end{equation}
Then the Continuous LWE hardness assumption (\cref{ass:clwe}) is false. 
\end{theorem}

Since the generalized Block MDPs in $\MM_n$ have horizon $H \leq n$ and action space of size $O(1)$, this result rules out algorithms with time complexity $\poly(n,H,|\MA|)$, and therefore proves the claimed hardness result in \cref{thm:halfspace-main}. Also, the description length of each $x \in \MX = \RR^n$ is $O(n)$ (again, ignoring issues of bit complexity). Thus, \cref{thm:halfspace-rl-hard}, in conjunction with \cref{thm:halfspace-reg}, rules out any reduction from reward-free RL with resets (in Generalized Block MDPs) to one-context regression that is efficient, i.e. that has oracle time complexity $\poly(H,|\MS|,|\MA|,\epsilon^{-1},\delta^{-1},n)$.

\paragraph{Proof overview.} To prove \cref{thm:halfspace-rl-hard}, we need the following result, which we prove (in \sssref{subsec:halfspace-technical}) by unpacking various proofs from \cite{tiegel2023hardness}. Essentially, it states that there are two parametric families of distributions $\{\nu_{w,0}: w \in \fS^{t-1}\}$ and $\{\nu_{w,1}: w \in \fS^{t-1}\}$ that are each computationally indistinguishable from some null distribution under Continuous LWE (\cref{item:no-adv}), and yet for each $w$, the two corresponding distributions are approximately separated by a halfspace (\cref{item:class-error-bounds}).

\begin{theorem}\label{thm:ltf-distributions}
Suppose \cref{ass:clwe} holds. Let $\nlarge \in\NN$ and define $\nsmall := \frac{\log^2 \nlarge}{\log^5(\log \nlarge)}$. There is a distribution $\nunull$ and families $\{\theta(w): w \in \fS^{\nsmall-1}\}$, $\{\nu_{w,0}:w \in \fS^{\nsmall-1}\}$, $\{\nu_{w,1}: w \in \fS^{\nsmall-1}\}$ where $\theta(w) \in \RR^{\nlarge}$, $\nu_{w,0},\nu_{w,1} \in \Delta(\RR^{\nlarge})$ are distributions, and the following properties hold:
\begin{enumerate}
\item\label{item:no-adv} For every algorithm $\Alg^{\MO_0,\MO_1}$ with time complexity $\poly(\nlarge)$, access to sampling oracles $\MO_0,\MO_1$, and outputs in $[0,1]$,
\[\left|\EE_{w \sim \Unif(\fS^{\nsmall-1})} \EE[\Alg^{\nu_{w,0},\nu_{w,1}}] - \EE[\Alg^{\nunull,\nunull}]\right| \leq \nlarge^{-\omega(1)}.\]
\item\label{item:class-error-bounds} There is $\gamma(\nlarge) \leq (\log \nlarge)^{-\Omega(\log^2 \log \nlarge)}$ so that for every $w \in \fS^{\nsmall-1}$,
\[\Pr_{x \sim \nu_{w,1}}[\langle x, \theta(w) \rangle < 0] \leq \gamma(\nlarge)\]
and
\[\Pr_{x \sim \nu_{w,0}}[\langle x,\theta(w) \rangle \geq 0] \leq \gamma(\nlarge)\]
and
\[\Pr_{x\sim\nunull}[\langle x,\theta(w)\rangle \geq 0] \geq 1/2.\]

\item\label{item:eval-halfspace} There is a $\poly(\nlarge)$-time algorithm that takes as input $w \in \fS^{\nsmall-1}$ and $x \in \RR^{\nlarge}$, and outputs $\sgn(\langle x,\theta(w)\rangle)$.

\item\label{item:nu-samp} There is a $\poly(\nlarge)$-time algorithm that takes as input $w \in \fS^{\nsmall-1}$ and $a \in \{0,1\}$, and outputs $x \in \RR^{\nlarge}$ where $\Law(x)$ is $n^{-\omega(1)}$-close to $\nu_{w,a}$ in total variation distance.
\item\label{item:null-samp} There is a $\poly(\nlarge)$-time algorithm that samples from $\nunull$.
\end{enumerate}
\end{theorem}

We exploit these distribution families by designing a family of \emph{approximate} combination lock MDPs within $\MM_n$. Concretely, for each sequence of vectors $(w_1,\dots,w_H)$ and action sequence $(a^\st_1,\dots,a^\st_H)$ we design a $\Phi_n$-decodable MDP $M^{:H}_{\bw,\bast}$, defined formally below, for which playing an action sequence similar to $\bast$ is necessary in order to reach the latent state $\phi^{\theta(w_H)}(x_H) = 1$ with decent probability. This idea is formalized in \cref{lemma:best-policy,lemma:expert-agreement}. We then use the indistinguishability property of the distributions families (\cref{item:no-adv}) together with a careful hybrid argument to prove that any computationally efficient RL algorithm has similar behavior on $M^{:H}_{\bw,\bast}$ as on a ``null MDP'' that is independent of $\bw$ and $\bast$. In particular,  \cref{lemma:hybrid} shows that the transitions of $M^{:H}_{\bw,\bast}$ can be replaced with ``null'' transitions layer-by-layer, starting with the last layer, and no efficient algorithm can detect the difference under \cref{ass:clwe}. We then complete the proof of \cref{thm:halfspace-rl-hard} by observing that an RL algorithm which is independent of $\bast$ cannot always play a policy similar to $\bast$.

\begin{definition}[MDPs for hardness construction]
Set $H := (\log n)^{\frac{1}{3}\log \log n}$. For any $\bw = (w_1,\dots,w_H) \in (\fS^{t-1})^H$ and $\bast = (a^\st_1,\dots,a^\st_H) \in \MA^H$ and $k \in [H]$, we define an MDP $M_{\bw,\ba}^{:k}$ with observation space $\MX=\RR^n$, action space $\MA=\{0,1\}$, and horizon $H = (\log n)^{\log \log n}$ as follows. The initial distribution is $\nunull$. For each $x_h \in \MX$ and $a_h \in \MA$ and $h \in [H-1]$, the transition distribution is
\[\BP_{h+1}^{M_{\bw,\ba}^{:k}}(x_{h+1}|x_h,a_h) := \begin{cases} 
\nu_{w_{h+1}, 1}(x_{h+1}) & \text{ if } \langle x_h,\theta(w_h)\rangle \geq 0 \text{ and } a_h = a^\st_h \text{ and } h < k\\ 
\nu_{w_{h+1}, 0}(x_{h+1}) & \text{ if } (\langle x_h,\theta(w_h)\rangle<0 \text{ or } a_h \neq a^\st_h) \text{ and } h < k \\
\nunull & \text{ if } h \geq k
\end{cases}.
\]
\end{definition}

\begin{fact}
For any $\bw=w_{1:H} \in \WH$, $\bast =a^\st_{1:H}\in \MA^H$, and $k \in [H]$, it holds that $M_{\bw,\ba}^{:k} \in \MM_n$.
\end{fact}

The following lemma shows that the fixed action sequence $\bast$ is a policy that reaches the latent state $\phi^{\theta(w_H)}(x_H) = 1$ with constant probability.

\begin{lemma}\label{lemma:best-policy}
For any $\bw=w_{1:H} \in \WH$ and $\bast =a^\st_{1:H}\in \MA^H$, 
\[\max_{\pi\in\Pi} \Pr^{M_{\bw,\ba}^{:H},\pi}[\phi^{\theta(w_H)}(x_H) = 1] \geq \frac{1}{4}\]
so long as $n$ is sufficiently large.
\end{lemma}

\begin{proof}
For notational convenience, write $M := M_{\bw,\bast}^{:H}$. Consider the policy $\pi$ defined by $\pi_h(x_h) := a^\st_h$ for all $h \in [H]$ and $x_h \in \MX$. The initial distribution of $M$ is $\nunull$, so $\Pr^{M,\pi}[\phi^{\theta(w_1)}(x_1) = 1] \geq 1/2$ by \cref{item:class-error-bounds} of \cref{thm:ltf-distributions}. For each $h < H$, conditioned on $x_h\in\MX$ with $\phi^{\theta(w_h)}(x_h) = 1$, the distribution of $x_{h+1}$ under policy $\pi$ is $\nu_{w_{h+1},1}$, so
\[\Pr^{M,\pi}[\phi^{\theta(w_{h+1})}(x_{h+1}) = 1 \mid{} \phi^{\theta(w_h)}(x_h) = 1] \geq 1 - \gamma(n).\]
Therefore by induction, 
\[\Pr^{M,\pi}[\phi^{\theta(w_H)}(x_H) = 1] \geq \frac{1}{2}(1-\gamma(n))^{H-1} \geq \frac{1}{4}\]
where the final inequality holds for all sufficiently large $n$, since $H = H(n) = o(\gamma(n))$.
\end{proof}

The following lemma shows that any policy that reaches $\phi^{\theta(w_H)}(x_H) = 1$ with decent probability must also often play the action sequence $\bast$.

\begin{lemma}\label{lemma:expert-agreement}
For any $\bw=w_{1:H} \in \WH$,  $\bast =a^\st_{1:H}\in \MA^H$, and policy $\pi\in\Pi$,
\[\Pr^{M^{:H}_{\bw,\bast},\pi}[\forall h \in [H-1]: \pi_h(x_h) = a^\st_h] \geq \Pr^{M_{\bw,\bast}^{:H},\pi}\left[\phi^{\theta(w_H)}(x_H) = 1\right] - H^2 \gamma(n).\]
\end{lemma}

\begin{proof}
For notational convenience, write $M = M^{:H}_{\bw,\bast}$. For each $h \in [H-1]$ let $\ME_h$ be the event that $\pi_h(x_h) \neq a^\st_h$. Then
\[\Pr\left[\left(\phi^{\theta(w_H)}(x_H) = 1\right) \land \left(\pi_h(x_h)\neq a^\st_h\right)\right] \leq \Pr\left[\phi^{\theta(w_H)}(x_H)=1 \middle| \pi_h(x_h)\neq a^\st_h\right] \leq H\gamma(n)\]
since conditioned on any $x_h \in \MX$ with $\pi_h(x_h) \neq a^\st_h$, $x_{h+1}$ is distributed according to $\nu_{w_{h+1},0}$, and thus $\Pr^{M,\pi}[\phi^{\theta(w_{h+1})}(x_{h+1})=1 \mid{}x_h] \leq \gamma(n)$ (\cref{item:class-error-bounds} of \cref{thm:ltf-distributions}), and at each subsequent step $k \geq h+1$, conditioned on any $x_k \in \MX$ with $\phi^{\theta(w_k)}(x_k) = 0$, $x_{k+1}$ is distributed according ot $\nu_{w_k,0}$ so once again $\Pr^{M,\pi}[\phi^{\theta(w_{k+1})}(x_{k+1})=1 \mid{}x_k] \leq \gamma(n)$. 

By the union bound,
\begin{align*}
\Pr\left[\left(\phi^{\theta(w_H)}(x_H) = 1\right) \land \left(\exists h\in[H-1]: \pi_h(x_h)\neq a^\st_h\right)\right] 
&\leq H^2\gamma(n).
\end{align*}
Therefore
\begin{align*}
&\Pr^{M,\pi}[\forall h \in [H-1]: \pi_h(x_h) = a^\st_h] \\ 
&\geq \Pr^{M,\pi}\left[\left(\phi^{\theta(w_H)}(x_H) = 1\right) \land \left(\forall h\in[H-1]: \pi_h(x_h)= a^\st_h\right)\right] \\
&\geq \Pr^{M,\pi}\left[\phi^{\theta(w_H)}(x_H) = 1\right] - H^2\gamma(n)
\end{align*}
as needed.
\end{proof}

We now implement a hybrid argument to show that any efficient RL algorithm has similar behavior on $M^{:1}_{\bw,\bast},\dots,M^{:H}_{\bw,\bast}$. In particular, we need that the \emph{value} of its output is similar on these different MDPs, where we define value as the maximum agreement probability (over all policies output by the algorithm) with the true action sequence $\bast$: 

\begin{definition}
For an MDP $M$, set of policies $\Psi$, and action sequence $\bast$, define 
\[\Val^M_{\bast}(\Psi) = \max_{\pi\in\Psi} \Pr^{M,\pi}[\forall h \in [H-1]: \pi_h(x_h) = a^\st_h].\]
\end{definition}

\begin{lemma}\label{lemma:hybrid}
Let $\Alg^M$ be a $\poly(n)$-time algorithm that uses interactive reset access to an MDP $M \in \MM_n$ and outputs a set of policies. Fix $\bast \in \MA^H$ and $k \in [H-1]$. Then
\[\left| \EE_{\bw\sim\Unif(\fS^{t-1})^{\otimes H}}\EE[\Val^{M_{\bw}}_{\bast}(\Alg^{M_{\bw}})] - \EE_{\bw\sim\Unif(\fS^{t-1})^{\otimes H}}\EE[\Val^{M'_{\bw}}_{\bast}(\Alg^{M'_{\bw}})]\right| \leq O(n^{-5}).\]
where $M_{\bw} = M^{:k}_{\bw,\bast}$ and $M'_{\bw} = M^{:k+1}_{\bw,\bast}$.
\end{lemma}

\begin{proof}
We define an algorithm $\Algbar^{\MO_0,\MO_1}$ with access to sampling oracles $\MO_0,\MO_1$ as follows. Sample $\bw = (w_1,\dots,w_H)$ from $\Unif(\fS^{t-1})^{\otimes H}$. Then simulate $\Alg$ with RL-with-resets access to an MDP $M^{\MO_0,\MO_1}_{w_{1:k}}$ defined as follows. First, define an initial sampling subroutine that samples from $\nunull$ (this can be implemented in time $\poly(n)$ by \cref{item:null-samp}). Next, define a conditional sampling subroutine that, given $h$, $x_h$, and $a_h$, samples $x_{h+1}$ according to the following distribution:
\[\BP(x_{h+1}|x_h,a_h) = \begin{cases} 
\nu_{w_{h+1}, 1}(x_{h+1}) & \text{ if } \phi^{\theta(w_h)}(x_h) = 1 \text{ and } a_h = a^\st_h \text{ and } h < k\\ 
\nu_{w_{h+1}, 0}(x_{h+1}) & \text{ if } (\phi^{\theta(w_h)}(x_h) = 0 \text{ or } a \neq a^\st_h) \text{ and } h < k \\
\MO_0(x_{h+1}) & \text{ if } \phi^{\theta(w_h)}(x_h) = 1 \text{ and } a_h = a^\st_h \text{ and } h = k \\ 
\MO_1(x_{h+1}) &\text{ if } (\phi^{\theta(w_h)}(x_h) = 0 \text{ or } a_h \neq a^\st_h) \text{ and } h = k \\
\nunull & \text{ if } h > k
\end{cases}.
\]
This sampler can be implemented in time $\poly(n)$, up to total variation error $n^{-\omega(1)}$: in particular, for any $x_h \in \MX$ and $w_h \in \fS^{t-1}$, we can compute $\sgn(\langle x_h,\theta(w_h)\rangle)$ in time $\poly(n)$ (\cref{item:eval-halfspace}) and hence evaluate $\phi^{\theta(w_h)}(x_h) = \mathbbm{1}[\langle x_h, \theta(w_h)\rangle \geq 0]$; moreover, we can sample from $\nu_{w_{h+1},0}$ and $\nu_{w_{h+1},1}$ for any given $w_{h+1} \in \fS^{t-1}$ (\cref{item:nu-samp}) up to total variation error $n^{-\omega(1)}$, we can sample from $\nunull$ (\cref{item:null-samp}), and we are given access to $\MO_0,\MO_1$. 

Let $\Psi$ be the output of $\Alg$ after interaction with this MDP. Next, $\Algbar^{\MO_0,\MO_1}$ computes an estimate $\wh \Val$ of $\Val^{M^{\MO_0,\MO_1}_{w_{1:k}}}_{\bast}(\Psi)$ by enumerating over $\Psi$ and performing Monte Carlo estimation: note that we can draw trajectories from this MDP, and we know $\bast$. By drawing $N = n^{10}$ trajectories, we can guarantee that conditioned on all prior randomness,
\[\EE \left|\wh \Val - \Val^{M^{\MO_0,\MO_1}_{w_{1:k}}}_{\bast}(\Psi)\right| \leq O(1/\sqrt{N}).\]
Finally, $\Algbar^{\MO_0,\MO_1}$ output $\wh \Val$.

\paragraph{Analysis.} Suppose $(\MO_0,\MO_1) = (\nunull,\nunull)$. If the conditional sampler defined above had no error, then it would hold for any choice of $\bw=w_{1:H}$ that $M^{\MO_0,\MO_1}_{w_{1:k}}$ is exactly $M^{:k}_{\bw,\bast}$. Since the error is $n^{-\omega(1)}$ per sample and $\Alg$ has time complexity $\poly(n)$, this error affects $\EE[\Algbar^{\nunull,\nunull}]$ by at most $n^{-\omega(1)}$. Therefore
\begin{align*}
\left|\EE[\Algbar^{\nunull,\nunull}]
- \EE_{\bw\sim\Unif(\fS^{t-1})^{\otimes H}}\EE[\Val^{M^{:k}_{\bw,\bast}}_{\bast}(\Alg^{M^{:k}_{\bw,\bast}})]\right| \leq O(1/\sqrt{N}),
\end{align*}
where the error due to Monte Carlo estimation dominates the error in the conditional sampler. 

Next, for any $w^\st_{k+1} \in\fS^{t-1}$, suppose that $(\MO_0,\MO_1) = (\nu_{w^\st_{k+1},0},\nu_{w^\st_{k+1},1})$. If the conditional sampler had no error, then it would hold for any choice of $\bw$ that $M^{\MO_0,\MO_1}_{w_{1:k}}$ is the same as $M^{:k+1}_{\bw^\st,\bast}$ where $\bw^\st = (w_1,\dots,w_k,w^\st_{k+1},w_{k+2},\dots,w_H)$. Again the error is at most $n^{-\omega(1)}$ per sample, and hence affects $\EE[\Algbar^{\nu_{w^\st_{k+1},0},\nu_{w^\st_{k+1},1}}]$ by at most $n^{-\omega(1)}$. Therefore
\begin{align*}
&\left|\EE_{w^\st_{k+1}\sim\Unif(\fS^{t-1})}\EE[\Algbar^{\nu_{w^\st_{k+1},0},\nu_{w^\st_{k+1},1}}]
- \EE_{w^\st_{k+1},w_1,\dots,w_H\sim\Unif(\fS^{t-1})}\EE[\Val^{M^{:k+1}_{\bw^\st,\bast}}_{\bast}(\Alg^{M^{:k+1}_{\bw^\st,\bast}})]\right| \\
&\leq \EE_{w^\st_{k+1}\sim\Unif(\fS^{t-1})} \left|\EE[\Algbar^{\nu_{w^\st_{k+1},0},\nu_{w^\st_{k+1},1}}]
- \EE_{w_1,\dots,w_H\sim\Unif(\fS^{t-1})}\EE[\Val^{M^{:k+1}_{\bw^\st,\bast}}_{\bast}(\Alg^{M^{:k+1}_{\bw^\st,\bast}})]\right| \\
&\leq O(1/\sqrt{N}).
\end{align*}
But $\Algbar$ has time complexity $\poly(n)$, so by \cref{item:no-adv},
\[\left|\EE_{w^\st_{k+1}\sim\Unif(\fS^{t-1})}\EE[\Algbar^{\nu_{w^\st_{k+1},0},\nu_{w^\st_{k+1},1}}]
- \EE[\Algbar^{\nunull,\nunull}]\right| \leq n^{-\omega(1)}.\]
The result follows by the triangle inequality.
\end{proof}

We can now put together \cref{lemma:hybrid} (applied for all $1 \leq k \leq H-1$) with \cref{lemma:best-policy,lemma:expert-agreement} and the coverage guarantee \cref{eq:genblock-cover} assumed in the theorem statement.
\vspace{1em}
\begin{proof}[Proof of \cref{thm:halfspace-rl-hard}]
For any $\bw \in (\fS^{t-1})^H$ and $\ba \in \MA^H$, let $\Psi^{:k}_{\bw,\bast}$ denote the (random) output of $\Alg^{M^{:k}_{\bw,\bast}}$. By the theorem assumption and the fact that $|\MA|=|\MS|=2$, we know that for any $\bw,\bast$, with probability at least $1/2$ over the execution of $\Alg^{M^{:H}_{\bw,\bast}}$, there is some $\pihat\in \Psi^{:H}_{\bw,\bast}$ such that 
\[\Pr^{M^{:H}_{\bw,\bast},\pihat}[\phi^{\theta(w_H)}(x_H)=1] \geq \frac{1}{\poly(H)}\left(\max_{\pi\in\Pi} \Pr^{M^{:H}_{\bw,\ba},\pi}[\phi^{\theta(w_H)}(x_H)=1] - \frac{1}{8}\right) \geq \frac{1}{\poly(H)}\]
where the second inequality is by \cref{lemma:best-policy}. In this event, by \cref{lemma:expert-agreement}, the policy $\pihat$ satisfies
\[\Pr^{M^{:H}_{\bw,\bast},\pihat}[\forall h \in [H-1]: \pihat_h(x_h) = a^\st_h] \geq \frac{1}{\poly(H)} - H^2 \gamma(n) \geq \frac{1}{\poly(H)}\]
where the second inequality holds for sufficiently large $n$, and uses the fact that $H = (\log n)^{\log \log n}$ whereas $\gamma(n) = (\log n)^{-\Omega(\log^2 \log n)}$.
Therefore we have
\[\EE[\Val^{M^{:H}_{\bw,\bast}}_{\bast}(\Psi^{:H}_{\bw,\bast})] \geq \frac{1}{\poly(H)}\]
where the expectation is over the execution of $\Alg^{M^{:H}_{\bw,\bast}}$.
Since this bound holds for any fixed $\bw$, it also holds in expectation over $\bw \sim \Unif(\fS^{t-1})^{\otimes H}$. By \cref{lemma:hybrid}, we get that for any $\bast \in \MA^H$,
\begin{equation} \EE_{\bw \sim \Unif(\fS^{t-1})^{\otimes H}} \EE[\Val^{M^{:1}_{\bw,\bast}}_{\bast}(\Psi^{:1}_{\bw,\bast})] \geq \frac{1}{\poly(H)} - O(H \cdot n^{-5}) \geq \frac{1}{\poly(H)}
\label{eq:val-lb}
\end{equation}
since $H = (\log n)^{\log \log n} = n^{o(1)}$. But the initial distribution and transition distributions of $M^{:1}_{\bw,\bast}$ are independent of $\bast$ (and $\bw$), i.e. we can write $M^{:1}_{\bw,\bast} = M$ for a fixed MDP $M$. It follows that the random variables $\{\Psi^{:1}_{\bw,\ba}: \ba \in \MA^H\}$ are identically distributed. So for any $\bw \in (\fS^{t-1})^H$ and $\bast \in \MA^H$,
\begin{align*}
\sum_{\ba \in \MA^H} \EE[\Val^{M^{:1}_{\bw,\ba}}_{\ba}(\Psi^{:1}_{\bw,\ba})]
&= \sum_{\ba \in \MA^H} \EE[\Val^{M^{:1}_{\bw,\ba}}_{\ba}(\Psi^{:1}_{\bw,\bast})] \\ 
&= \sum_{\ba \in \MA^H} \EE[\Val^{M}_{\ba}(\Psi^{:1}_{\bw,\bast})] \\ 
&= \EE\left[\sum_{\ba \in \MA^H} \max_{\pi \in \Psi^{:1}_{\bw,\bast}}\Pr^{M, \pi}[\forall h \in [H-1]: \pi_h(x_h) = a_h]\right] \\ 
&\leq \EE\left[\sum_{\ba \in \MA^H} \sum_{\pi \in \Psi^{:1}_{\bw,\bast}}\Pr^{M, \pi}[\forall h \in [H-1]: \pi_h(x_h) = a_h]\right] \\ 
&\leq \EE[|\Psi^{:1}_{\bw,\bast}|] \\ 
&\leq \poly(n)
\end{align*}
since the size of the output of $\Alg$ is bounded by its runtime. It follows that
\[\EE_{\bw\sim\Unif(\fS^{t-1})^{\otimes H}}\sum_{\ba \in \MA^H} \EE[\Val^{M^{:1}_{\bw,\ba}}_{\ba}(\Psi^{:1}_{\bw,\ba})] \leq \poly(n)\]
and thus there is some $\ba \in \MA^H$ with
\[\EE_{\bw\sim\Unif(\fS^{t-1})^{\otimes H}}\EE[\Val^{M^{:1}_{\bw,\ba}}_{\ba}(\Psi^{:1}_{\bw,\ba})] \leq \frac{\poly(n)}{2^H} \leq n^{-\omega(1)}\]
where the final inequality uses that $H = \omega(\log n)$. But this contradicts \cref{eq:val-lb}.
\end{proof}

\subsubsection{Technical Results Regarding CLWE and PTFs}\label{subsec:halfspace-technical}

To prove \cref{thm:ltf-distributions}, we follow the strategy used by \cite{tiegel2023hardness}: we first prove an analogous result for polynomial threshold functions (\cref{lemma:ptf-distributions}, below), and then lift to linear threshold functions (albeit suffering a blowup in the dimension) via the Veronese mapping. To reiterate, these results essentially follow from inspecting various statements and proofs from \cite{tiegel2023hardness}.

\begin{lemma}\label{lemma:ptf-distributions}
Let $\nsmall,\ell \in\NN$ with $\ell \geq 2\sqrt{\nsmall}$, and $\delta \in (0,1)$. Define $\MDnull := N(0, \frac{1}{2\pi} I_{\nsmall})$. There is a family of degree-$\ell$ PTFs $\{f_w: \RR^{\nsmall} \to \{-1,1\}\}$ and two families of distributions $\{\MD_w^+\}$ and $\{\MD_w^-\}$ over $\RR^{\nsmall}$, all indexed by $w \in \RR^{\nsmall}$, with the following properties:
\begin{enumerate}
\item\label{item:indist} Under \cite[Assumption 3.4]{tiegel2023hardness}, for every algorithm $\Alg^{\MO_0,\MO_1}$ with time complexity $2^{\nsmall^\delta}$, access to sampling oracles $\MO_0,\MO_1$, and outputs in $[0,1]$,
\[\left|\EE_{w \sim \Unif(\fS^{\nsmall-1})} \EE[\Alg^{\MD_w^+,\MD_w^-}] - \EE[\Alg^{\MDnull,\MDnull}]\right| \leq 2^{-\nsmall^\delta}.\]
\item For every $w \in \fS^{\nsmall-1}$,
\begin{equation} \Pr_{x \sim \MD_w^+}[f_w(x) \neq 1] \leq \exp(-\Omega(\ell^2/\nsmall))\label{eq:misspec-plus}\end{equation}
and
\begin{equation} \Pr_{x \sim \MD_w^-}[f_w(x) \neq -1] \leq \exp(-\Omega(\ell^2/\nsmall))\label{eq:misspec-minus}\end{equation}
and
\begin{equation} \Pr_{x \sim \MDnull}[f_w(x) = 1] \geq 1/2.\label{eq:null-bias}
\end{equation}
\item\label{item:samp} There is a $\poly(\nsmall)$-time algorithm that takes as input $w \in \fS^{\nsmall-1}$ and $a \in \{-,+\}$, and outputs $(x,f_w(x))$ with $\TV(\Law(x),\MD_w^a) \leq 2^{-\Omega(t)}$.
\end{enumerate}
\end{lemma}

\begin{proof}
Set $\beta = 1/t$ and $\gamma = 2\sqrt{t}$. For each $w \in \fS^{t-1}$, let $\MD_w^+$ be the distribution $\NH_{w,\beta,\gamma,0}$, and let $\MD_w^-$ be the distribution $\NH_{w,\beta,\gamma,1/2}$, where both of these \emph{non-overlapping homogeneous CLWE} distributions are defined in \cite[Definition 3.3]{tiegel2023hardness}. By \cite[Lemma 4.3]{tiegel2023hardness}, for each $w \in \fS^{t-1}$ there is some degree-$\ell$ polynomial threshold function $f_w: \RR^t \to \{-1,1\}$ such that 
\[\frac{1}{2}\Pr_{x \sim \MD_w^+}[f_w(x) \neq 1] + \frac{1}{2}\Pr_{x \sim \MD_w^-}[f_w(x) \neq 1] \leq \exp(-\Omega(\ell^2/t)),\]
which proves \cref{eq:misspec-plus,eq:misspec-minus}. Moreover, inspecting the proof of \cite[Lemma 4.3]{tiegel2023hardness} we see that $f_w(x) = \sgn(p(\langle w,x\rangle))$ for a fixed degree-$\ell$ polynomial $p: \RR \to \RR$ (i.e. depending only on $t$, $\ell$, $\beta$, $\gamma$). Since the distribution $\MDnull$ is radially symmetric, the distribution of $\langle w,x\rangle$ under $x \sim \MDnull$ does not depend on $w$, so $\Pr_{x \sim \MDnull}[f_w(x) = 1] = r$ for some constant $r \in [0,1]$. If $r \geq 1/2$ then \cref{eq:null-bias} holds for all $w \in \fS^{t-1}$. If $r < 1/2$, then we can simply swap the definitions of $\MD_w^+$ and $\MD_w^-$ for all $w$, and replace $f_w$ with $-f_w$ for all $w$. This preserves \cref{eq:misspec-plus,eq:misspec-minus}, and after this swap \cref{eq:null-bias} is now satisfied for all $w \in \fS^{t-1}$.

Since the sign pattern of $p$ is explicit (in terms of $t$, $\ell$, $\beta$, and $\gamma$), for any $w \in \fS^{t-1}$ and $x \in \RR^t$ we can efficiently compute $\sgn(p(\langle w,x\rangle)) = f_w(x)$. To prove \cref{item:samp} it remains to show that given $w \in \fS^{t-1}$ we can approximately sample from $\NH_{w,\beta,\gamma,0}$ and $\NH_{w,\beta,\gamma,1/2}$. By \cite[Lemma 6.2]{tiegel2023hardness}, for any $c \in [0,1)$, given a sampling oracle for the CLWE distribution $C_{w,\beta,\gamma}$ there is a $\poly(t)$-time algorithm for sampling from a distribution that is $2^{-t}$-close to $\HH_{w,\beta,\gamma,c}$ in total variation distance, where $\HH_{w,\beta,\gamma,c}$ is the homogeneous CLWE distribution (\cite[Definition 3.2]{tiegel2023hardness}) and $C_{w,\beta,\gamma}$ is the CLWE distribution (\cite[Definition 3.1]{tiegel2023hardness}). But by definition, there is an explicit, $\poly(t)$-time algorithm for sampling from $C_{w,\beta,\gamma}$ (assuming the ability to sample from a standard Gaussian distribution). Moreover, \cite[Lemma A.1]{tiegel2023hardness} shows that $\TV(\NH_{w,\beta,\gamma,c},\HH_{w,\beta,\gamma,c}) \leq 4 \cdot \exp(-\frac{1}{100\beta^2}) \leq 2^{-\Omega(t)}$ by choice of $\beta$. Thus, given $w \in \fS^{t-1}$, we can sample from a distribution that is $2^{-\Omega(t)}$-close to $\NH_{w,\beta,\gamma,c}$ in total variation distance, in time $\poly(t)$. This proves \cref{item:samp}.

By \cite[Theorem 6.1]{tiegel2023hardness}, it holds under \cite[Assumption 3.4]{tiegel2023hardness} that there is no constant $c>0$ and $2^{t^\delta}$-time algorithm $\Alg^\MO$ with outputs in $[0,1]$ and
\[\left|\EE_{w \sim \Unif(\fS^{t-1})}[\Alg^{\frac{1}{2}(D_w^+, 1) + \frac{1}{2}(D_w^-, -1)}] - \EE[\Alg^{N(0, \frac{1}{2\pi}I_t)\times \Ber(1/2)}]\right| \geq \Omega(n^{-c}).\]
In fact, by a standard boosting argument, the advantage can be driven down to $2^{-t^\delta}$. Moreover, since samples from $D_w^+$ and $D_w^-$ can be individually extracted using a sampling oracle for $\frac{1}{2}(D_w^+,1)+\frac{1}{2}(D_w^-,-1)$, it follows that there is no $2^{t^\delta}$-time algorithm $\Alg^{\MO_0,\MO_1}$ with outputs in $[0,1]$ and
\[\left|\EE_{w \sim \Unif(\fS^{t-1})}[\Alg^{D_w^+,D_w^-}] - \EE[\Alg^{N(0, \frac{1}{2\pi}I_t),N(0, \frac{1}{2\pi}I_t)}]\right| \geq \Omega(2^{-t^\delta}).\]
This proves \cref{item:indist}.
\end{proof}

\begin{proof}[Proof of \cref{thm:ltf-distributions}]
Invoke \cref{lemma:ptf-distributions} with $t = \frac{\log^2 n}{\log^4(\log n)}$ and $\ell = \frac{\log n}{3\log \log n} \geq 2\sqrt{t}$, where the inequality holds for sufficiently large $n$. For each $w \in \fS^{t-1}$, let $\MD_w^+,\MD_w^- \in \Delta(\RR^t)$ be the distributions and let $f_w: \RR^t \to \{-1,1\}$ be the $2\sqrt{t}$-PTF given by the lemma. Let $\tau: \RR^t \to \RR^n$ be defined as follows. Identify the first $t^0 + t^1 + \dots + t^\ell$ coordinates of $[n]$ with sequences $\alpha = (\alpha_1,\dots,\alpha_k) \in [t]^k$ of length at most $\ell$, and for $x \in \RR^t$ and any such sequence $\alpha$, define
\[\tau(x)_\alpha := \prod_{i=1}^{|\alpha|} x_i.\]
Define the remaining $n - (t^0+\dots+t^\ell)$ coordinates of $\tau(x)$ to be $0$. Note that \[t^0+\dots+t^\ell \leq 2t^\ell \leq 2(\log^2 n)^{\log(n)/(3\log\log n)} \leq 2n^{2/3}\]
so this map is well-defined for sufficiently large $n$. Next, observe that for any polynomial $q: \RR^t \to \RR$ of degree at most $\ell$, there is some $\theta \in \RR^n$ such that $q(x) = \langle \tau(x), \theta\rangle$ for all $x \in \RR^t$. Accordingly, for each $w \in \fS^{t-1}$, define $\theta(w) \in \RR^n$ so that $f_w(x) = \sgn(\langle \tau(x), \theta(w)\rangle)$ for all $x \in \RR^t$. Let $\nu_{w,1}$ be the distribution of $\tau(x)$ under $x \sim \MD_w^+$, and let $\nu_{w,0}$ be the distribution of $\tau(x)$ under $x \sim \MD_w^-$. Let $\nunull$ be the distribution of $\tau(x)$ under $x \sim \MDnull$. We now check the theorem claims:
\begin{enumerate}
\item Let $\Alg^{\MO_0,\MO_1}$ be an $\poly(n)$-time algorithm with access to sampling oracles $\MO_0,\MO_1 \in \Delta(\RR^n)$. We define an algorithm $\Algbar^{\MObar_0,\MObar_1}$ with access to sampling oracles $\MObar_0,\MObar_1 \in \Delta(\RR^t)$ as follows. Simulate $\Alg$; when it queries $\MO_0$, draw $x \sim \MObar_0$ and pass $\tau(x)$; when it queries $\MO_1$, draw $x \sim \MObar_1$ and pass $\tau(x)$. Finally, output the final output of $\Alg$.

Since $\tau$ can be evaluated in time $\poly(n)$, the time complexity of $\Algbar$ is $\poly(n) = 2^{O(\log n)} \leq 2^{O(t^{2/3})}$. Moreover, the distribution of $\Algbar^{\MDnull,\MDnull}$ is exactly the distribution of $\Alg^{\nunull,\nunull}$, and for each $w \in \fS^{t-1}$, the distribution of $\Algbar^{\MD_w^+,\MD_w^-}$ is exactly the distribution of $\Alg^{\nu_{w,1},\nu_{w,0}}$. Thus, \cref{item:indist} of \cref{lemma:ptf-distributions}, applied to $\Algbar$, implies that
\[\left|\EE_{w \sim \Unif(\fS^{\nsmall-1})} \EE[\Alg^{\nu_{w,0},\nu_{w,1}}] - \EE[\Alg^{\nunull,\nunull}]\right| \leq 2^{-t^{2/3}}.\]
By definition of $t$, this bound decays super-polynomially in $n$.

\item Fix $w \in \fS^{t-1}$. We have
\begin{align*}
\Pr_{x \sim \nu_{w,1}}[\langle x,\theta(w)\rangle < 0]
&= \Pr_{x \sim \MD_w^+}[\langle \tau(x), \theta(w)\rangle < 0] \\ 
&= \Pr_{x \sim \MD_w^+}[f_w(x) = -1] \\ 
&\leq \exp(-\Omega(\ell^2/t)) \\ 
&\leq \exp(-\Omega(\log^2 \log n)) \\
&\leq (\log n)^{-\Omega(\log \log n)}
\end{align*}
by \cref{eq:misspec-plus} and the definitions of $\ell$ and $t$. A symmetric argument with $\MD_w^-$ and \cref{eq:misspec-minus} shows that
\[\Pr_{x \sim \nu_{w,0}}[\langle x,\theta(w) \rangle \geq 0] \leq (\log n)^{-\Omega(\log \log n)}.\]
Finally,
\begin{align*}
\Pr_{x \sim \nunull}[\langle x,\theta(w)\rangle \geq 0]
&= \Pr_{x \sim \MDnull}[\langle \tau(x), \theta(w)\rangle \geq 0] \\ 
&= \Pr_{x \sim \MDnull}[f_w(x) = 1] \\ 
&\geq 1/2
\end{align*}
by \cref{eq:null-bias}.

\item Given $w \in \fS^{t-1}$, we can sample $x \in \RR^t$ with $\TV(\Law(x),\MD_w^+) \leq 2^{-\Omega(t)}$ in time $\poly(t)$ (by \cref{item:samp}), and we can then compute $\tau(x)$ in time $\poly(n)$. By the data processing inequality, $\TV(\Law(\tau(x)), \nu_{w,1}) \leq 2^{-\Omega(t)} \leq n^{-\omega(1)}$. Moreover, we can compute $f_w(x) = \sgn(\langle \tau(x),\theta(w)\rangle)$ in time $\poly(t)$ by \cref{item:samp}. The same argument works for $\nu_{w,0}$.

\item Since $\MDnull = N(0,\frac{1}{2\pi}I_t)$, we can efficiently sample $x \sim \MDnull$ and then compute $\tau(x)$, which has distribution $\nunull$ by definition.
\end{enumerate}
This completes the proof.
\end{proof}

\subsection{Statistical Tractability of Generalized Block MDPs}\label{sec:halfspace-statistical}

In this section we prove \cref{prop:halfspace-rl-statistical}, which asserts that reward-free RL in the family of Generalized Block MDPs (and hence Low-Rank MDPs) $\MM_n$ is \emph{statistically} tractable---and hence the hardness result from the preceding section is purely a \emph{computational} phenomenon. 

\begin{proposition}\label{prop:halfspace-rl-statistical}
There is a (computationally \emph{inefficient}) algorithm $\Alg^M$ that, given episodic access to any MDP $M \in \MM_n$, has sample complexity $\poly(n)$ and produces a set of policies $\Psi$ satisfying the following guarantee with probability at least $1 - o(1)$:
\[\forall s \in \MS: \max_{\pi\in\Psi} d^{M,\pi}_H(s) \geq \frac{1}{64} \cdot \left(\max_{\pi\in\Pi} d^{M,\pi}_H(s) - \frac{1}{8}\right).\]
\end{proposition}

Since $\Phi_n$ has infinite cardinality, this does not follow from a black-box application of prior results on learning Low-Rank MDPs. However, as we verify in \cref{thm:mbfr-pdim}, the main result of \cite{mhammedi2023efficient} can be extended to this setting so long as an appropriate function class based on $\Phi_n$ has bounded \emph{pseudo-dimension} (\cref{def:pdim}). We then verify the requisite pseudo-dimension bound (\cref{lemma:halfspace-pdim}). From these, the proof of \cref{prop:halfspace-rl-statistical} is straightforward.

\begin{definition}[Pseudo-dimension \citep{haussler2018decision}]\label{def:pdim}
For any set $\MX$ and function class $\MF \subseteq \{\MX \to \RR\}$, the pseudo-dimension of $\MF$ is defined as the VC dimension of $\MF^+$, where $\MF^+ \subseteq \{\MX\times\RR \to \{0,1\}\}$ is defined as $\MF^+ := \{(x,\xi) \mapsto \mathbbm{1}[f(x) > \xi]: f \in \MF\}$.
\end{definition}

We prove the following straightforward extension of a result by \cite{mhammedi2023efficient}:

\begin{theorem}[Extension of {\cite[Theorem 3.2]{mhammedi2023efficient}}]\label{thm:mbfr-pdim}
Let $\MS,\MX, \MA$ be sets, let $H \in \NN$, and let $\Phi$ be a set of functions $\phi: \MX\to\MS$. Let $\fd$ denote the pseudo-dimension (\cref{def:pdim}) of the function class $\MF \subseteq (\MX\times\MA\times\MX \to \RR)$ defined by
\begin{equation} \MF := \left\{(x,a,x') \mapsto f(\phi^1(x),\phi^2(x),\phi^3(x'),\phi^4(x'),a) \mid{} f: \MS^4\times\MA \to [0,1], \phi^1,\phi^2,\phi^3,\phi^4 \in \Phi\right\}.\label{eq:disc-class-mbfr}\end{equation}

There is an algorithm $\Alg^M$ that takes input $\epsilon,\delta \in (0,1/2)$ and episodic access to an MDP $M$ with observation space $\MX$, action space $\MA$, and horizon $H$, and has the following property. If $M$ is a generalized $\Phi$-decodable block MDP, then $\Alg^M(\epsilon,\delta)$, with probability at least $1-\delta$, produces sets $\Psi_{1:H}$ such that
\begin{align} &\EE_{\pi\sim\Unif(\Psi_h)} d^{M,\pi}_h(x) \geq \frac{1}{8|\MA|^2|\MS|} \max_{\pi\in\Pi} d^{M,\pi}_h(x) \nonumber\\
&\forall h \in \{2,\dots,H\}, x \in \MX: \max_{\pi\in\Pi} d^{M,\pi}_h(x) \geq \epsilon \sum_{(s,a)\in\MS\times\MA} \BP^M_{h}(x|s,a).\label{eq:vox-guarantee}
\end{align}
Moreover, the algorithm has sample complexity $\poly(H,|\MS|,|\MA|,1/\epsilon,\fd,\log(1/\delta))$.
\end{theorem}

\begin{proof}
First suppose that $|\Phi| < \infty$ and we replace the pseudo-dimension term $\fd$ in the sample complexity with $\log |\Phi|$. Then we claim that this theorem is essentially immediate from \cite[Theorem 3.2]{mhammedi2023efficient}: by \cref{prop:genblock-is-lowrank}, any generalized $\Phi$-decodable block MDP is low-rank with rank $d := |\MS||\MA|$ and function class $\Philin$ of size exactly $|\Phi|$. It can be checked that the norm bounds required by \cite{mhammedi2023efficient} are satisfied: $\norm{\philin(x,a)}_2 \leq 1$ for all $x\in\MX$, $a \in \MA$, and $\philin\in\Philin$, and 
\[\norm{\int_\MX g(x) d\mu^\st_{h+1}}_2 \leq \sqrt{d}\]
for any $h \in [H-1]$ and $g: \MX \to [0,1]$, since $(\mu^\st_{h+1})_{s,a}$ is a distribution over $\MX$ for each $s\in\MS$ and $a \in \MA$. Thus, \cite[Theorem 3.2]{mhammedi2023efficient} shows that with probability at least $1-\delta$, the output $\Psi_{1:H}$ is a $(1/(8|\MA|^2||\MS|),\epsilon)$-policy cover. However, inspecting the proof, in fact the stronger guarantee is shown that the output is a $(1/(8|\MA|^2|\MS|), \epsilon)$-\emph{randomized} policy cover \--- see \cite[Definition 2.2]{mhammedi2023efficient}. This gives \cref{eq:vox-guarantee}, using the fact that $\norm{\mu^\st_{h+1}(x)}_2 \leq \sum_{s,a} \BP_{h+1}(x\mid{}s,a)$ for any $x \in \MX$.

It remains to argue that the result can be generalized to infinite classes $\Phi$ using pseudo-dimension. The dependence on $\log|\Phi|$ in the sample complexity bound in \cite[Theorem 3.2]{mhammedi2023efficient} arises due to the need to prove uniform concentration over $\Phi$. In particular, it arises in two places:
\begin{enumerate}
\item Analysis of $\PSDP$: \cite[Lemma D.2]{mhammedi2023efficient} incurs dependence on $\log|\Phi|$ while proving that the regression estimate $\hat g^{(t)}$ computed on line 7 of \cite[Algorithm 3]{mhammedi2023efficient} is close to the $Q$-function $Q^{\hat \pi^{t+1}}_t$. However, we can remove this dependence as follows. With notation as in line 7 of \cite[Algorithm 3]{mhammedi2023efficient}, it suffices to prove that so long as $|\MD^{(t)}| \geq \poly(\fd, 1/\epsilon,\log(1/\delta))$, we have with probability at least $1-\delta$ that
\[\sup_{g \in \MG_t} \left|\frac{1}{|\MD^{(t)}|} \sum_{(x,a,R) \in \MD^{(t)}} (g(x,a) - R)^2 - \EE[(g(x,a) - R)^2]\right| \leq \epsilon.\]
In particular, we need this to hold for $\MG_t := \{(x,a) \mapsto f(\phi(x),a) \mid{} f: \MS\times\MA \to [0,1], \phi\in\Phi\}$. By \cite[Corollary 4.3]{modi2021model} and the triangle inequality, it suffices to bound the pseudo-dimension of the function classes
\[\MF_1 := \{(x,a) \mapsto f(\phi(x),a)^2 \mid{} f: \MS\times\MA \to [0,1], \phi\in\Phi\},\]
\[\MF_2 := \{(x,a,R) \mapsto f(\phi(x),a) \cdot R \mid{} f: \MS\times\MA \to [0,1], \phi\in\Phi\},\]
and
\[\MF_3 := \{R \mapsto R^2\}.\]
Note that $\MF_3$ has constant size and hence constant pseudo-dimension; each function in $\MF_2$ can be expressed as a product of some function in $\MG_t$ with the function $R \mapsto R$; and each function in $\MF_1$ can be expressed as a product of two functions in $\MG_t$. Thus, by \cite[Lemma 50]{modi2021model}, it suffices to bound the pseudo-dimension of $\MG_t$. Since $\MG_t$ can be embedded in $\MF$, the pseudo-dimension of $\MG_t$ is bounded by $\fd$.

\item Analysis of \texttt{RepLearn}: \cite[Lemma F.1]{mhammedi2023efficient} incurs dependence on $\log |\Phi|$ through application of \cite[Lemma 14]{modi2021model}, which in turn uses \cite[Lemma 17]{modi2021model}, which in turn uses \cite[Lemma 34]{modi2021model}. In the proof of \cite[Lemma 34]{modi2021model}, a factor of $\log |\Phi|$ is incurred to bound the log covering number of a particular function class $\MH$, which consists of certain functions $h: \MX\times\MA\times\MX\to\RR$. For our application, the function classes $\Phi,\Phi'$ in the lemma statement are both $\Philin$, and the reward function class $\MR$ is $\{(x,a) \mapsto f(\phi(x),a): f:\MS\times\MA \to [0,1], \phi\in\Phi\}$. From this, we can check that each function $h \in \MH$ depends on its input $(x,a,x')$ only through $a$ and $\phi^1(x),\phi^2(x),\phi^3(x'),\phi^4(x')$ for four functions $\phi^1,\phi^2,\phi^3,\phi^4 \in \Phi$. Thus, $\MH$ can be embedded in $\MF$, and so the pseudo-dimension of $\MH$ can be bounded by $\fd$. This means that in the proof of \cite[Lemma 34]{modi2021model}, instead of bounding the log covering number, we can apply \cite[Corollary 43]{modi2021model} using this bound on the pseudo-dimension of $\MH$. The rest of the proof of the lemma is unchanged.
\end{enumerate}
This completes the proof.
\end{proof}

\begin{lemma}\label{lemma:halfspace-pdim}
Fix $n \in \NN$ and let $\Phi_n$ denote the class of linear threshold functions (\cref{def:halfspace-mdps-apx}). The function class $\MF_n \subseteq (\RR^n\times\{0,1\}\times\RR^n \to \RR)$ defined by
\[\MF_n := \left\{(x,a,x') \mapsto f(\phi^1(x),\phi^2(x),\phi^3(x'),\phi^4(x'),a) \mid{} f: \{0,1\}^5 \to [0,1], \phi^1,\phi^2,\phi^3,\phi^4 \in \Phi_n\right\}\]
has pseudo-dimension at most $O(n\log n)$.
\end{lemma}

\begin{proof}
By \cref{def:pdim}, we need to bound the VC dimension of the function class $\MF^+_n \subseteq (\RR^n\times\{0,1\}\times\RR^n\times\RR \to \RR)$
\[\MF^+_n := \left\{(x,a,x',\xi) \mapsto \mathbbm{1}[f(\phi^{1:2}(x),\phi^{3:4}(x'),a)>\xi] \mid{} f: \{0,1\}^5 \to [0,1], \phi^1,\phi^2,\phi^3,\phi^4 \in \Phi_n\right\}.\]
Fix any set $\MD = (x_i,a_i,x'_i,\xi_i)_{i=1}^m$. We would like to upper bound the number of attainable vectors
\[(\mathbbm{1}[f(\phi^{1:2}(x_1),\phi^{3:4}(x'_1),a_1)>\xi_1],\dots,\mathbbm{1}[f(\phi^{1:2}(x_m),\phi^{3:4}(x'_m),a_m)>\xi_m]) \in \{0,1\}^m\]
as $f$ and $\phi^1,\dots,\phi^4$ vary. First, note that the numbers $\xi_1,\dots,\xi_m$ partition $\RR$ into $m+1$ intervals, and two functions $f,f'$ such that $f(b)$ and $f'(b)$ lie in the same interval for all $b \in \{0,1\}^5$ induce the same vector (for any fixed $\phi^1,\dots,\phi^4$). Thus, we can restrict focus to $(m+1)^{32} = \poly(m)$ choices of $f$.

Fix one such $f$. By the Milnor-Warren bound \cite{milnor1964betti}, the set $\{(\phi(x_1),\dots,\phi(x_m)): \phi \in \Phi_n\}$ has size at most $m^{O(n)}$. Thus, as $\phi^1,\dots,\phi^4 \in \Phi_n$ vary, the number of attainable vectors is bounded by $m^{O(n)}$. Summing over the $\poly(m)$ choices of $f$, we find that the total number of attainable vectors as all parameters vary is still $m^{O(n)}$, so $\MD$ cannot be shattered by $\MF_n^+$ unless $m \leq O(n\log n)$. Thus, the pseudo-dimension of $\MF_n$ is at most $O(n\log n)$.
\end{proof}

\begin{proof}[Proof of \cref{prop:halfspace-rl-statistical}]
We apply \cref{thm:mbfr-pdim} with $\MS := \MA := \{0,1\}$, $\MX := \RR^n$, $H := (\log n)^{\log \log n}$, and $\Phi := \Phi_n$. By \cref{lemma:halfspace-pdim}, the pseudo-dimension of the resulting set $\MF$ defined in \cref{eq:disc-class-mbfr} is at most $O(n\log n)$. Invoke the algorithm $\Alg^M$ guaranteed by \cref{thm:mbfr-pdim} with parameters $\epsilon := \frac{1}{512}$ and $\delta := 1/2$. For any $M \in \MM_n$, since $M$ is a generalized $\Phi_n$-decodable block MDP, we get that with probability at least $1/2$, the sets $\Psi_{1:H}$ produced by $\Alg^M$ satisfy \cref{eq:vox-guarantee}. Let $\Xgood$ be the set of $x \in \MX$ satisfying
\[\max_{\pi\in\Pi} d^{M,\pi}_H(x) \geq \epsilon \sum_{(s,a)\in\MS\times\MA} \BP^M_H(x\mid{}s,a).\]
Suppose that $M$ has decoding functions $\phist_1,\dots,\phist_H$. For any $s \in \MS$, we have
\begin{align*}
\max_{\pi\in\Psi} d^{M,\pi}_H(s)
&\geq \EE_{\pi\sim\Unif(\Psi_H)} \left[d^{M,\pi}_H(s) \right]\\ 
&\geq \sum_{x\in\Xgood: \phist_H(x)=s} \EE_{\pi\sim\Unif(\Psi_H)} \left[ d^{M,\pi}_H(x) \right]\\ 
&\geq \frac{1}{64} \sum_{x\in\Xgood: \phist_H(x)=s} \max_{\pi\in\Pi} d^{M,\pi}_H(x) \\ 
&\geq \frac{1}{64} \max_{\pi\in\Pi} \sum_{x\in\Xgood: \phist_H(x)=s} d^{M,\pi}_H(x) \\ 
&\geq \frac{1}{64} \max_{\pi\in\Pi} \left(d^{M,\pi}_H(s) - \sum_{x\in\MX\setminus\Xgood: \phist_H(x)=s} d^{M,\pi}_H(x)\right) \\ 
&\geq \frac{1}{64} \max_{\pi\in\Pi} \left(d^{M,\pi}_H(s) - \epsilon \sum_{x\in\MX} \sum_{(s,a)\in\MS\times\MA} \BP^M_H(x|s,a)\right) \\ 
&= \frac{1}{64} \max_{\pi\in\Pi} \left(d^{M,\pi}_H(s) - \epsilon |\MS||\MA|\right) \\ 
&= \frac{1}{64} \max_{\pi\in\Pi} \left(d^{M,\pi}_H(s) - \frac{1}{8}\right)
\end{align*}
by choice of $\epsilon$. Finally, the sample complexity of the algorithm is $\poly(n)$ by \cref{thm:mbfr-pdim}, the fact that the pseudo-dimension is at most $O(n\log n)$, the choice of parameters $\epsilon,\delta = \Omega(1)$, and the fact that $H,|\MS|,|\MA| \leq n$.
\end{proof}

\newpage

\section{Supporting Technical Results}

This section contains supporting technical results. \cref{sec:psdp}
gives a self-contained presentation of the $\PSDP$ algorithm, while
\cref{sec:misc} gives miscellaneous regression reductions used
throughout our main results.

\subsection{Policy Search by Dynamic Programming (PSDP)}
\label{sec:psdp}
\begin{algorithm}[tp]
	\caption{$\PSDPB(k, \Reg, R, \Psi_{1:k}, \Gamma, N)$: Policy Search by Dynamic Programming \\(variant of \citet{bagnell2003policy}; see also \cite{mhammedi2023representation,golowich2024exploring})}
	\label{alg:psdpb}
	\begin{algorithmic}[1]\onehalfspacing
          \State \textbf{input:} Step $k \in [H]$; regression oracle $\Reg$; reward function $R: \MX \to [0,1]$; policy covers $\Psi_1, \ldots, \Psi_{k}$; backup policy cover $\Gamma$; number of samples $N\in \mathbb{N}$.
		\For{$h=k, \dots, 1$} 
        \For{$a \in \MA$}
    		\State $\MD_{h,a} \gets\emptyset$.
    		\For{$N$ times}
            \State Sample policy $\pi \sim \frac{1}{2}\left(\Unif(\Psi_h) + \Unif(\Gamma)\right)$.
    		\State Sample trajectory $(x_1, a_1, \dots, x_k, a_k, x_{k+1})\sim
    		\pi \circ_h a \circ_{h+1} \pihat^{h+1:k}$.
            \State Sample $r_{k+1} \sim \Ber(R(x_{k+1}))$.
    		\State Update dataset: $\MD_{h,a} \gets \MD_{h,a} \cup \{ (x_h, r_{k+1})\}$.
    		\EndFor
    		\State Solve regression:
		\[\wh Q_h(\cdot,a) \gets \Reg(\MD_{h,a}).\] \label{eq:psdpb-regression}
        \EndFor
		\State Define $\pihat_h : \MX \ra \MA$ by
		\[
		\pihat_h(x)  := 
			\argmax_{a\in \MA} \wh Q_h(x,a),
          \] 
          \quad\, and write $\pihat^{h:k} = (\pihat_h, \ldots, \pihat_k)$. \label{line:psdpb-pihat-def}
		\EndFor
		\State \textbf{return:} Policy $\pihat^{1:k} \in \Pi$. 
	\end{algorithmic}
\end{algorithm}

The following lemma provides an analysis of \emph{Policy Search by Dynamic Programming (PSDP)} \citep{bagnell2003policy}---specifically, an implementation where the $Q$-functions are fit using one-context regression (\cref{alg:psdpb}). This shows that (approximate) policy optimization with a given reward function is efficiently reducible to one-context regression, and is a key element in both \cref{cor:online-rl-to-regression} and \cref{cor:reset-rl-to-regression}, as a subroutine in $\PCO$ (\cref{alg:pco}) and $\PCR$ (\cref{alg:pcr}) respectively. While the below statement is technically novel, since it abstracts generalization arguments into the regression oracle, at a technical level the analysis is entirely standard, see e.g. \cite{mhammedi2023representation,golowich2024exploring}.

\paragraph{Value functions.} For a Block MDP $M$, policy $\pi$, and collection $\bfr = (\bfr_h)_{h=1}^H$ of reward functions $\bfr_h: \MX \times \MA \to \RR$, we define value functions $Q^{M,\pi,\bfr}_h(x,a) := \EE^{M,\pi}[\sum_{k=h}^H \bfr_k(x_k,a_k) \mid{} x_h=x, a_h=a]$ and $V^{M,\pi,\bfr}_h(x) := \EE^{M,\pi}[\sum_{k=h}^H \bfr_k(x_k,a_k) \mid{} x_h=x]$.

\begin{lemma}[PSDP analysis]\label{lemma:psdp-trunc-online}
Fix $\alpha,\epsilon,\delta \in (0,1)$, $k \in [H]$, and $N \in \NN$. Let $\Phi_{1:k}$ be $\alpha$-truncated policy covers for $M$ at steps $1,\dots,k$ (\cref{defn:trunc-pc}), let $\Reg$ be an oracle that solves $\Nreg$-efficient one-context regression over $\Phi$, and let $R: \Xbar\to[0,1]$ be a function with $R(\term)=0$. Let $\Gamma \subset \Pi$ be a finite set of policies. If $N \geq \Nreg(\epsilon,\delta)$, then $\pihat \gets \PSDPB(k,\Reg,R,\Psi_{1:k},\Gamma,N)$ satisfies
\[\E^{M,\pihat}[R(x_{k+1})] \geq \max_{\pi\in\Pi} \E^{\Mbar(\Gamma),\pi}[R(x_{k+1})] - \frac{4H\sqrt{|\MA|\epsilon}}{\min(\alpha\trunc,\tsmall)}\]
with probability at least $1 - H|\MA|\delta$.
\end{lemma}

\begin{proof}
  Define reward function $\bfr = (\bfr_1,\dots,\bfr_H)$ by 
  \[\bfr_h(x,a) = \begin{cases} R(x) & \text{ if } h = k+1 \\ 0 & \text{ otherwise } \end{cases}.\] 
  Let $\pi^\st \in \argmax_{\pi\in\Pi} \E^{M,\pi}[R(x_{k+1})]$. Applying \cref{lemma:perf-diff-trunc} with reward function $R$ and policies $\pihat$ and $\pi^\st$ gives
  \begin{align*}
  &\E^{\Mbar(\Gamma),\pi^\st}[R(x_{k+1})] - \E^{M,\pihat}[R(x_{k+1})] \\
  &\leq \sum_{h=1}^k \E^{\Mbar(\Gamma),\pi^\st} \left[Q^{M,\pihat,\bfr}_h(x_h, a_h) - V^{M,\pihat,\bfr}_h(x_h)\right].
  \end{align*}
Fix $h \in [k]$. By definition (\lineref{line:psdpb-pihat-def}) we have $\pihat_h(x) \in \argmax_{a\in\MA} \wh Q_h(x,a)$. Let us define $\Delta_h:\Xbar\to\RR_{\geq 0}$ by \[\Delta_h(x) := \begin{cases} \max_{a \in \MA} | Q^{M,\pihat,\bfr}_h(x,a) - \wh Q_h(x,a)| & \text{ if } x \in \MX \\ 0 & \text{ if } x = \term\end{cases}.\] Then for any $x \in \MX$, we have
  \begin{align*}
    &Q^{M,\pihat,\bfr}_h(x, \pi_h^\st(x)) - V^{M,\pihat,\bfr}_h(x) \\
    &= Q^{M,\pihat,\bfr}_h(x, \pi_h^\st(x)) - Q^{M,\pihat,\bfr}_h(x, \pihat_h(x))  \\
    &\leq \wh Q_h(x,\pi^\st_h(x)) - \wh Q_h(x,\wh\pi_h(x))+ 2\Delta_h(x) \\ 
    &\leq 2\Delta_h(x)
  \end{align*}
  where the first inequality uses the definition of $\Delta_h(x)$ and the second inequality uses the fact that $\pihat_h(x) \in \argmax_{a \in \MA} \wh Q_h(x,a)$. Note that the above inequality also holds for $x=\term$, since $Q^{M,\pi,\bfr}_h(\term,a) = V^{M,\pi,\bfr}_h(\term) = 0$ for any $\pi\in\Pi$, $a \in \MA$. It follows that
  \begin{equation} 
  \E^{\Mbar(\Gamma),\pi^\st}[R(x_{k+1})] - \E^{M,\pihat}[R(x_{k+1})]
  \leq \sum_{h=1}^{k-1} \E^{\Mbar(\Gamma),\pi^\st}[2\Delta_h(x_h)],
  \label{eq:psdp-perf-diff}
  \end{equation}
  and it only remains to upper bound each term $\E^{\Mbar(\Gamma),\pi^\st}[\Delta_h(x_h)]$. Once more, fix $h \in [k]$ and $a \in \MA$. The dataset $\MD_{h,a}$ consists of $N$ independent and identically distributed samples $(x,r)$ with
  \[\E[r|x] = \E^{M,a \circ_h \pihat^{h+1:k}}[R(x_{k+1}) | x_h = x] = \E^{M,a \circ_h \wh \pi^{h+1:k}}[R(x_{k+1})|s_h = \phi^\st(x)] = Q^{M,\pihat,\bfr}_h(x,a).\]
  In particular, by the penultimate equality, $\E[r|x]$ only depends on $\phi^\st(x)$, so by the guarantee on $\Reg$ (\cref{def:one-con-regression}) and the assumption that $N \geq \Nreg(\epsilon,\delta)$, it holds with probability at least $1-\delta$ that
  \[\EE_{\pi \sim \frac{1}{2}(\Unif(\Psi_h)+\Unif(\Gamma))} \E^{M,\pi}(\wh Q_h(x_h,a) - Q^{M,\pihat,\bfr}_h(x_h,a))^2 \leq \epsilon.\]
  Let $\ME_{h,a}$ be the event that this inequality holds. Under $\bigcap_{a \in \MA} \ME_{h,a}$, we have
  \begin{align*}
\EE_{\pi \sim \frac{1}{2}(\Unif(\Psi_h)+\Unif(\Gamma))} \E^{M,\pi}[\Delta_h(x_h)^2] 
&\leq \sum_{a\in\MA} \EE_{\pi \sim \frac{1}{2}(\Unif(\Psi_h)+\Unif(\Gamma))} \E^{M,\pi}(\wh Q_h(x_h,a) - Q^{M,\pihat,\bfr}_h(x_h,a))^2 \\ 
&\leq |\MA|\epsilon,
\end{align*}
  which yields, via Jensen's inequality, that
  \[\EE_{\pi \sim \frac{1}{2}(\Unif(\Psi_h)+\Unif(\Gamma))} \E^{M,\pi}[\Delta_h(x_h)] \leq \sqrt{|\MA|\epsilon}.\] It follows that under event $\bigcap_{a\in\MA} \ME_{h,a}$,
  \begin{align}
    \E^{\Mbar(\Gamma),\pi^\st}[\Delta_h(x_h)] 
    &\leq \frac{2}{\min(\alpha\trunc,\tsmall)} \EE_{\pi\sim\frac{1}{2}(\Unif(\Psi_h)+\Unif(\Gamma))} \E^{M,\pi}[\Delta_h(x_h)] \nonumber \\ 
    &\leq \frac{2\sqrt{|\MA|\epsilon}}{\min(\alpha\trunc,\tsmall)} \nonumber.
  \end{align}
  where the first inequality uses \cref{item:h-plus-one-cov-lb} of \cref{lemma:srch-gamma-covering} together with the assumption that $\Psi_h$ is an $\alpha$-truncated cover (\cref{defn:trunc-pc}) and the fact that $\Delta_h(\term)=0$ and $\Delta_h(x) \geq 0$ for all $x \in \MX$.
Substituting into \cref{eq:psdp-perf-diff}, we conclude that, under the event $\bigcap_{h=1}^{k} \bigcap_{a\in\MA} \ME_{h,a}$ (which occurs with probability at least $1-H|\MA|\delta$), 
  \[
    \E^{\Mbar,\pi^\st}[R(x_{k+1})] - \E^{M,\pihat}[R(x_{k+1})] \leq \frac{4H\sqrt{|\MA|\epsilon}}{\min(\alpha\trunc,\tsmall)}
  \]
  as claimed.
\end{proof}

\subsection{Miscellaneous Reductions}
\label{sec:misc}

In \sssref{sec:onered}, we prove that one-context regression is necessary for reward-free episodic RL, modifying a proof of \cite{golowich2024exploration}; this reduction is one piece in the proof of \cref{cor:regression-to-online-rl}. In \sssref{sec:noiselessonered}, we introduce noiseless one-context regression and prove that it is necessary for reward-free RL in the reset access model, which complements our result from \cref{sec:resets}. In \sssref{sec:onetwo}, we show that one-context regression is a special case of two-context regression, which is needed for our episodic RL algorithm $\PCO$ (see \cref{sec:online}). In \sssref{sec:oneaug} and \sssref{sec:twoaug}, we show that one-context regression and two-context regression for concept class $\Phiaug$ reduce to one-context regression and two-context regression for $\Phi$; the latter is one piece in the proof of \cref{cor:regression-to-online-rl}.

In these reductions, fix a concept class $\Phi \subseteq (\MX\to\MS)$ and recall the definition of sets $\Xaug,\Saug$ and augmented concept class $\Phiaug \subseteq (\Xaug\to\Saug)$ from \cref{def:phiaug}.

\subsubsection{The $\OneRed$ Reduction}\label{sec:onered}

In recent work, \cite{golowich2024exploration} showed that one-context regression is necessary for reward-directed episodic RL. Here we adapt their argument to reward-free RL; the modification is straightforward under regularity (using the ``extra'' states $\{0,1\}$).

Concretely, the following theorem shows that $\OneRed$ (\cref{alg:onered}) reduces one-context regression for concept class $\Phi$ to reward-free episodic RL for concept class $\Phiaug$. We use this reduction as a component of our reduction from \emph{two}-context regression to reward-free RL (see \cref{sec:app_minimality}).

\begin{algorithm}[t]
	\caption{$\OneRed(\MO, (x^{(i)},y^{(i)})_{i=1}^n,\epsilon,\delta)$: Reduction from one-context regression to reward-free episodic RL}
	\label{alg:onered}
	\begin{algorithmic}[1]\onehalfspacing
		          \State \textbf{input:} Oracle $\MO$ for reward-free episodic RL; samples $(x^{(i)},y^{(i)})_{i=1}^n$; tolerances $\epsilon,\delta$.
		\State Set $\epa := \sqrt{\epsilon/4}$ and $i = 1$. Initialize $\MO$ with tolerance $\epsilon/4$, failure probability $\delta/2$, horizon $H := 2$, and action set $\MA := \{0,\epa,\dots,1-\epa\}$. Simulate $\MO$ as follows:
        \Repeat
            \State When $\MO$ queries for a new episode, pass $x^{(i)}$.
            \State When $\MO$ plays an action $a_1 \in \MA$,  pass observation $0$ with probability $(a-y^{(i)})^2$. Otherwise, pass observation $1$. In either case, set $i \gets i+1$.
        \Until{$\MO$ returns policy cover $\Psi$}
        \State $m \gets 16\epsilon^{-2}\log(4|\Psi|/\delta)$.
        \For{$\pi \in \Psi$}
            \State Compute $\wh E(\pi) := \frac{1}{m}\sum_{i=n-m+1}^n (\pi(x^{(i)}) - y^{(i)})^2$.\label{line:one-red-emp}
        \EndFor
        \State \textbf{return:} $\pihat := \argmin_{\pi \in \Psi}  \wh E(\pi)$.
	\end{algorithmic}
\end{algorithm}

\begin{proposition}[Modification of {\cite[Proposition B.2]{golowich2024exploration}}]\label{prop:onered}
Suppose that $\MO$ is a $(\Nrl,\Krl)$-efficient reward-free episodic RL oracle for $\Phiaug$. Then $\OneRed(\MO,\cdot)$ is a $\Nreg$-efficient one-context regression algorithm for $\Phi$ with
\[\Nreg(\epsilon,\delta) = \Nrl\left(\frac{\epsilon}{4}, \frac{\delta}{2}, 2, \sqrt{\frac{4}{\epsilon}}\right) + \frac{16\log(4\Krl\left(\frac{\epsilon}{4}, \frac{\delta}{2}, 2, \sqrt{\frac{4}{\epsilon}}\right)/\delta)}{\epsilon^2}.\]
\end{proposition}

\begin{proof}
Let $\epsilon,\delta>0$, $\MD \in \Delta(\MX)$, and $f: \MS \to \{0,1\}$. Let $n \geq \Nreg(\epsilon,\delta)$. Let $(x^{(i)},y^{(i)})_{i=1}^n$ be i.i.d. samples with $x^{(i)} \sim \MD$, $y^{(i)} \in \{0,1\}$, and $\EE[y^{(i)}\mid{}x^{(i)}] = f(\phi(x^{(i)}))$ for some $\phi \in \Phi$. We analyze the execution of $\OneRed(\MO,(x^{(i)},y^{(i)})_{i=1}^n, \epsilon,\delta)$. We know that $n-m \geq \Nrl(\epsilon/4,\delta/2,2,\sqrt{4/\epsilon})$, since $n \geq \Nreg(\epsilon,\delta)$. Fix any episode $i$ of interaction with the oracle $\MO$. Observe that conditioned on the initial observation $x^{(i)}$ and action $a_1$,
\begin{align*}
\EE[(a_1-y^{(i)})^2\mid{} x^{(i)}, a_1]
&= (a_1)^2 - 2 a_1 \EE[y^{(i)}\mid{} x^{(i)}] + \EE[(y^{(i)})^2 \mid{} x^{(i)}] \\ 
&= (a_1)^2 + (1 - 2a_1) f(\phi(x^{(i)}))
\end{align*}
where the final equality uses that $y^{(i)} \in \{0,1\}$. Thus, $\OneRed$ simulates $\MO$ on a $\Phiaug$-decodable block MDP $M$ with horizon $2$, observation space $\Xaug$, latent state space $\Saug$, initial observation distribution $\MD$, and transition distribution defined by
\[\BP_2(0\mid{} x_1,a_1) := a_1^2 + (1-2a_1)f(\phi(x_1)),\]
\[\BP_2(1\mid{} x_1,a_1) := 1 - \BP_2(0\mid{} x_1,a_1).\]
By \cref{def:strong-rf-rl}, the output of $\MO$ is a set of policies $\Psi$ of size at most $\Krl(\epsilon/4,\delta/2,2,\sqrt{4/\epsilon})$, such that with probability at least $1-\delta/2$, there is some $\pi^\st \in \Psi$ such that
\[d^{M,\pi^\st}_2(1) \geq \max_{\pi\in\Pi} d^{M,\pi}_2(1) - \frac{\epsilon}{4}.\]
Condition on this event, and observe that for any $\pi \in \Pi$,
\begin{align*}
d^{M,\pi}_2(0) 
&= \E^{M,\pi}[a_1^2 + (1-2a_1)f(\phi(x_1))] \\ 
&= \EE_{x \sim \MD}[(\pi(x) - f(\phi(x)))^2 + f(\phi(x)) - f(\phi(x))^2] \\ 
&= \EE_{x,y}[(\pi(x) - y)^2] + Z
\end{align*}
where $Z := \EE_{x,y}[f(\phi(x)) - f(\phi(x))^2 - (f(\phi(x))-y)^2]$, and the expectations are over a fresh sample $(x,y)$ from the same distribution as $(x^{(i)},y^{(i)})$. But now by Hoeffding's inequality, the union bound, and choice of $m := 16\epsilon^{-2}\log(4|\Psi|/\delta)$, we have with probability at least $1-\delta/2$ that for all $\pi\in\Psi$,
\[\left|\wh E(\pi) - \EE_{x,y}[(\pi(x)-y)^2]\right| \leq \epsilon/4,\]
where $\wh E(\pi)$ is the empirical loss for $\pi$ computed in \cref{line:one-red-emp} of $\OneRed$. In this event, we get that 
\begin{align*}
\EE_{x,y}[(\pihat(x)-y)^2]
&\leq \frac{\epsilon}{2} + \EE_{x,y}[(\pi^\st(x) - y)^2] \\ 
&= \frac{\epsilon}{2} + 1-d^{M,\pi^\st}_2(1) - Z \\ 
&\leq \frac{\epsilon}{2} + 1 - \max_{\pi\in\Pi} d^{M,\pi}_2(1) - Z \\ 
&\leq \frac{3\epsilon}{4} + \min_{\pi\in\Pi} \EE_{x,y}[(\pi(x)-y)^2].
\end{align*}
It follows that
\[\EE_x[(\pihat(x) - f(\phi(x)))^2] \leq \frac{3\epsilon}{4} + \min_{\pi\in\Pi} \EE_x[(\pi(x)-f(\phi(x)))^2] \leq \epsilon\]
since the policy $\pi(x) = \sqrt{4/\epsilon} \lfloor f(\phi(x) \cdot \sqrt{4/\epsilon}\rfloor$ has squared error at most $\epsilon/4$.
\end{proof}

\subsubsection{The $\NoiselessOneRed$ Reduction}\label{sec:noiselessonered}

In this section, we adapt the reduction from \sssref{sec:onered} to the reset access model. However, this requires weakening the regression problem to be \emph{noiseless}:

\begin{definition}[Noiseless one-context regression]\label{def:noiseless-one-con-regression}
Let $\Nreg: (0,1/2)^2 \to \NN$ be a function. An algorithm $\Alg$ is an $\Nreg$-efficient noiseless one-context regression algorithm for $\Phi$ if the following holds. Fix $\epsilon,\delta \in (0,1/2)$, $n \in \NN$, and $\phi \in \Phi$. Let $\MD \in \Delta(\MX)$ be a distribution, and let $f: \MS \to \{0,1\}$. Let $(x^{(i)},y^{(i)})_{i=1}^n$ be i.i.d. samples with $x^{(i)} \sim \MD$, $y^{(i)} \in \{0,1\}$, and $\E[y^{(i)}\mid{}x^{(i)}] = f(\phi(x^{(i)}))$. If $n \geq \Nreg(\epsilon,\delta)$, then with probability at least $1-\delta$, the output of $\Alg((x^{(i)},y^{(i)})_{i=1}^n, \epsilon,\delta)$ is a circuit $\MR: \MX \to [0,1]$ satisfying 
\[\EE_{x \sim \MD} (\MR(x) - f(\phi(x)))^2 \leq \epsilon.\]
\end{definition}

With this definition, the following theorem shows that $\NoiselessOneRed$ (\cref{alg:noiselessonered}) reduces noiseless one-context regression for concept class $\Phi$ to reward-free RL for concept class $\Phiaug$ in the reset model. By combining with \cref{prop:oneaug}, this implies that there is a reduction to reward-free RL for concept class $\Phi$ itself, so long as $\Phi$ is regular. We leave it as an open problem whether the reduction can be strengthened to work with noisy one-context regression.

\begin{algorithm}[t]
	\caption{$\NoiselessOneRed(\MO, (x^{(i)},y^{(i)})_{i=1}^n,\epsilon,\delta)$: Reduction from noiseless one-context regression to RL with resets}
	\label{alg:noiselessonered}
	\begin{algorithmic}[1]\onehalfspacing
		          \State \textbf{input:} Oracle $\MO$ for reward-free RL with resets; samples $(x^{(i)},y^{(i)})_{i=1}^n$; tolerances $\epsilon,\delta$.
		\State Set $\epa := \sqrt{\epsilon/4}$ and $i = 1$. Initialize $\MO$ with tolerance $\epsilon/4$, failure probability $\delta/2$, horizon $H := 2$, and action set $\MA := \{0,\epa,\dots,1-\epa\}$. Simulate $\MO$ as follows:
        \Repeat
            \State When $\MO$ queries the first sampling oracle, pass $x^{(i)}$ and set $i \gets i+1$.
            \State When $\MO$ queries the second sampling oracle with inputs $x_1\in\Xaug$ and $a_1\in\MA$, identify any $j < i$ with $x_1 = x^{(j)}$. With probability $(a-y^{(j)})^2$, pass observation $0$. Otherwise, pass observation $1$.
        \Until{$\MO$ returns policy cover $\Psi$}
        \State $m \gets 16\epsilon^{-2}\log(4|\Psi|/\delta)$.
        \For{$\pi \in \Psi$}
            \State Compute $\wh E(\pi) := \frac{1}{m}\sum_{i=n-m+1}^n (\pi(x^{(i)}) - y^{(i)})^2$.\label{line:noiseless-one-red-emp}
        \EndFor
        \State \textbf{return:} $\pihat := \argmin_{\pi \in \Psi}  \wh E(\pi)$.
	\end{algorithmic}
\end{algorithm}

\begin{proposition}\label{prop:noiseless-onered}
Suppose that $\MO$ is a $(\Nrl,\Krl)$-efficient reward-free reset RL algorithm for $\Phiaug$. Then $\NoiselessOneRed(\MO,\cdot)$ is a $\Nreg$-efficient noiseless one-context regression algorithm for $\Phi$ with
\[\Nreg(\epsilon,\delta) = \Nrl\left(\frac{\epsilon}{4}, \frac{\delta}{2}, 2, \sqrt{\frac{4}{\epsilon}}\right) + \frac{16\log(4\Krl\left(\frac{\epsilon}{4}, \frac{\delta}{2}, 2, \sqrt{\frac{4}{\epsilon}}\right)/\delta)}{\epsilon^2}.\]
\end{proposition}

\begin{proof}
Let $\epsilon,\delta>0$, $\MD \in \Delta(\MX)$, and $f: \MS \to \{0,1\}$. Let $n \geq \Nreg(\epsilon,\delta)$. Let $(x^{(i)},y^{(i)})_{i=1}^n$ be i.i.d. samples with $x^{(i)} \sim \MD$, $y^{(i)} \in \{0,1\}$, and $\EE[y^{(i)}\mid{}x^{(i)}] = f(\phi(x^{(i)}))$ for some $\phi \in \Phi$. We analyze the execution of $\NoiselessOneRed(\MO,(x^{(i)},y^{(i)})_{i=1}^n, \epsilon,\delta)$. We know that $n-m \geq \Nrl(\epsilon/4,\delta/2,2,\sqrt{4/\epsilon})$, since $n \geq \Nreg(\epsilon,\delta)$. Now the first sampling oracle provides independent samples from $\MD$. For the second sampling oracle, since $\MO$ can only query $x_1 \in \Xaug$ which it has previously seen, it must be that there exists $j < i$ with $x_1 = x^{(j)}$. Moreover, since the range of $f$ is in $\{0,1\}$, we have deterministically that $y^{(j)} = f(\phi(x^{(j)}))$. Conditioned on the queries $x_1$ and $a_1$, the output of the second sampling oracle is therefore independent of all prior queries, and the probability of observing $0$ is
\begin{align*}
\EE[(a_1-y^{(j)})^2\mid{} x^{(j)}, a_1]
&= (a_1)^2 - 2 a_1 \EE[y^{(j)}\mid{} x^{(j)}] + \EE[(y^{(j)})^2 \mid{} x^{(j)}] \\ 
&= (a_1)^2 + (1 - 2a_1) f(\phi(x_1))
\end{align*}
where the final equality uses that $y^{(i)} \in \{0,1\}$. Thus, $\NoiselessOneRed$ simulates $\MO$ on a $\Phiaug$-decodable block MDP $M$ with horizon $2$, observation space $\Xaug$, latent state space $\Saug$, initial observation distribution $\MD$, and transition distribution defined by
\[\BP_2(0\mid{} x_1,a_1) := a_1^2 + (1-2a_1)f(\phi(x_1)),\]
\[\BP_2(1\mid{} x_1,a_1) := 1 - \BP_2(0\mid{} x_1,a_1).\]
The remainder of the proof is identical to that of \cref{prop:onered}.
\end{proof}

\subsubsection{The $\OneTwo$ Reduction}\label{sec:onetwo}

\begin{algorithm}[t]
	\caption{$\OneTwo(\MO, (x^{(i)},y^{(i)})_{i=1}^n,\epsilon,\delta)$: One-context regression to two-context regression reduction}
	\label{alg:onetwo}
	\begin{algorithmic}[1]\onehalfspacing
		          \State \textbf{input:} Two-context regression oracle $\MO$; samples $(x^{(i)},y^{(i)})_{i=1}^n$; tolerances $\epsilon,\delta$.
		\State Pick arbitrary $\xbar \in \MX$.
        \State Compute
        \[\til \MR \gets \MO((x_1^{(i)},\xbar, y^{(i)})_{i=1}^{n},\epsilon, \delta).\]\label{line:tilmr}
        \State \textbf{return:} $\MR$ defined by $\MR(x) := \til\MR(x,\xbar)$.
	\end{algorithmic}
\end{algorithm}

The following proposition shows that $\OneTwo$ (\cref{alg:onetwo}) is an efficient reduction from one-context regression to two-context regression.

\begin{proposition}\label{prop:onetwo}
Fix sets $\MX,\MS$ and $\Phi \subseteq (\MX\to\MS)$. Suppose that $\MO$ is an $\Nreg$-efficient two-context regression oracle for $\Phi$. Then $\OneTwo(\MO,\cdot)$ is an $\Nreg$-efficient one-context regression oracle for $\Phi$.
\end{proposition}

\begin{proof}
Let $(x^{(i)},y^{(i)})_{i=1}^n$ be i.i.d. samples with $x^{(i)} \sim \MD$, $y^{(i)} \in \{0,1\}$, and $\EE[y^{(i)}\mid{}x^{(i)}] = f(\phi(x^{(i)}))$ for some $\phi \in\Phi$ and $f:\MS\to[0,1]$. Then for any fixed $\xbar \in \MX$, the samples $(x_1^{(i)},\xbar,y^{(i)})_{i=1}^n$ are i.i.d., the distribution of $(x_1^{(i)},\xbar)$ is $\phi$-realizable (since $\xbar$ is independent of $x_1^{(i)}$), and $\EE[y^{(i)}\mid{}x_1^{(i)},\xbar] = g(\phi(x^{(i)}),\phi(\xbar))$ where $g(s_1,s_2) := f(s_1)$. By \cref{def:two-con-regression}, so long as $N \geq \Nreg(\epsilon,\delta)$, it holds with probability at least $1-\delta$ that the predictor $\til \MR$ computed in \lineref{line:tilmr} satisfies
\[\EE_{x_1 \sim \MD} (\til\MR(x_1,\xbar) - g(\phi(x_1), \phi(\xbar)))^2 \leq \epsilon.\]
In this event, by definition of $g$, the output $\MR(\cdot) := \til\MR(\cdot,\xbar)$ of $\OneTwo$ satisfies
\[\EE_{x_1 \sim \MD} (\MR(x_1) - f(\phi(x_1)))^2 \leq \epsilon\]
as required for one-context regression (\cref{def:one-con-regression}).
\end{proof}

\subsubsection{The $\OneAug$ Reduction}\label{sec:oneaug}

\begin{algorithm}[t]
	\caption{$\OneAug(\MO, (x^{(i)},y^{(i)})_{i=1}^n,\epsilon,\delta)$: One-context regression for $\Phiaug$}
	\label{alg:oneaug}
	\begin{algorithmic}[1]\onehalfspacing
		\State\textbf{input:} One-context regression oracle $\MO$ for $\Phi$; samples $(x^{(i)},y^{(i)})_{i=1}^n$; tolerances $\epsilon,\delta$.
		\State Let $S := \{i \in [n]: x^{(i)} \in \MX\}$.
        \State Compute
        \[\MR_\MX \gets \MO((x^{(i)}, y^{(i)})_{i \in S},\epsilon/6, \delta/6).\]\label{line:mrmx}
        \State For $b \in \{0,1\}$, let $S_b := \{i \in [n]: x^{(i)} = b\}$ and compute
        \[\MR_b := \frac{1}{|S_b|} \sum_{i \in S_b} y^{(i)}.\]\label{line:mrb}
        \State \textbf{return:} $\MR:\Xaug\to[0,1]$ defined by \[ \MR(x) := 
        \begin{cases} 
        \MR_\MX(x) & \text{ if } x \in \MX \\ 
        \MR_0 & \text{ if } x = 0 \\ 
        \MR_1 & \text{ if } x = 1
        \end{cases}.\]
	\end{algorithmic}
\end{algorithm}

The following proposition shows that $\OneAug$ (\cref{alg:oneaug}) is an efficient reduction from one-context reduction for concept class $\Phiaug$ to the same problem for concept class $\Phi$. The basic idea is that the states $\{0,1\}$ are fully observed, so they can be regressed separately via mean estimation.

\begin{proposition}\label{prop:oneaug}
There is a constant $C_{\ref{prop:oneaug}}>0$ so that the following holds. Fix sets $\MX,\MS$ and $\Phi \subseteq (\MX \to \MS)$. Suppose that $\MO$ is an $\Nreg$-efficient one-context regression oracle for $\Phi$. Then $\OneAug(\MO,\cdot)$ is an $\Nreg'$-efficient one-context regression oracle for $\Phiaug$ with
\[\Nreg'(\epsilon,\delta) = C_{\ref{prop:oneaug}}\left(\epsilon^{-1}\Nreg(\epsilon/6,\delta/6) + \epsilon^{-2}\log(12/\delta)\right).\]
\end{proposition}

\begin{proof}
Fix $\epsilon,\delta \in (0,1)$, $\MD \in \Delta(\Xaug)$, $f: \Saug \to [0,1]$, and $\phiaug\in\Phiaug$. By definition there is $\phi\in\Phi$ with $\phiaug = \aug(\phi)$. We invoke \cref{lemma:compose-predictors} with the following parameters. Set $\MZ := \Xaug \times \{0,1\}$, and define $\MZ_0 := \{0\} \times \{0,1\}$, $\MZ_1 = \{1\}\times\{0,1\}$, and $\MZ_2 := \MX \times \{0,1\}$. Let $\mu \in \Delta(\MZ)$ be the distribution of $(x,y)$ where $x \sim \MD$ and $y \in \{0,1\}$ with $\EE[y\mid{}x] = f(\phiaug(x)]$. Let $\mu_0,\mu_1,\mu_2$ be the conditional distributions associated with $\MZ_0,\MZ_1,\MZ_2$. Finally, define $h_0,h_1,h_2$ by \[h_0((x^{(i)},y^{(i)})_{i=1}^m)(x,y) := \left(\frac{1}{m}\sum_{i=1}^m y^{(i)} -f(0)\right)^2,\]
\[h_1((x^{(i)},y^{(i)})_{i=1}^m)(x,y) := \left(\frac{1}{m}\sum_{i=1}^m y^{(i)} -f(1)\right)^2,\]
\[h_2((x^{(i)},y^{(i)})_{i=1}^m)(x,y) := \left(\MO((x^{(i)},y^{(i)})_{i=1}^m,\epsilon/6,\delta/6)(x) - f(\phiaug(x))\right)^2.\]
By Hoeffding's inequality, if $(x^{(i)},y^{(i)})_{i=1}^m$ are i.i.d. samples from $\mu_0$ and $m \geq 6\epsilon^{-1}\log(12/\delta)$, then since $\EE[y^{(i)}] = f(\phiaug(0)) = f(0)$, it holds with probability at least $1-\delta/6$ that \[\EE_{(x,y)\sim\mu_0}[h_0((x^{(i)},y^{(i)})_{i=1}^m)(x,y)] \leq \epsilon/6.\]
The same argument holds for $h_1$. Finally, if $(x^{(i)},y^{(i)})_{i=1}^m$ are i.i.d. samples from $\mu_2$ and $m \geq \Nreg(\epsilon/6,\delta/6)$, then by the assumption on $\MO$ and the fact that $\EE[y^{(i)}\mid{}x^{(i)}] = f(\phiaug(x^{(i)})) = f(\phi(x^{(i)}))$ since $x^{(i)} \in\MX$, it holds that with probability at least $1-\delta/6$,
\[\EE_{(x,y)\sim\mu_2}[h_2((x^{(i)},y^{(i)})_{i=1}^m)(x,y)] \leq \epsilon/6.\]
We conclude from \cref{lemma:compose-predictors} that if $n \geq \Nreg'(\epsilon,\delta) = C_{\ref{prop:oneaug}}\left(\epsilon^{-1}\Nreg(\epsilon/6,\delta/6) + \epsilon^{-2}\log(12/\delta)\right)$ and $C_{\ref{prop:oneaug}}$ is a sufficiently large constant, then with probability at least $1-\delta$ over i.i.d. samples $(x^{(i)},y^{(i)})_{i=1}^n$ from $\mu$, the function $H$ defined in \cref{lemma:compose-predictors} satisfies $\EE_{(x,y)\sim\mu}[H(x,y)] \leq \epsilon$. But we can write
\[H(x,y) = (\MR(x) - f(\phiaug(x)))^2\]
where
\[\MR(x) = \begin{cases} 
\frac{1}{\#\{i: x^{(i)}=0\}}\sum_{i:x^{(i)}=0} y^{(i)} & \text{ if } x = 0 \\ 
\frac{1}{\#\{i: x^{(i)}=1\}}\sum_{i:x^{(i)}=1} y^{(i)} & \text{ if } x = 1 \\
\MO((x^{(i)},y^{(i)})_{i: x^{(i)}\in\MX},\epsilon/6,\delta/6)(x) & \text{ if } x \in \MX
\end{cases}.\]
This is precisely the predictor computed by $\OneAug(\MO,(x^{(i)},y^{(i)})_{i=1}^n,\epsilon,\delta)$, so we have shown that $\OneAug(\MO,\cdot)$ is an $\Nreg'$-efficient one-context regression algorithm for $\Phiaug$. 
\end{proof}

\noindent The preceding proof used the following convenient technical lemma about composing statistical predictors on different subsets of a space:

\begin{lemma}\label{lemma:compose-predictors}
Let $\MZ$ be a set, and let $\MZ_1 \sqcup \dots \sqcup \MZ_k$ be a partition of $\MZ$. Let $\mu \in \Delta(\MZ)$ be a distribution. For each $i \in [k]$ let $\mu_i$ be the distribution of $z \sim \mu$ conditioned on $z \in \MZ_i$, and let $h_i: (\MZ_i)^\st \to (\MZ_i\to[0,1])$ be a function with the following property: given at least $m$ i.i.d. samples $(z^{(j)}_i)_j$ from $\mu_i$, it holds with probability at least $1-\delta$ that $\EE_{z\sim\mu_i}[h_i((z_i^{(j)})_j)(z)] \leq \epsilon$.

Let $(z^{(j)})_{j=1}^n$ be $n$ i.i.d. samples from $\mu$, and for each $i \in [k]$ let $S_i = \{j: z^{(j)} \in \MZ_i\}$. Define $H: \MZ \to [0,1]$ by
\[H(z) := h_i((z^{(j)})_{j \in S_i})(z) \text{ for } z \in \MZ_i.\]
If $n \geq 2m/\epsilon + 8\log(1/\delta)/\epsilon$, then with probability at least $1-2k\delta$, it holds that
\[\EE_{z \sim \mu}[H(z)] \leq 2k\epsilon.\]
\end{lemma}

\begin{proof}
Let $\Igood = \{i \in [k]: \mu(\MZ_i) \geq \epsilon\}$. Let $\ME$ be the event that $|S_i| \geq m$ for all $i \in \Igood$. For each such $i$, we have by a Chernoff bound and choice of $n$ that,
\[\Pr[|S_i| < m] \leq \Pr\left[|S_i| < \frac{1}{2} \mu(\MZ_i) n\right] \leq e^{-\mu(\MZ_i) n/8} \leq \delta.\]
Therefore $\Pr[\ME] \geq 1-\delta k$. Condition on $S_1,\dots,S_k$ and suppose that $\ME$ holds. For each $i \in \Igood$, the tuple $(z^{(j)})_{j \in S_i}$ consists of at least $m$ i.i.d. samples from $\mu_i$. Therefore by the lemma assumption, with probability at least $1-\delta k$, we have for all $i \in \Igood$ that
\[\EE_{z \sim \mu_i}[h_i((z^{(j)})_{j \in S_i})(z)] \leq \epsilon.\]
Condition additionally on this event. Then
\begin{align*}
\EE_{z \sim \mu}[H(z)]
&= \sum_{i = 1}^k \mu(\MZ_i) \EE_{z \sim \mu_i}[h_i((z^{(j)})_{j \in S_i}(z)] \\ 
&\leq \sum_{i\in\Igood} \mu(\MZ_i) \epsilon + \sum_{i \in [k]\setminus\Igood} \mu(\MZ_i) \\ 
&\leq (k+1)\epsilon
\end{align*}
which suffices for the claimed bound, and holds in an event with probability at least $1-2\delta k$.
\end{proof}

\subsubsection{The $\TwoAug$ Reduction}\label{sec:twoaug}

\begin{algorithm}[t]
	\caption{$\TwoAug(\MO, (x_1^{(i)},x_2^{(i)},y^{(i)})_{i=1}^n,\epsilon,\delta)$: Two-context regression for $\Phiaug$}
	\label{alg:twoaug}
	\begin{algorithmic}[1]\onehalfspacing
          \State \textbf{input:} Two-context regression oracle $\MO$ for $\Phi$; samples $(x_1^{(i)},x_2^{(i)},y^{(i)})_{i=1}^n$; tolerances $\epsilon,\delta$.
        \State Fix $\xbar \in \MX$.
		\State For all $i \in [n]$, define $\til x_1^{(i)} = x_1^{(i)}$ if $x_1^{(i)} \in \MX$, and otherwise $\xbar$. Similarly define $\til x_2^{(i)}$.
        \State For all pairs $B,B' \in \{\{0\},\{1\},\MX\}$, define
        \[\MR_{B,B'} \gets 
\MO((\til x_1^{(i)},\til x_2^{(i)},y^{(i)})_{i:x_1^{(i)}\in B, x_2^{(i)} \in B'},\epsilon/18,\delta/18).\]
        \State \textbf{return:} the predictor $\MR: \Xaug\times\Xaug \to [0,1]$ defined as follows. Given $x_1,x_2$, define $\til x_1 = x_1$ if $x_1 \in \MX$ and $\til x_1 = \xbar$ otherwise; similarly define $\til x_2$. Then output $\MR_{B,B'}(\til x_1,\til x_2)$ for the unique $B,B'$ with $x_1 \in B$ and $x_2 \in B'$.
	\end{algorithmic}
\end{algorithm}

The following proposition shows that $\TwoAug$ (\cref{alg:twoaug}) is an efficient reduction from two-context reduction for concept class $\Phiaug$ to the same problem for concept class $\Phi$. Similar to $\OneAug$, the reduction decomposes the regression problem into several parts, though now some of the parts are effectively one-context regression problems.

\begin{proposition}\label{prop:twoaug}
There is a constant $C_{\ref{prop:twoaug}}>0$ so that the following holds. Fix sets $\MX,\MS$ and $\Phi \subseteq (\MX \to \MS)$. Suppose that $\MO$ is an $\Nreg$-efficient two-context regression oracle for $\Phi$. Then $\TwoAug(\MO,\cdot)$ is an $\Nreg'$-efficient two-context regression oracle for $\Phiaug$ with
\[\Nreg'(\epsilon,\delta) = C_{\ref{prop:twoaug}}\left(\epsilon^{-1}\Nreg(\epsilon/18,\delta/18) + \epsilon^{-2}\log(36/\delta)\right).\]
\end{proposition}

\begin{proof}
Fix $\epsilon,\delta \in (0,1)$, $\MD \in \Delta(\Xaug\times\Xaug)$, $f: \Saug\times\Saug \to [0,1]$, and $\phiaug\in\Phiaug$. Suppose that $\MD$ is $\phiaug$-realizable (\cref{def:realizable-distribution}). By definition of $\Phiaug$, there is $\phi\in\Phi$ with $\phiaug = \aug(\phi)$. We invoke \cref{lemma:compose-predictors} with the following parameters. Set $\MZ := \Xaug \times\Xaug \times \{0,1\}$. Let $\mu \in \Delta(\MZ)$ be the distribution of $(x_1,x_2,y)$ where $(x_1,x_2) \sim \MD$ and $y \in \{0,1\}$ with $\EE[y\mid{}x_1,x_2] = f(\phiaug(x_1),\phiaug(x_2)]$. For each $B,B' \in \{\{0\},\{1\},\MX\}$ define $\MZ_{B,B'} := \MB\times\MB'\times \{0,1\}$, and let $\mu_{B,B'}$ be the associated conditional distribution. Fix $\xbar \in \MX$, and define $h_{B,B'}$ by
\begin{align*}
&h_{B,B'}((x_1^{(i)},x_2^{(i)},y^{(i)})_{i=1}^m)(x_1,x_2,y) \\ 
&:= \begin{cases}
\left(\MO((x_1^{(i)},x_2^{(i)},y^{(i)})_{i=1}^m,\epsilon/18,\delta/18)(x_1,x_2) - f(\phiaug(x_1),\phiaug(x_2))\right)^2 & \text{ if } B,B' = \MX \\ 
\left(\MO((x_1^{(i)},\xbar,y^{(i)})_{i=1}^m,\epsilon/18,\delta/18)(x_1,\xbar) - f(\phiaug(x_1),\phiaug(x_2))\right)^2 & \text{ if } B = \MX, B' \neq \MX \\
\left(\MO((\xbar,x_2^{(i)},y^{(i)})_{i=1}^m,\epsilon/18,\delta/18)(\xbar,x_2) - f(\phiaug(x_1),\phiaug(x_2))\right)^2 & \text{ if } B \neq \MX, B' = \MX \\
\left(\MO((\xbar,\xbar,y^{(i)})_{i=1}^m,\epsilon/18,\delta/18)(\xbar,\xbar) - f(\phiaug(x_1),\phiaug(x_2))\right)^2 & \text{ if } B,B' \neq \MX
\end{cases}.
\end{align*}
Fix $B,B'$. Let $(x_1^{(i)},x_2^{(i)},y^{(i)})_{i=1}^m$ be i.i.d. samples from $\mu_{B,B'}$ and suppose $m \geq \Nreg(\epsilon/18,\delta/18)$. If $B,B'=\MX$ then $\EE[y^{(i)}\mid{}x_1^{(i)},x_2^{(i)}] = f(\phiaug(x_1^{(i)}),\phiaug(x_2^{(i)})) = f(\phi(x_1^{(i)}),\phi(x_2^{(i)}))$. Moreover, the marginal distribution of $(x_1^{(i)},x_2^{(i)})$ is $\phi$-realizable since it can be expressed as the conditional distribution of $\MD$ under the event that $\phiaug(x_1^{(i)}),\phiaug(x_2^{(i)}) \in \MS$. Thus, by the assumption on $\MO$, it holds with probability at least $1-\delta/18$ that
\[\EE_{(x_1,x_2,y)\sim\mu_{B,B'}}\left[\left(\MO((x_1^{(i)},x_2^{(i)},y^{(i)})_{i=1}^m,\epsilon/18,\delta/18)(x_1,x_2) - f(\phiaug(x_1),\phiaug(x_2))\right)^2 \right] \leq \epsilon/18.\]
If $B=\MX,B'=\{0\}$ then $\EE[y^{(i)}\mid{}x_1^{(i)},\xbar] = \EE[y^{(i)}\mid{}x_1^{(i)}] = f(\phi(x_1^{(i)}),0)$ since $x_2^{(i)}=\phiaug(x_2^{(i)})=0$ is fixed under $\mu_{B,B'}$. Moreover, the marginal distribution of $(x_1^{(i)},\xbar)$ is $\phi$-realizable since $\xbar$ is fixed and hence independent of $x_1^{(i)})$. Thus, by the assumption on $\MO$, it holds with probability at least $1-\delta/18$ that 
\begin{align*}
&\EE_{(x_1,x_2,y)\sim\mu_{B,B'}}\left[\left(\MO((x_1^{(i)},x_2^{(i)},y^{(i)})_{i=1}^m,\epsilon/18,\delta/18)(x_1,x_2) - f(\phiaug(x_1),\phiaug(x_2))\right)^2 \right] \\
&= \EE_{(x,y)\sim\mu_{B,B'}}\left[\left(\MO((x_1^{(i)},x_2^{(i)},y^{(i)})_{i=1}^m,\epsilon/18,\delta/18)(x_1,x_2) - f(\phi(x_1),0)\right)^2 \right] \\
&\leq \epsilon/18.
\end{align*}
The remaining cases follow by analogous arguments. Thus, we can apply \cref{lemma:compose-predictors}. If $n \geq \Nreg'(\epsilon,\delta) = C_{\ref{prop:twoaug}}(\epsilon^{-1}\Nreg(\epsilon/18,\delta/18)+\epsilon^{-1}\log(36/\delta))$, where $C_{\ref{prop:twoaug}}$ is a sufficiently large constant, then with probability at least $1-\delta$ over i.i.d. samples $(x_1^{(i)},x_2^{(i)},y^{(i)})_{i=1}^n$ from $\mu$, the function $H$ defined in \cref{lemma:compose-predictors} satisfies $\EE_{(x_1,x_2,y)\sim\mu}[H(x_1,x_2,y)] \leq \epsilon$. But we can write
\[H(x_1,x_2,y) = (\MR(x) - f(\phiaug(x_1),\phiaug(x_2)))^2\]
where
\[\MR(x_1,x_2) = \begin{cases} 
\MO((x_1^{(i)},x_2^{(i)},y^{(i)})_{i:x_1^{(i)},x_2^{(i)}\in\MX},\epsilon/18,\delta/18)(x_1,x_2) & \text{ if } x_1,x_2 \in \MX \\ 
\MO((x_1^{(i)},\xbar,y^{(i)})_{i:x_1^{(i)}\in\MX,x_2^{(i)}=x_2},\epsilon/18,\delta/18)(x_1,\xbar) & \text{ if } x_1 \in \MX, x_2\not \in\MX \\
\MO((\xbar,x_2^{(i)},y^{(i)})_{i:x_1^{(i)}=x_1,x_2^{(i)}\in\MX},\epsilon/18,\delta/18)(\xbar,x_2) & \text{ if } x_1 \not\in \MX, x_2 \in\MX \\
\MO((\xbar,\xbar,y^{(i)})_{i:x_1^{(i)}=x_1,x_2^{(i)}=x_2},\epsilon/18,\delta/18)(\xbar,\xbar) & \text{ if } x_1, x_2\not \in\MX
\end{cases}.\]
This is exactly the predictor computed by $\TwoAug(\MO,(x_1^{(i)},x_2^{(i)},y^{(i)})_{i=1}^n,\epsilon,\delta)$, so we have shown that $\TwoAug(\MO,\cdot)$ is an $\Nreg'$-efficient two-context regression algorithm for $\Phiaug$.
\end{proof}

\colt{
\newpage
\section{Discussion and Future Work}\label{sec:discussion}

}

\end{document}